\documentclass[runningheads]{llncs}
\usepackage[T1]{fontenc}
\usepackage{graphicx}
\usepackage{subfig}
\usepackage{dsfont}
\usepackage{multirow}
\usepackage{makecell}
\usepackage{algorithm}
\usepackage[noend]{algpseudocode}
\usepackage{mathrsfs}
\usepackage{booktabs}
\usepackage{enumitem}
\usepackage{mathtools}
\usepackage{amsmath, amssymb}
\usepackage{bm}
\usepackage[cal=cm]{mathalfa}
\usepackage[misc,geometry]{ifsym}

\let\L\relax

\let\define\relax

\DeclareMathOperator{\E}{\mathbb{E}}

\DeclareMathOperator{\G}{\mathcal{G}}

\DeclareMathOperator{\D}{\mathcal{D}}
\DeclareMathOperator{\X}{\mathcal{X}}
\DeclareMathOperator{\Y}{\mathcal{Y}}
\DeclareMathOperator{\L}{\mathcal{L}}

\DeclareMathOperator*{\argmin}{arg\,min}

\DeclareMathOperator{\define}{\coloneqq}

\newcommand\norm[1]{\left\lVert#1\right\rVert}
\newcommand*{\tran}{^{\mkern-1.5mu\mathsf{T}}}

\begin{document}
\title{Query-decision Regression between Shortest Path and Minimum Steiner Tree}
%
%
\author{Guangmo Tong (\Letter)\orcidID{0000-0003-3247-4019} \and
Peng Zhao\orcidID{0000-0001-7043-0389} \and
Mina Samizadeh\orcidID{0000-0002-1082-966X}}
%
%
\institute{University of Delaware, USA \\
\email{\{amotong, pzhao, minasmz\}@udel.edu}
}
\maketitle              
\begin{abstract}
Considering a graph with unknown weights, can we find the shortest path for a pair of nodes if we know the minimal Steiner trees associated with some subset of nodes? That is, with respect to a fixed latent decision-making system (e.g., a weighted graph), we seek to solve one optimization problem (e.g., the shortest path problem) by leveraging information associated with another optimization problem (e.g., the minimal Steiner tree problem). In this paper, we study such a prototype problem called \textit{query-decision regression with task shifts}, focusing on the shortest path problem and the minimum Steiner tree problem. We provide theoretical insights regarding the design of realizable hypothesis spaces for building scoring models, and present two principled learning frameworks. Our experimental studies show that such problems can be solved to a decent extent with statistical significance.

\end{abstract}
\section{Introduction}
In its most general sense, a decision-making problem seeks to find the best decision for an input query in terms of an objective function that quantifies the decision qualities \cite{edwards1954theory}. Traditionally, the objective function is given a prior, and we thus focus primarily on its optimization hardness. However, real-world systems are often subject to uncertainties, making the latent objective function not completely known to us \cite{kochenderfer2015decision}; this creates room for data-driven approaches to play a key role in building decision-making pipelines \cite{provost2013data}. 

\textbf{Query-decision Regression with Task Shifts (QRTS).} When facing an unknown objective function, one can adopt the learn-and-optimize framework where we first learn the unknown objective function from data and then solve the target optimization problem based on the learned function \cite{bertsimas2020predictive}, which can be dated back to Bengio’s work twenty
years ago \cite{bengio1997using}. Nevertheless, the learn-and-optimize framework suffers from the fact that the learning process is often driven by the \textit{average} accuracy while good optimization effects demand \textit{worst-case} guarantees \cite{ford2015beware}. Query-decision regression (QR), as an alternative decision-making diagram, seeks to infer good decisions by learning directly from successful optimization results. The feasibility of such a diagram has been proved by a few existing works \cite{chen2015deepdriving}. Such a success points out an interesting way of generalizing QR called query-decision regression with task shifts (QRTS): assuming that there are two query-decision tasks associated with a latent system, can we solve one task (i.e., the target task) by using optimization results associated with the other task (i.e., the source task) -- Figure \ref{fig: example}? Proving the feasibility of such problems is theoretically appealing, as it suggests that one can translate the optimal solutions between different optimization problems. 

\begin{figure}[t]
\centering
\subfloat[]{\label{fig: path_1}\includegraphics[width=0.10\textwidth]{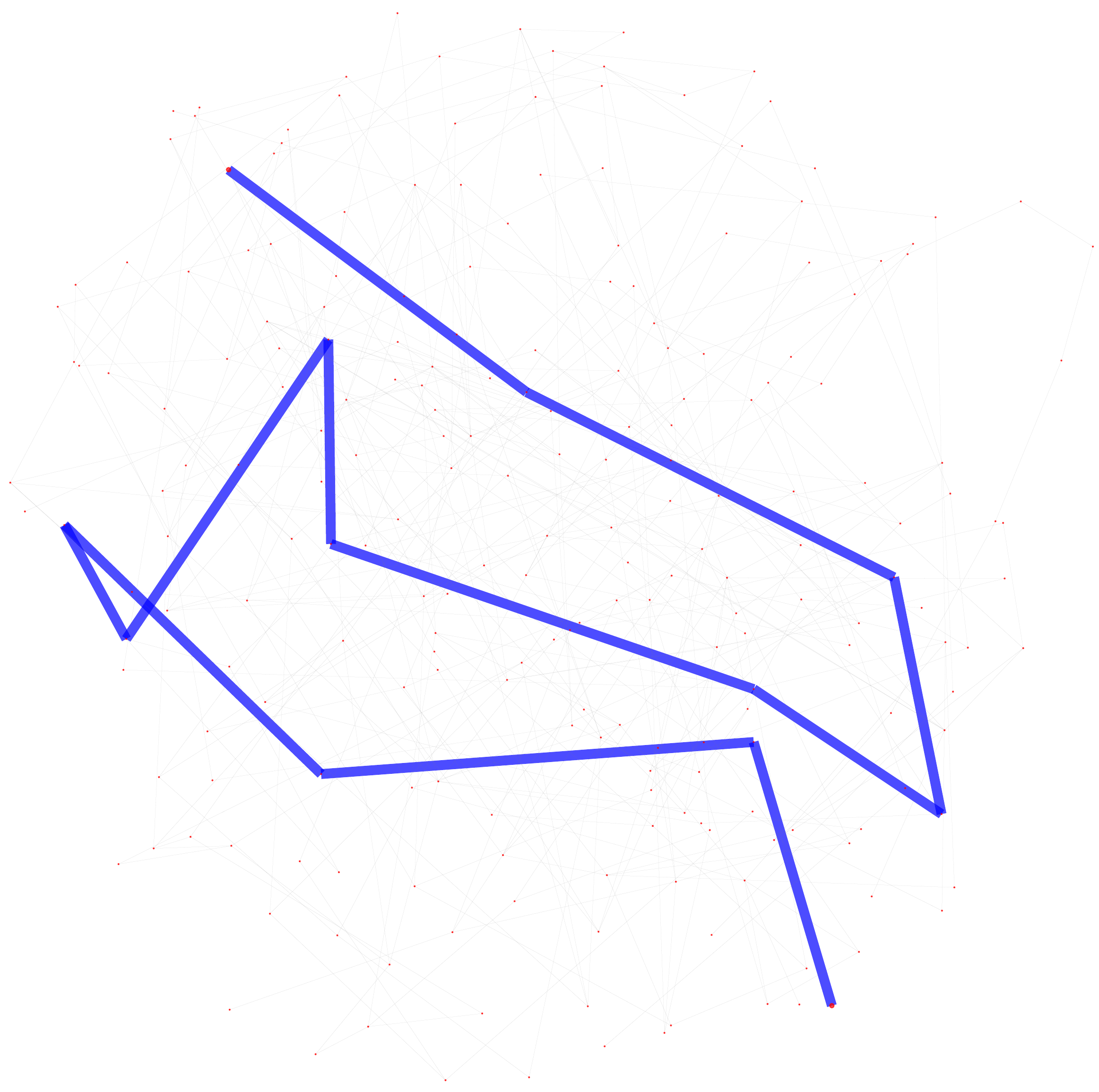}}\hspace{1mm}
\subfloat[]{\label{fig: path_2}\includegraphics[width=0.10\textwidth]{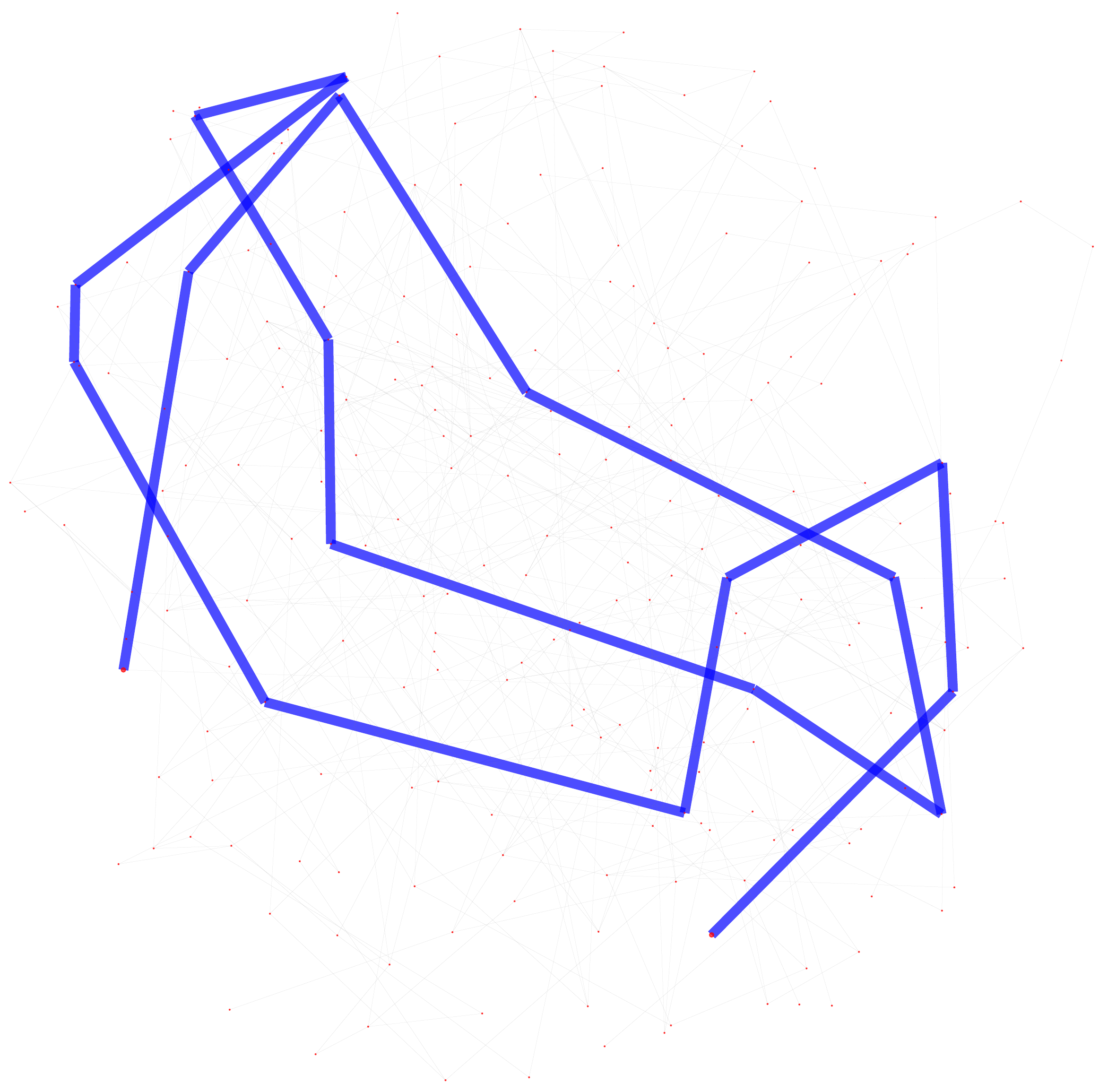}}\hspace{1mm}
\subfloat[]{\label{fig: path_3}\includegraphics[width=0.10\textwidth]{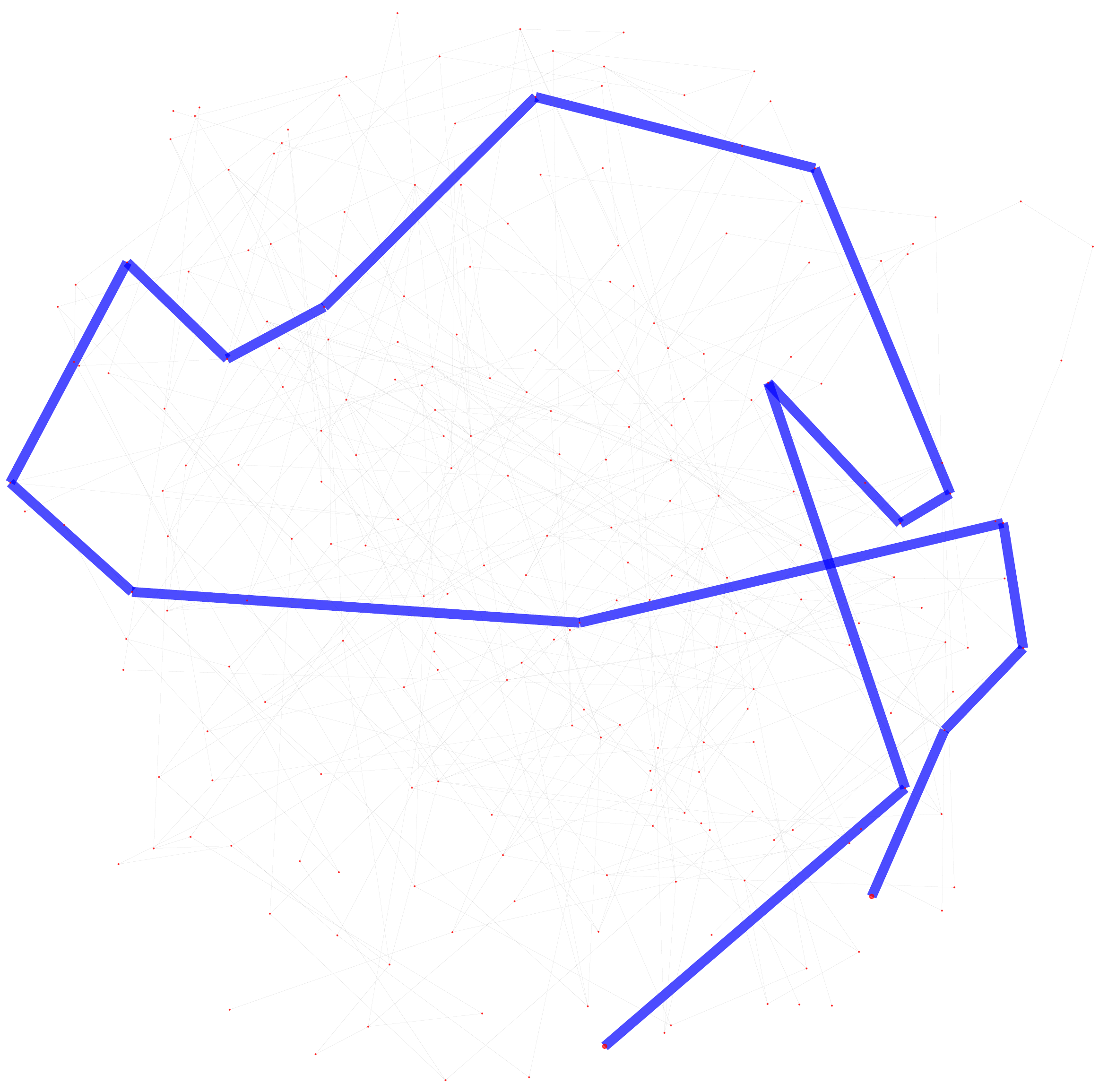}}\hspace{1mm}
\subfloat[]{\label{fig: path_4}\includegraphics[width=0.10\textwidth]{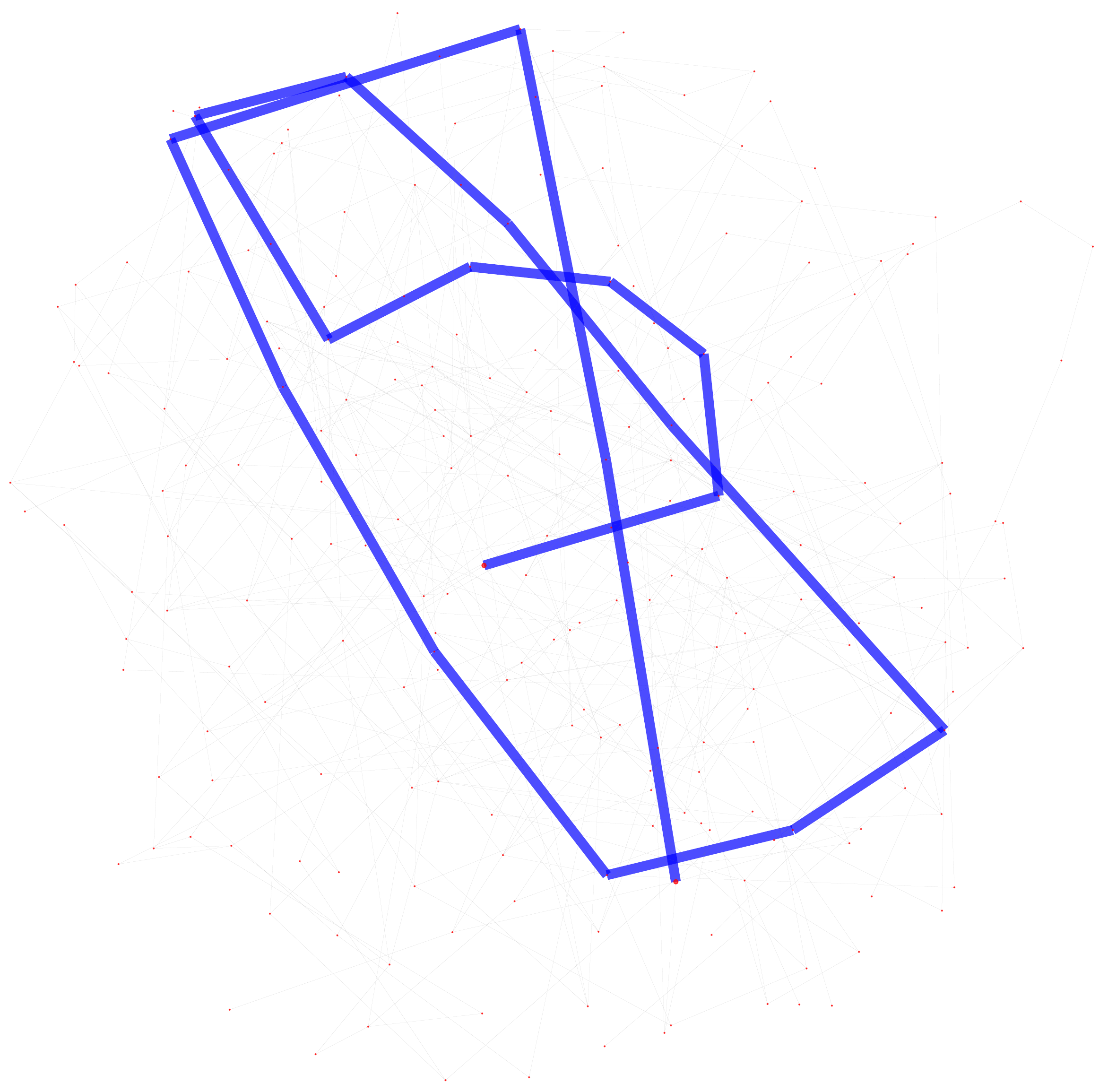}}\hspace{1mm}
\subfloat[]{\label{fig: tree_1}\includegraphics[width=0.10\textwidth]{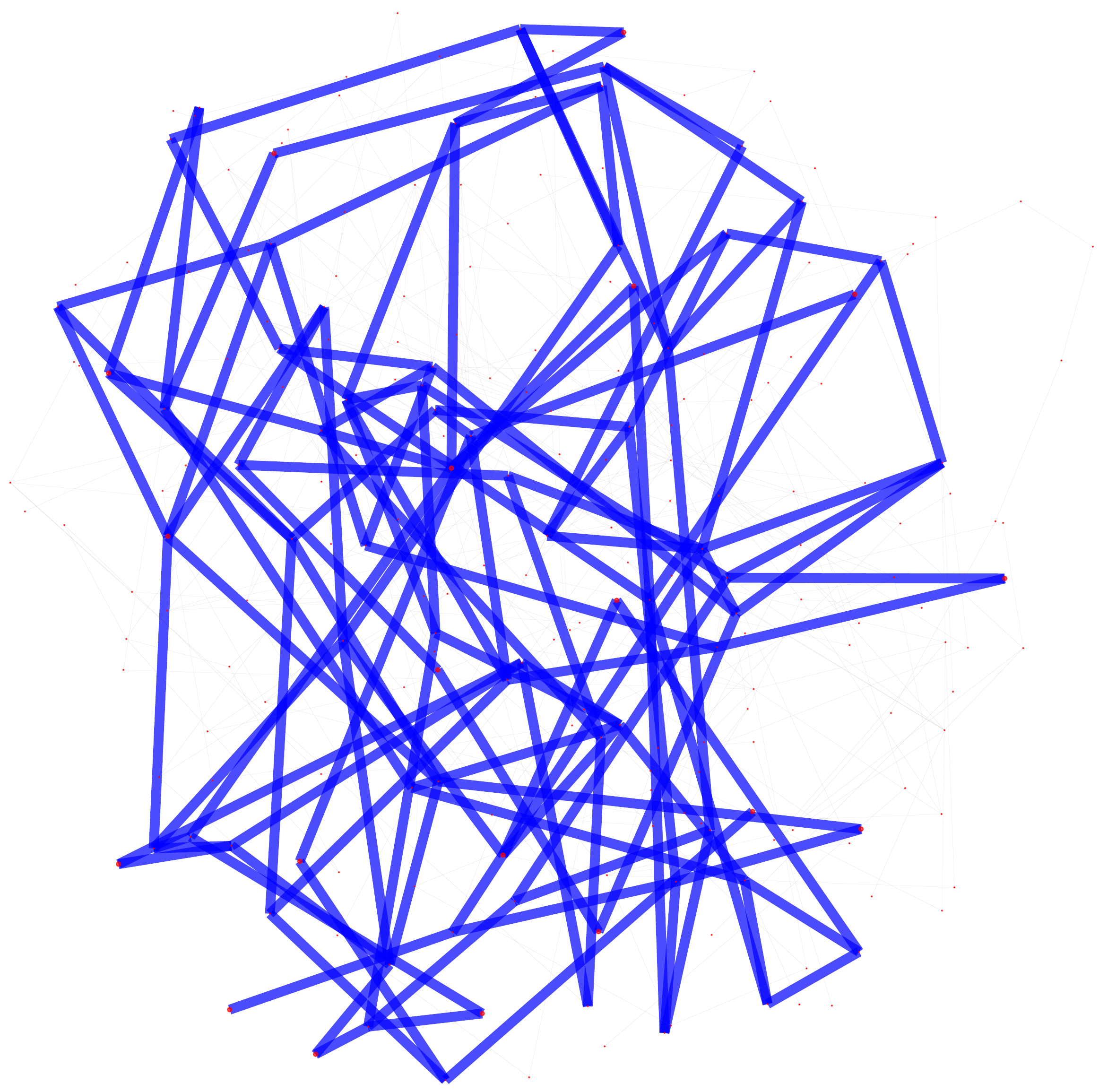}}\hspace{1mm}
\subfloat[]{\label{fig: tree_2}\includegraphics[width=0.10\textwidth]{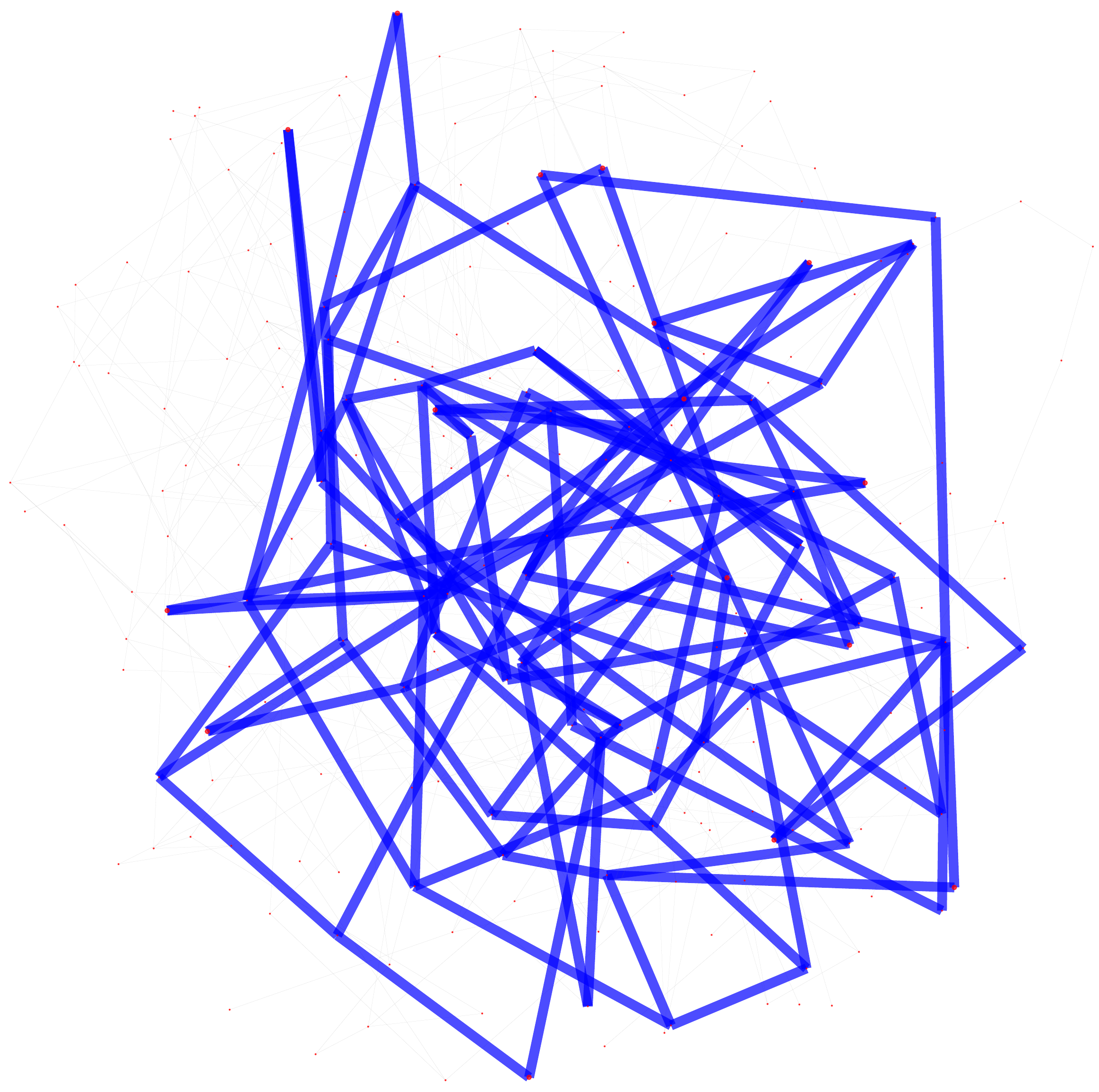}}\hspace{1mm}
\subfloat[]{\label{fig: tree_3}\includegraphics[width=0.10\textwidth]{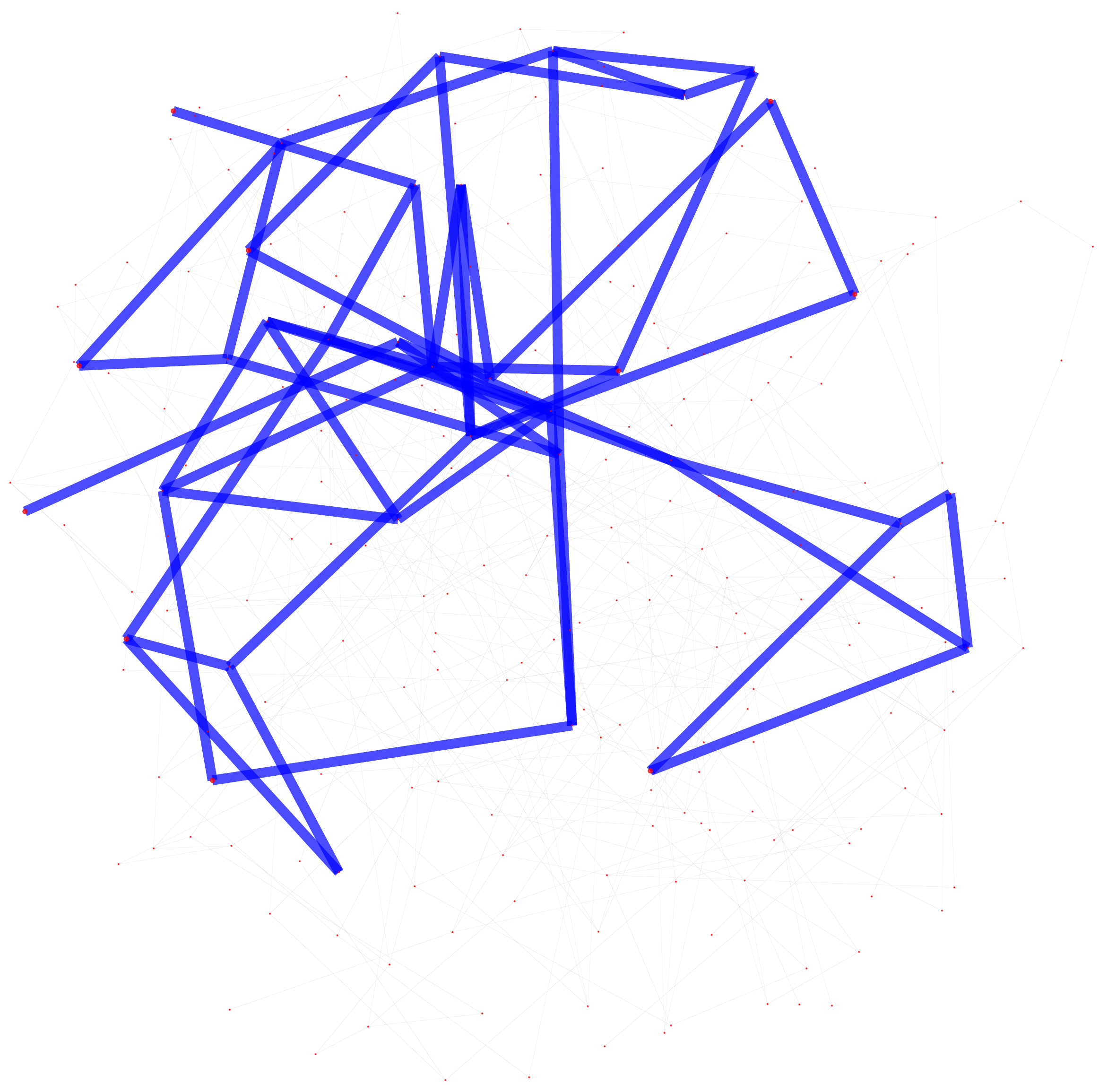}}\hspace{1mm}
\subfloat[]{\label{fig: tree_4}\includegraphics[width=0.10\textwidth]{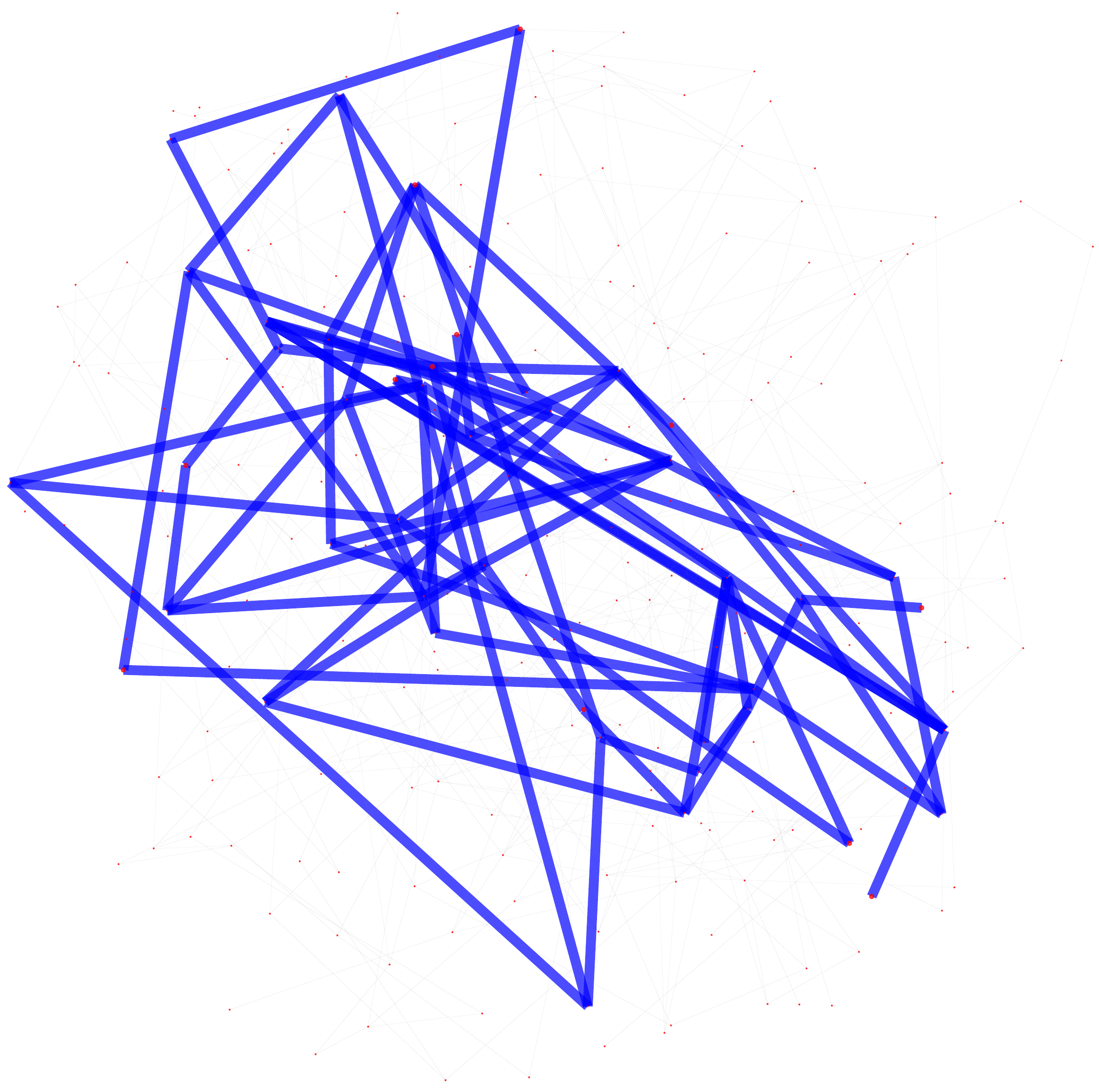}}
\vspace{-2mm}
\caption{\small Associated with a fixed weighted graph, (a)-(d) show the shortest paths of four pairs of nodes, and (e)-(h) show the minimal Steiner trees for four node subsets. Can we leverage the information in (a)-(d) to compute the solutions in (e)-(h), or vice versa?}
\label{fig: example}
\vspace{-5mm}
\end{figure}

\textbf{Contribution.} This paper presents the first study on QRTS over stochastic graphs for two specific problems, the shortest path problem and the minimum Steiner tree problem. Taking QRTS as a statistical learning problem, we provide theoretical analysis regarding the creation of realizable hypothesis spaces for designing score functions, seeking to integrate the latent decision objective into the learning pipeline for better optimization effects. Based on the proposed hypothesis space, we design two principled methods QRTS-P and QRTS-D, where P stands for \textbf{p}oint estimation and D stands for \textbf{d}istribution learning. In particular,  QRTS-P is designed based on the principle of point estimation that implicitly searches for the best mean graph, while QRTS-D leverages distribution learning to compute the pattern of the edge weights that can best fit the samples. We present empirical studies using graphs of classic families. As one of the main contributions of this paper, our results confirm that QRTS can be solved to a satisfactory extent, which may be the first piece of evidence showing that one can successfully translate knowledge between different optimization problems sharing the same underlying system. The supplementary materials\footnote{\url{https://github.com/cdslabamotong/QRTS}} include technical proofs, more discussions on experiments, source code, and data.

\section{Preliminaries}
\label{sec: pre}

We consider a countable family $\G$ of weighted directed graphs sharing the same graph structure $G=(V, E)$. For each weighted graph $g \in \G$, let $g(): E \rightarrow \mathbb{R}^+$ be its weight function. Without loss of generality, we assume that $G$ has no multiple edge when the edge directions are omitted; therefore, each graph $g \in \G$ can also be taken as an undirected graph, without causing confusion regarding the edge weights. Associated with $\G$, there is an \textit{unknown} distribution $\D_{\G}$ over $\G$. We consider optimization problems in the following form.
\begin{definition}[\textbf{Query-decision Optimization}]
Let $\X \subseteq 2^V$ be a query space and $\Y \subseteq 2^E$ be a decision space. In addition, let $f(x,y,g) \in  \mathbb{R}^{+}$ be the decision value associated with a query $x \in \X$, a decision $y \in \Y$, and a graph $g \in \G$ (either directed or undirected). For a given query $x \in \X$, we seek to find the decision that can minimize the expected decision value:
{\small
\begin{align}
\label{eq: true_problem}
{\argmin}_{y \in \Y} F_{f,\D_{\G}}(x,y) \text{~where~} \hspace{0.0cm}  F_{f,\D_{\G}}(x,y) \define  \E_{g \sim  \D
_{\G}} [f(x,y,g)].
\end{align}}%
Such a task is specified by a three-tuple $(f, \X, \Y)$. It reduces to the deterministic case with $|\G|=1$ (e.g., Figure \ref{fig: example}).
\end{definition}
When the distribution $\D_{\G}$ is known to us, the above problems fall into stochastic combinatorial optimization \cite{rajkumar2014online}. In addressing the case when $\D_{\G}$ is unknown, query-decision regression emphasizes the scenario when there is no proper data to learn $\D_{\G}$, and it is motivated by the aspiration to \textit{learn directly from successful optimization results}. Since the optimization problem in question (i.e., Eq. (\ref{eq: true_problem})) can be computationally hard under common complexity assumptions (e.g., NP$\neq$P), we assume that an approximate solution is observed. Accordingly, we will utilize samples in the form of
{\small
\begin{align}
\label{eq: data}
D_{f,\alpha}= \Big\{(x_j,y_j) \big|  F_{f,\D_{\G}}(x_j,y_j) \leq \alpha \cdot \min_{y \in \Y} F_{f,\D_{\G}}(x_j,y) \Big\},
\end{align}}%
where the quality of the observed decision is controlled by a nominal ratio $\alpha \geq 1$. With such, we formulate query-decision regression with/without task shifts as statistical learning problems. 

\begin{definition}[\textbf{Query-decision Regression (QR)}]
Associated with a query-decision optimization problem $(f,\X,\Y)$ and a distribution $\D_{\X}$ over $\X$, given a collection $D_{f,\alpha}=\{(x_j,y_j)\}$ of query-decision pairs with $x_j$ being iid from $\D_{\X}$, we aim to learn a decision-making model $M: \X \rightarrow \Y$ that can predict high-quality decisions for future queries. 
\end{definition}

\begin{definition}[\textbf{Query-decision Regression with Task Shifts (QRTS)}]
Consider a source task $(f_S, \X_S, \Y_S)$ and a target task $(f_T, \X_T, \Y_T)$ sharing the same $\G$ and $\D_{\G}$. Suppose that we are provided with query-decision samples $D_{f_S,\alpha}$ associated with the source task. We aim to learn a decision-making model $M: \X_T \rightarrow \Y_T$ for the target task, with the goal of maximizing the optimization effect:
{\small
\begin{align}
\label{eq: loss}
    \L(M,\D_{\X_T}) \define \E_{x\sim \D_{\X_T}} \Big[\frac{\min_{y \in \Y_T} F_{f_T,\D_{\G}}(x,y)}{F_{f_T,\D_{\G}}(x,M(x))}\Big],
\end{align}}%
where $\D_{\X_T}$ is the query distribution of the target task.
\end{definition}

\begin{remark} [\textbf{Technical Challenge}]
\label{remark: challenge}
In principle, QR falls into the setting of supervised learning, in the sense that it attempts to learn a mapping using labeled data. Therefore, standard supervised learning methods can solve such problems with statistical significance more or less, although they may not be the optimal methods for specific tasks \cite{tong2021usco}. QRTS reduces to QR when the source task is identical to the target one, but it is arguably more challenging: one can no longer apply standard supervised learning methods because the query-decision mapping we seek to infer is associated with the target task but the samples are of the source task. To the best of our knowledge, no existing method can be directly applied to solve problems like the one in Figure \ref{fig: example}.
\end{remark}

\section{A Warm-up Method: QRTS-P}
\label{sec: qrts-p}
In this section, we present a simple and intuitive method called QRTS-P for solving QRTS. 
To illustrate the idea, we notice that the latent objective function $F_{f,\D_{\G}}$ can be expressed as a function of the mean weights of the edges, which is due to the linearity of expectation. 
\begin{example}
Suppose that the considered query-decision optimization problem is the stochastic shortest path problem. For each node pair $x=(u,v)$, let $\Y_x$ be the set of all paths from $u$ to $v$. The latent objective function can thus be expressed as 
{\small
\begin{align}
\label{eq: example}
F_{f,\D_{\G}}(x,y) =
\begin{cases}
  \sum_{e \in y}\E_{g\sim\D_{\G}}[g(e)] & \text{if $y \in \Y_{x}$}  \\
 +\infty & \text{otherwise}. 
 \end{cases}
\end{align}}%
Similarly, for the stochastic minimal Steiner tree problem \cite{vazirani2001approximation}, which finds the min-weight subgraph that connects a given set of nodes, we may define $\Y_x$ as the set of valid Steiner trees of a node set $x$, and with such, the latent objective has the identical form as Eq. (\ref{eq: example}).
\end{example}
In abstract, let $\Y_{f, x} \subseteq \Y_{f}$ be the set of the feasible solutions associated with a query $x$, and $\mathds{1}_S \in \{0,1\}$ be the set indicator function, i.e., $\mathds{1}_S(x) = 1 \iff x\in S$. The query-decision optimization now has the generic form of
{\small
\begin{align}
\label{eq: linear_objective}
\argmin_{y \in \Y_{f,x}} {\sum}_{e \in y}\E_{g\sim\D_{\G}}[g(e)] = {\sum}_{e \in E}\E_{g\sim\D_{\G}}[g(e)] \mathds{1}_y(e)
\end{align}}%
Such an abstraction suggests that it would be sufficient for solving the target task if one can find the mean graph induced by $\D_{\G}$, which essentially asks for good estimations of $\{\E_{g\sim\D_{\G}}[g(e)]| e \in E\}$ -- leading to a point estimation problem \cite{lehmann2006theory}. In what follows, we will see how samples $D_{f_S,\alpha}$ associated with the source task can be helpful for such a purpose.

For each $e\in E$, let $w_e \in \mathbb{R}^+$ be the sought-after estimation of $\E_{g\sim\D_{\G}}[g(e)]$. Since each sample $(x_j,y_j)$ in $D_{f_S,\alpha}$ is an $\alpha$-approximation, in light of Eq. (\ref{eq: linear_objective}), a desired set $\{w_e\} \define \{w_e| e \in \E\}$ should satisfy the linear constraint
{\small
\begin{align}
\label{eq: qrts_p_constraint}
\alpha \cdot \min_{y \in \Y_{f_S,x_j}}    \sum_{e \in E} w_{e}\mathds{1}_{y} \geq    \sum_{e \in E} w_{e} \mathds{1}_{y_j}, 
\end{align}}%
which means that the sample decision $y_j$ is also an $\alpha$-approximation in the mean graph induced by $\{w_{e}\}.$
Applying the standard large-margin training to Eq. (\ref{eq: qrts_p_constraint}) \cite{suthaharan2016support}, a robust estimation can be inferred by solving the following quadratic program
{\small
\begin{align}
    \label{eq: qrts_p_training}
    \min_{w_{e}, \eta_j} \sum_{e \in E} w_{e}^2+ C \cdot \sum_j \eta_j \hspace{0.0cm}
    \text{~s.t.~} \hspace{0.0cm} \alpha \cdot \underbrace{\min_{y \in \Y_{f_S,x_j}}  \sum_{e \in E} w_{e} \mathds{1}_{y}(e)}_{\text{\textbf{source inference}}} -\sum_{e \in E} w_{e} \mathds{1}_{y_j}(e)  \geq -\eta_j, \forall j.
\end{align}}%
where $C$ is a hyperparamter. Optimization problems in the above form have been widely discussed for training structured prediction models, and they can be solved efficiently as long as the source inference problem in Eq. (\ref{eq: qrts_p_training}) can be effectively solved for a given $\{w_e\}$ \cite{lucchi2013learning}. We adopt the cutting plane algorithm in our experiments and defer the details to Appendix \ref{sup: cutting}. With the learned weights $\{w_e\}$, the inference for a query $x^* \in \X_T$ of the target problem can be computed through
{\small
\begin{align}
\label{eq: target_inference_qrts_p}
\text{\textbf{target inference}:} \hspace{0.3cm} \min_{y \in \Y_{f_T,x^*}}  {\sum}_{e \in E} w_{e} \mathds{1}_{y}(e)
\end{align}}%
We denote such an approach as QRTS-P. The source and target inferences will be discussed later in Remark \ref{remark: oracle}, as they are special cases of later problems.

\section{A Probabilistic Perspective: QRTS-D}
\label{sec: qrts-d}
In this section, we present a more involved method called QRTS-D for solving QRTS. It turns out that QRTS-D subtly subsumes QRTS-P as a special case.

\subsection{Overall Framework}
QRTS-D follows the standard scoring model, where we assign each decision a score and make a prediction by selecting the decision with \textit{the lowest score}:
{\small
\begin{align}
\label{eq: qrts_d_score}
\text{\textbf{score function}:} \hspace{0.5cm} & h: \X_T \times \Y_{T} \rightarrow \mathbb{R}\\ \nonumber
\text{\textbf{inference}:} \hspace{0.5cm}&{\argmin}_{y \in \Y_{T}} h(x,y).
\end{align}}%
Such a framework is expected to solve QRTS well, provided that for each pair $(x,y) \in \X_T \times \Y_{T}$, a low score $h(x,y)$ can imply a small objective value $F_{f_T,\D_{\G}}(x,y)$. Implementing such an idea hinges on three integral parts: \textbf{a)} a hypothesis space $H$ of $h$; \textbf{b)} training methods to search for the best score function within $H$ based on the empirical evidence $D_{f_S,\alpha}$; \textbf{c)} algorithms for solving the inference problem.
With such a framework, we will first discuss insights for designing a desired hypothesis space and then present training methods.

\subsection{Hypothesis Design}
In designing a desired score function $h$, the key observation is that the true objective function $F_{f_T,\D_{\G}}$ of the target task is a perfect score function, in that the inference over $F_{f_T,\D_{\G}}(x,y)$ recovers the exact optimal solution. While $\D_{\G}$ is unknown to us, the technique of importance sampling offers a means of deriving a parameterized approximation \cite{tokdar2010importance}. In particular, for any empirical distribution $\D_{\G}^{em}$ over $\G$, we have 
{\small
\begin{align}
\label{eq: key}
 F_{f_T,\D_{\G}}(x,y) =  \int_{g \in \G} \frac{\D_{\G}[g]}{\D_{\G}^{em}[g]} f_T(x,y,g) d \D_{\G}^{em}, 
\end{align}}%
which immediately implies the function approximation guarantee between  $F_{f_T,\D_{\G}}$ and an affine combination of $f_T$.
\begin{theorem}
\label{theorem: approximation}
Let  $\norm{\cdot}$ denote the function distance with respect to the Lebesgue measure associated with any distribution $\D$ over $\X_T \times \Y_T$. For each $\epsilon \geq 0$, $\lim_{K \rightarrow \infty}  \Pr_{g_i \sim  \D_{\G}^{em}}\Big[ \inf_{  w_i \in \mathbb{R}} \norm{ F_{f_T,\D_{\G}}(x,y) - \sum_{i=1}^{K} w_i f_T(x,y,g_i) } \leq \epsilon  \Big] =1$.
\end{theorem}
Theorem \ref{theorem: approximation} justifies the following hypothesis space for the score function of which the complexity is controlled by its dimension $K \in \mathbb{Z}$.
{\small
\begin{align*}
H_{K,\D_{\G}^{em}} \define \Big\{ h_{\vec{w},\{g_i\}}(x,y) \define \sum_{i=1}^K w_i f_T(x,y,g_i) \big| 
 g_i \sim  \D_{\G}^{em}, \vec{w} = (w_1,...,w_K) \in \mathbb{R}^K \Big\}.
\end{align*}}%
These score functions are very reminiscent of the principled kernel machines \cite{hofmann2008kernel}, with the distinction that our kernel function, namely $f_T$, is inherited from the latent optimization problem rather than standard kernels \cite{murphy2012machine}. For such a score function $h_{\vec{w},\{g_i\}}(x,y)$, the inference process is further specialized as 
{\small
\begin{align}
\label{eq: target_inference_qrts_d}
\text{\textbf{target inference}:}~M_{\vec{w},\{g_i\}}(x) \define {\min}_{y \in \Y_T}  {\sum}_{i} w_i f_T(x,y,g_i).
\end{align}}%
With the construction of $H_{K,\D_{\G}^{em}}$, a realizable space can be achieved provided that the dimension $K$ is sufficiently large, which allows us to characterize the generalization loss Eq. (\ref{eq: loss}).

\begin{theorem}
\label{theorem: consistency}    
Suppose that a $\beta$-approximation is adopted to solve the target inference problem Eq. (\ref{eq: target_inference_qrts_d}). Let $D_{\infty} \in \mathbb{R}$ be the $\infty$-order Rényi divergence between $\D_{\G}$ and $\D_{\G}^{em}$. For each $\epsilon \geq 0$ and $\delta>0$, there exists 
{\small
\begin{align*}
C = O\Big(\frac{\ln |\X_T|+\ln |\Y_T|}{\epsilon^2} \cdot \ln \frac{1}{\delta} \cdot \exp(D_{\infty})\Big)
\end{align*}}%
such that when $K\geq C$, with probability at least $1-\delta$ over the selection of $\{g_i\}$, we have $\sup_{ \vec{w} \in \mathbb{R}^K }  \L(M_{\vec{w},\{g_i\}},\D_{\X_T})\geq \beta\cdot \frac{1-\epsilon}{1+\epsilon}$.
\end{theorem}
Theorem \ref{theorem: consistency} suggests that a high dimension (i.e., $K$) may be needed when a) the spaces are large and/or b) the deviation between $\D_{\G}$ and $\D_{\G}^{em}$ is high, which is intuitive. The proof of Theorem \ref{theorem: consistency} leverages point-wise concentration to acquire the desired guarantees, while Theorem \ref{theorem: approximation} is proved through concentrations in function spaces (Appendix \ref{sup: proof}). 

\subsection{QRTS-D}
With the design of $H_{K,\D_{\G}^{em}}$, we now present methods for computing a concrete score function $h_{\vec{w},\{g_i\}}$, which is to decide a collection $\{g_i\}$ of subgraphs as well as the associated weights $\vec{w}$. In light of Eq. (\ref{eq: key}), $\vec{w}_i$ can be viewed as the importance of graph $g_i$. We will not restrict ourselves to a specific choice of $\D_{\G}^{em}$ and thus assume that a nominal parametric family $\D_{\G,\theta}^{em}$ is adopted, with an extra subscription $\theta$ added to denote the parameter set. Assuming that the hyperparameter $K$ is given, the framework of QRTS-D loops over three phases: \textbf{a) graph sampling}, to sample $\{g_1,...,g_K\}$ iid from $\D_{\G,\theta}^{em}$; \textbf{b) importance learning}, to compute $\vec{w}$ for $\{g_i\}$; \textbf{c) distribution tuning}, to update $\D_{\G,\theta}^{em}$.
The first phase is trivial, and we will therefore focus on the other two phases.

\textbf{Importance Learning.} In computing the weights $\vec{w}$ for a given $\{g_i\}$, we have reached a key point to attack the challenges mentioned in Remark \ref{remark: challenge}: the function approximation guarantee in Theorem \ref{theorem: approximation} holds not only for the target task but also for the source task. That is, $\sum_{i=1}^K w_i f_T(x,y,g_i)$ is a desired score function (for solving the target task) if and only if $\sum_{i=1}^K w_i f_S(x,y,g_i)$ can well approximate the true objective function $F_{f_S,\D_{\G}}(x,y)$ of the source task. Therefore, since the samples in $D_{f_{S},\alpha}=\{(x_j,y_j)\}$ are $\alpha$-approximations to the source task, the ideal weights $\vec{w}$ should satisfy 
{\small
\begin{align*}
\alpha \min_{y \in \Y_S} {\sum}_{i=1}^K w_i f_S(x_j,y,g_i) \geq {\sum}_{i=1}^K w_i f_S(x_j,y_j,g_i).  
\end{align*}}%
In this way, we have been able to leverage the samples from the source task to decide the best $\vec{w}$ associated with $\{g_i\}$. This owes to the fact that our design $H_{K,\D_{\G}^{em}}$ allows us to separate $\D_{\G}$ from the task-dependent kernels (i.e., $f_S$ and $f_T$), which is otherwise not possible if we parametrized the score function (i.e., Eq. \ref{eq: qrts_d_score}) using naive methods (e.g., neural networks). Following the same logic behind the translation from Eq. (\ref{eq: qrts_p_constraint}) to Eq. (\ref{eq: qrts_p_training}), the above constraints lead to a similar optimization program: 
{\small\begin{align}
\label{eq: qrts_d_training}
\min_{\vec{w}, \eta_i} \norm{\vec{w}}^2+ C \cdot \sum_i \eta_i 
\text{~\textbf{s.t}.}   \underbrace{\min_{y \in \Y_S}  \alpha\sum_{i=1}^K w_i f_S(x_j,y,g_i)}_{\textbf{source inference}}  - \sum_{i=1}^K w_i f_S(x_j,y_j,g_i)   \geq -\eta_j, \forall j.
\end{align}}%
The above program shares the same type with Eq. (\ref{eq: qrts_p_training}), and we again defer the optimization details to Appendix \ref{sup: cutting}.

\textbf{Distribution Tuning.} With the importance vector $\vec{w}$ learned based on the subgraphs $\{g_i\}$ sampled from the current $\theta$, we seek to fine-tune $\D_{\G, \theta}^{em}$ to make it aligned more with the latent distribution $\D_{\G}$, which is desired as suggested by the proof of Theorem \ref{theorem: approximation}. Inspired by Eq. (\ref{eq: key}), the true likelihood associated with $g_i$ is approximated by $w_i^* \define w_i \D_{\G, \theta}^{em} [g_i]$. Consequently, one possible way to reshape $\D_{\G, \theta}^{em}$ is to find the $\theta^*$ that can minimize the discrepancy between $\D_{\G, \theta^*}^{em}$ and $\D_{\vec{w}^*}$, i.e., $\theta^* = \argmin_{\theta} D(\D_{\G, \theta}^{em} || \D_{\vec{w}^*})_{|\{g_i\}}$, where $\D_{\vec{w}^*}$ is the discrete distribution over $\{g_i\}$ defined by normalizing $(w_1^*,...,w_K^*)$, and 
the distance measure $D(||)$ can be selected at the convenience of the choice of the $\D_{\G, \theta}^{em}$ -- for example, cosine similarity or cross-entropy. For such problems, standard methods can be directly applied when $\D_{\G, \theta}^{em}$ is parameterized by common distribution families; first- and second-order methods can be readily used if $\D_{\G, \theta}^{em}$ has a complex form such as neural networks.

\begin{algorithm}[t]
\caption{QRTS-D}\label{alg: qrts_d}
\small
\begin{algorithmic}[1]
\State \textbf{Input:} $D_{f_S,\alpha}=\{x_y,y_i\}, C, K, T, \alpha$, $\D_{\G,\theta}^{em}$;
\State \textbf{Output:} $\vec{w}=(w_1,...,w_K)$ and $\{g_1,...,g_K\}$
\State Initialize $\theta$, $t=0$;
\Repeat
\State $\{g_1,...,g_K\}$ iid from $\D_{\G,\theta}^{em}$;
\State Compute $\vec{w}$ via Eq. (\ref{eq: qrts_d_training}) based on $D_{f_S,\alpha}$ and $C$;
\State Update $\theta$ via $\theta^* = \argmin_{\theta} D(\D_{\G, \theta}^{em} || \D_{\vec{w}^*})_{|\{g_i\}}$;
\State $t=t+1$
\Until $t=T$
\State \textbf{Return}  $\{g_1,...,g_K\}$ and $\vec{w}$
\end{algorithmic}
\end{algorithm}

The QRTS-D method is conceptually simple, as summarized in Alg. \ref{alg: qrts_d}. Similar to QRTS-P, using such a method requires algorithms for solving the source and target inferences in Eqs. (\ref{eq: target_inference_qrts_d}) and (\ref{eq: qrts_d_training}). In what follows, we discuss such issues as well as the possibility of enhancing QRTS-D using QRTS-P

\begin{remark}[\textbf{Source and Target Inferences}]
\label{remark: oracle}
For QRTS-P, the source (resp., target) inference problem is nothing but to solve the source (resp., target) query-decision optimization task in its deterministic case. For QRTS-D, the inference problems are to solve the source and target tasks over a weighted combination of deterministic graphs. For the shortest path problem, such inference problems can be solved in polynomial time; for the minimum Steiner tree problem, such inference problems admit 2-approximation \cite{vazirani2001approximation}.
\end{remark}


\begin{remark}[\textbf{QRTS-P vs QRTS-D}]
\label{remark: capacity}
As one may have noticed, QRTS-P is a natural special case of the importance learning phase of QRTS-D, in the sense that each $w_e$ in QRTS-P corresponds to the importance of the subgraph with one edge  (i.e., $e$). In other words, QRTS-P can be viewed as the QRTS-D where the support of $\D_{\G, \theta}^{em}$ is the span of single-edge subgraphs with unit weights. Notably, the dimension of QRTS-P is fixed and thus limited by the number of edges, while the dimension $K$ of QRTS-D can be made arbitrarily large. For this reason, QRTS-P may be preferred if the sample size is small, while QRTS-D can better handle large sample sets, which is evidenced by our experimental studies. 
\end{remark}

\begin{remark}[\textbf{QRTS-PD}]
 \label{remark: QRTS-PD}
In QRTS-D, the initialization of $\D_{\G,\theta}^{em}$ is an open issue, and this creates the possibility of integrating QRTS-P into QRTS-D by initializing  $\D_{\G,\theta}^{em}$ using the weights $\{w_e\}$ learned from QRTS-P. This leads to another approach called QRTS-PD consisting of three steps: \textbf{a)} run QRTS-P to acquire the estimations $\{w_e\}$; \textbf{b)} stabilize $\theta$ based on $\{w_e\}$ through maximum likelihood estimation, i.e., $\min_{\theta} -\sum_{e \in E} \log \sum_{g \in \G} \D_{\G,\theta}^{em}[g|g(e)=w_e]$; \textbf{c)} run QRTS-D. 
From such a perspective, QRTS-PD can be taken as a continuation of QRTS-P to further improve the generalization performance by building models that are more expressive.
\end{remark}

\section{Empirical Studies}
\label{sec: exp}
In this section, we present empirical studies demonstrating that QRTS can be solved with statistical significance using the presented methods.

\subsection{Experimental settings}

\textbf{Source and Target Tasks.} We specifically focus on two query-decision optimization tasks: shortest path and minimum Steiner tree \cite{hwang1992steiner}. Depending on the selection of the source and target tasks, we have two possible settings: \textit{Path-to-Tree} and \textit{Tree-to-Path}. The source and target inferences can be approximated effectively, as discussed in Remark \ref{remark: oracle}. These algorithms are also used to generate samples of query-decision pairs (i.e., Eq. (\ref{eq: data})).

\begin{table*}[!pt]
\renewcommand{\arraystretch}{1.0} 
\caption{\small \textbf{Results for Path-to-Tree on Kro, Col and BA.}  Each cell shows the mean ratio together with the standard deviation (std). The top three results in each column are highlighted.}
\label{table: PT_1}
\centering
\scalebox{0.98}{\small
\begin{tabular}{@{}  l@{\hspace{1mm}} r@{\hspace{3mm}} c@{\hspace{1mm}} c @{\hspace{1mm}} c @{\hspace{3mm}} c @{\hspace{1mm}} c @{\hspace{1mm}} c @{\hspace{3mm}} c @{\hspace{1mm}} c   @{\hspace{1mm}} c  @{}}

\toprule
 &    &\multicolumn{3}{c}{\textbf{Kro}}& 
 \multicolumn{3}{c}{\textbf{Col}} & \multicolumn{3}{c}{\textbf{BA}}  \\
\multicolumn{2}{c}{Train Size}  & 60 & 240 & 2400 &  60 & 240 & 2400 & 60 & 240 & 2400 \\

 \midrule
 \multicolumn{2}{c}{\textbf{QRTS-P}} & \textbf{4.0}\tiny{(0.3)} & 3.4\tiny{(0.5)}&  2.4\tiny{(0.3)}&  291\tiny{(78)} &  \textbf{150}\tiny{(29)}&  128\tiny{(7.4)}&   181\tiny{(28)}&  144\tiny{(41)} & 63\tiny{(35)} \\
 \midrule
 \multirow{3}{*}{\makecell{\textbf{QRTS}\\ \textbf{-PD}$^{-}$}}  & 60 & 4.4\tiny{(0.3)} & 3.4\tiny{(0.8)} & 2.3\tiny{(0.4)} & \textbf{250}\tiny{(82)} & 367\tiny{(56)} & 123\tiny{(2.3)} & 209\tiny{(22)} & 195\tiny{(22)} & 40\tiny{(13)} \\
 & 240 & 4.6\tiny{(0.7)} & 3.7\tiny{(0.1)} & 2.7\tiny{(0.2)} & 330\tiny{(65)} & 227\tiny{(14)} & 117\tiny{(6.1)} & \textbf{129}\tiny{(22)} & 149\tiny{(19)} & 35\tiny{(12)} \\
 & 2400 & 4.3\tiny{(0.9)} & \textbf{3.2}\tiny{(0.5)} & 2.4\tiny{(0.1)} & \textbf{214}\tiny{(68)} & 183\tiny{(69)} & \textbf{88}\tiny{(12)} & \textbf{139}\tiny{(38)} & \textbf{131}\tiny{(44)} & \textbf{27}\tiny{(7.2)} \\
 \midrule
 \multirow{3}{*}{\makecell{\textbf{QRTS}\\ \textbf{-PD-1}}}  & 60 & \textbf{3.9}\tiny{(0.7)} & 3.4\tiny{(0.8)} & \textbf{2.3}\tiny{(0.4)} & 361\tiny{(55)} & 225\tiny{(11)} & 115\tiny{(5.1)} & 177\tiny{(31)} & 130\tiny{(34)} & 43\tiny{(14)} \\
 & 240 & 4.2\tiny{(0.4)} & 3.2\tiny{(0.5)} & 2.4\tiny{(0.2)} & 350\tiny{(88)} & 245\tiny{(31)} & 129\tiny{(16)} & 166\tiny{(34)} & 132\tiny{(49)} & 34\tiny{(13)} \\
 & 2400 & 4.3\tiny{(0.5)} & \textbf{3.1}\tiny{(0.3)} & \textbf{2.3}\tiny{(0.4)} & \textbf{261}\tiny{(93)} & 186\tiny{(24)} & \textbf{106}\tiny{(6.2)} & \textbf{160}\tiny{(29)} & \textbf{119}\tiny{(3)} & {34}\tiny{(7.3)} \\
 \midrule
  \multirow{3}{*}{\makecell{\textbf{QRTS}\\ \textbf{-PD-3}}}  & 60 & 4.3\tiny{(1.1)} & 3.2\tiny{(0.6)} & 2.6\tiny{(0.5)} & 431\tiny{(27)} & \textbf{167}\tiny{(99)} & \textbf{110}\tiny{(15)} & 391\tiny{(46)} & 144\tiny{(13)} & 60\tiny{(13)} \\
 & 240 & 4.3\tiny{(0.9)} & 3.3\tiny{(0.6)} & 2.3\tiny{(0.4)} & 317\tiny{(87)} & {183}\tiny{(35)} & 126\tiny{(11)} & 202\tiny{(38)} & 138\tiny{(17)} & \textbf{30}\tiny{(17)} \\
 & 2400 & \textbf{4.0}\tiny{(0.4)} & \textbf{3.2}\tiny{(0.4)} & \textbf{2.2}\tiny{(0.2)} & 324\tiny{(38)} & \textbf{107}\tiny{(12)} & 113\tiny{(6.1)} & 184\tiny{(49)} & \textbf{120}\tiny{(2.4)} & \textbf{23}\tiny{(2.7)} \\
  \midrule
 \multicolumn{2}{c}{\textbf{Unit} \& \textbf{Rand}} & \multicolumn{3}{c}{5.2{\tiny(0.2)} \& 10.4{\tiny(0.3)}} &\multicolumn{3}{c}{990{\tiny(68)} \& 2231{\tiny(45)}} &\multicolumn{3}{c}{349{\tiny(4.7)} \& 749\tiny{(13)}} \\
 

\bottomrule
\end{tabular}%
} 
\vspace{-6mm}
\end{table*}

\textbf{Graph, True Distribution, and Samples.} We adopt a collection of graphs of classic types: a Kronecker graph (\textbf{Kro}) \cite{leskovec2010kronecker}, a road network of Colorado (\textbf{Col}) \cite{demetrescu2008implementation}, a Barabasi–Albert graph (\textbf{BA}) \cite{barabasi1999emergence}, and two Watts–Strogatz graphs with different densities (\textbf{WS-dense} and \textbf{WS-sparse}) \cite{watts1998collective}.
The statistics of these graphs can be found in Appendix \ref{sup: exp}. To have a diverse graph pattern, we generate the ground truth distribution $\D_{\G}$ by assigning each edge a Weibull distribution \cite{lai2011weibull} with parameters randomly selected from $\{1,...,20\}$. For each graph and each problem instance, we generate a pool of $10,000$ query-decision pairs.

\textbf{QRTS Methods.} We use QRTS-PD-1 (resp.,  QRTS-PD-3) to denote the QRTS-PD method when one (resp., three) iterations over the three phases are used. Based on QRTS-PD-1, we implement QRTS-PD$^{-}$ that foregoes the distribution tuning phase. For these methods, the model dimensions $K$ are selected from $\{60, 240, 2400\}$. We utilize the one-slack cutting plane algorithm \cite{joachims2009cutting} for the large-margin training (i.e., Eqs. (\ref{eq: qrts_p_training}) and (\ref{eq: qrts_d_training})). The empirical distribution $\D_{\G}^{em}$ is parameterized by assigning each edge an exponential distribution. 


\textbf{Baselines.} We set up two baselines \textbf{Unit} and \textbf{Rand}. Unit computes the predictions based on the graph with unit-weight edges. Rand computes the decision based on the graph with random weights. Unit essentially leverages only the graph structure to compute predictions. We note that none of the methods in existing papers can be directly applied to QRTS (Remark \ref{remark: challenge}). 


\textbf{Training and Testing.} In each run, the training size is selected from $\{60, 240, 2400\}$, and the testing size is $1000$, where samples are randomly selected from the sample pool. Given a testing set $\{x_i, y_i\}$ of the target task, the performance is measured by the ratio $\sum_{i} F_{f_T,\D_{\G}}(x_i,y_i^*)/\sum_{i} F_{f_T,\D_{\G}}(x_i,y_i)$, where $y_i^*$ is the predicted decision associated with $x_i$; a lower ratio implies better performance. We report the average ratios and the standard deviations over five runs for each method.

\begin{table*}[!pt]
\renewcommand{\arraystretch}{1.0} 
\caption{\small \textbf{Results on Tree-to-Path.} Each cell shows the mean ratio together with the standard deviation (std). Small stds ($<$ 0.1) are denoted as 0.0.  The top three results in each column are highlighted.}
\label{table: TP_1}
\centering
\scalebox{0.90}{\small
\begin{tabular}{@{}  l@{\hspace{1mm}} r@{\hspace{2mm}} c@{\hspace{1mm}} c @{\hspace{1mm}} c @{\hspace{2mm}} c @{\hspace{1mm}} c @{\hspace{1mm}} c @{\hspace{2mm}} c @{\hspace{1mm}} c   @{\hspace{1mm}} c  @{}}

\toprule
 &    &\multicolumn{3}{c}{\textbf{Kro}}& 
 \multicolumn{3}{c}{\textbf{Col}} & \multicolumn{3}{c}{\textbf{BA}}  \\
\multicolumn{2}{c}{Train Size}  & 60 & 240 & 2400 &  60 & 240 & 2400 & 60 & 240 & 2400 \\

 \midrule
 \multicolumn{2}{c}{\textbf{QRTS-P}} & \textbf{1.46} \tiny{(0.0)} & \textbf{1.37} \tiny{(0.0)} & 1.60 \tiny{(0.0)} & 9.8 \tiny{(0.2)} & 6.6 \tiny{(0.4)} & \textbf{6.1} \tiny{(0.1)} & 1.6 \tiny{(0.1)} & 1.4 \tiny{(0.2)} & 1.4 \tiny{(0.1)} \\
 \midrule
 \multirow{3}{*}{\makecell{\textbf{QRTS}\\ \textbf{-PD}$^{-}$}}  & 60 & \textbf{1.44} \tiny{(0.1)} & 1.41 \tiny{(0.1)} & 1.39 \tiny{(0.0)} & 11 \tiny{(5.8)} & 8.6 \tiny{(1.5)} & 7.1 \tiny{(0.6)} & 1.9 \tiny{(0.4)} & 1.4 \tiny{(0.2)} & 1.5 \tiny{(0.1)} \\
 & 240  & \textbf{1.42} \tiny{(0.1)} & 1.46 \tiny{(0.0)} & 1.39 \tiny{(0.0)} & 9.9 \tiny{(4.2)} & 6.8 \tiny{(3.6)} & 6.5 \tiny{(0.4)} & 1.7 \tiny{(0.3)} & 1.5 \tiny{(0.2)} & 1.5 \tiny{(0.1)} \\
 & 2400  & 1.48 \tiny{(0.1)} & 1.45 \tiny{(0.0)} & 1.38 \tiny{(0.0)} & \textbf{8.9} \tiny{(3.1)} & \textbf{6.0} \tiny{(1.3)} & 6.2 \tiny{(0.9)} & \textbf{1.5} \tiny{(0.2)} & \textbf{1.3} \tiny{(0.1)} & \textbf{1.3} \tiny{(0.1)} \\
 \midrule
 \multirow{3}{*}{\makecell{\textbf{QRTS}\\ \textbf{-PD-1}}}  & 60 & 1.55 \tiny{(0.0)} & 1.42 \tiny{(0.1)} & 1.37 \tiny{(0.0)} & 6.6 \tiny{(0.6)} & 6.3 \tiny{(2.4)} & 6.8 \tiny{(0.6)} & 1.6 \tiny{(0.3)} & 1.5 \tiny{(0.0)} & 1.4 \tiny{(0.1)} \\
 & 240 & 1.56 \tiny{(0.1)} & 1.44 \tiny{(0.1)} & 1.36 \tiny{(0.0)} & 11 \tiny{(3.2)} & 7.6 \tiny{(1.4)} & 6.6 \tiny{(0.0)} & 1.6 \tiny{(0.3)} & 1.5 \tiny{(0.3)} & 1.5 \tiny{(0.1)} \\
 & 2400 & 1.52 \tiny{(0.1)} & 1.39 \tiny{(0.1)} & 1.37 \tiny{(0.0)} & 14 \tiny{(2.4)} & 7.7 \tiny{(3.9)} & 7.2 \tiny{(0.3)} & 1.7 \tiny{(0.1)} & 1.5 \tiny{(0.3)} & 1.7 \tiny{(0.1)} \\
 \midrule
  \multirow{3}{*}{\makecell{\textbf{QRTS}\\ \textbf{-PD-3}}}  & 60 & 1.53 \tiny{(0.1)} & \textbf{1.41} \tiny{(0.1)} & \textbf{1.34} \tiny{(0.1)} & 13 \tiny{(4.2)} & \textbf{5.9} \tiny{(1.5)} & \textbf{5.5} \tiny{(0.9)} & 1.7 \tiny{(0.2)} & 1.5 \tiny{(0.1)} & 1.4 \tiny{(0.1)} \\
 & 240 & 1.50 \tiny{(0.1)} & 1.42 \tiny{(0.0)} & \textbf{1.36} \tiny{(0.0)} & \textbf{9.4 }\tiny{(1.4)} & \textbf{6.6} \tiny{(0.2)} & \textbf{5.9} \tiny{(0.1)} & \textbf{1.6} \tiny{(0.1)} & \textbf{1.4} \tiny{(0.2)} & \textbf{1.2} \tiny{(0.1)} \\
 & 2400 & 1.47 \tiny{(0.1)} & \textbf{1.41} \tiny{(0.1)} & \textbf{1.32} \tiny{(0.0)} & \textbf{7.8} \tiny{(0.6)} & 8.5 \tiny{(1.4)} & 7.7 \tiny{(2.0)} & \textbf{1.3} \tiny{(0.1)} & \textbf{1.3} \tiny{(0.1)} & \textbf{1.3} \tiny{(0.2)} \\
  \midrule
 \multicolumn{2}{c}{\textbf{Unit} \& \textbf{Rand}} & \multicolumn{3}{c}{1.57 {\tiny(0.1)} \& 1.78 {\tiny(0.1)}} &  \multicolumn{3}{c}{9.2 {\tiny(0.1)} \& 19 {\tiny(0.1)}} & \multicolumn{3}{c}{1.57{\tiny(0.0)} \& 2.2{\tiny(0.3)}} \\


\bottomrule
\end{tabular}%
} 
\vspace{-4mm}
\end{table*}

\subsection{Analysis}
The results on Kro, Col, and BA are given in Tables \ref{table: PT_1} and \ref{table: TP_1}. The results on WS-sparse and WS-dense can be found in Appendix \ref{sup: exp}.  The main observations are listed below, and the minor observations are given in Appendix \ref{sup: exp}.

\textbf{O1: The proposed methods behave reasonably with promising performance.} First, we observe that the proposed methods perform significantly better when more samples are given, which suggests that they are able to infer meaningful information from the samples toward solving the target task. On the other hand, all the proposed methods are clearly better than Rand, implying that the model efficacy is non-trivial. In addition, they easily outperform Unit by an evident margin in most cases. For example, for Path-to-Tree on BA in Table \ref{table: PT_1}, the best ratio achieved by QRTS-PD-3 is $23$, while Unit and Random cannot produce a ratio smaller than $300$.

\textbf{O2: QRTS-P offers an effective initialization for QRTS-D.} With very few exceptions, QRTS-PD performs much better than QRTS-P under the same sample size, which confirms that QRTS-P can indeed be improved by using importance learning through re-sampling, which echos Remark \ref{remark: QRTS-PD}. This is especially true when the sample size is large; for example, for Path-to-Tree on BA with $2400$ samples, methods based on QRTS-PD with a dimension of $2400$ are at least twice better than QRTS-P in terms of the performance ratio, demonstrating that QRTS-PD of a high dimension can better consume large datasets.

\textbf{O3: Distribution tuning is helpful after multiple iterations.} Since QRTS-D can be used without the distribution tuning phase, we are wondering if the distribution turning phase is necessary. By comparing QRTS-PD-1 with QRTS-PD$^{-}$, we see that the distribution tuning phase can be useful in many cases, but its efficacy is not very significant. However, combined with the results of QRTS-PD-3, we observe that the distribution turning phase can better reinforce the optimization effect when multiple iterations are used. Finally, by comparing QRTS-PD-1 and QRTS-PD-3, we find that training more iterations is useful mostly when the model dimension is large, which is especially the case for Tree-to-Path (Table \ref{table: TP_1}). 

\textbf{O4: The learning process is smooth.} Figure \ref{fig: visua} visualizes the learning process of QRTS-P for two example testing queries. One can see that QRTS-P tends to select solutions with fewer edges under the initial random weights $\vec{w}$, and it gradually finds better solutions (possibly with more edges) when better weights are learned. We have such observations for most of the samples, which suggest that the proposed method works the way it is supposed to. More visualizations can be found in Appendix \ref{sup: exp}.

\begin{figure}[t]
\centering
\captionsetup[subfloat]{labelfont=scriptsize,textfont=scriptsize,labelformat=empty}
\subfloat[  { [30, 0]} ]{\label{fig: log_2_20_8}\includegraphics[width=0.08\textwidth]{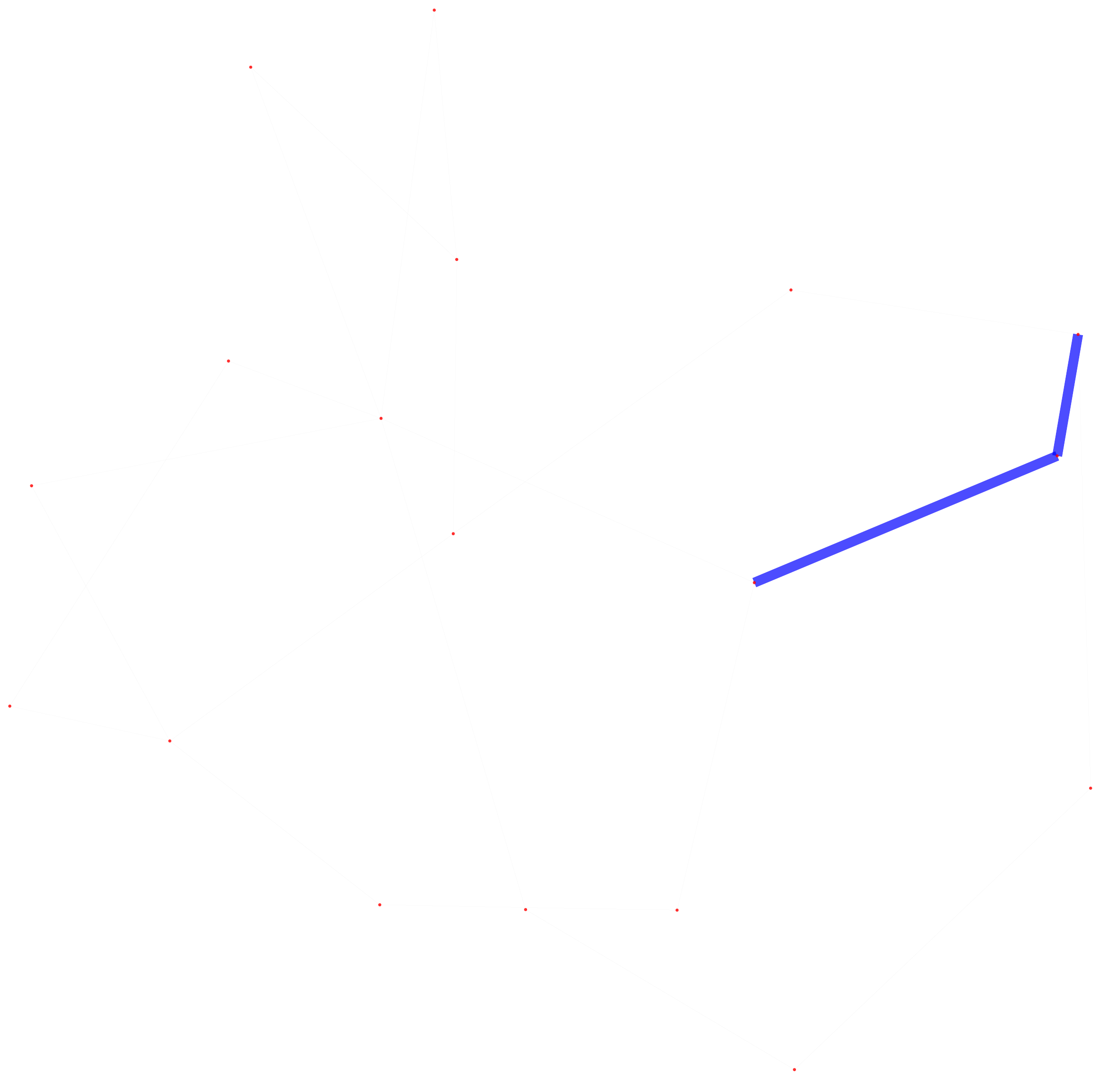}}\hspace{1mm}
\subfloat[ {[30, 1]} ]{\label{fig: log_2_21_8}\includegraphics[width=0.08\textwidth]{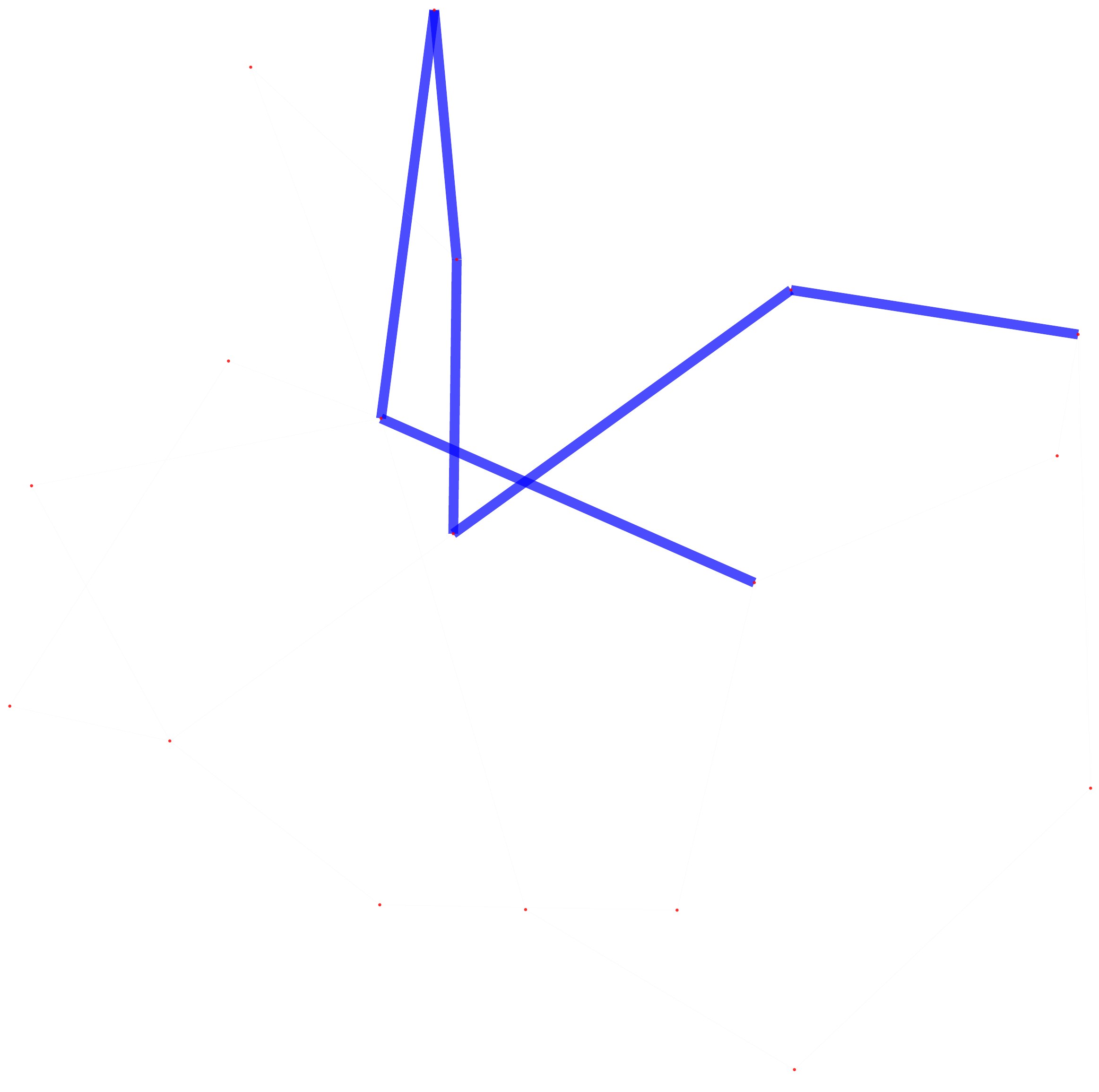}}\hspace{1mm}
\subfloat[ {[30, 2]} ]{\label{fig: log_2_22_8}\includegraphics[width=0.08\textwidth]{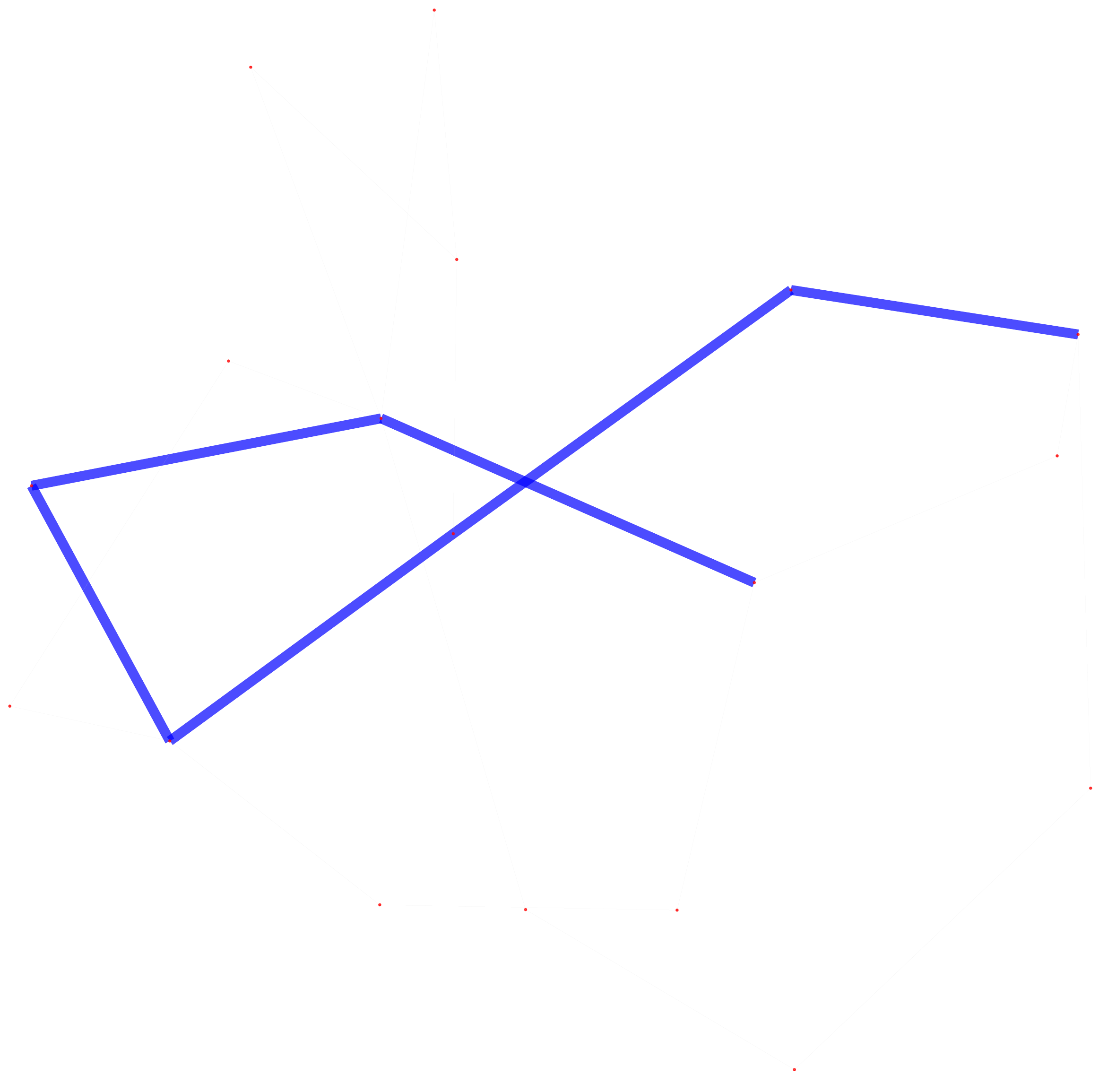}}\hspace{1mm}
\subfloat[ {[30, 5]} ]{\label{fig: log_2_24_8}\includegraphics[width=0.08\textwidth]{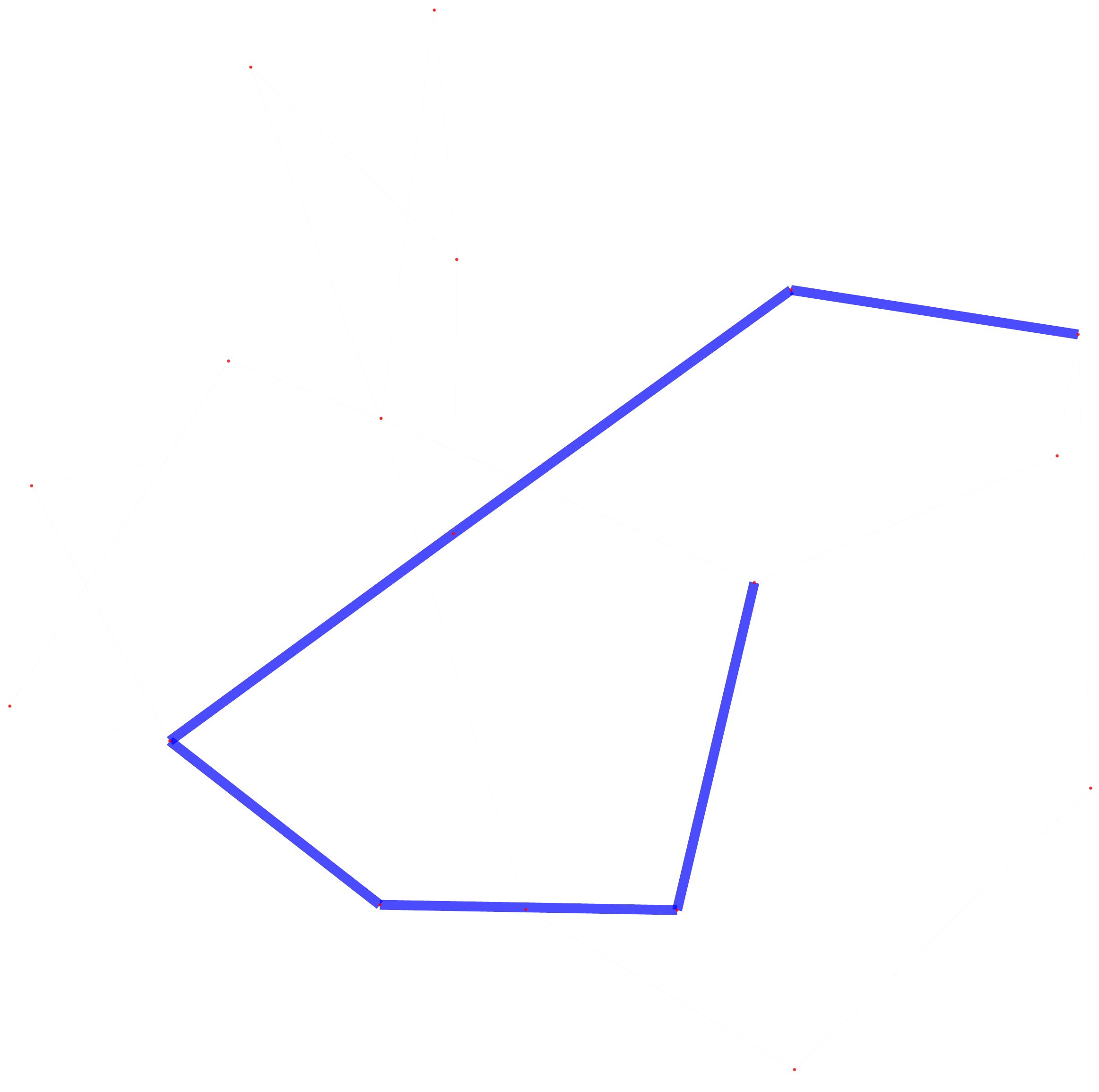}}\hspace{1mm}
\subfloat[Target]{\label{fig: log_2_19_8}\includegraphics[width=0.08\textwidth]{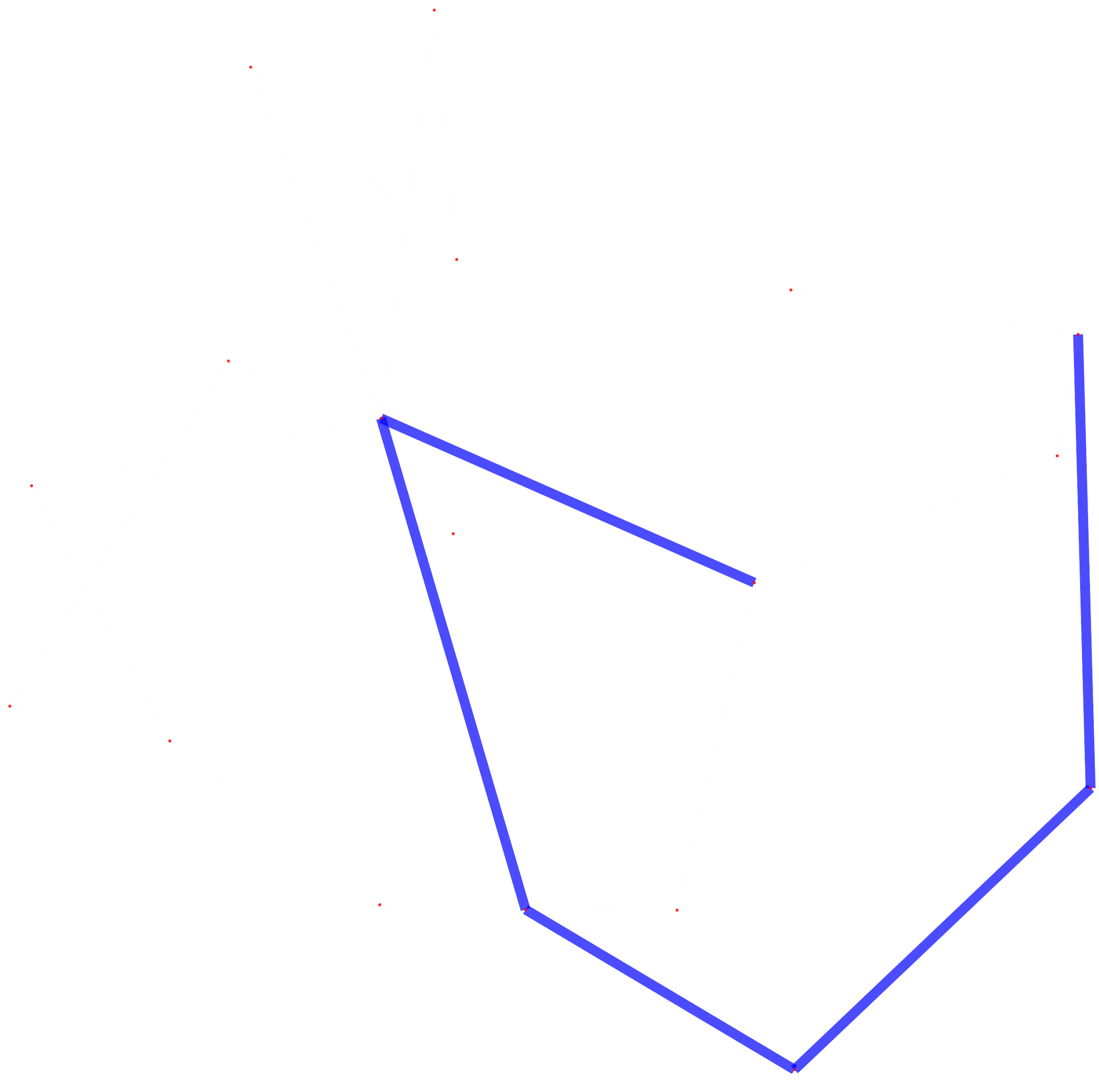}}\hspace{5mm}
\subfloat[  { [30, 0]} ]{\label{fig: 4_30_0}\includegraphics[width=0.10\textwidth]{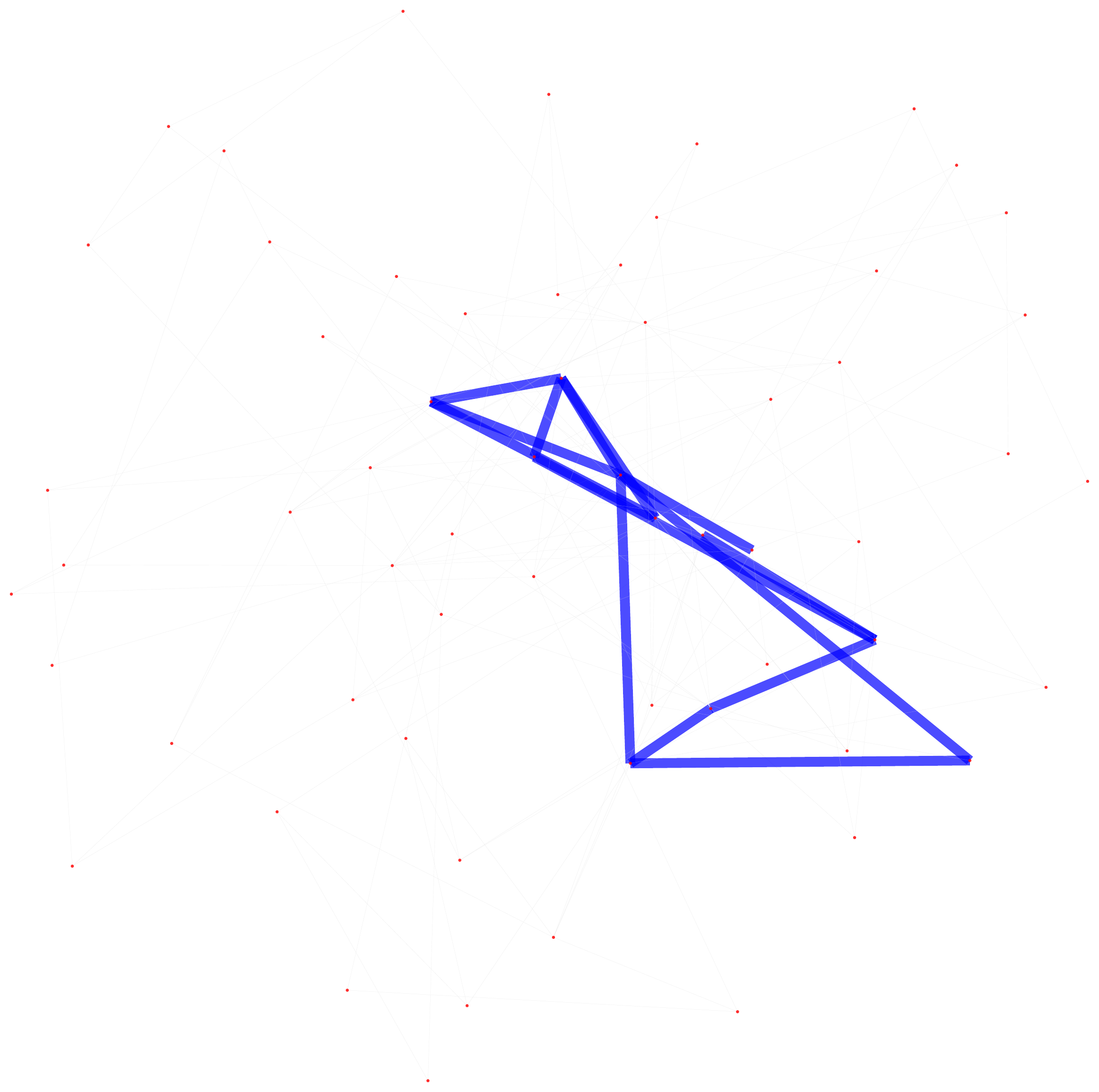}}\hspace{1mm}
\subfloat[  { [30, 1]} ]{\label{fig: 4_30_1}\includegraphics[width=0.10\textwidth]{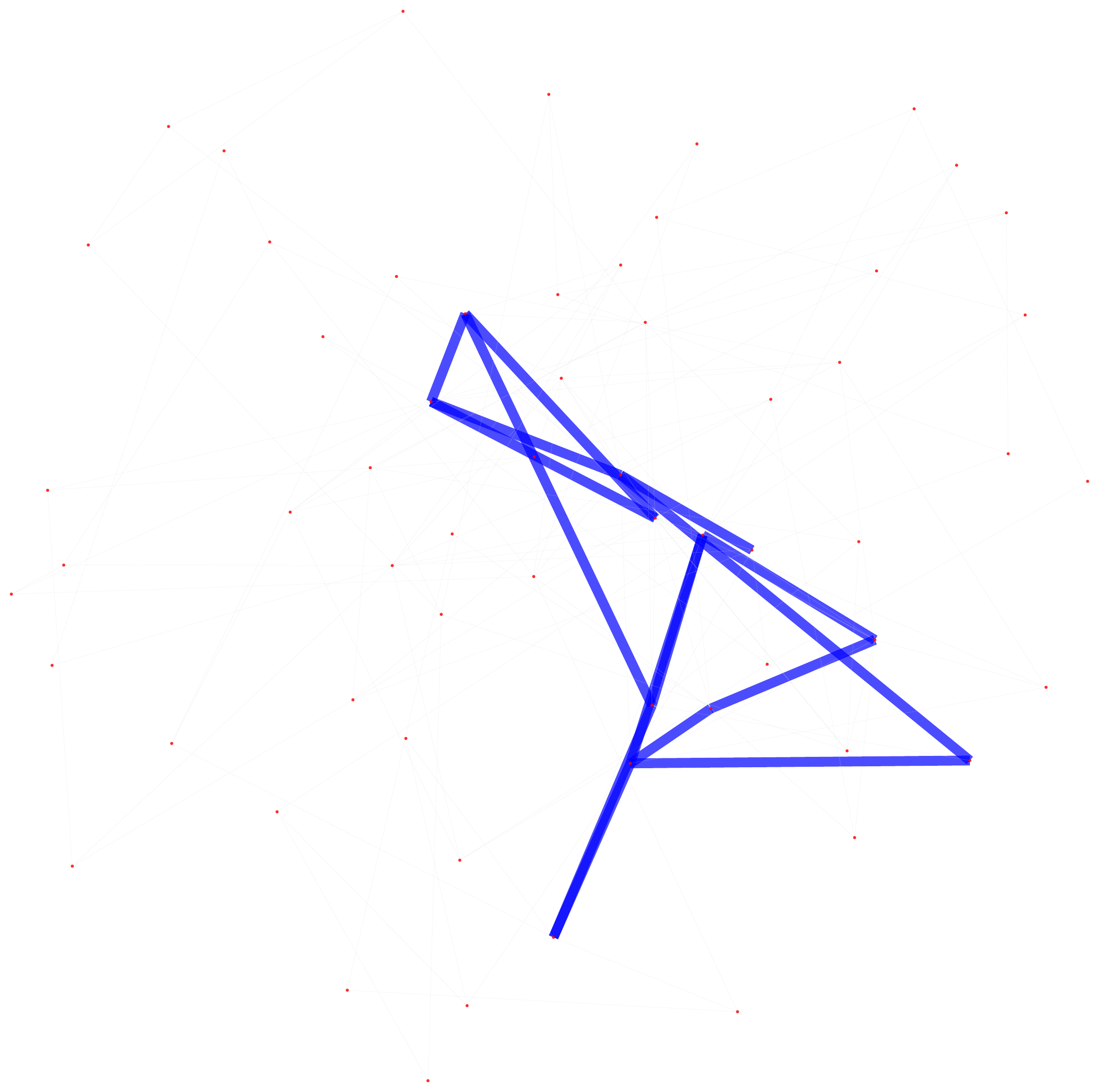}}\hspace{1mm}
\subfloat[  { [30, 2]} ]{\label{fig: 4_30_2}\includegraphics[width=0.10\textwidth]{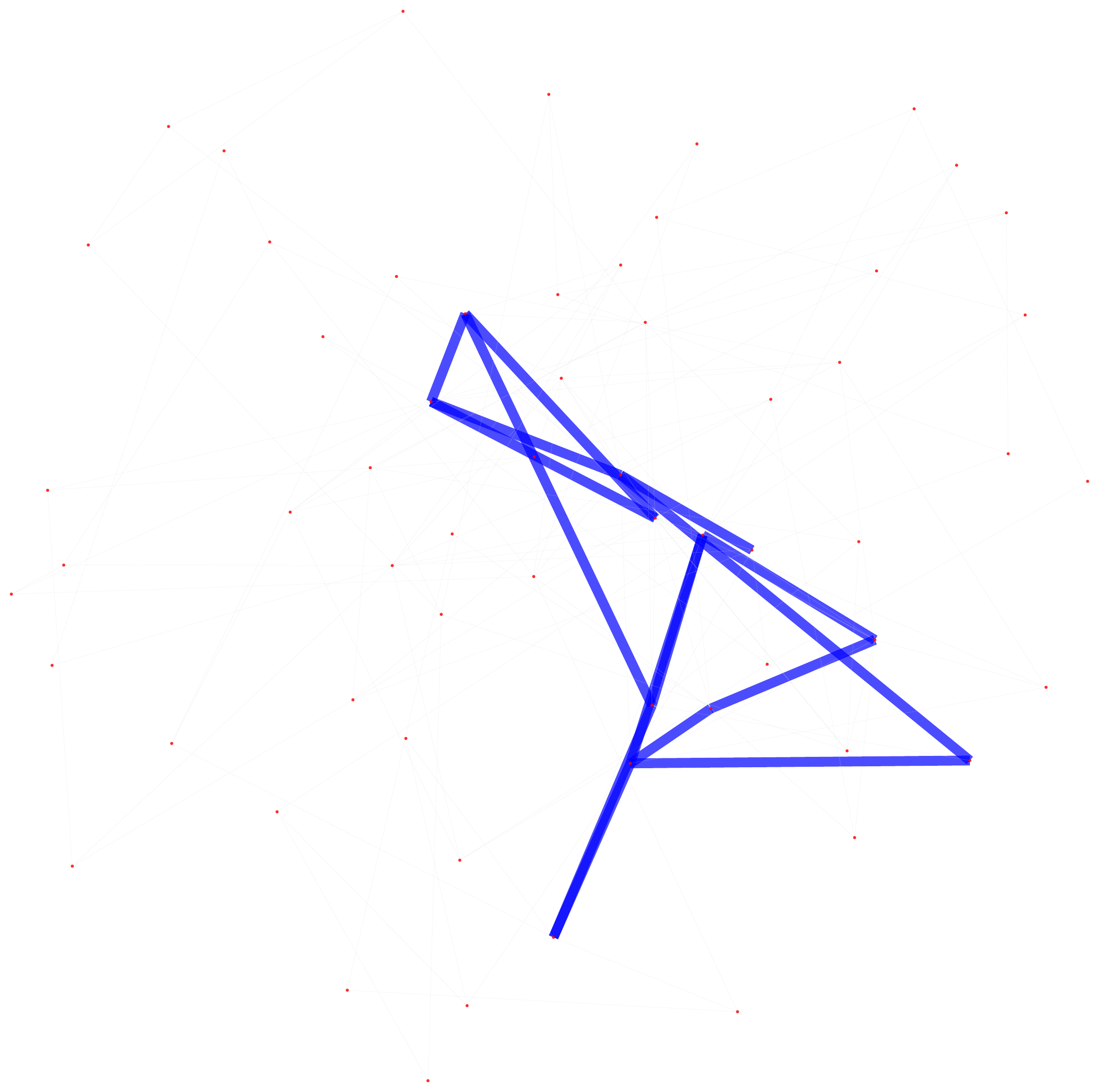}}\hspace{1mm}
\subfloat[  Target ]{\label{fig: 4_30_t}\includegraphics[width=0.10\textwidth]{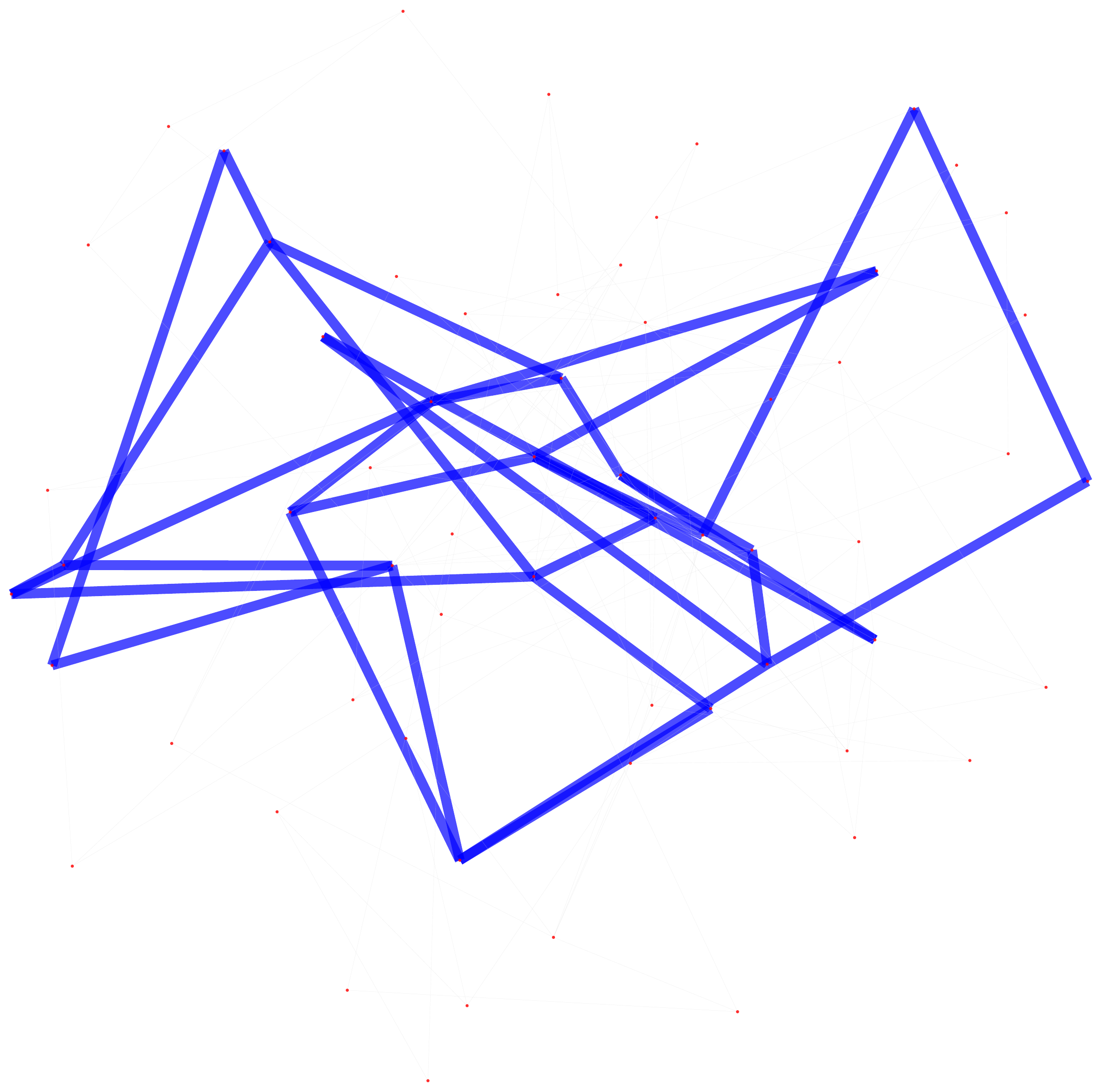}}\hspace{0mm}


\subfloat[ {[60, 0]} ]{\label{fig: log_2_14_8}\includegraphics[width=0.08\textwidth]{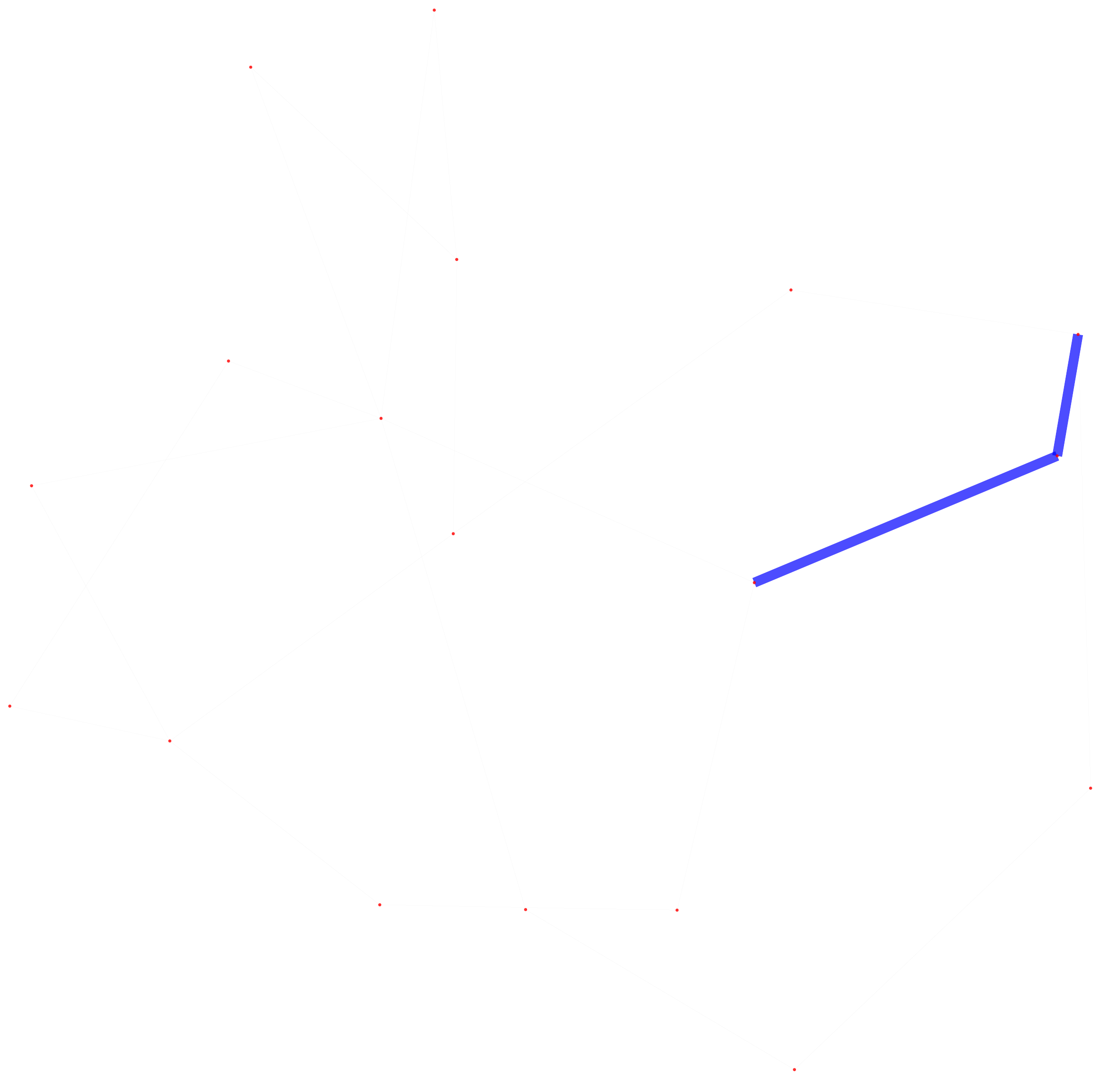}}\hspace{1mm}
\subfloat[ {[60, 1]} ]{\label{fig: log_2_15_8}\includegraphics[width=0.08\textwidth]{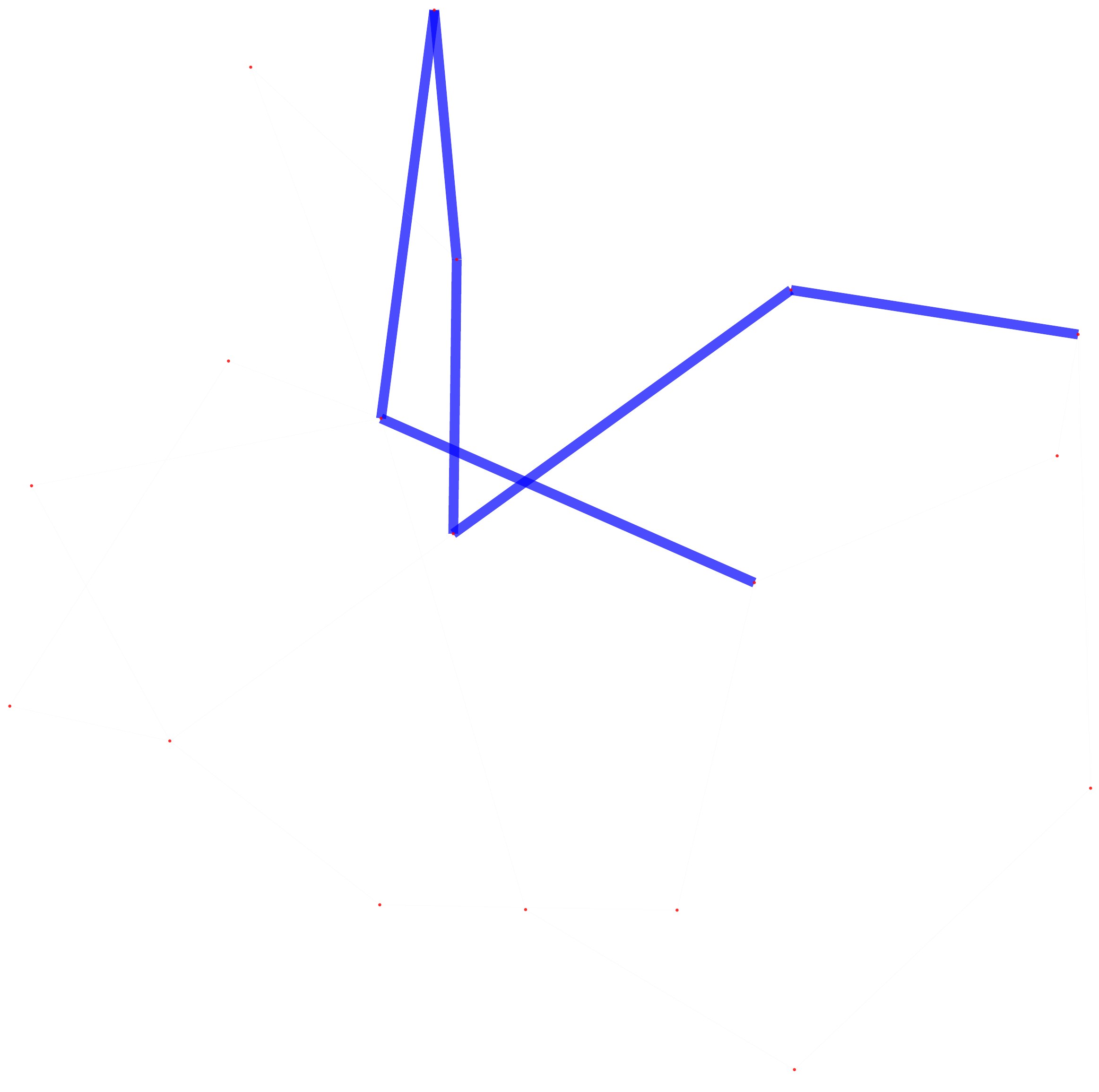}}\hspace{1mm}
\subfloat[ {[60, 2]} ]{\label{fig: log_2_16_8}\includegraphics[width=0.08\textwidth]{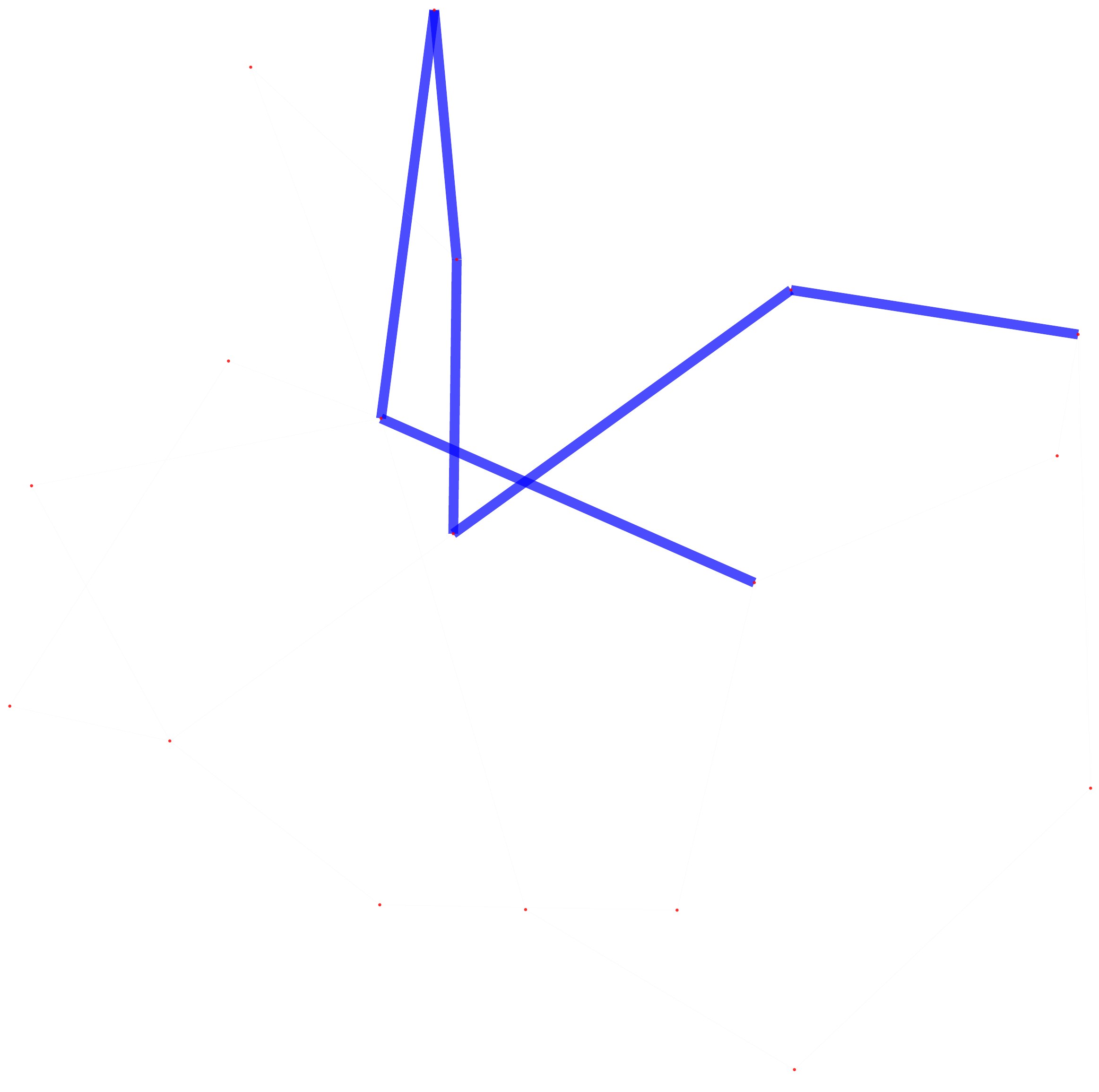}}\hspace{1mm}
\subfloat[ {[60, 5]} ]{\label{fig: log_2_18_8}\includegraphics[width=0.08\textwidth]{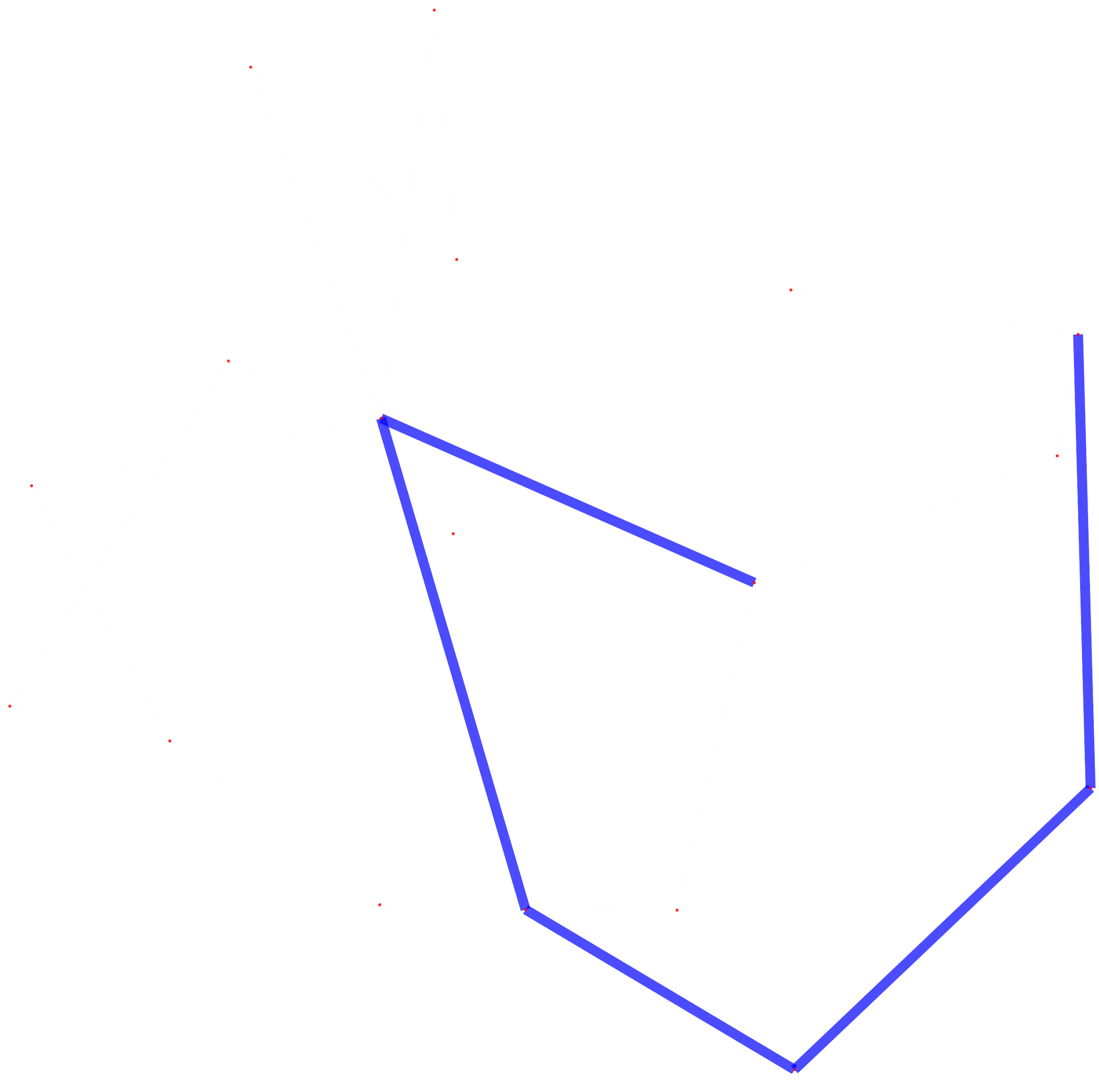}}\hspace{1mm}
\subfloat[Target]{\label{fig: log_2_13_8}\includegraphics[width=0.08\textwidth]{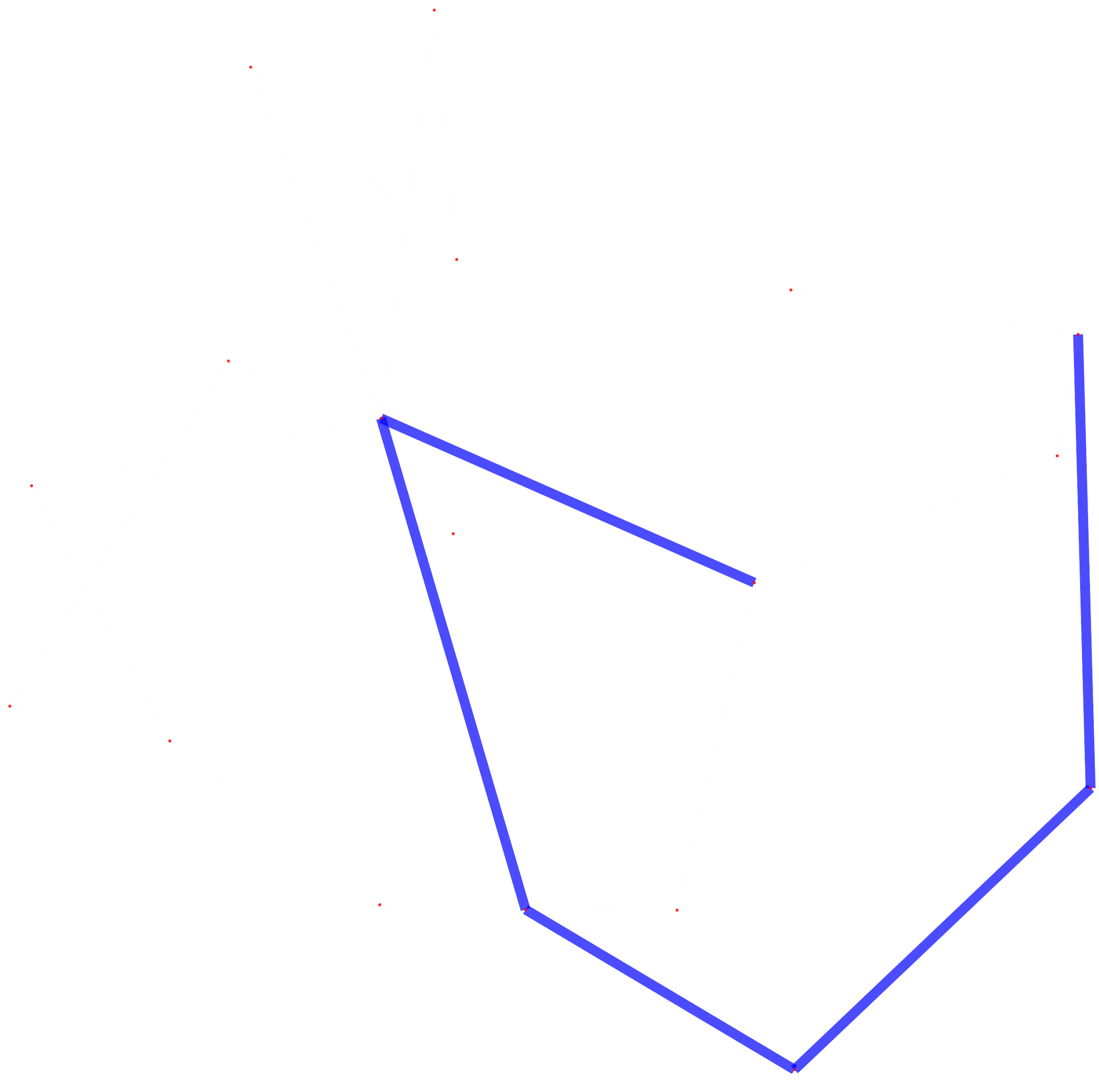}}\hspace{5mm}
\subfloat[  { [240, 0]} ]{\label{fig: 4_240_0}\includegraphics[width=0.10\textwidth]{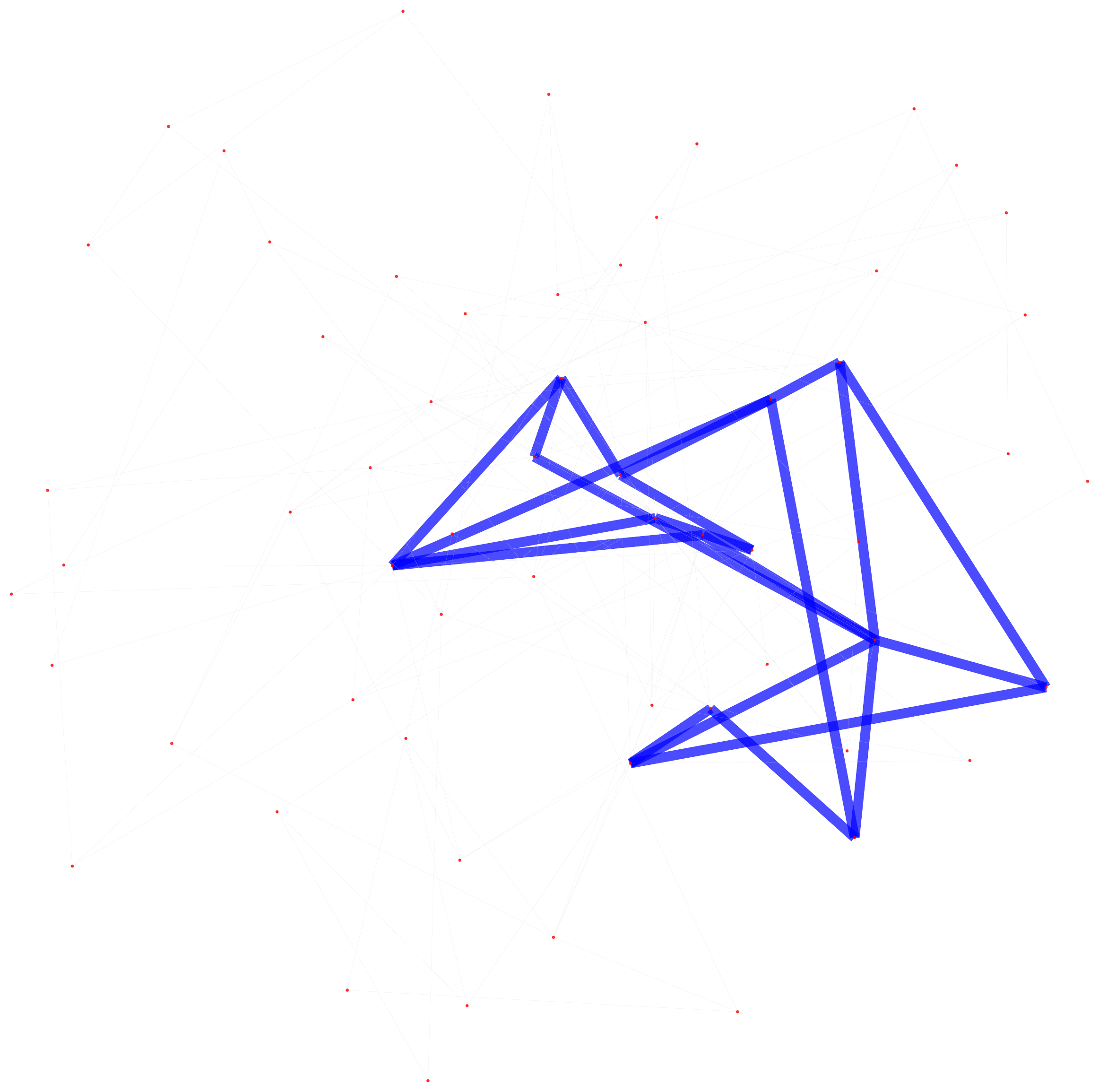}}\hspace{1mm}
\subfloat[  { [240, 1]} ]{\label{fig: 4_240_1}\includegraphics[width=0.10\textwidth]{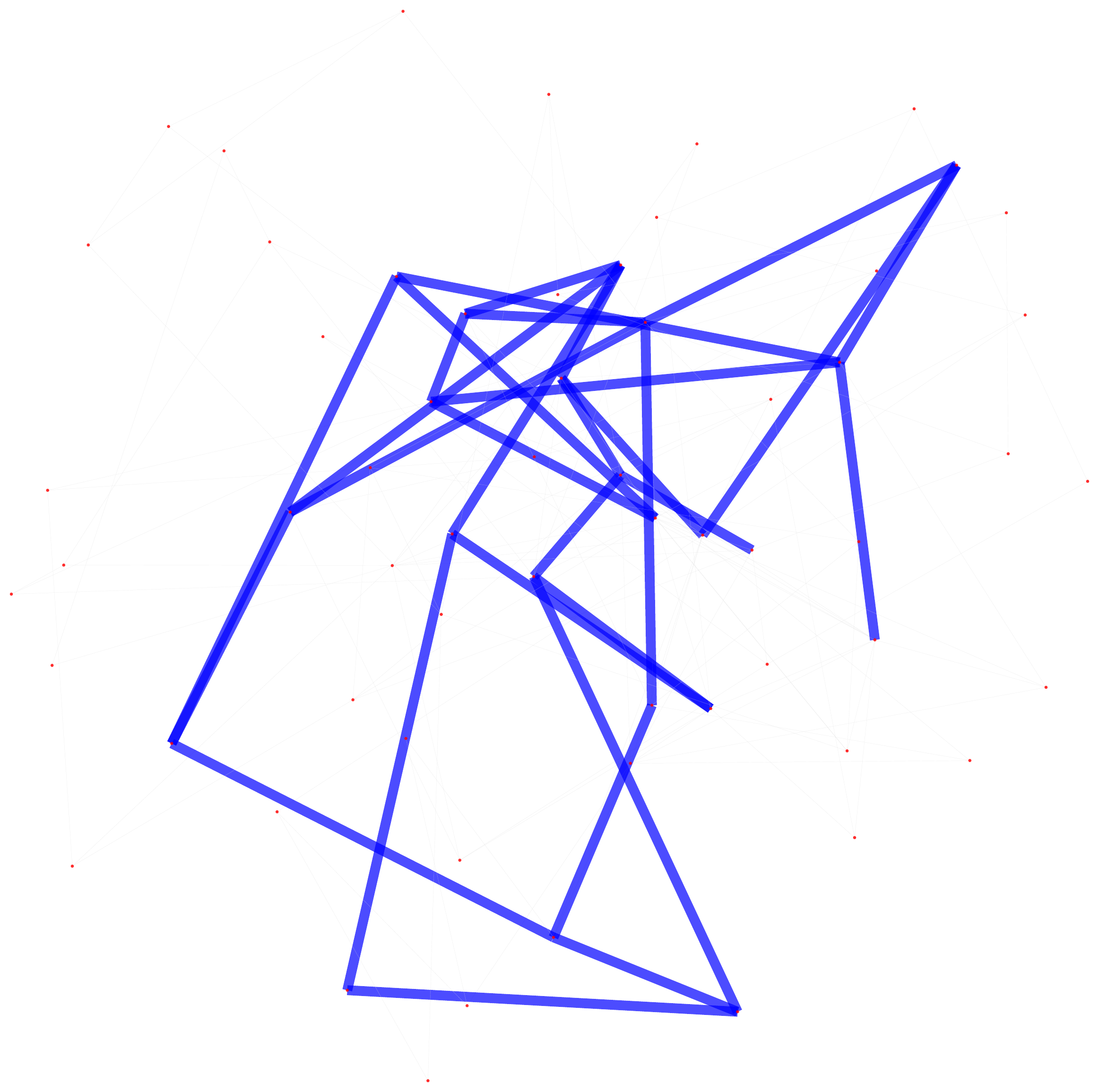}}\hspace{1mm}
\subfloat[  { [240, 2]} ]{\label{fig: 4_240_2}\includegraphics[width=0.10\textwidth]{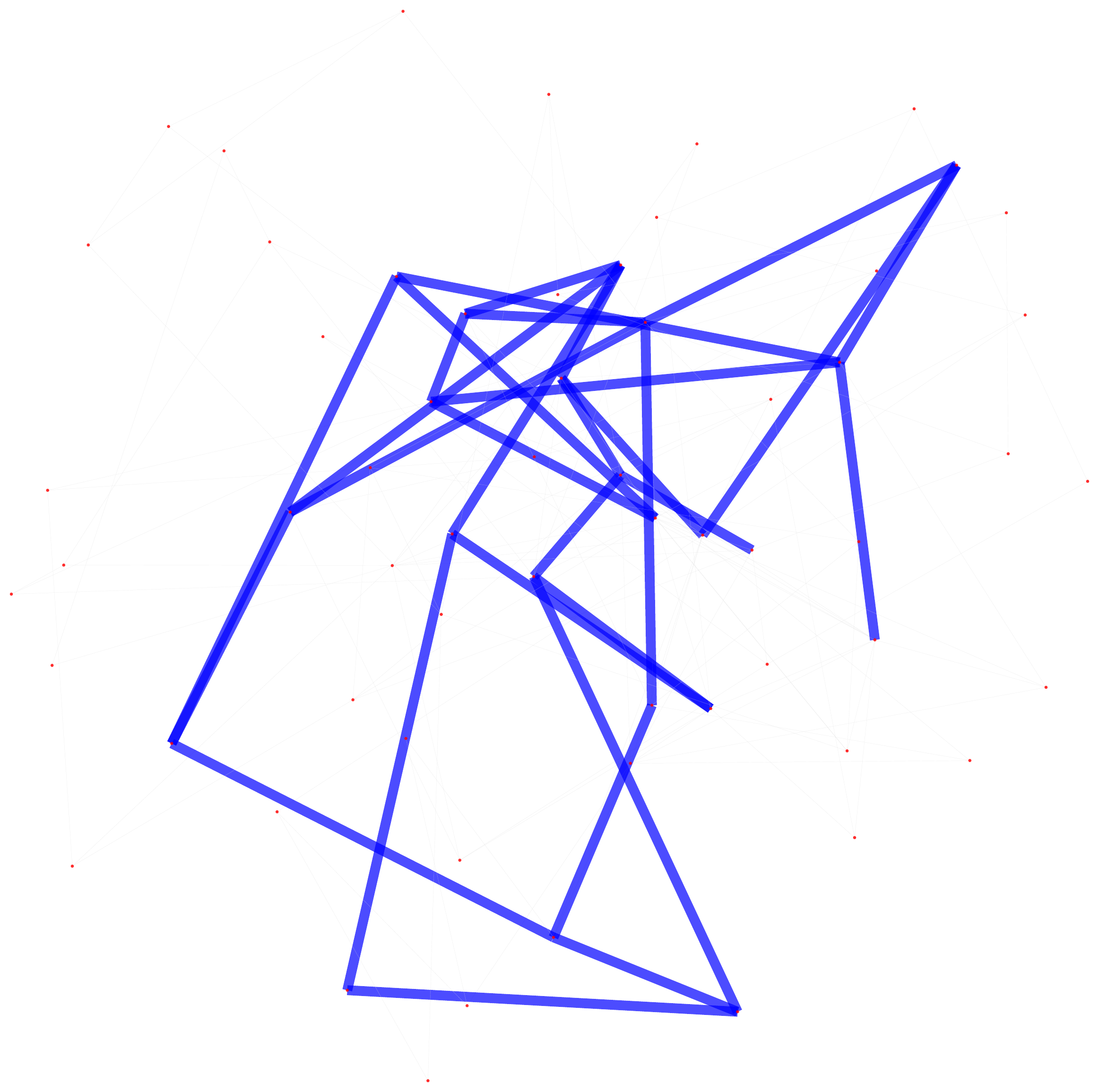}}\hspace{1mm}
\subfloat[  Target ]{\label{fig: 4_240_t}\includegraphics[width=0.10\textwidth]{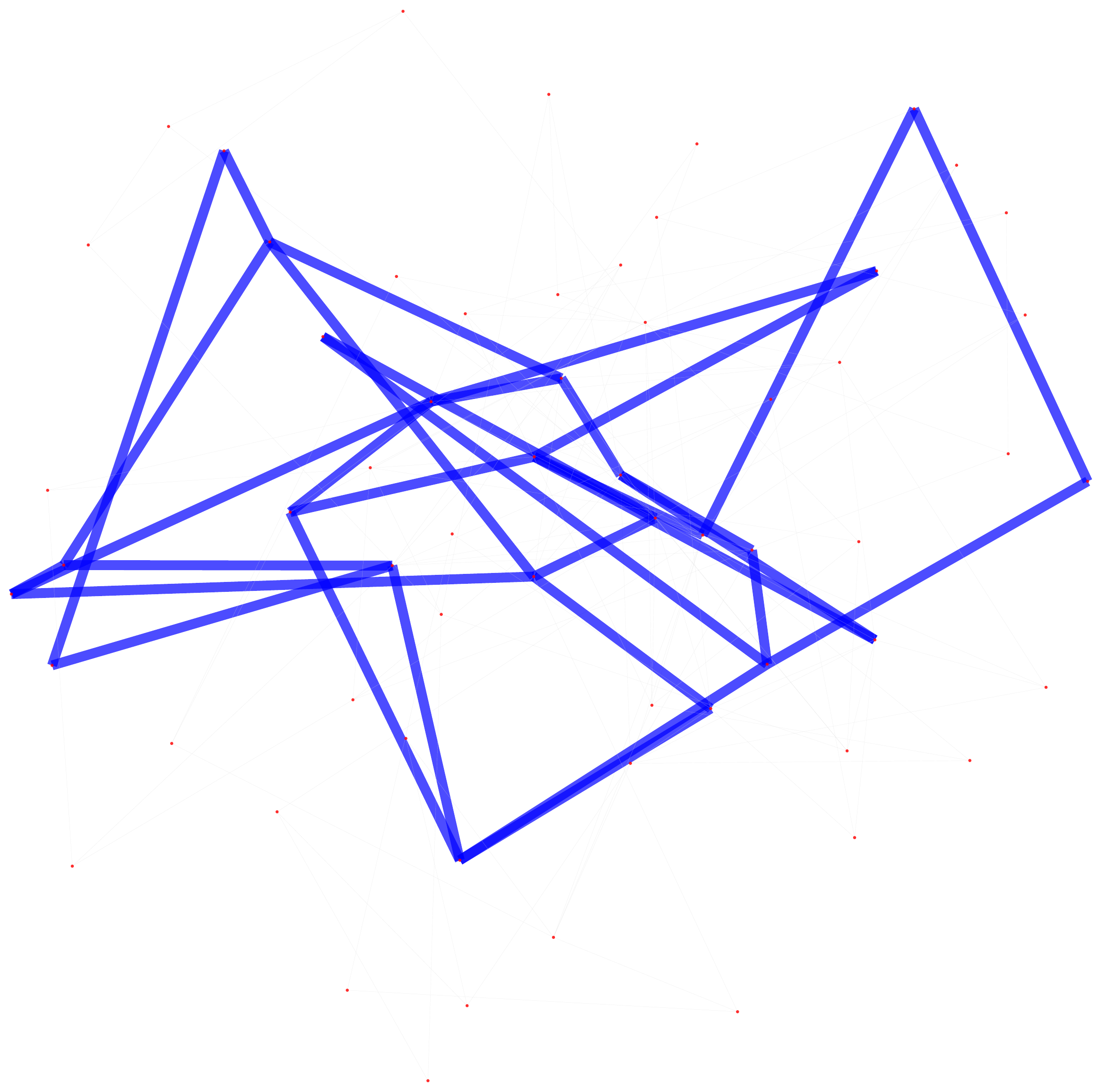}}\hspace{0mm}


\subfloat[ {[240, 0]} ]{\label{fig: log_2_8_8}\includegraphics[width=0.08\textwidth]{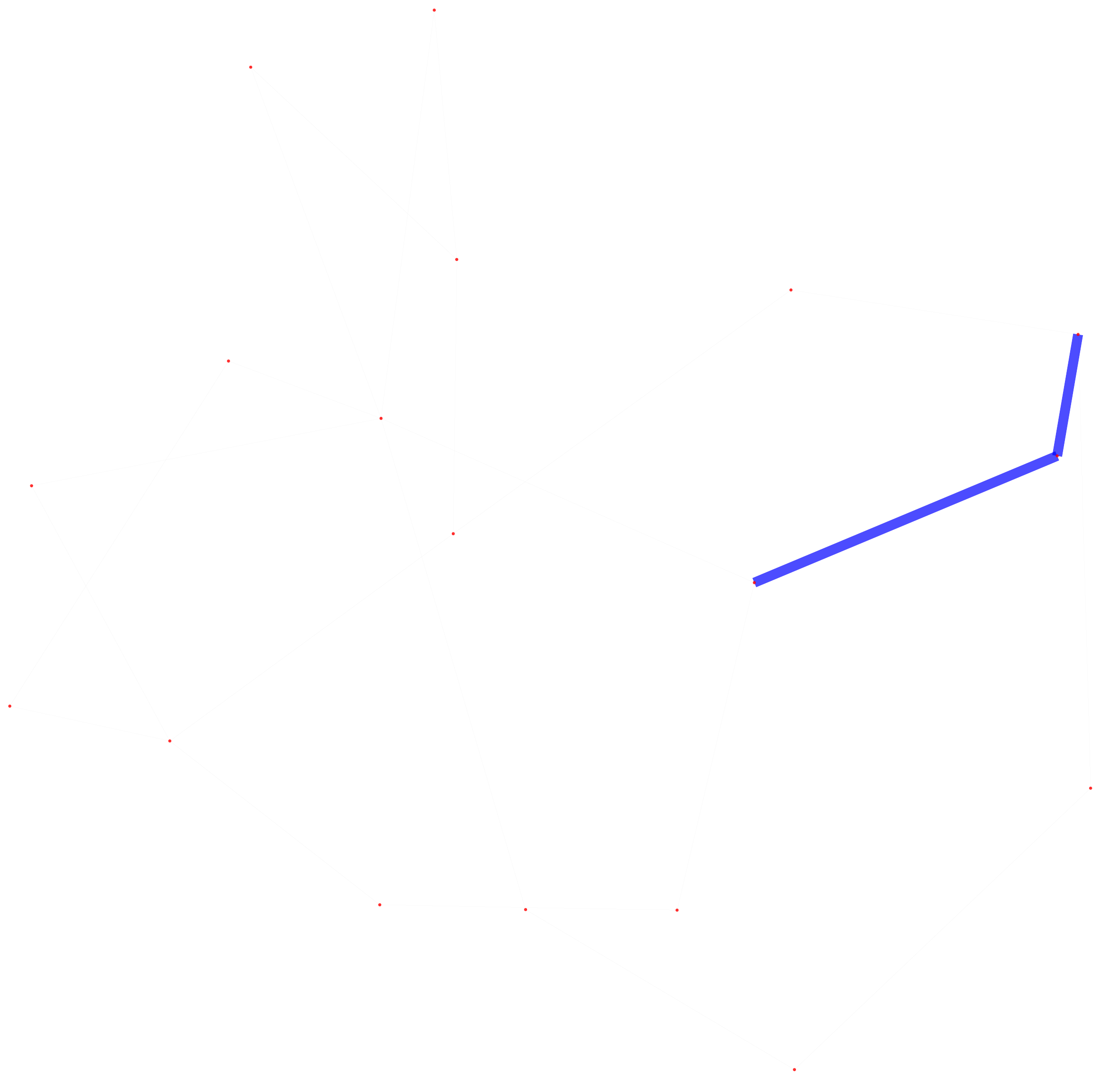}}\hspace{1mm}
\subfloat[ {[240, 1]} ]{\label{fig: log_2_9_8}\includegraphics[width=0.08\textwidth]{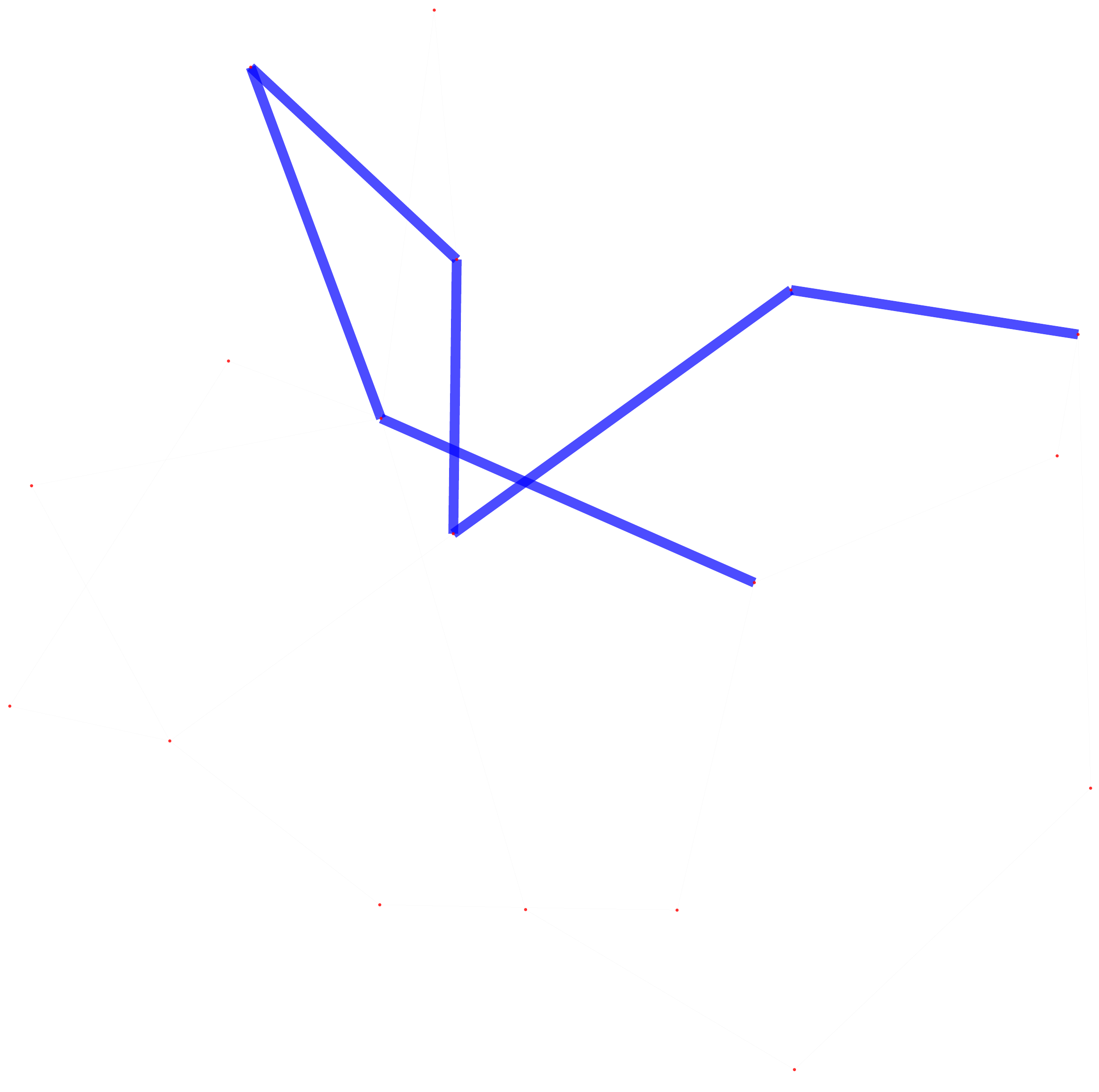}}\hspace{1mm}
\subfloat[ {[240, 2]} ]{\label{fig: log_2_10_8}\includegraphics[width=0.08\textwidth]{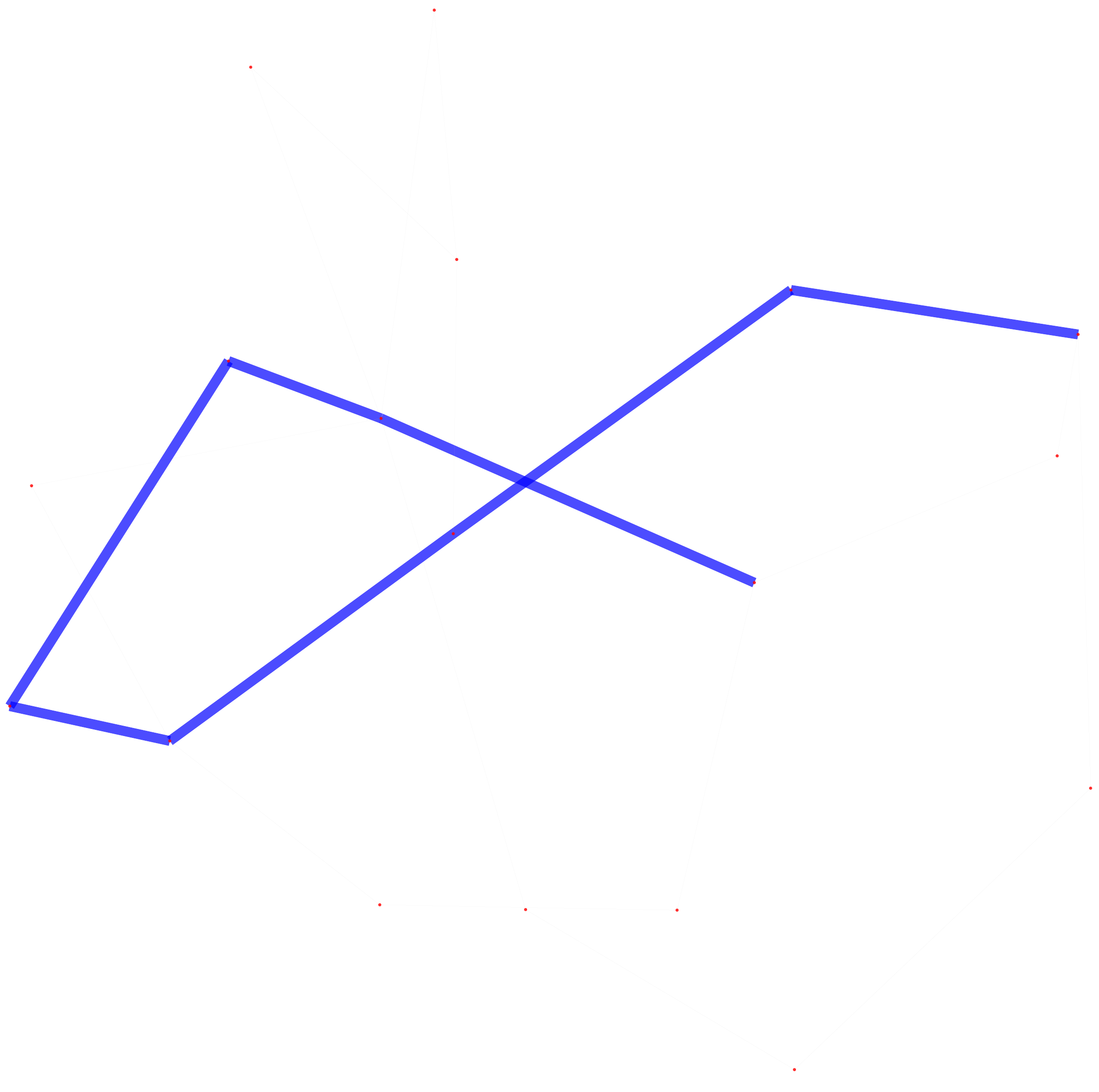}}\hspace{1mm}
\subfloat[ {[240, 5]} ]{\label{fig: log_2_12_8}\includegraphics[width=0.08\textwidth]{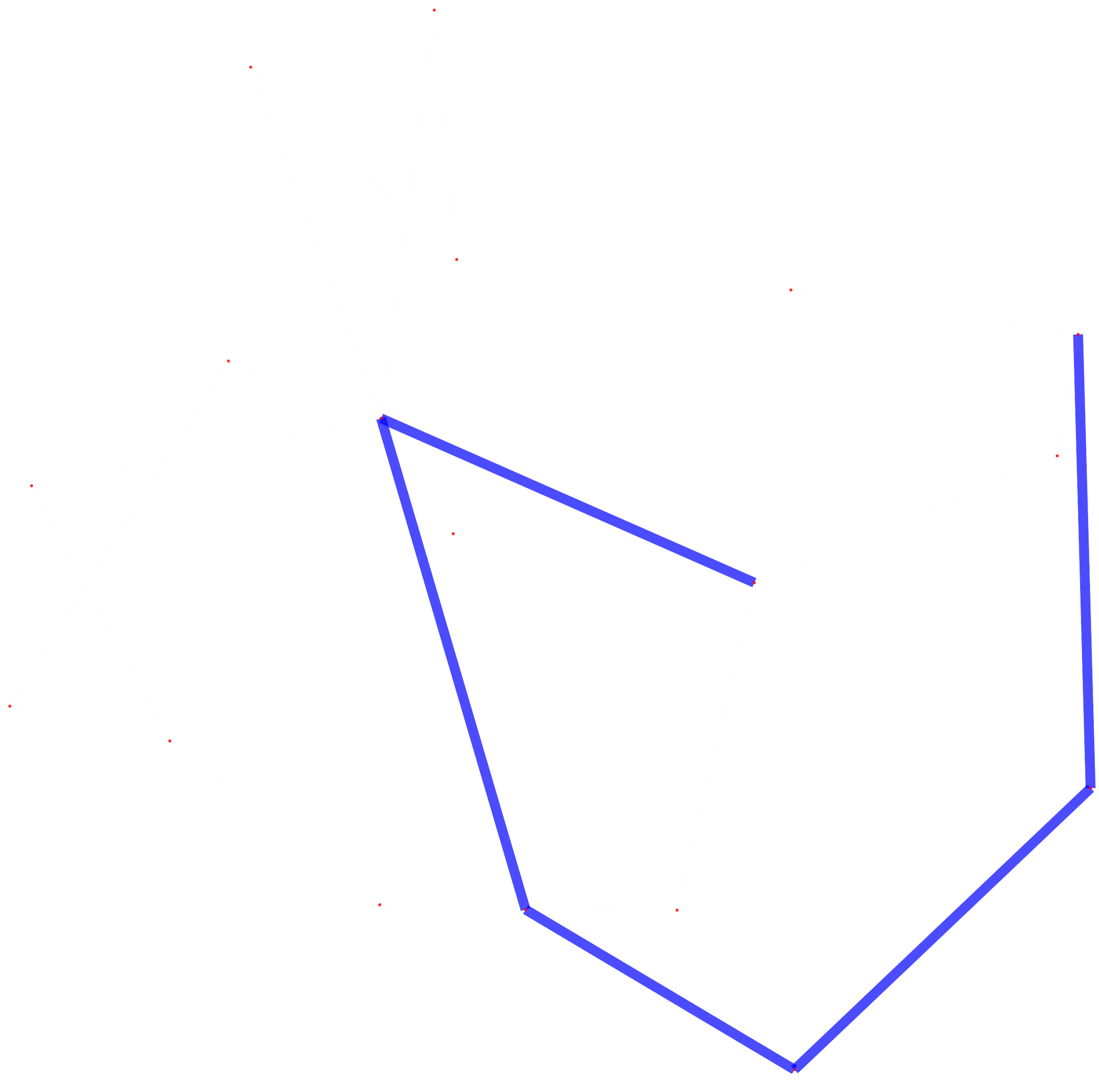}}\hspace{1mm}
\subfloat[Target]{\label{fig: log_2_7_8}\includegraphics[width=0.08\textwidth]{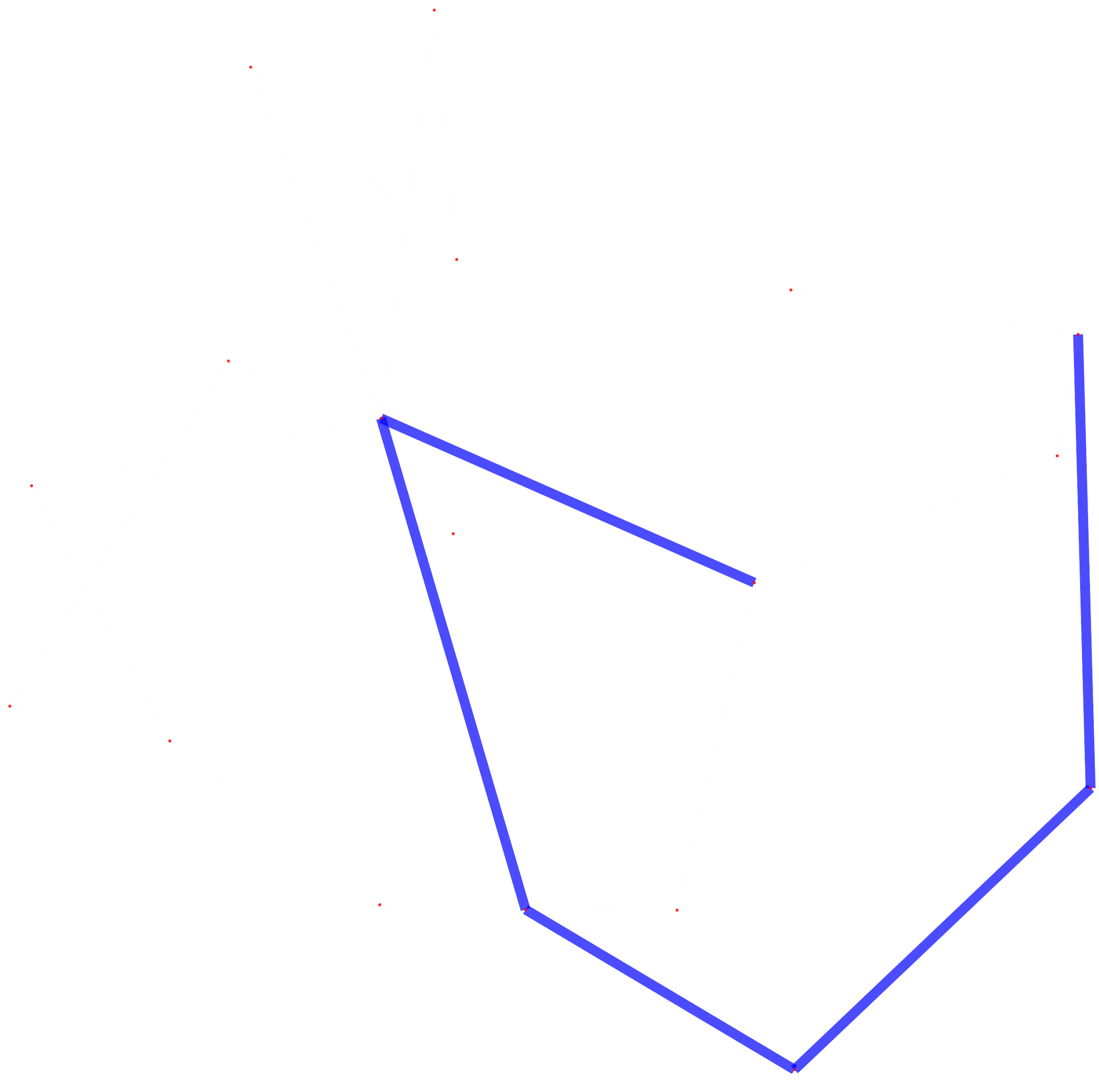}}\hspace{5mm}
\subfloat[  { [480, 0]} ]{\label{fig: 4_480_0}\includegraphics[width=0.10\textwidth]{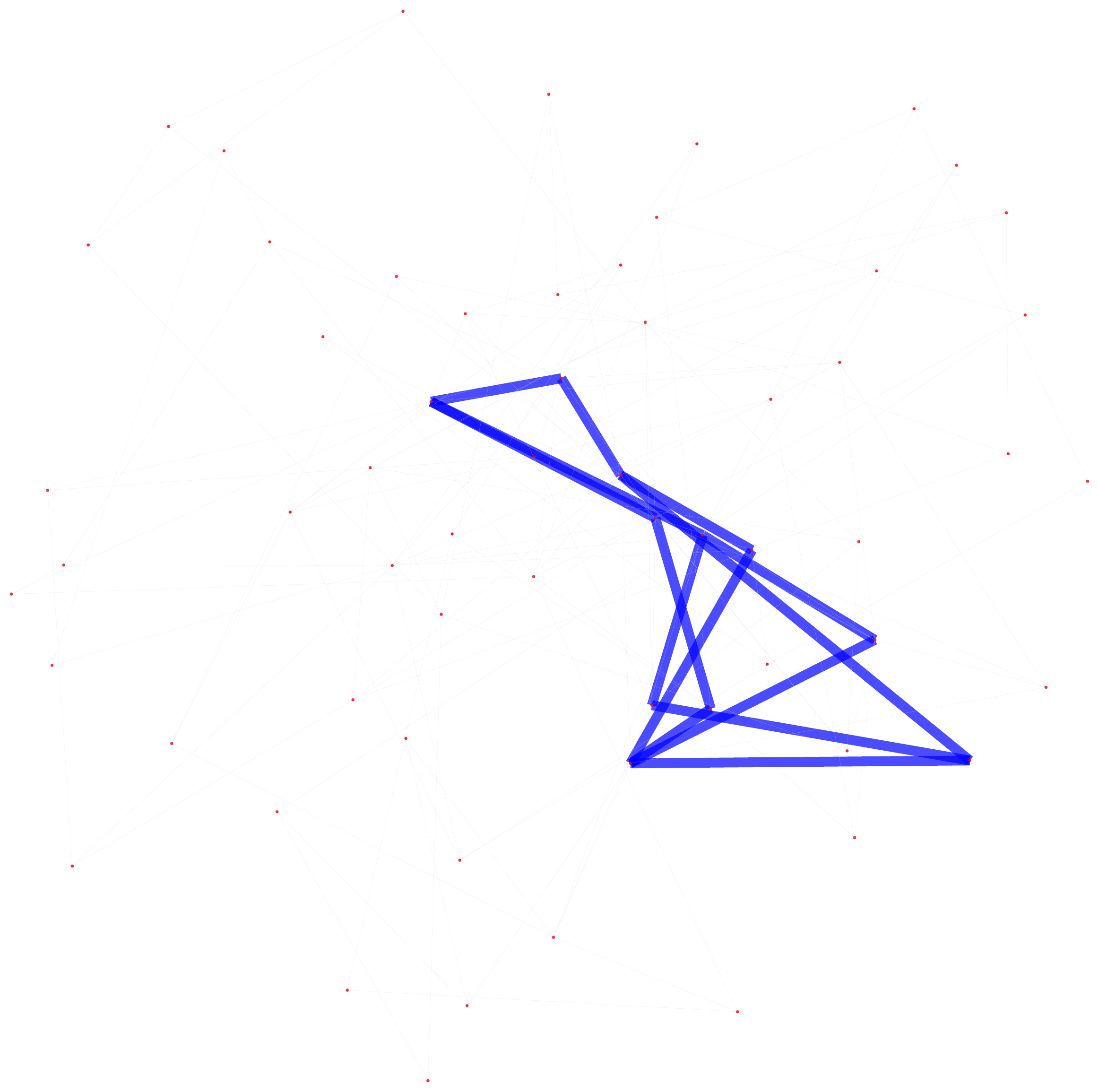}}\hspace{1mm}
\subfloat[  { [480, 1]} ]{\label{fig: 4_480_1}\includegraphics[width=0.10\textwidth]{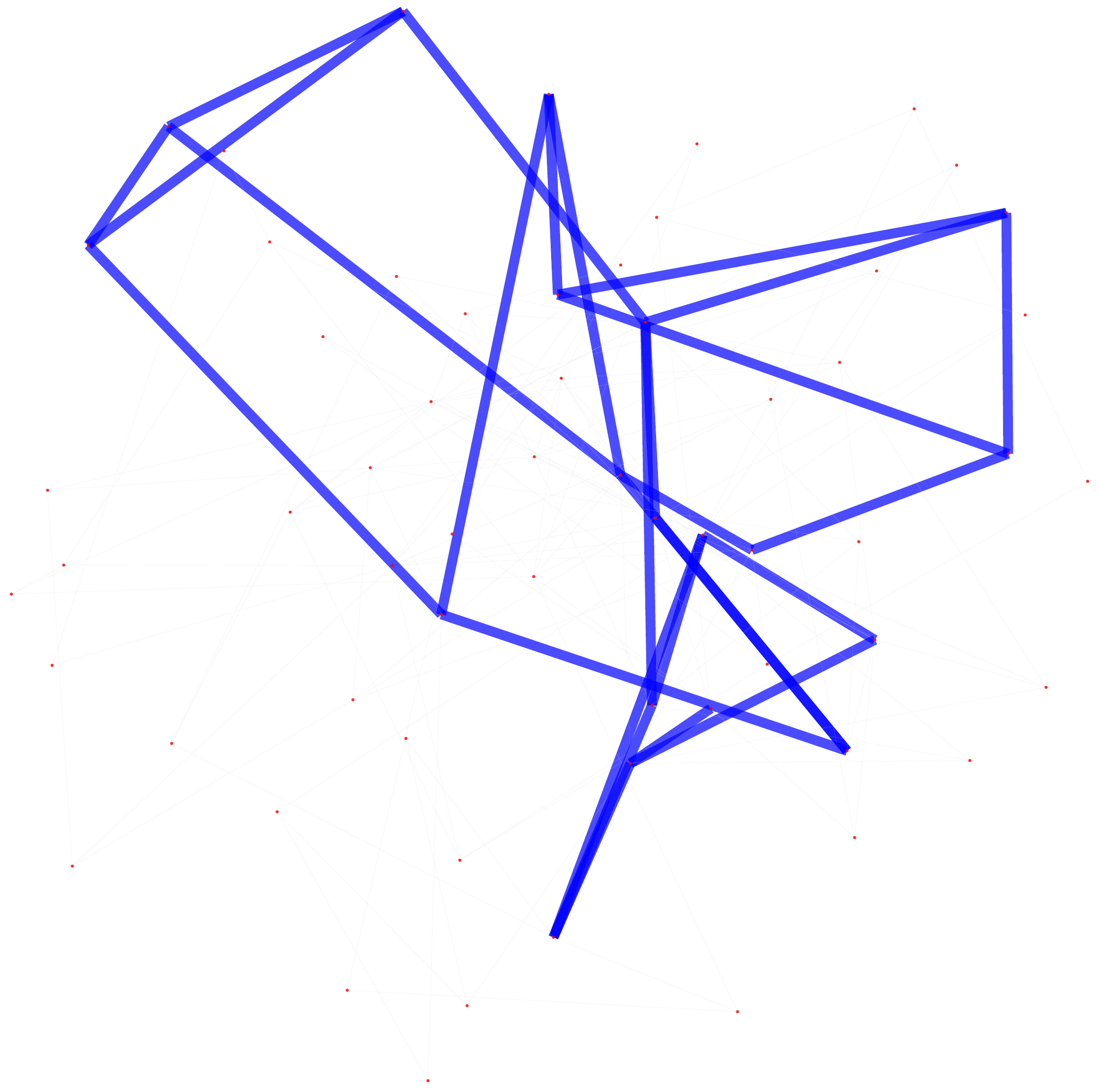}}\hspace{1mm}
\subfloat[  { [480, 2]} ]{\label{fig: 4_480_2}\includegraphics[width=0.10\textwidth]{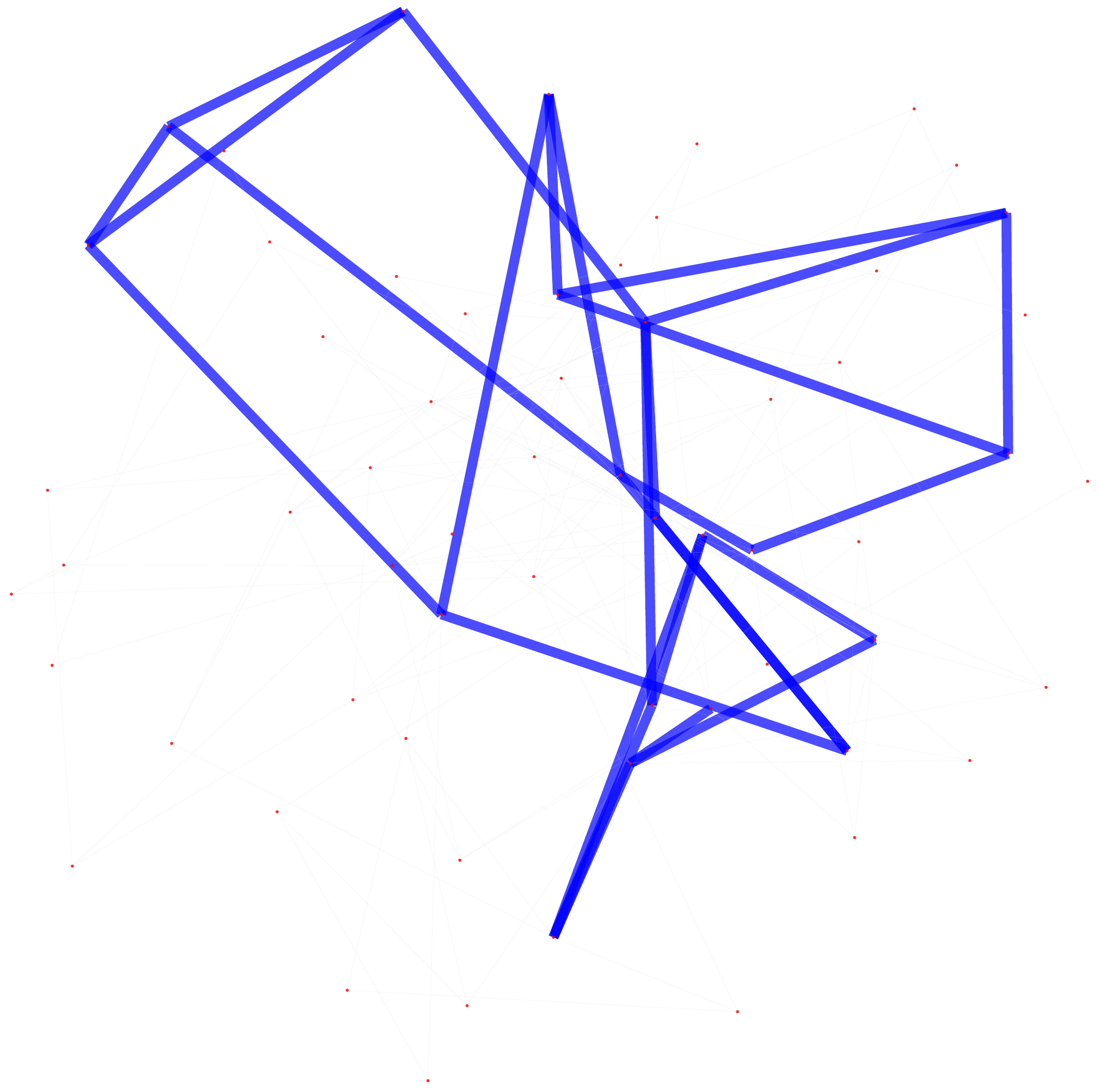}}\hspace{1mm}
\subfloat[  Target ]{\label{fig: 4_480_t}\includegraphics[width=0.10\textwidth]{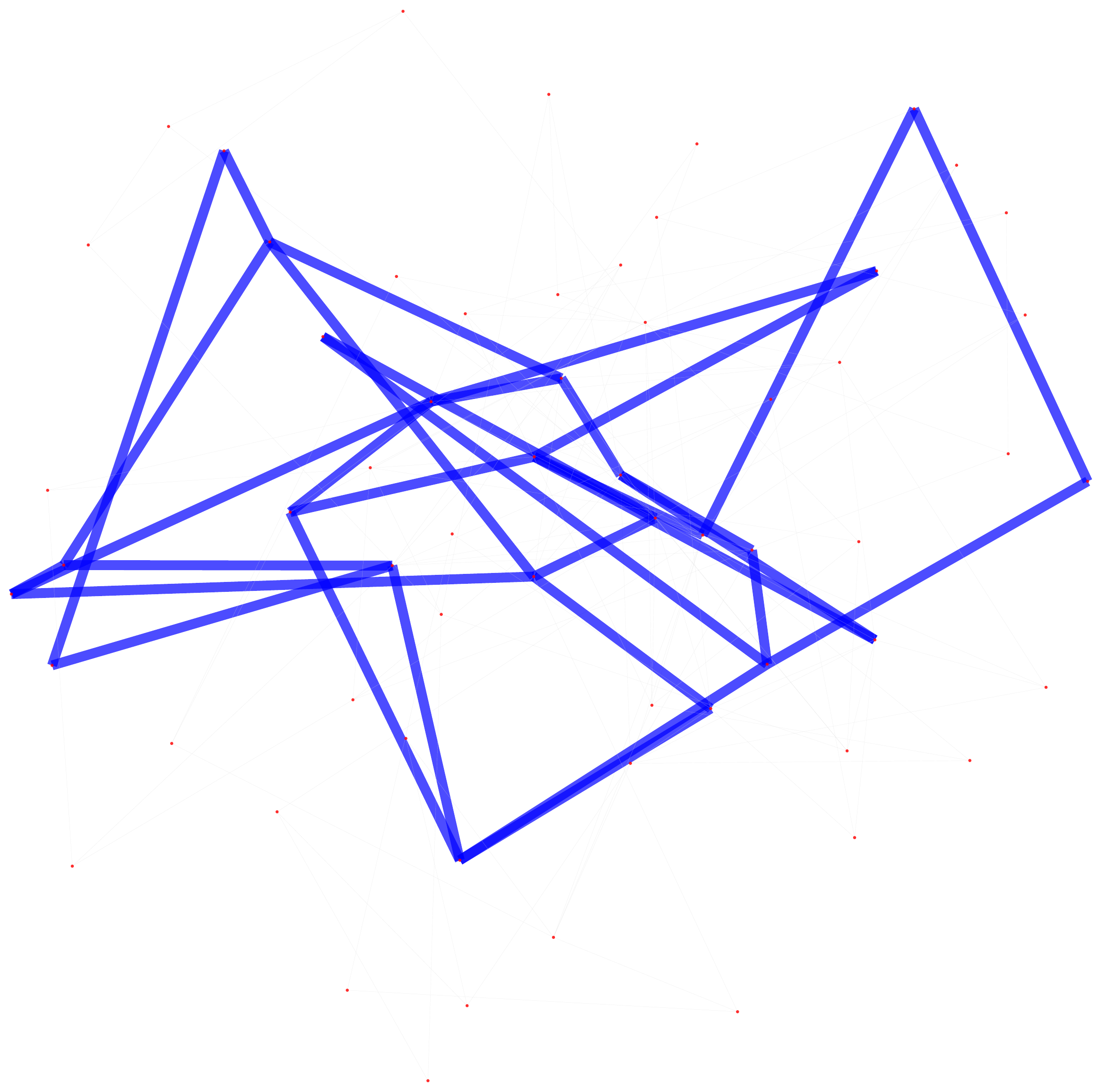}}\hspace{0mm}


\caption{\small The left (resp., right) shows the visualizations of the solution to one testing query for the Tree-to-Path (resp., Path-to-Tree) problem under QRTS-P on Kro. The figure labeled by $[a,b]$ shows the result with training size $a$ after $b$ iterations in the cutting plane algorithm (Appendix \ref{sup: cutting} for solving Eq. (\ref{eq: qrts_p_training}). Each row shows the results under one training size, where the last figure shows the optimal solution.}
\label{fig: visua}
\vspace{-4mm}
\end{figure}

\section{Future Directions}
\label{sec: conc}
Although we observed that the approximation ratio becomes better with the increase in training size and training iteration, it is not always the case that the solution converges to the optimal one -- for example, Figure \ref{fig: visua}-Right and the visualizations in Appendix \ref{sup: exp}. This is reasonable because two solutions may have similar costs but with very different edge sets. Depending on the needs of the agents, other metrics can be adopted, which may require new designs of the hypothesis space and training methods. Another important future direction is to enhance the proposed method through (deep) representation learning. 


\subsubsection{\ackname} This project is supported in part by National Science Foundation
under Award Career IIS-2144285.

\bibliographystyle{splncs04}
\bibliography{PAKDD}

\newpage

\newpage
\appendix

\section*{Query-decision Regression  between Shortest Path and Minimum Steiner Tree (Technical Appendix)}
\section{Proofs}
\label{sup: proof}

\subsection{Proof of Theorem \ref{theorem: approximation}}
To prove the theorem, it suffices to prove that for each $\epsilon \geq 0$ and $\delta \in (0,1)$, there exists an $N \in \mathbb{Z}$ depending on $\epsilon$ and $\delta$, such that we have 
\begin{align}
\Pr_{g_i \sim  \D_{\G}^{em}}\Big[ \Delta \leq \epsilon \Big] \geq 1-\delta    
\end{align} 
when $K \geq N$. This follows from the following lemma.

\begin{lemma}
For each $\epsilon \geq 0$ and $\delta \in (0,1)$, let $\epsilon_1^*$ and $\epsilon_2^*$ be the solution to
\begin{align*}
&\min g(\epsilon_1,\epsilon_2) \\
\text{~s.t.~} \hspace{0.5cm} & g(\epsilon_1,\epsilon_2) = \max(\frac{1}{\epsilon_1^2},\frac{1}{\epsilon_2^2}\ln \frac{1}{\delta}) \\
\hspace{0.5cm} & \epsilon_1+\epsilon_2 =\epsilon \\
\hspace{0.5cm} &\epsilon_1\geq 0, \epsilon_2 \geq 0.
\end{align*}
When $K$ is no less than $g(\epsilon_1^*,\epsilon_2^*)\cdot e^{2D_{\infty}} \cdot \max f_T^2$, with probability least $1-\delta$ over the selection of $\{g_i\}$, there exists $\vec{w}=(w_1,...,w_K)$ such that 

{\small\begin{align*}
&\norm{ \E_{g \sim  \D_{\G}} [f_T(x,y,g)] - \sum_{i=1}^{K} w_i f_T(x,y,g_i) } \\
= & \sqrt{\int_{\X_T \times \Y_T} \Big(\E_{g \sim  \D_{\G}} [f_T(x,y,g)] - \sum_{i=1}^{K} w_i f_T(x,y,g_i)\Big)^2 d \D }\\
\leq &  \epsilon
\end{align*}}%
where $D_{\infty} = \max_{g} \frac{\D_{\G}^{em}[g]}{\D_{\G}[g]}$ is the $\infty$-order Rényi divergence between $\D_{\G}$ and $\D_{\G}^{em}$
\end{lemma}

\begin{proof}
We prove this lemma by showing that the desired $w_i$ is $\frac{\D_{\G}^{em}[g_i]}{K \D_{\G}[g_i]}$. Let us define $h_{\{g_i\}}(x,y)\define \sum_{i} \frac{\D_{\G}^{em}[g_i]}{K \D_{\G}[g_i]} f_T(x,y,g_i)$ and 
{\small\begin{align}
\Delta_{\{g_i\}}\define \sqrt{\int_{\X_T \times \Y_T} \Big(\E_{g \sim  \D_{\G}} [f_T(x,y,g)] - h_{\{g_i\}}(x,y)\Big)^2 d \D }.
\end{align}}%
Through elementary calculation, the expectation of $\Delta_{\{g_i\}}$ can be bounded as follows. 
\begin{align*}
    &\E_{g_i \sim \D_{\G}^{em}} [\Delta_{\{g_i\}}] \\
    \leq &\frac{1}{\sqrt{K}}\max_{g} \frac{\D_{\G}^{em}[g]}{\D_{\G}[g]}\cdot \max f_T(x,y,g) \\
    =& \frac{e^{D_{\infty}}}{\sqrt{K}}  \cdot \max f_T(x,y,g) \leq \epsilon_1^*.
\end{align*}
In addition, we noticed that for any two collections $\{g_i^1\}$ and $\{g_i^2\}$ that differ by at most one element, the change in $\Delta_{\{g_i\}}$ can be bounded by 
\begin{align*}
    &|\Delta_{\{g_i^1\}}-\Delta_{\{g_i^2\}}|  \\
    =&|\norm{ \E_{g \sim  \D_{\G}} [f_T(x,y,g)] - h_{\{g_i^1\}}(x,y) }_{\D^*} \\
    & \hspace{2cm} -\norm{ \E_{g \sim  \D_{\G}} [f_T(x,y,g)] - h_{\{g_i^2\}}(x,y) }_{\D}|\\
    \leq & \norm{h_{\{g_i^1\}}(x,y)- h_{\{g_i^2\}}(x,y)}_{\D}\\
    \leq & \max_{g} \sqrt{\int_{\X \times \Y} 2 \Big(\frac{\D_{\G}^{em}[g]}{K \D_{\G}[g]} f_T(x,y,g) \Big)^2 d \D }\\
    \leq & \frac{\sqrt{2}}{K}\max_{g} \frac{\D_{\G}^{em}[g]}{\D_{\G}[g]}\cdot \max f_T(x,y,g)\\
     = & \frac{\sqrt{2}}{K} \exp(D_{\infty}(\D_{\G}^{em} || \D_{\G}))\cdot \max f_T(x,y,g)
\end{align*}
The following arguments complete the proof.
\begin{align*}
&\Pr_{g_i\sim  \D_{\G}^{em} }[\Delta_{\{g_i\}}\geq \epsilon] \\
= & \Pr_{g_i  \sim \D_{\G}^{em}}[\Delta_{\{g_i\}} - \epsilon_1^* \geq \epsilon_2^*] \\
\leq & \Pr_{g_i \D_{\G}^{em} \sim }[\Delta_{\{g_i\}} - \E_{g_i \sim \D_{\G}^{em}} [\Delta_{\{g_i\}}] \geq \epsilon_2^*]\\
\leq&\exp\Big(\frac{-2K^2 (\epsilon_2^*)^2}{\big(\sqrt{2}\max_{g} \frac{\D_{\G}^{em}[g]}{\D_{\G}[g]}\cdot \max f_T \big)^2}\Big) \leq \delta
\end{align*}
\end{proof}

\subsection{Proof of Theorem \ref{theorem: consistency}}


Let $y_{x}^*$ be the optimal solution to $\min_{y \in \Y_T} F_{f_T,\D}(x,y)$, and $y_{x}^{\beta}$ be the $\beta$-approximation to $\min_{y \in \Y_T} h_{\vec{w},\{g_i\}}(x,y)$ . Suppose that for each $x$ we have 
\begin{align*}
\big|h_{\vec{w},\{g_i\}}(x,y_{x}^*)-F_{f_T,\D}(x,y_{x}^*)\big|\leq \epsilon \cdot F_{f_T,\D}(x,y_{x}^*)    
\end{align*}
and
\begin{align*}
\big|h_{\vec{w},\{g_i\}}(x,y_{x}^{\beta})-F_{f_T,\D}(x,y_{x}^{\beta}) \big|\leq \epsilon \cdot F_{f_T,\D}(x,y_{x}^{\beta}).    
\end{align*}
This implies that
\begin{align*}
 (1+\epsilon)F_{f_T,\D}(x,y_{x}^*)   \geq & h_{\vec{w},\{g_i\}}(x,y_{x}^*) \\
\geq & \frac{1}{\beta} h_{\vec{w},\{g_i\}}(x,y_{x}^{\beta})  \\
\geq & \frac{1-\epsilon}{\beta} F_{f_T,\D}(x,y_{x}^{\beta}).
\end{align*}
Therefore, Theorem \ref{theorem: consistency} follows directly from the following lemma.
\begin{lemma}
\label{lemma: pointwise}
With probability at least $1-\delta$ over the sampling of $\{g_1,...,g_K\}$ from $\D_{\G}^{em}$, there exists $\vec{w}$ such that for each $x \in \X_T$ and $y \in \Y_{T}$, we have 
\begin{align}
|h_{\vec{w},\{g_i\}}(x,y)-F_{f_T,\D_{\G}}(x,y)|\leq \epsilon \cdot F_{f_T,\D}(x,y)    
\end{align}
provided that $K$ is no less than
\begin{align*}
    \frac{2(C_1+\epsilon)(\ln 2 + \ln |\X_T| +\ln |\Y_T| -\ln \delta)}{C_2 \epsilon^2},
\end{align*}
where $C_1 =\max_{x,y} F_{f_T,\D_{\G}}(x,y) \cdot \exp(D_{\infty})$, $C_2 = \min_x F_{f_T,\D}(x,y_x^*)$, and  $D_{\infty} = \max_{g} \frac{\D_{\G}^{em}[g]}{\D_{\G}[g]}$.
\end{lemma}
\begin{proof}
 For each collection $\{g_1,...,g_K\}$, let us denote the associated weight as $w_i=\frac{\D_{\G}[g_i]}{K\D_{\G}^{em}[g_i]}$. For each $x \in \X_T$, $y \in \Y_{T}$ and $i$, we notice that $\E_{g_i \sim \D_{\G}^{em}}[w_i f_T(x,y, g_i)] = F_{f_T,\D}(x,y)$. Therefore, with the standard concentration inequalities (e.g., Chernoff bound), we have 
\begin{align*}
&\Pr\Big[|h_{\vec{w},\{g_i\}}(x,y)-F_{f_T,\D}(x,y)|\leq \epsilon \cdot F_{f_T,\D}(x,y)\Big]  \\ 
\leq & 2 \exp(\frac{-K F_{f_T,\D}(x,y) (\epsilon/C_1)^2}{C_1(2+\epsilon/C_1)}) \\
\leq &2 \exp(\frac{-K \min_x F_{f_T,\D}(x,y_x^*) (\epsilon/C_1)^2}{C_1(2+\epsilon/C_1)}) \\
=&  2 \exp(\frac{-K C_2 (\epsilon/C_1)^2}{C_1(2+\epsilon/C_1)})
\end{align*}
By the selection of $K$, the above probability is no larger than $\frac{\delta}{|\X_T|\cdot |\Y_T|}$. Taking the union bound over $\X_T$ and $\Y_T$, the above result suggests that $w_i$ is the sought-after weight.
   
\end{proof}

\begin{algorithm}[tp!]
	\caption{One-slack Cutting Plane}\label{alg: cutting}
	\begin{algorithmic}[1]
		\State \textbf{Input:} $(M_1,P_1),...,(M_n,P_n),C,\epsilon, \alpha$;
		\State $\mathcal{W}\leftarrow \emptyset$;
		\Repeat
		\State Solve the QP over constraints $\mathcal{W}$:
		\begin{align*}		\hspace{1cm}(\vec{w},\xi)\leftarrow &  \argmin   \frac{1}{2} \norm{\vec{w}}^2_2+C\cdot \xi_i  \\
		& \text{s.t.}   \forall (\overline{P}_1,...\overline{P}_n) \in \mathcal{W}:\\ 
		&   \frac{1}{n}\vec{w}{\tran} \sum_{i=1}^{n} \big(\vec{G}(M_i,P_i)-  \vec{G}(M_i,\overline{P}_i) \big) \\
  & \hspace{0.5cm}\geq  \frac{\alpha}{n}\cdot \sum_{i=1}^{n} L(P_i,\overline{P}_i)-\xi;\\ 
		&    \vec{w} \geq 0.
		\end{align*}
		\For {$i=1,...,n$}
		\State $\hat{P}_i \leftarrow \argmin_{|S|\leq k}  -\vec{w}{\tran} \vec{G}(M_i,S)$; 
		\EndFor 
		\State $\mathcal{W}\leftarrow \mathcal{W} \cup \{(\hat{P}_1,...\hat{P}_n)\}$;
		\Until $\frac{\alpha}{n}\cdot \sum_{i=1}^{n} L(P_i,\hat{P}_i)-\frac{1}{n}\vec{w}{\tran} \sum_{i=1}^{n} \big(\vec{G}(M_i,P_i)-  \vec{G}(M_i,\hat{P_i} \big)\leq \xi+\epsilon$
	\end{algorithmic}
	\vspace{-1mm}
\end{algorithm}

\section{One-slack Cutting Plane}
\label{sup: cutting}
For the optimization problems in Eqs. (\ref{eq: qrts_p_training}) and (\ref{eq: qrts_d_training}), we adopt the one-slack cutting plane algorithm \cite{joachims2009cutting} for the large-margin training, which is implemented through PyStruct \cite{muller2014pystruct}. A generic pseudo-code is given in  Algorithm \ref{alg: cutting}, where the kernel function $\vec{G}$ corresponds to $h_{\vec{w},\{g_i\}}$ under our context and $L$ is the zero-one loss.

\newpage

\section{Experiments}
\label{sup: exp}

\subsection{Settings}
More details about the adopted graphs can be found in Table \ref{table: data}. Our experiment was executed on Amazon EC2 C5 Instance with 96 virtual CPUs and a memory of 192G. All the experiments were able to finish within a reasonable time, although denser graphs may take longer time.

		


\begin{table}[t]
	\renewcommand{\arraystretch}{1.2} 
	\small
	\caption{\textbf{Datasets.} }
	\centering
	\label{table: data}
	\begin{tabular}{  @{}  l   @{\hspace{1mm}}   l @{\hspace{1mm}}  l @{\hspace{1mm}}  l@{\hspace{0mm}} l @{\hspace{0mm}}  l @{\hspace{2mm}}  l }
  & Nodes &    Edges  & References   	\\
		
		\midrule
   Kro &    1024  & 2655  & \cite{leskovec2010kronecker}	\\
		Col & 512& 	1000		&  \cite{demetrescu2008implementation}   \\
		
		 BA & 1000  & 	  500	&  \cite{barabasi1999emergence}    \\

   WS-dense & 1000  & 	 10000	&  \cite{watts1998collective}      \\

   WS-sparse &  1000 & 	 5000	&  \cite{watts1998collective}     \\

		\bottomrule
	\end{tabular}
\end{table}

\subsection{Results on WS}
The results on WS-sparse and WS-dense can be found in Tables \ref{table: WS_sparse} and \ref{table: WS_dense}. Overall, we have similar observations mentioned in the main paper. In particular, for Path-to-Tree, the models can successfully learn from large datasets; in addition, QRTS-P is still comparable to QRTS-PD when the training size is large, which is different from most other cases. 

\subsection{More visualizations}
More visualizations can be found in Figures \ref{fig: more_path_0}-\ref{fig: more_tree_4}. In general, we see that with a better weight being learned, QRTS-P can find solutions with lower costs even with more edges.

\subsection{Minor Observations}
\begin{itemize}
\item \textbf{Computation burdens should not be overlooked.} While a larger sample set or a higher model dimension can -- in theory -- result in better performance, they also increase the computation burden in the larger-margin training, thereby not necessarily improving the generalization effect. Indeed, with overly many samples, QRTS-P can even result in larger ratios (e.g., Kro and BA in Table \ref{table: TP_1}). Similar arguments apply to Path-to-Tree on WS-dense (Table \ref{table: WS_dense} in Appendix \ref{sup: exp}), where we observe that increasing the model dimension in QRTS-PD can hurt the performance in some cases. 

\item \textbf{Simple baselines can occasionally produce good solutions.} While Unit is in general less powerful than the proposed methods, it does perform well in certain corner cases. For example, for Tree-to-Path on BA, Unit has a very good performance when the sample size is $60$. One plausible reason is that BA is relatively small and thus the graph structure has a higher influence on the optimal decision. 

    \item \textbf{The scale of the ratio varies over datasets.} We can see that the absolute scale of the ratio may vary drastically over different datasets, which is primarily due to the graph density and connectivity pattern. For the cases with large ratios (e.g., Path-to-Tree on WS-dense), the efficacy of QRTS-PD is very significant, and additionally, increasing the sample size is more useful than increasing the model dimension. We have similar observations on WS-sparse (Table \ref{table: WS_sparse}).
    \item \textbf{Importance learning is necessary.} The core part of QRTS-P utilizes the cutting-plane algorithm for solving Eq. (\ref{eq: qrts_p_training}), where the parameter $\vec{w}$ is randomly initialized and then iteratively updated through the alternation between convex optimization and loss-augmented inference. To justify that such a training process is effective, we compare the generalization ratio resulted by the $\vec{w}$ before and after the training process. This can rule out the scenario where a random $\vec{w}$ can result in good solutions. The results on Kro are given in Table \ref{table: further}; it is clear that the training of QRTS-P can evidently increase the generalization performance.
\end{itemize}

\begin{table*}[!t]
	\renewcommand{\arraystretch}{1.2} 
	\caption{\textbf{Results on WS-sparse.} The top results are highlighted.}
	\label{table: WS_sparse}
	\centering
	\scalebox{1.05}{\tiny
		\begin{tabular}{@{}  l@{\hspace{2mm}} r@{\hspace{3mm}} r @{\hspace{1mm}} r @{\hspace{1mm}}r @{\hspace{2mm}} r @{\hspace{1mm}} r @{\hspace{1mm}} r @{\hspace{2mm}} r @{\hspace{1mm}} r @{\hspace{1mm}} r    @{}}
			\toprule

			\multicolumn{2}{@{}l}{\underline{\textbf{Path-to-Tree}}}   & \multicolumn{3}{c}{QRTS-PD$^{-}$} & \multicolumn{3}{c}{QRTS-PD-1} & \multicolumn{3}{@{}c}{QRTS-PD-3} \\
			
			\multirow{1}{*}{\makecell{Train Size}}  & \multirow{1}{*}{QRTS-P} & \multicolumn{1}{c}{60}  &   \multicolumn{1}{c}{240} &\multicolumn{1}{c}{480} &\multicolumn{1}{c}{60}  &   \multicolumn{1}{c}{240} &\multicolumn{1}{c}{480}  & \multicolumn{1}{c}{60}  &   \multicolumn{1}{c}{240} &\multicolumn{1}{c}{480}  \\
			\midrule
			60  & 13E3  & 13E3 & 14E3 & 19E3 & 11E3 & 14E3 & 17E3  & 16E3 & 19E3 & 14E3\\
			
			240 &  7E3 & 6E3 & 5E3 & 7E3  & 8E3 & 6E3 & 7E3 & 5E3 & 7E3 & 8E3\\
			
			2400 &  1E3& \textbf{0.8E3 }& 5E3 & 7E3 &1E3 &\textbf{0.8E3} & 1E3 & 2E3 & \textbf{0.7E3} & \textbf{0.8E3} \\
			& & \multicolumn{3}{c}{Unit: 44E3   } & \multicolumn{3}{c}{Random: 109E3}\\
			
			\midrule
			\midrule

			\multicolumn{2}{@{}l}{\underline{\textbf{Tree-to-Path}}}   & \multicolumn{3}{c}{QRTS-PD$^{-}$} & \multicolumn{3}{c}{QRTS-PD-1} & \multicolumn{3}{@{}c}{QRTS-PD-3} \\
			
			\multirow{1}{*}{\makecell{Train Size}}  & \multirow{1}{*}{QRTS-P} & \multicolumn{1}{c}{60}  &   \multicolumn{1}{c}{240} &\multicolumn{1}{c}{480} &\multicolumn{1}{c}{60}  &   \multicolumn{1}{c}{240} &\multicolumn{1}{c}{480}  & \multicolumn{1}{c}{60}  &   \multicolumn{1}{c}{240} &\multicolumn{1}{c}{480}  \\
			\midrule
			60  & 26 \tiny{(4)} & 26 \tiny{(4.2)} & 25 \tiny{(2.4)} & 22 \tiny{(3.6)} & 23 \tiny{(4.1)} & 24 \tiny{(3.2)} & 22 \tiny{(4.9)} & 18 \tiny{(2.0)} & 17 \tiny{(4.0)} & 19 \tiny{(2.2)}\\
			
			240 &  26 \tiny{(2.1)} & 24 \tiny{(6.1)}& 23 \tiny{(8.1)}& 23 \tiny{(5.8)}& 27 \tiny{(5.6)}& 23 \tiny{(4.2)}& 21 \tiny{(9.3)}& 20 \tiny{(3.9)}& 20 \tiny{(5.2)}& 18 \tiny{(2.5)}\\
			
			2400 & 17 \tiny{(3.7)}& 19 \tiny{(4.1)}& \textbf{15 \tiny{(3.1)}}& \textbf{14 \tiny{(2.1)}}& 16 \tiny{(1.3)}& 18 \tiny{(3.9)} & 15 \tiny{(3.6)} & \textbf{14 \tiny{(1.8)}}& 17 \tiny{(1.2)}& \textbf{16\tiny{(1.3)}}\\
			
			& & \multicolumn{3}{c}{Unit: 25.1 \tiny{(6.3)}  } & \multicolumn{3}{c}{Random: 34.3 \tiny{(4.1)}}\\

			\bottomrule

		\end{tabular}%
	} 
\end{table*}

\begin{table*}[!t]
	\renewcommand{\arraystretch}{1.2} 
	\caption{\textbf{Results on WS-dense.} The top results are highlighted.}
	\label{table: WS_dense}
	\centering
	\scalebox{1.05}{\tiny
		\begin{tabular}{@{}  l@{\hspace{2mm}} r@{\hspace{3mm}} r @{\hspace{1mm}} r @{\hspace{1mm}}r @{\hspace{2mm}} r @{\hspace{1mm}} r @{\hspace{1mm}} r @{\hspace{2mm}} r @{\hspace{1mm}} r @{\hspace{1mm}} r    @{}}
			\toprule

			\multicolumn{2}{@{}l}{\underline{\textbf{Path-to-Tree}}}   & \multicolumn{3}{c}{QRTS-PD$^{-}$} & \multicolumn{3}{c}{QRTS-PD-1} & \multicolumn{3}{@{}c}{QRTS-PD-3} \\
			
			\multirow{1}{*}{\makecell{Train Size}}  & \multirow{1}{*}{QRTS-P} & \multicolumn{1}{c}{60}  &   \multicolumn{1}{c}{240} &\multicolumn{1}{c}{480} &\multicolumn{1}{c}{60}  &   \multicolumn{1}{c}{240} &\multicolumn{1}{c}{480}  & \multicolumn{1}{c}{60}  &   \multicolumn{1}{c}{240} &\multicolumn{1}{c}{480}  \\
			\midrule
			60  &  55E3  & 62E3  &  55E3  & 53E  & 67E3 & 42E3 &83E3 & 54E3 & 47E3 &58E3\\
                240 &   28E3   & 19E3& 21E3 & 22E3 & 18E3 &19E3 & 21E3 &26E3 & 14E3 & 30E3 \\
                2400 &   20E3& \textbf{2E3}  & 3E3  & \textbf{2E3}  &10E3 & \textbf{1E3} & 15E3 & \textbf{0.5E3} & \textbf{0.6E3} & \textbf{0.1E3}\\
			& & \multicolumn{3}{c}{Unit: 130E3  \tiny{(5E3)}  } & \multicolumn{3}{c}{Rand: 685E3  \tiny{(3E3)}}\\
			
			\midrule
			\midrule

			\multicolumn{2}{@{}l}{\underline{\textbf{Tree-to-Path}}}   & \multicolumn{3}{c}{QRTS-PD$^{-}$} & \multicolumn{3}{c}{QRTS-PD-1} & \multicolumn{3}{@{}c}{QRTS-PD-3} \\
			
			\multirow{1}{*}{\makecell{Train Size}}  & \multirow{1}{*}{QRTS-P} & \multicolumn{1}{c}{60}  &   \multicolumn{1}{c}{240} &\multicolumn{1}{c}{480} &\multicolumn{1}{c}{60}  &   \multicolumn{1}{c}{240} &\multicolumn{1}{c}{480}  & \multicolumn{1}{c}{60}  &   \multicolumn{1}{c}{240} &\multicolumn{1}{c}{480}  \\
			\midrule
			60  &  143 \tiny{(23)} & 132 \tiny{(18)} &140 \tiny{(15)}  & 113 \tiny{(11)}  & 149 \tiny{(13)}& 136 \tiny{(10)} & 124 \tiny{(9.0)} & 113 \tiny{(16)}& 141 \tiny{(10)} & 136 \tiny{(11)}\\

   240 &  80 \tiny{(12)} & 90 \tiny{(13)} &62 \tiny{(8.6)} & 76 \tiny{(7.7)}  & 91 \tiny{(7.0)}& 80 \tiny{(6.0)}& 83 \tiny{(9.0)}& 86 \tiny{(5.0)} &72 \tiny{(6.0)} & 92 \tiny{(8.0)}\\

   2400 &    \textbf{57} \tiny{(6.1)}  & 73 \tiny{(8.2)}  &72 \tiny{(7.3)}  &96 \tiny{(13)} & 86 \tiny{(11)}& 78 \tiny{(12)}& 72 \tiny{(10)} & \textbf{54} \tiny{(8.0)} & \textbf{41} \tiny{(4.0)} & \textbf{45} \tiny{(5.0)}  \\

 & & \multicolumn{3}{c}{Unit: 183 \tiny{(12)}  } & \multicolumn{3}{c}{Random: 583 \tiny{(32)} }\\
			\bottomrule		
		\end{tabular}%
	} 
\end{table*}

\begin{table}[tp!]
	\renewcommand{\arraystretch}{1.1} 
	\caption{\textbf{Results of QRTS-P.} For each training size, the results of five independent experiments are given. Each cell shows the performance ratios before and after training. }
	\centering
	\label{table: further}
	\begin{tabular}{  @{}  l   @{\hspace{3mm}}   l @{\hspace{5mm}}  c @{\hspace{3mm}}  c@{\hspace{3mm}} c @{\hspace{3mm}}  c @{\hspace{3mm}}  c @{\hspace{3mm}} c @{}  }
		\multicolumn{5}{@{}l}{\textbf{Path-to-Tree}} \\ 
		Train size &  &    \#1 &  \#2  &   \#3  &   \#4   & \#5 & \\ 
		\midrule
		\multirow{2}{*}{60} & after & 4.938&5.504&4.500&5.070&4.760&  \\ 
		
		 & before & 	5.007&5.585 &4.796 &5.317 &4.923 \\ 
		\midrule
		 \multirow{2}{*}{480} & after & 2.691&2.849&2.942&2.784&2.754 \\ 
		 & before & 4.757&5.196&5.177&5.292&4.809 \\ 
		 
		\bottomrule
    ~\\
  \multicolumn{5}{@{}l}{\textbf{Tree-to-Path}} \\ 
  Train size &  &    \#1 &  \#2  &   \#3  &   \#4   & \#5 & \\ 
		\midrule
		\multirow{2}{*}{60} & after & 1.382&1.337&1.413&1.438&1.271 \\
		
		 & before & 	1.639&1.645&1.609&1.686&1.555 \\
		\midrule
		 \multirow{2}{*}{480} & after &1.451&1.457&1.356&1.383&1.328\\
		 & before & 1.621&1.724&1.609&1.658&1.670\\
		 
  \bottomrule
	\end{tabular}
\end{table}

\begin{figure}[h]
\centering
\captionsetup[subfloat]{labelfont=scriptsize,textfont=scriptsize,labelformat=empty}
\subfloat[  { [30, 0]} ]{\label{fig: 0_30_0}\includegraphics[width=0.10\textwidth]{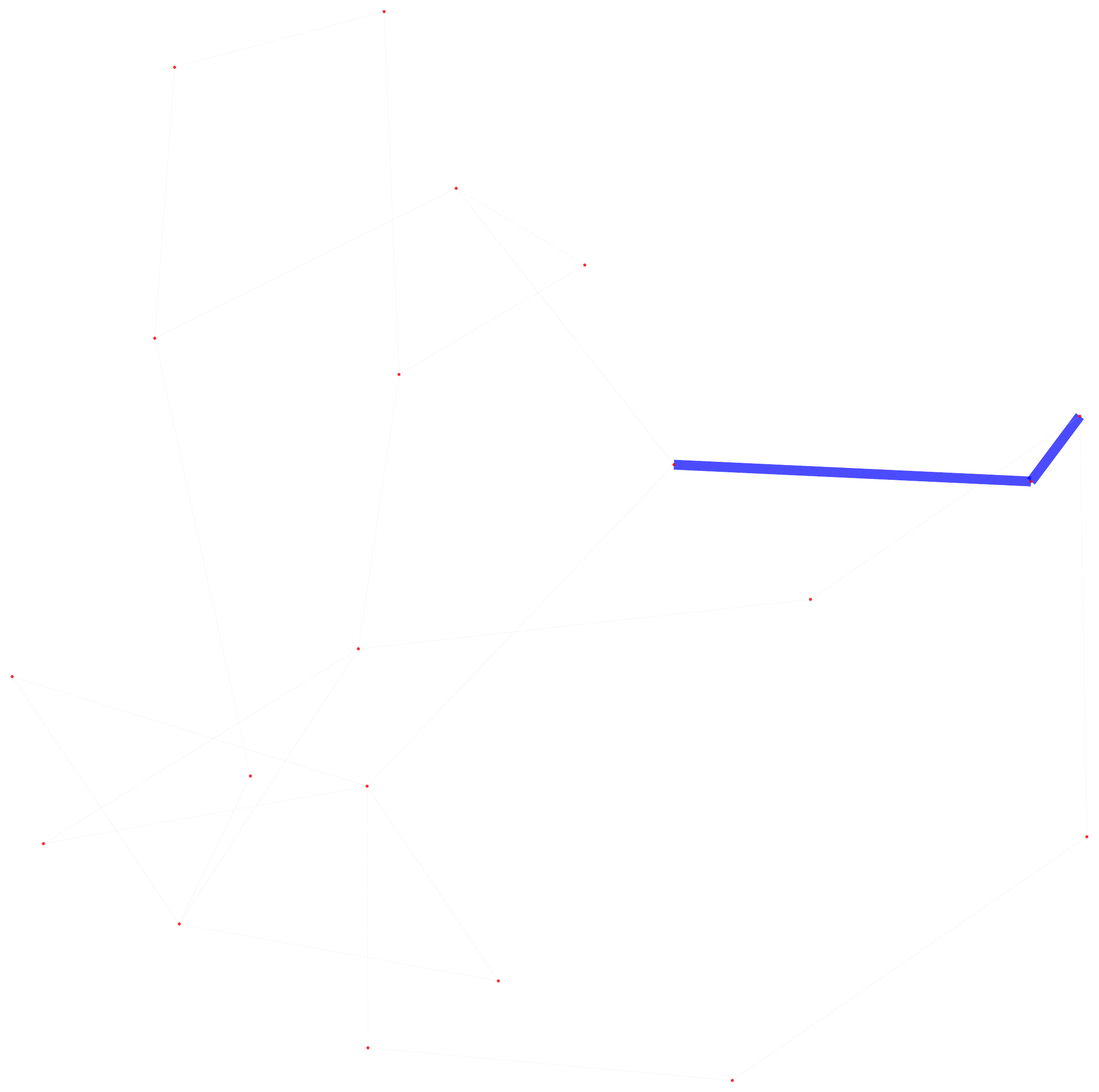}}\hspace{1mm}
\subfloat[  { [30, 1]} ]{\label{fig: 0_30_1}\includegraphics[width=0.10\textwidth]{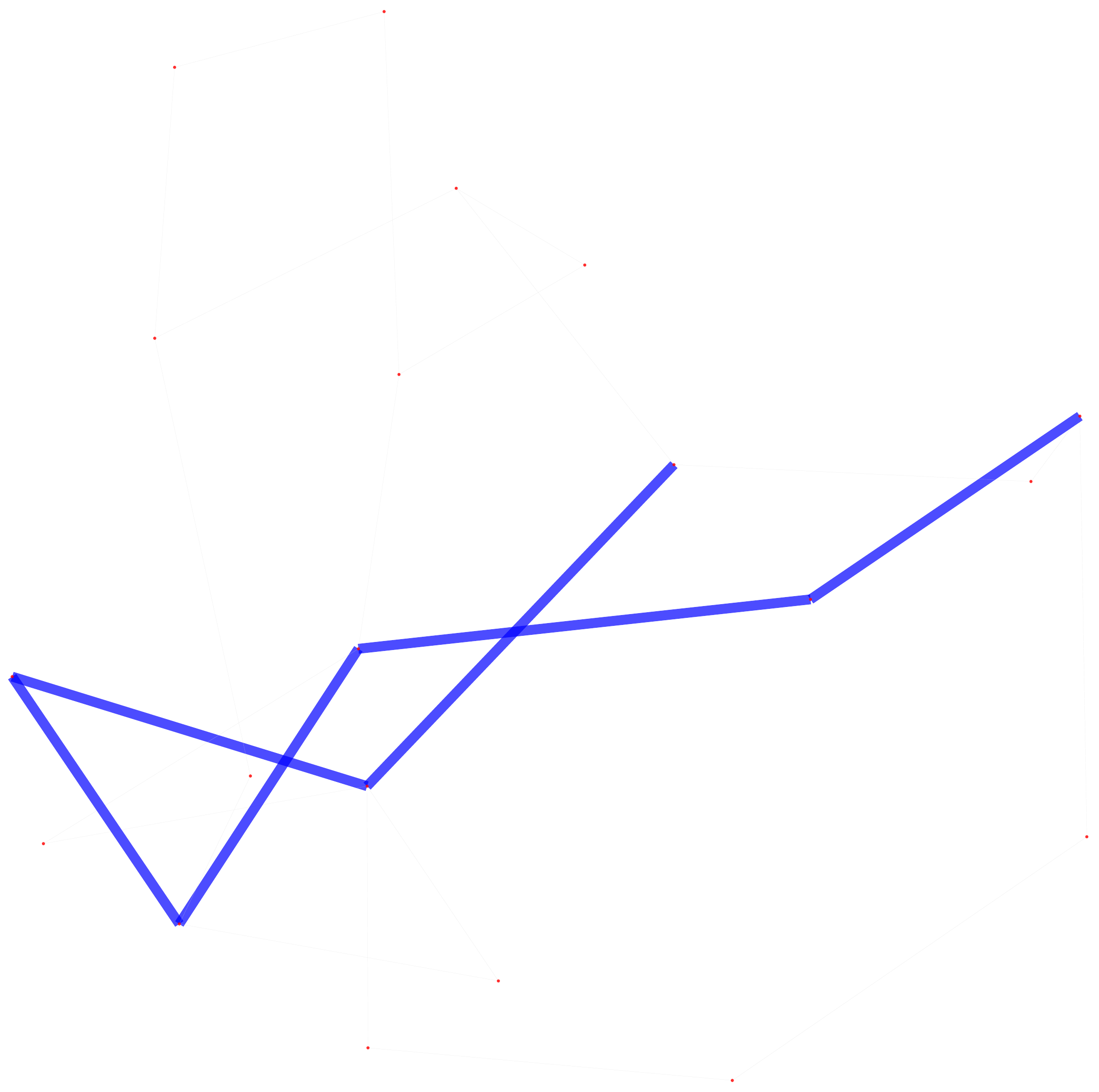}}\hspace{1mm}
\subfloat[  { [30, 4]} ]{\label{fig: 0_30_4}\includegraphics[width=0.10\textwidth]{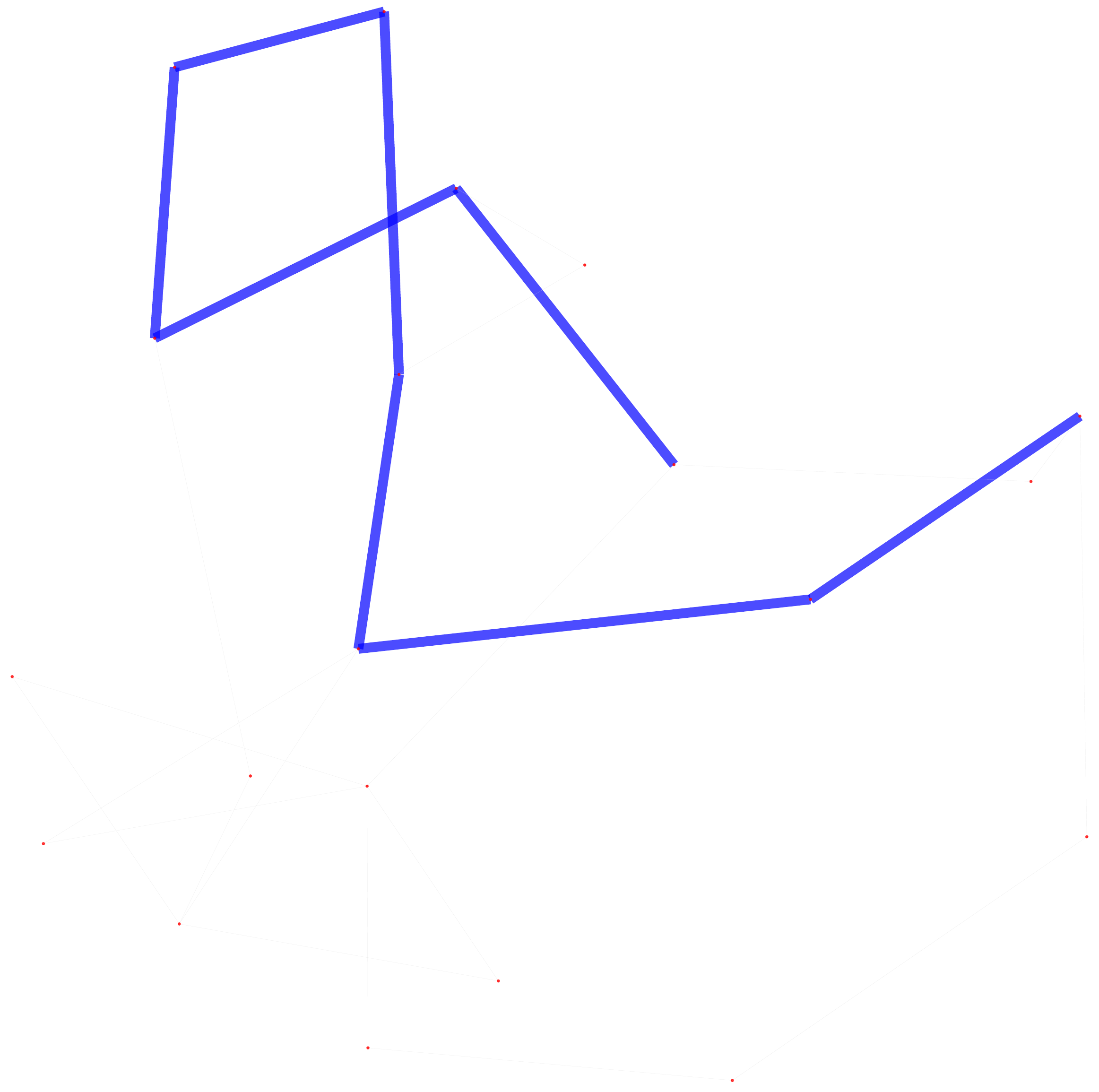}}\hspace{1mm}
\subfloat[  Target ]{\label{fig: 0_30_t}\includegraphics[width=0.10\textwidth]{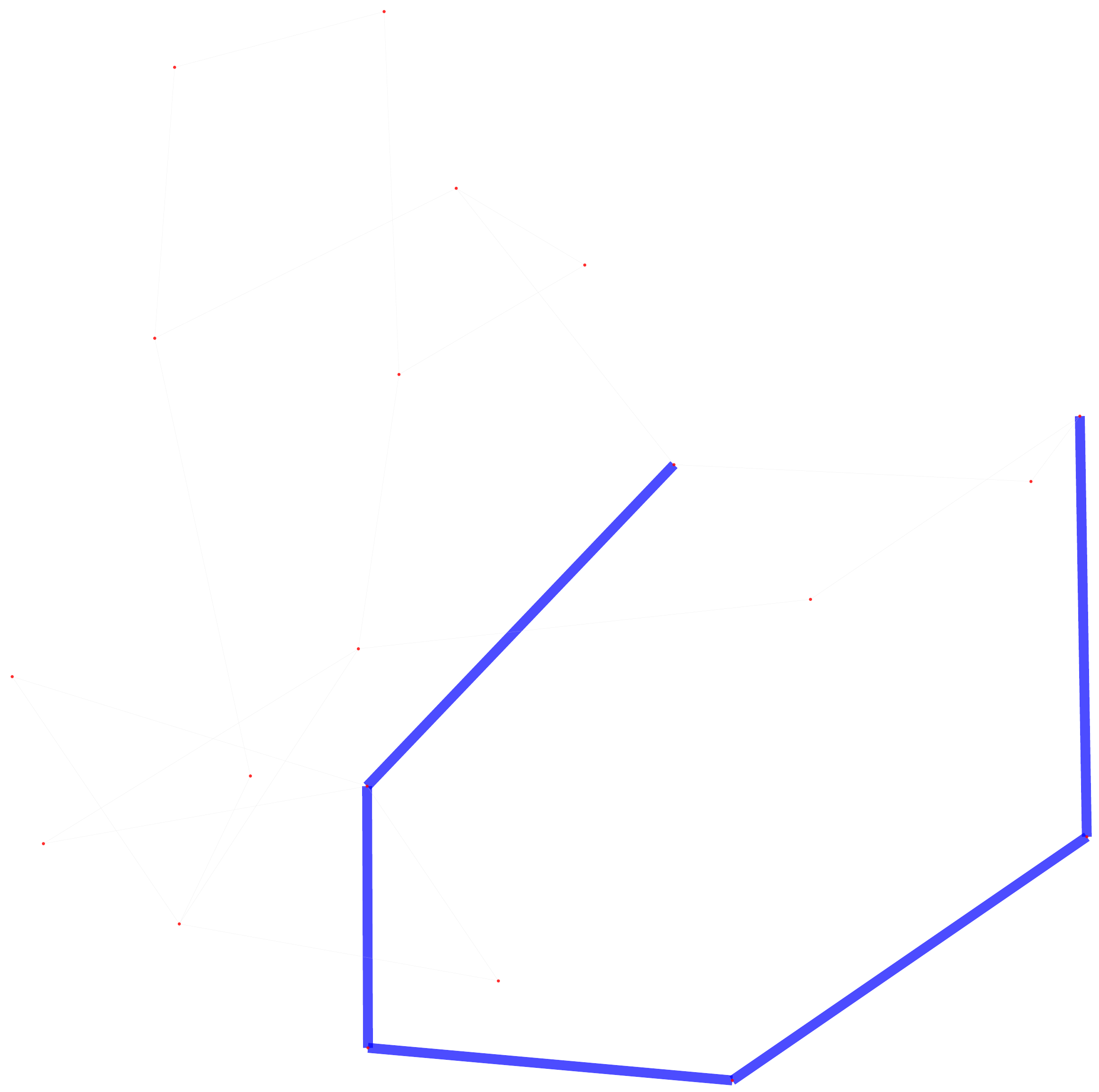}}\hspace{0mm}

\subfloat[  { [60, 0]} ]{\label{fig: 0_60_0}\includegraphics[width=0.10\textwidth]{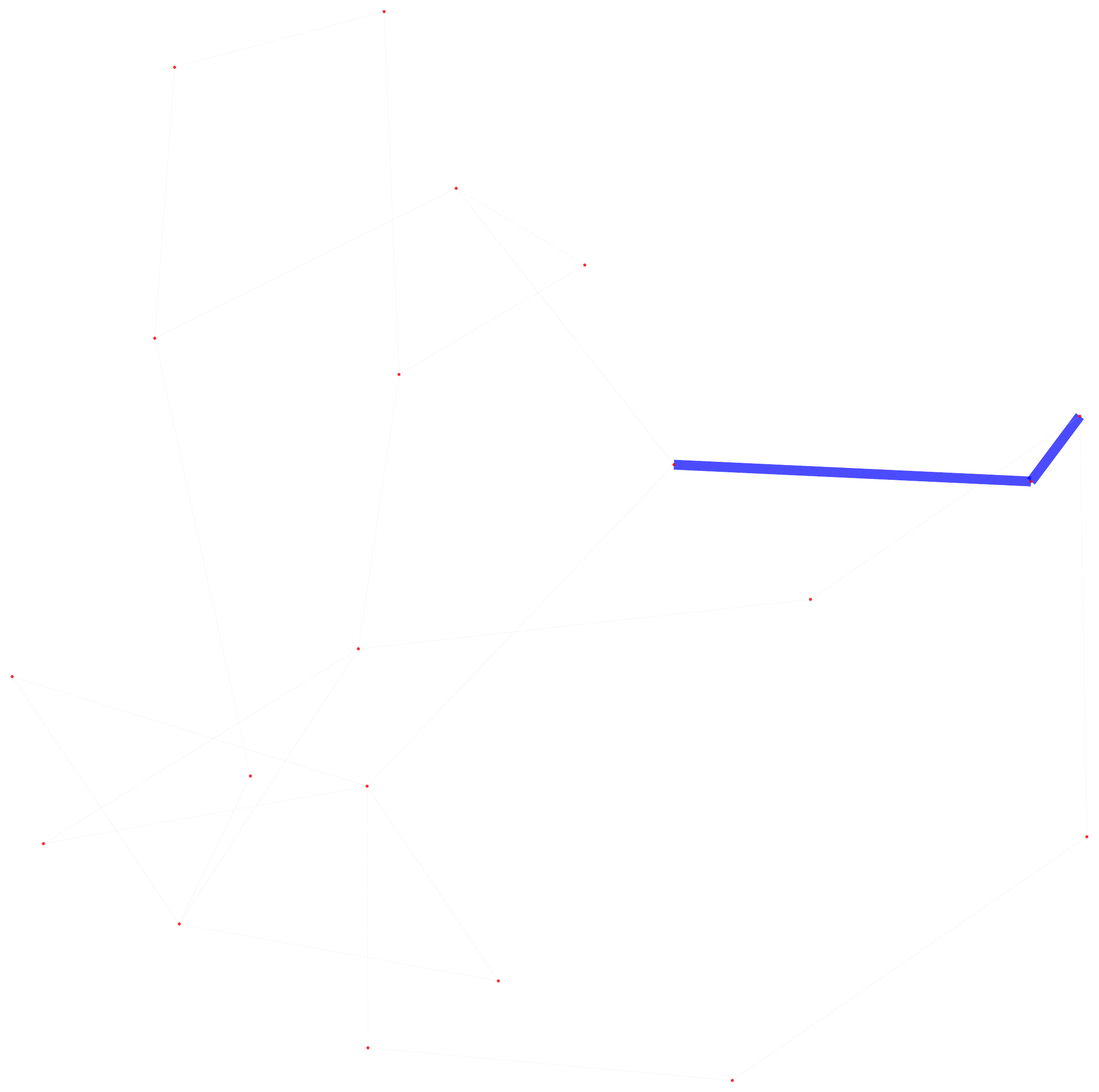}}\hspace{1mm}
\subfloat[  { [60, 1]} ]{\label{fig: 0_60_1}\includegraphics[width=0.10\textwidth]{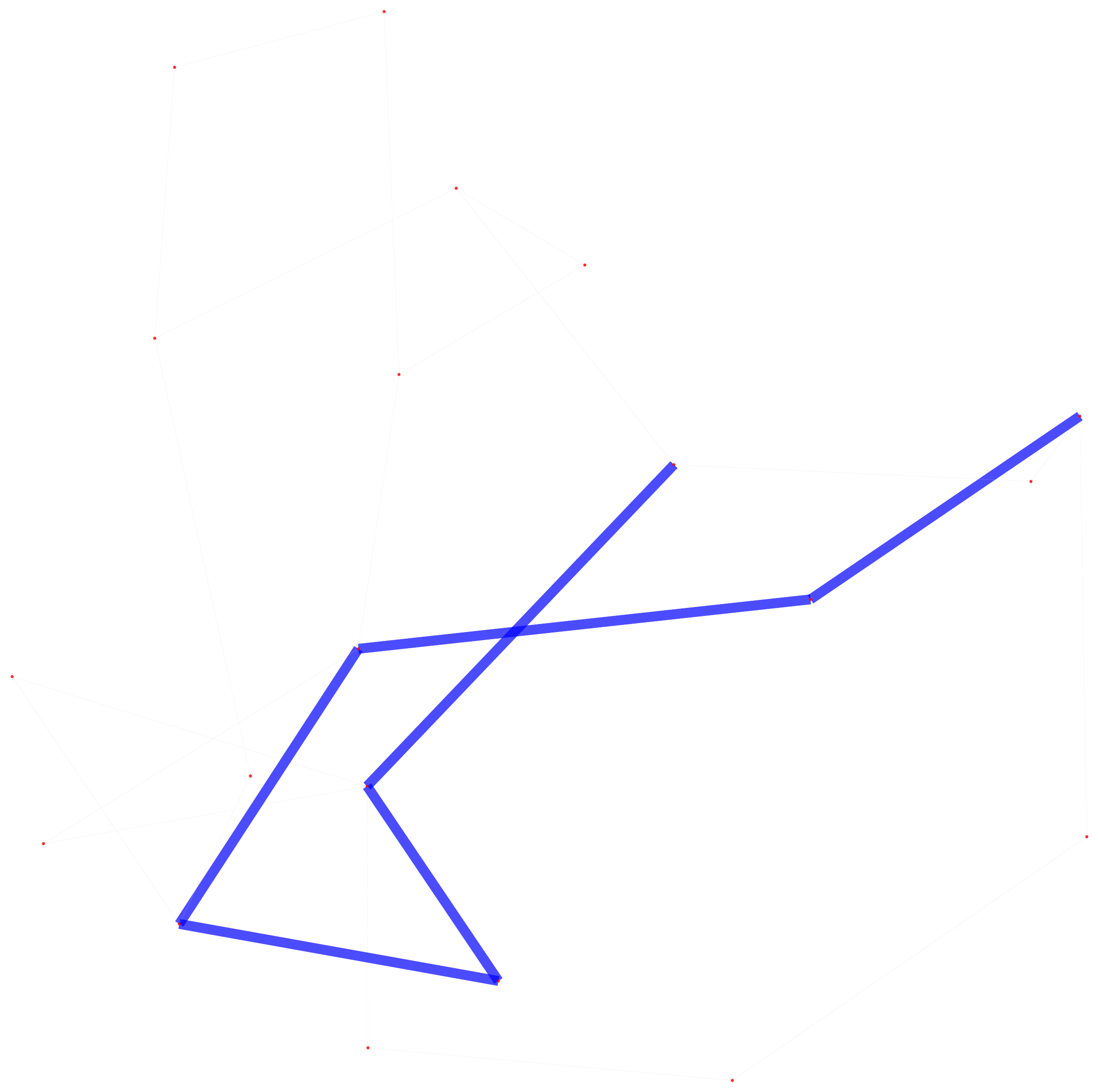}}\hspace{1mm}
\subfloat[  { [60, 4]} ]{\label{fig: 0_60_4}\includegraphics[width=0.10\textwidth]{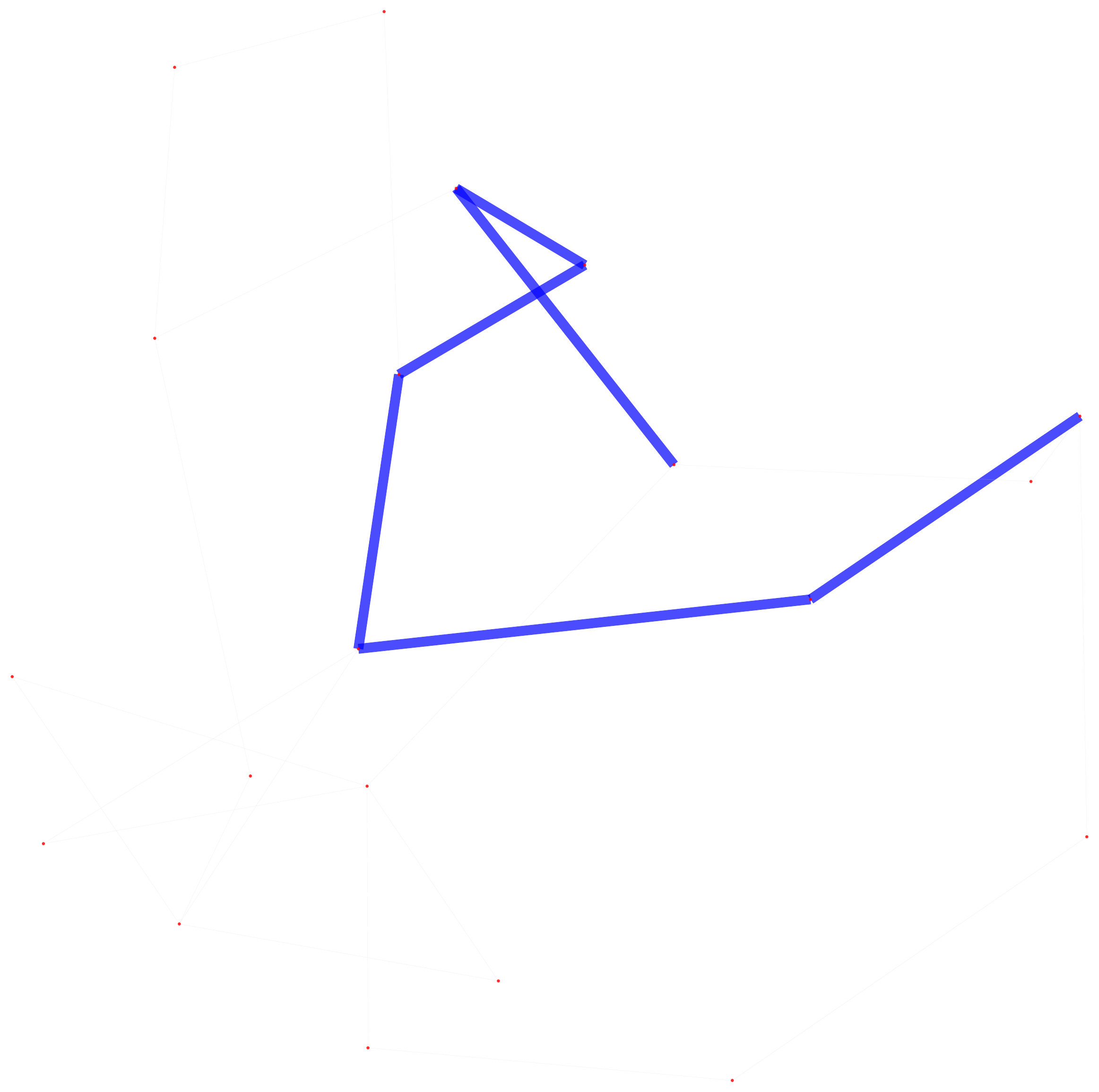}}\hspace{1mm}
\subfloat[  Target ]{\label{fig: 0_60_t}\includegraphics[width=0.10\textwidth]{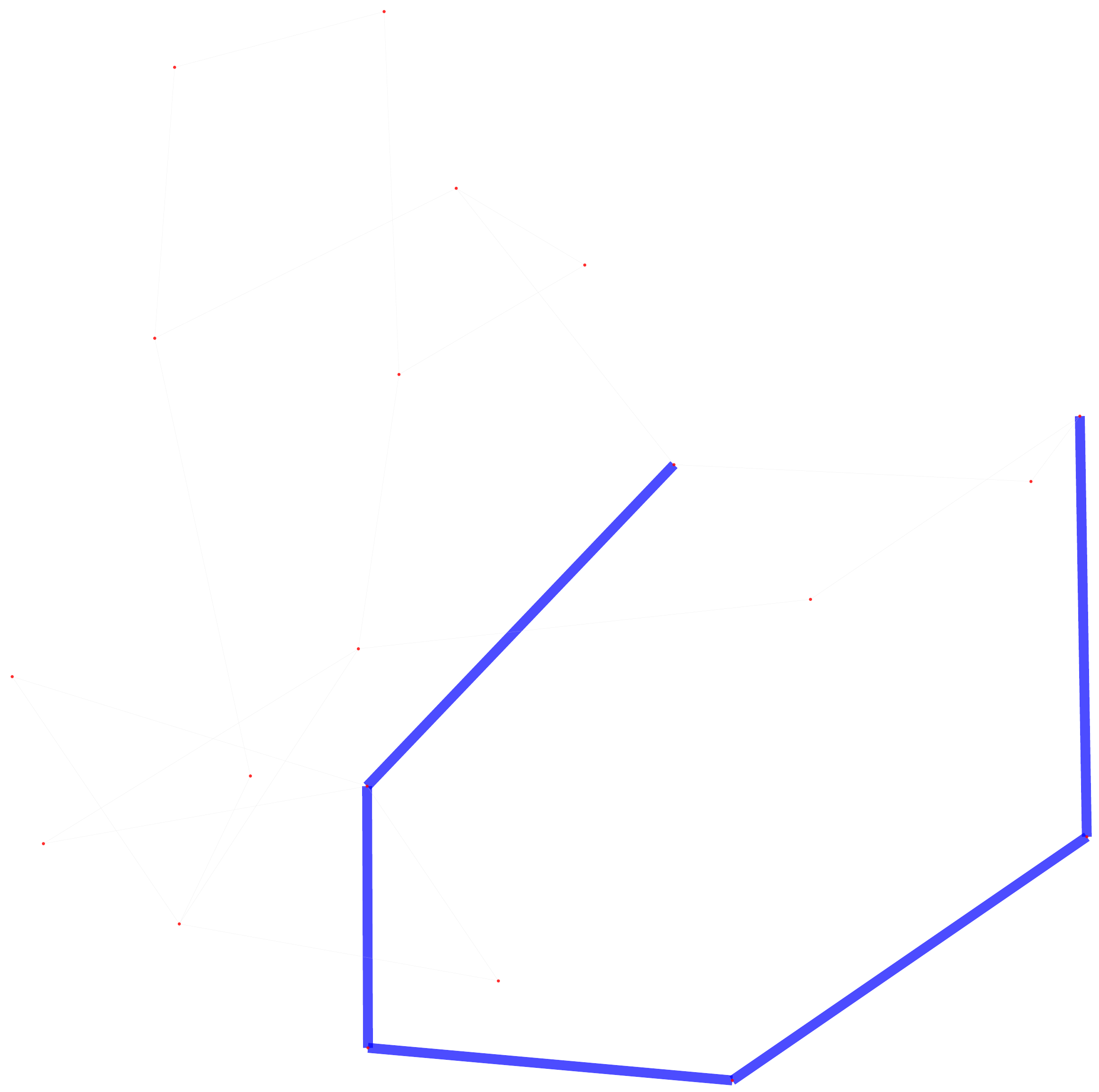}}\hspace{0mm}

\subfloat[  { [240, 0]} ]{\label{fig: 0_240_0}\includegraphics[width=0.10\textwidth]{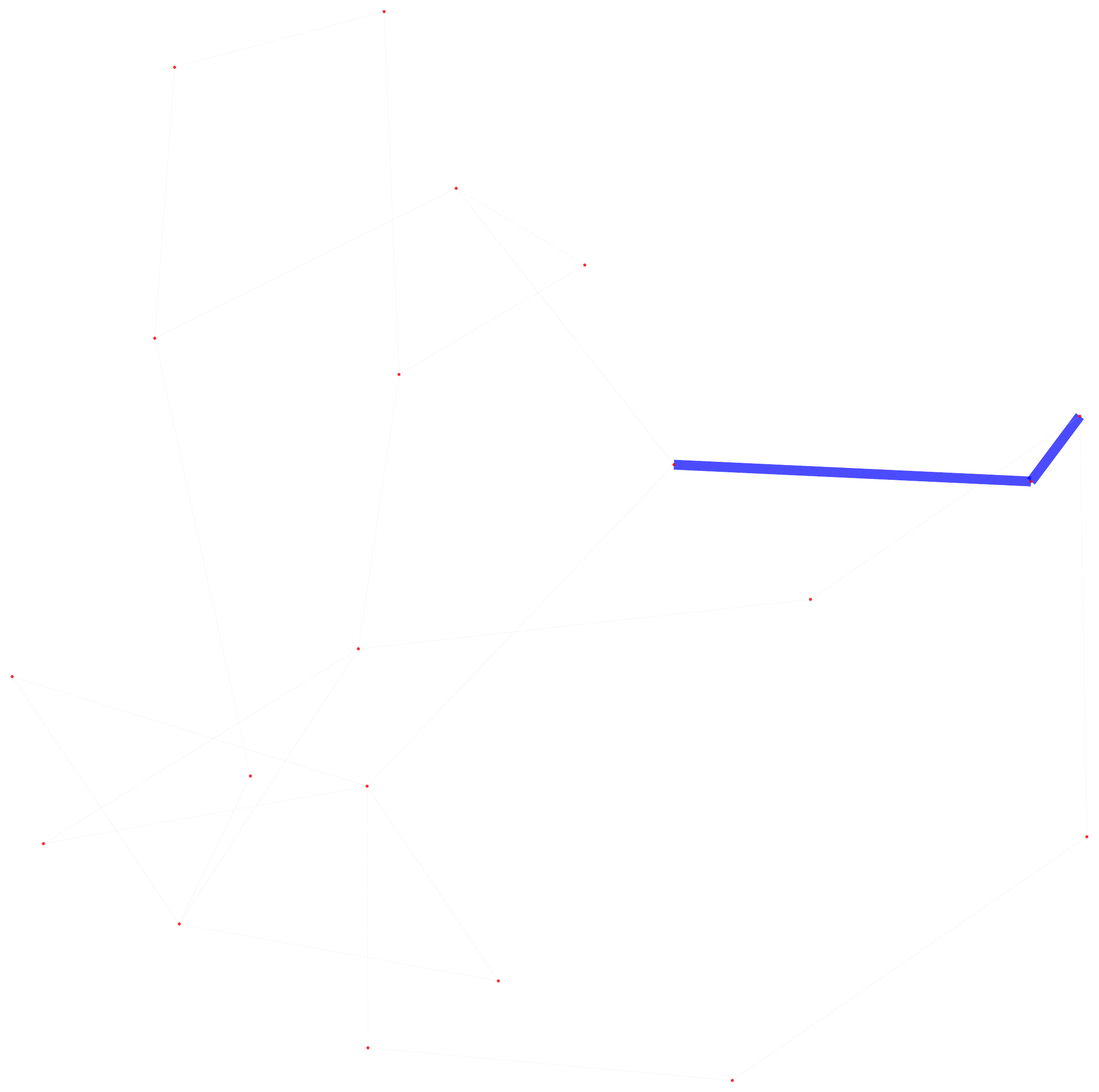}}\hspace{1mm}
\subfloat[  { [240, 1]} ]{\label{fig: 0_240_1}\includegraphics[width=0.10\textwidth]{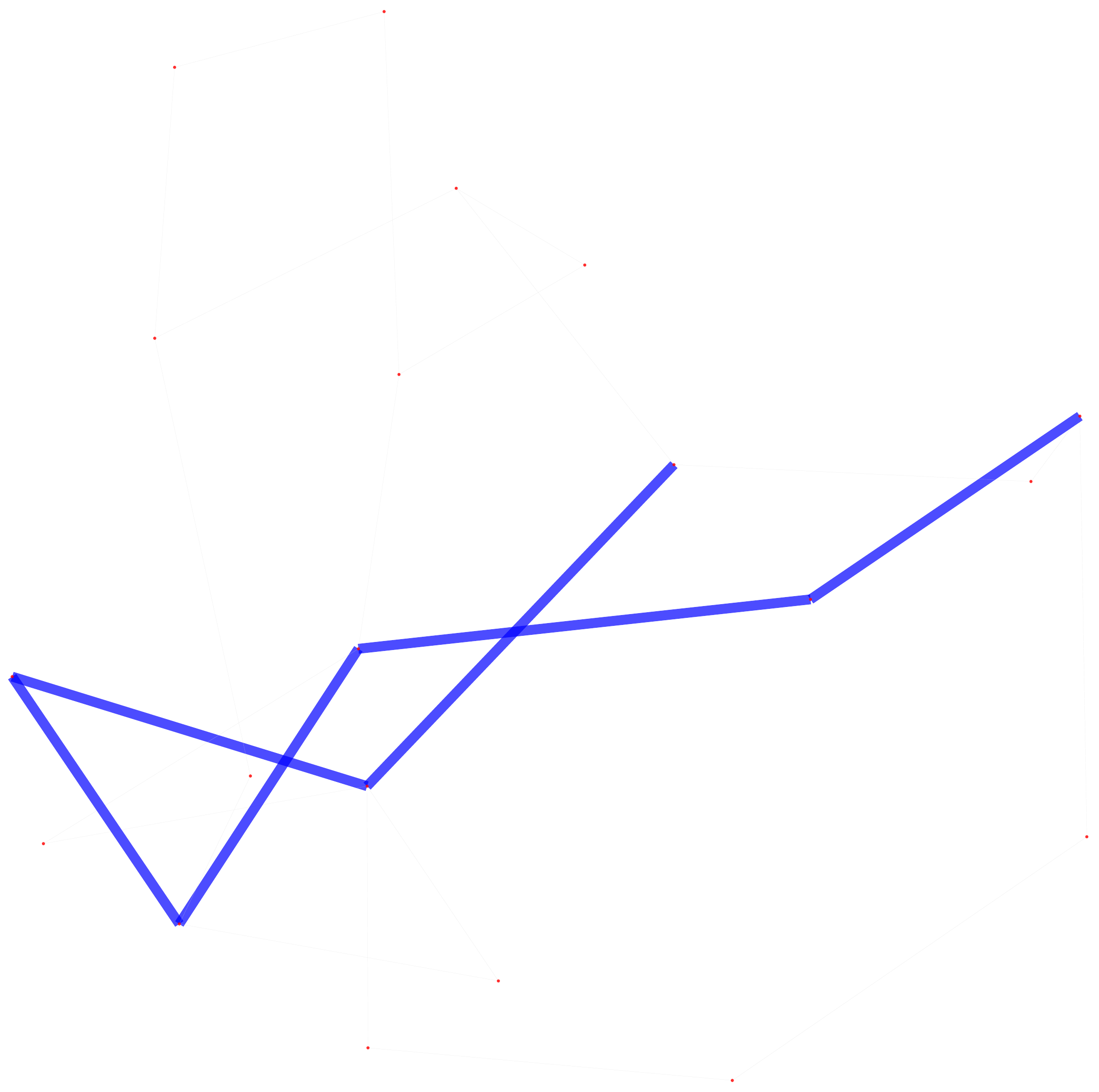}}\hspace{1mm}
\subfloat[  { [240, 4]} ]{\label{fig: 0_240_4}\includegraphics[width=0.10\textwidth]{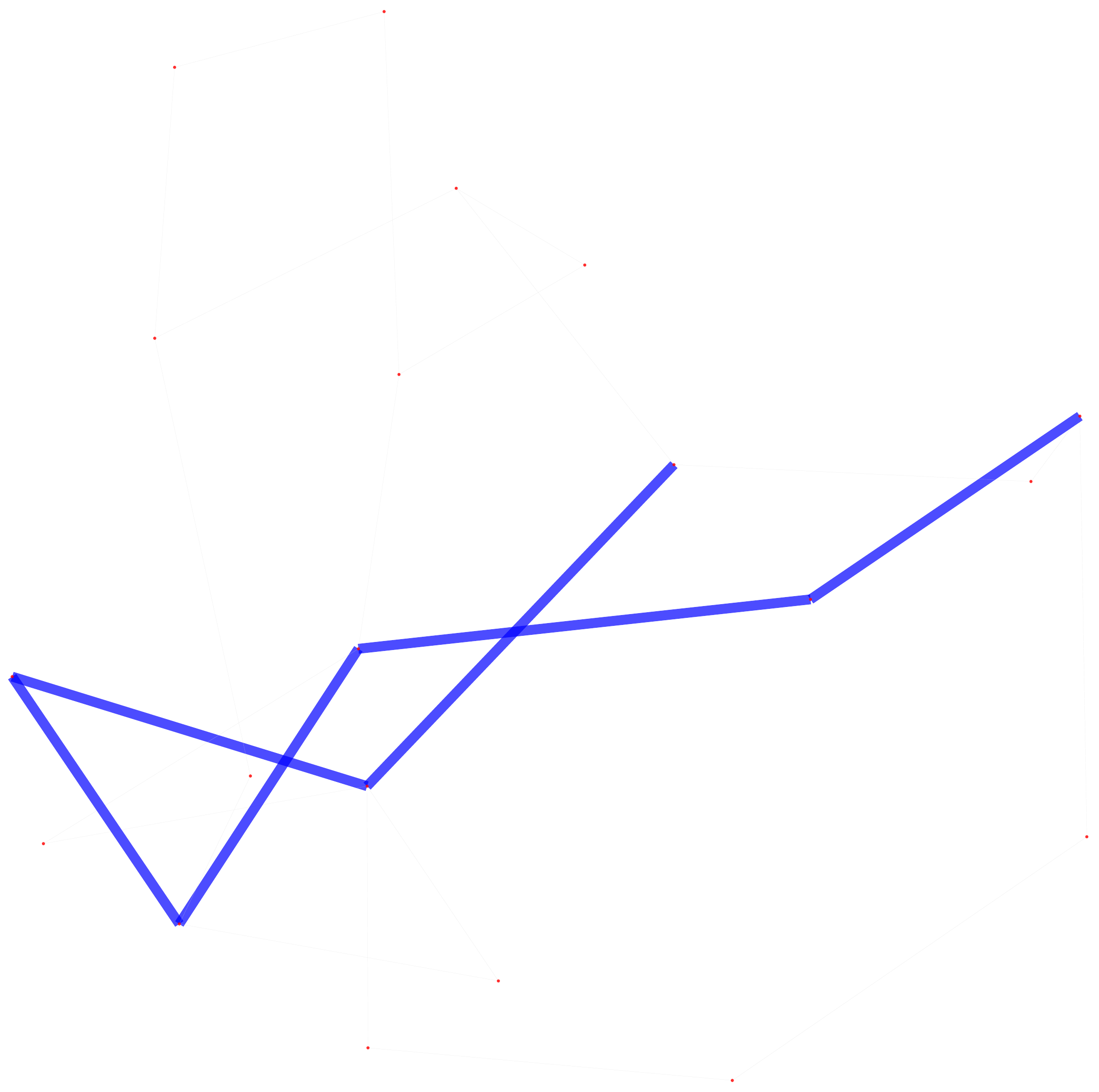}}\hspace{1mm}
\subfloat[  Target ]{\label{fig: 0_240_t}\includegraphics[width=0.10\textwidth]{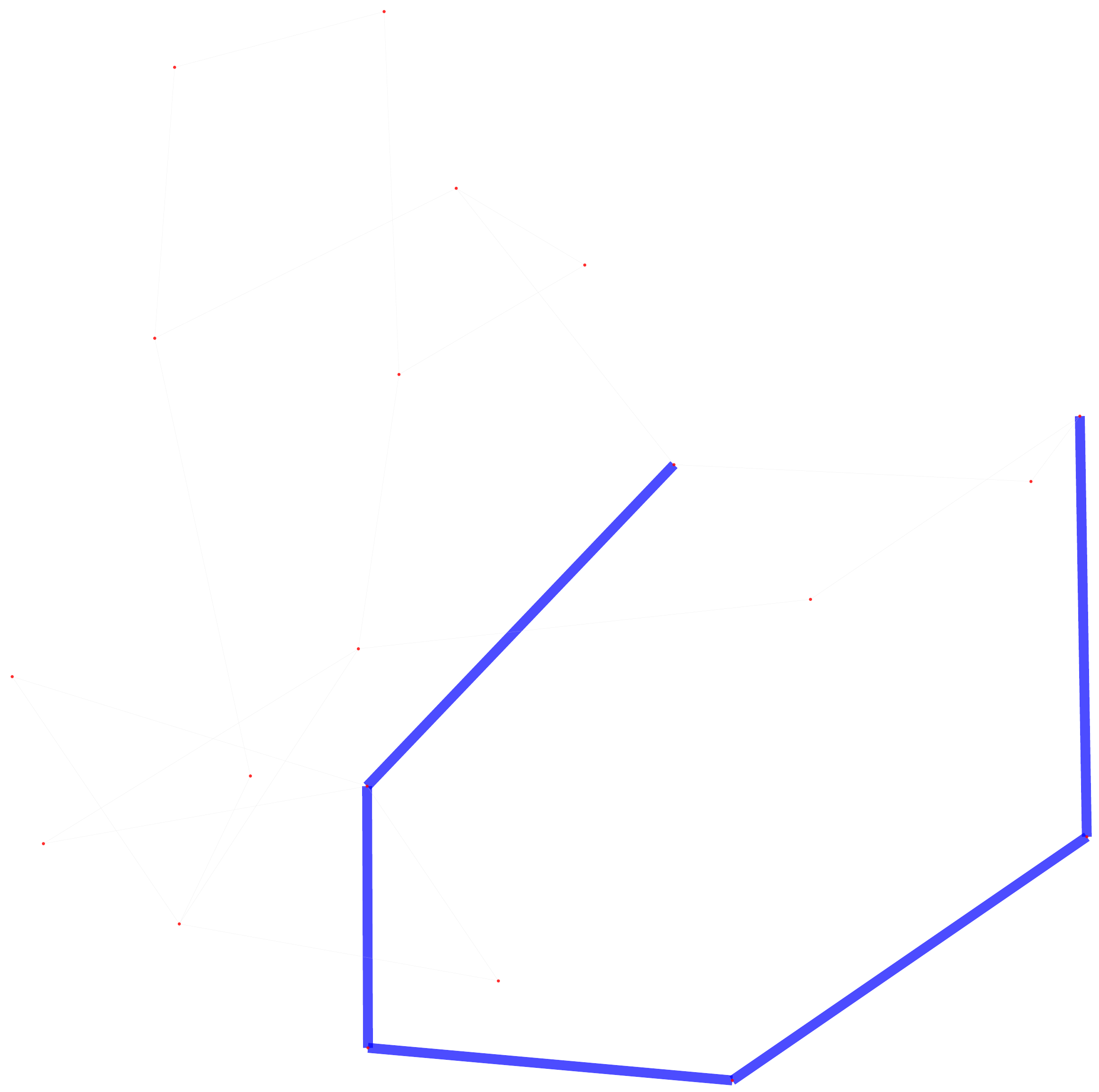}}\hspace{0mm}
\caption{\small Visualization of the results of QRTS-P for an example testing query for Tree-to-Path on Kro.}
\label{fig: more_path_0}
\end{figure}

\begin{figure}[t]
\centering
\captionsetup[subfloat]{labelfont=scriptsize,textfont=scriptsize,labelformat=empty}

\subfloat[  { [30, 0]} ]{\label{fig: path_1}\includegraphics[width=0.10\textwidth]{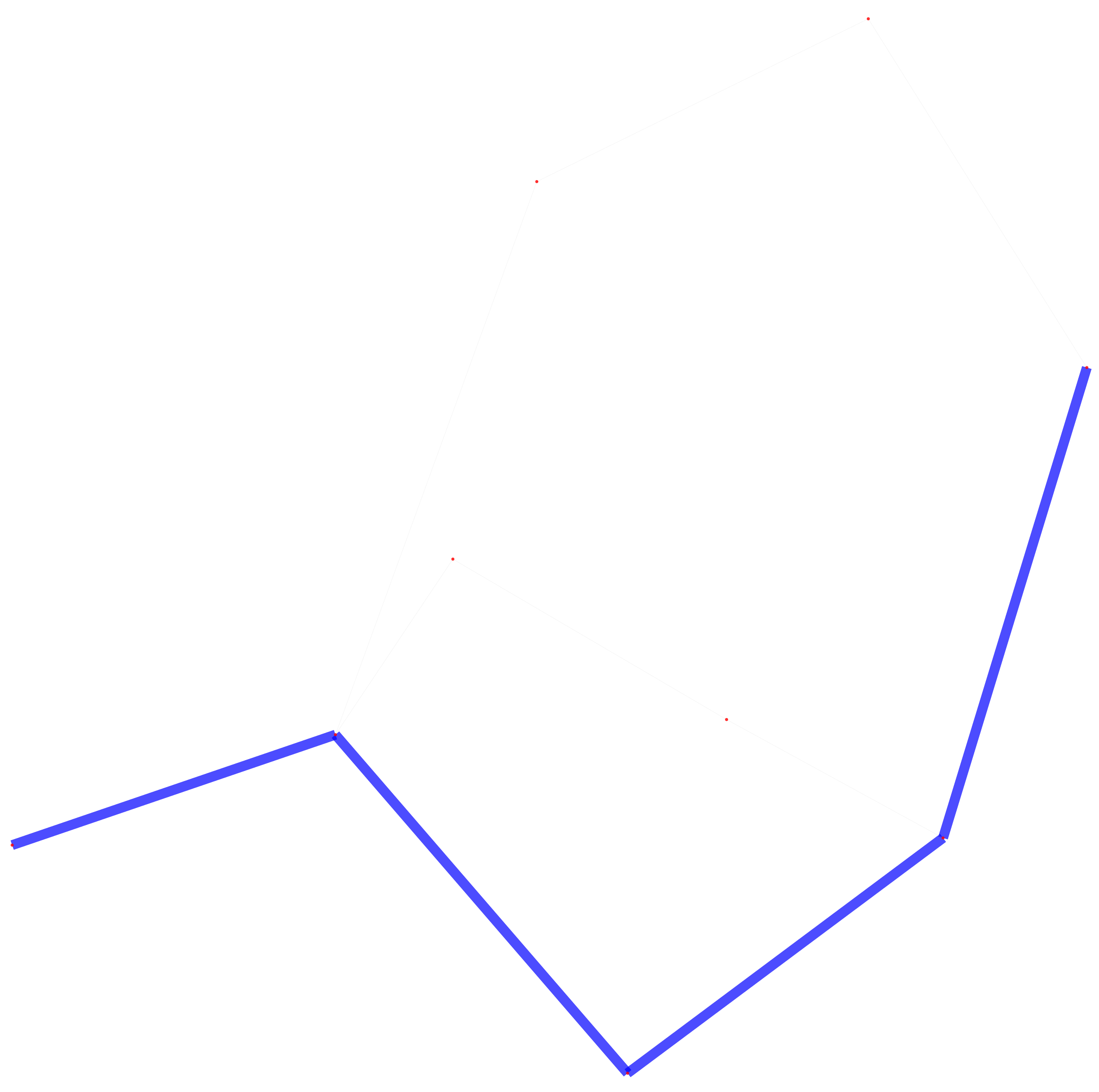}}\hspace{1mm}
\subfloat[  { [30, 1]} ]{\label{fig: path_1}\includegraphics[width=0.10\textwidth]{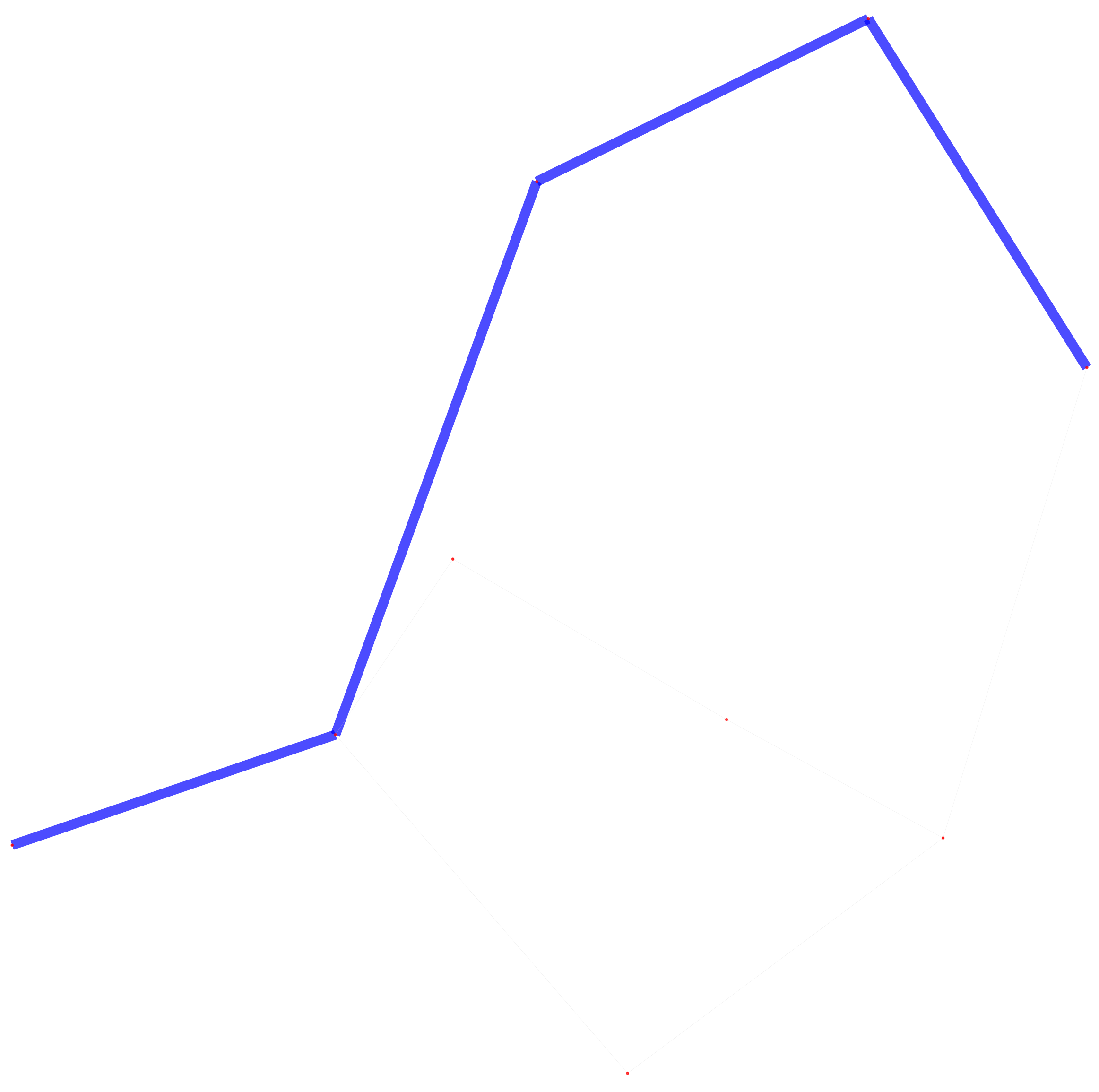}}\hspace{1mm}
\subfloat[  { [30, 4]} ]{\label{fig: path_1}\includegraphics[width=0.10\textwidth]{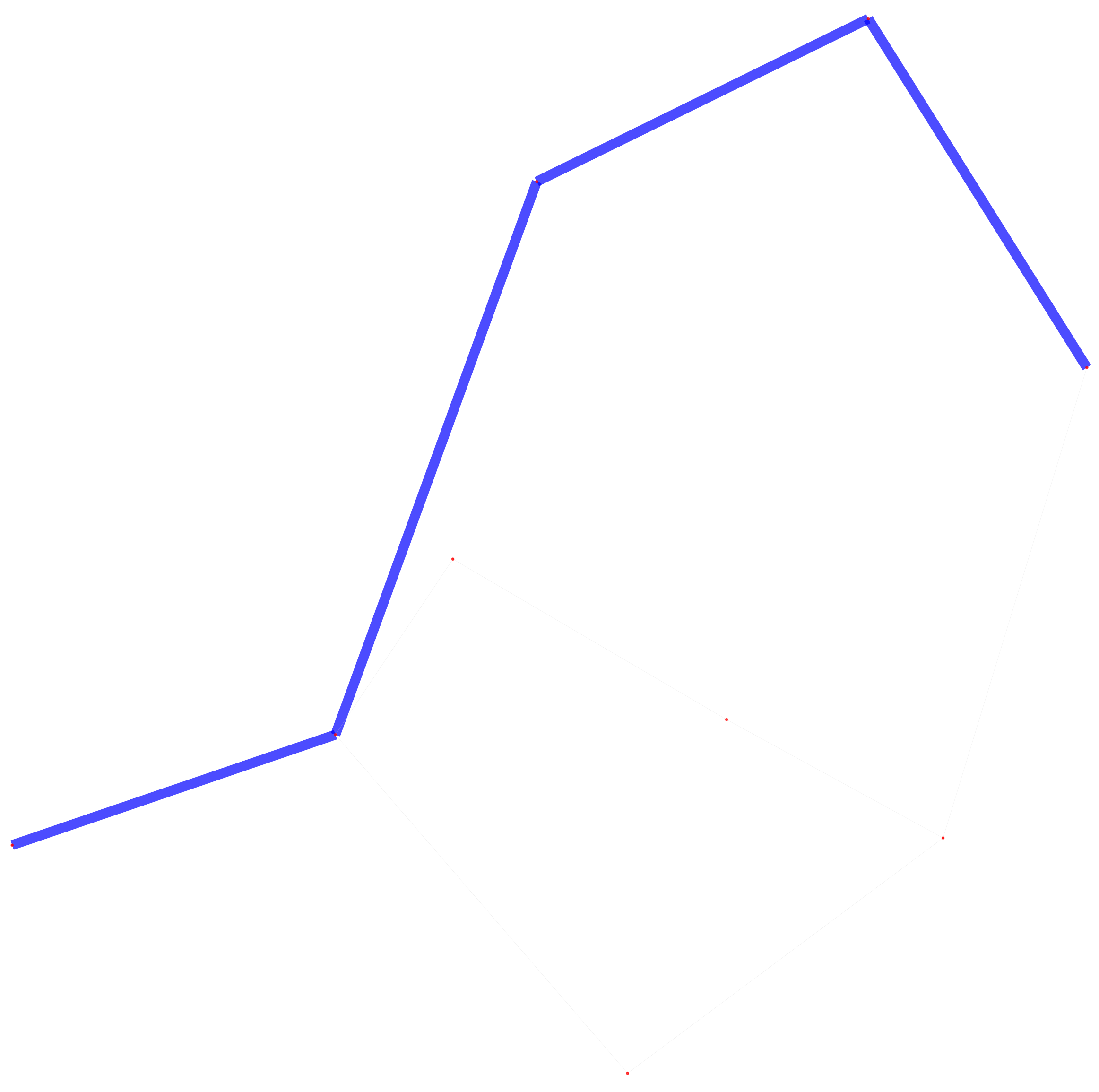}}\hspace{1mm}
\subfloat[  Target ]{\label{fig: path_1}\includegraphics[width=0.10\textwidth]{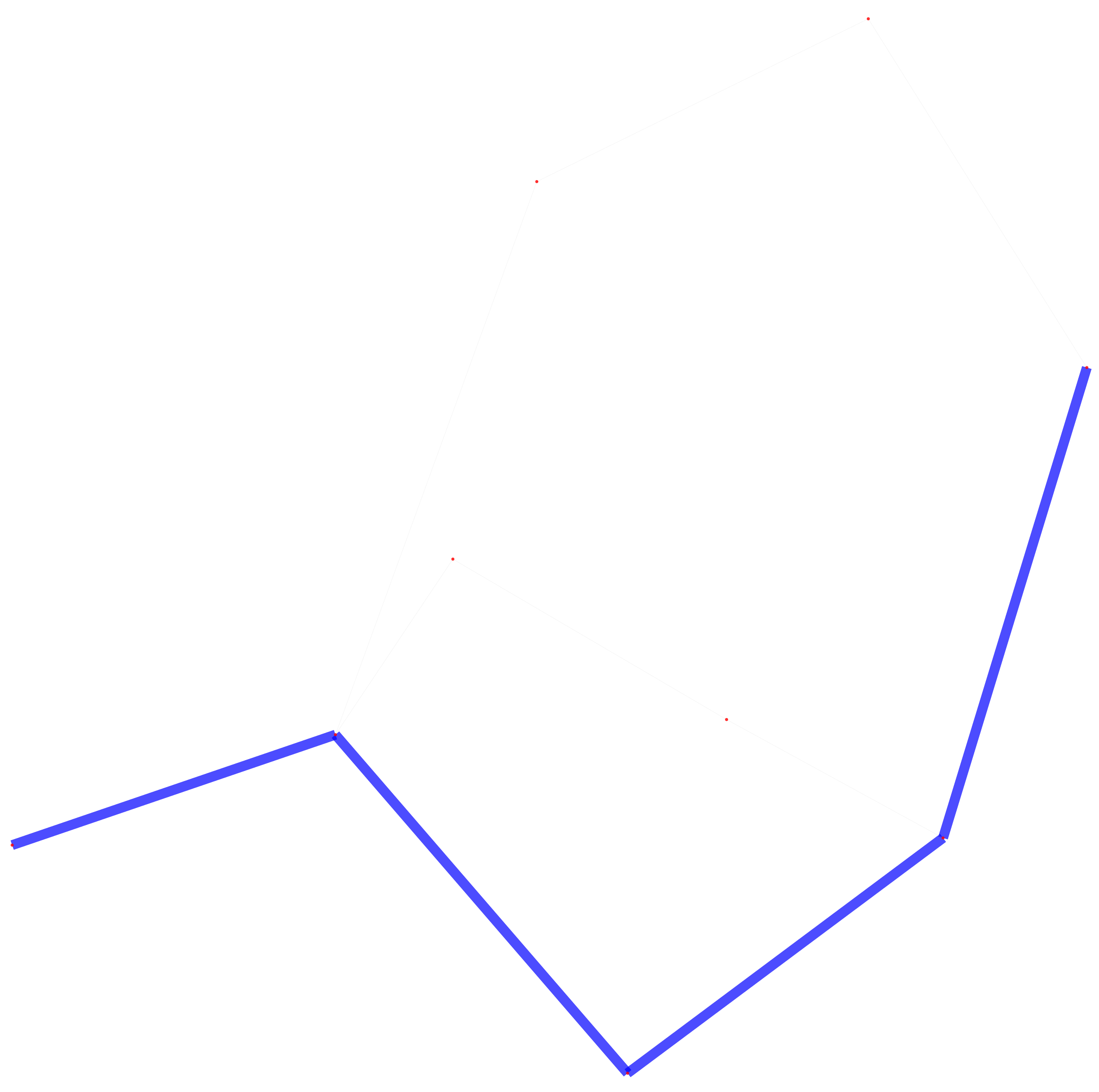}}\hspace{0mm}

\subfloat[  { [60, 0]} ]{\label{fig: path_1}\includegraphics[width=0.10\textwidth]{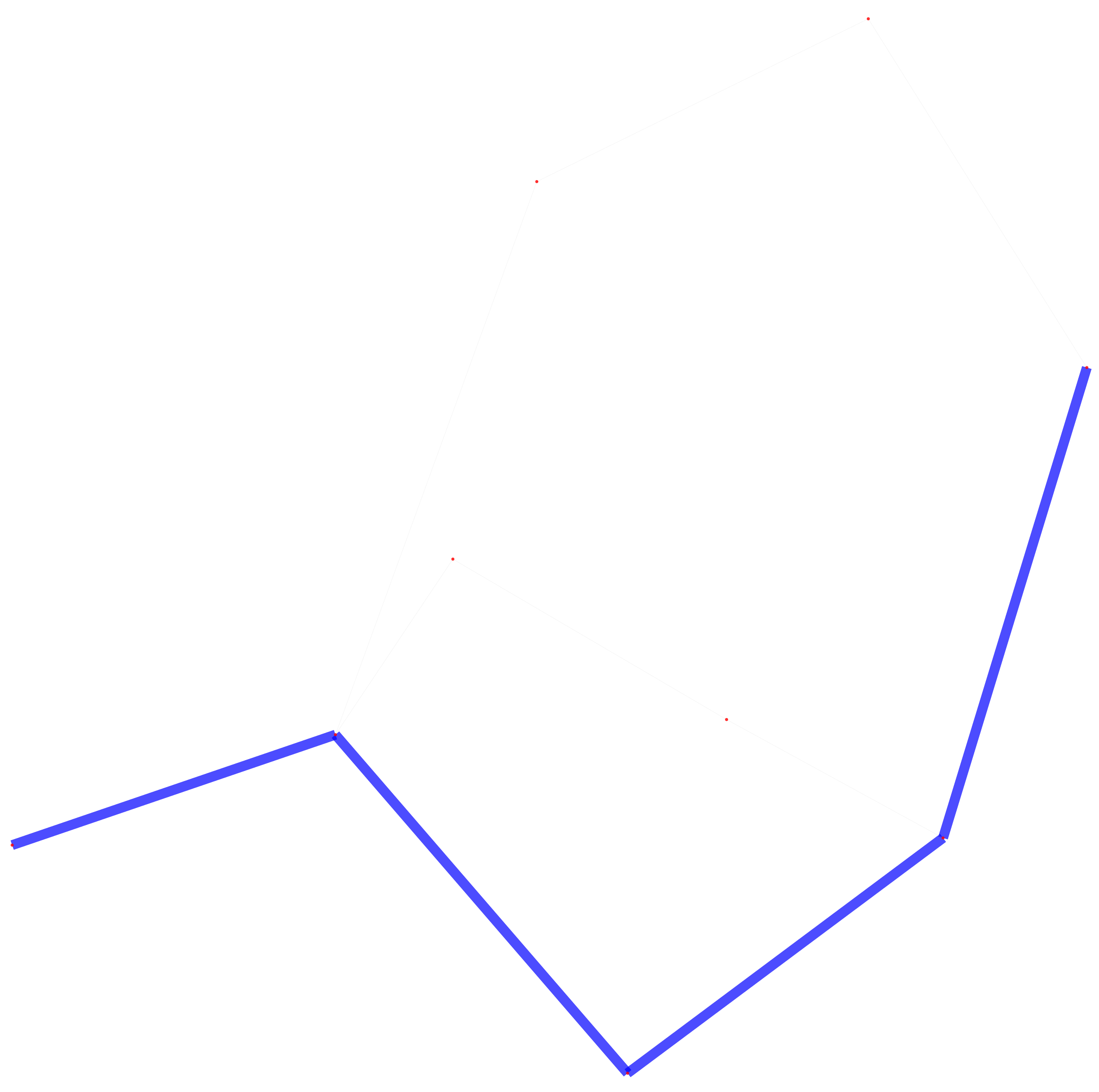}}\hspace{1mm}
\subfloat[  { [60, 1]} ]{\label{fig: path_1}\includegraphics[width=0.10\textwidth]{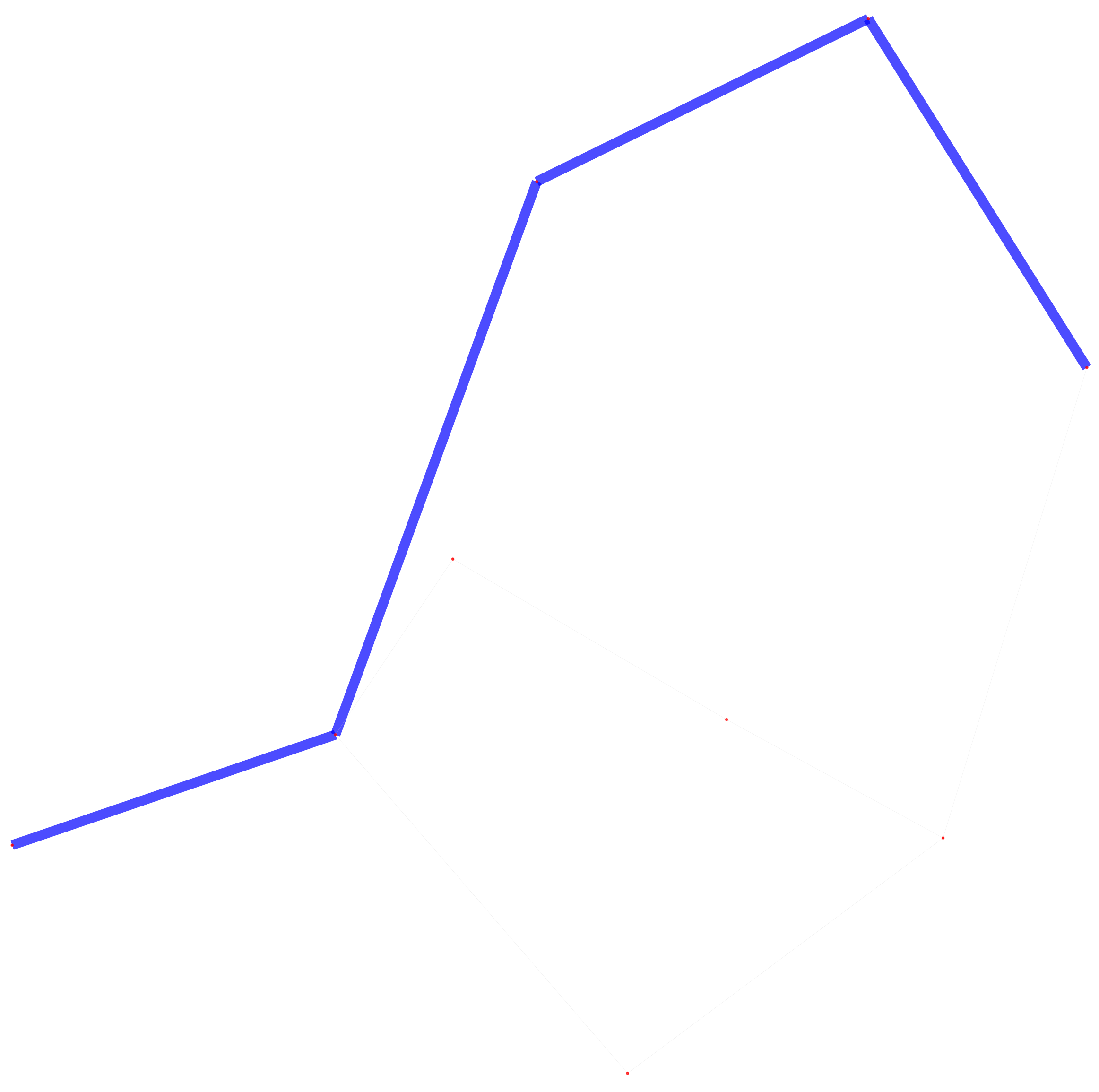}}\hspace{1mm}
\subfloat[  { [60, 4]} ]{\label{fig: path_1}\includegraphics[width=0.10\textwidth]{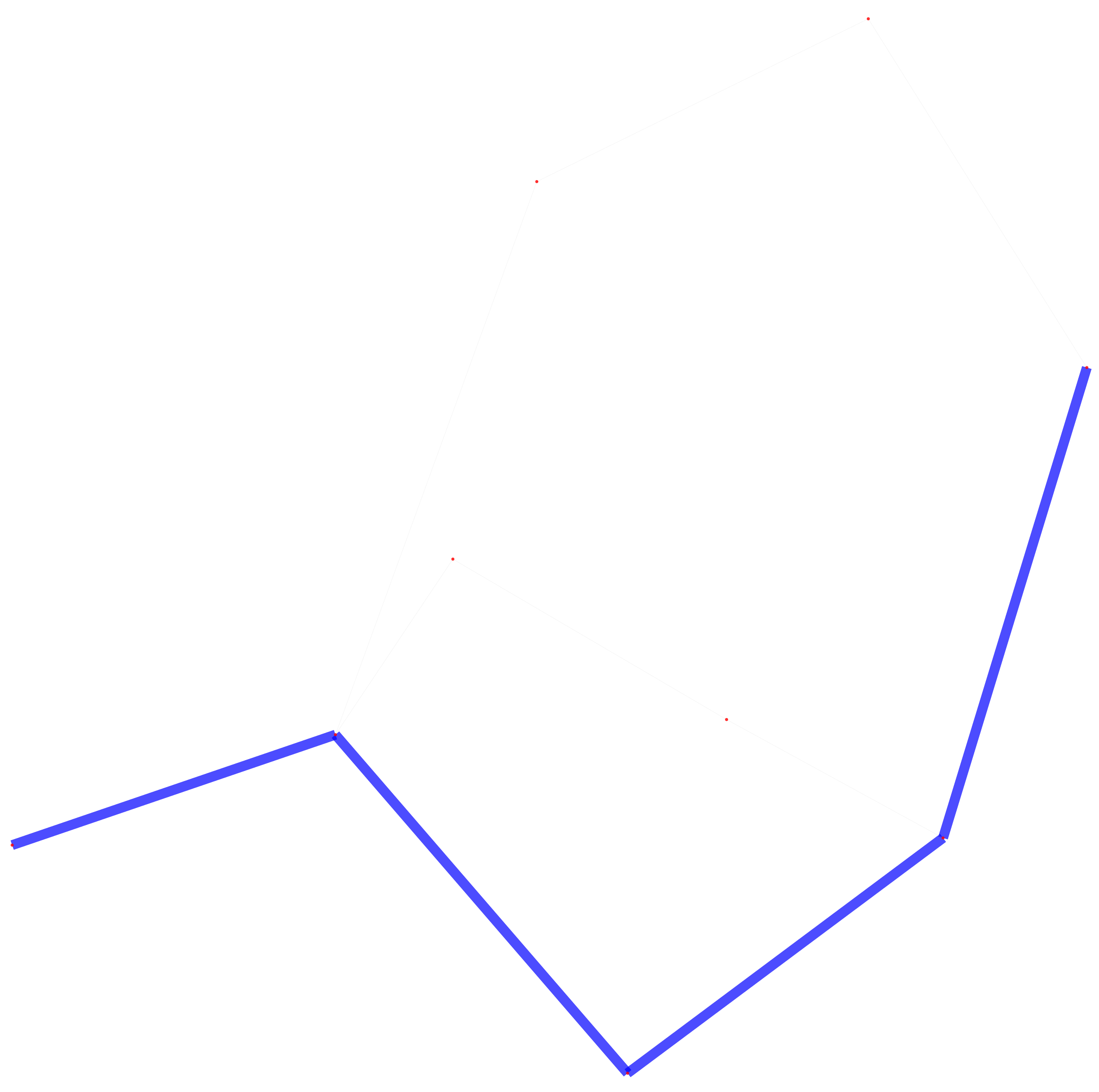}}\hspace{1mm}
\subfloat[  Target ]{\label{fig: path_1}\includegraphics[width=0.10\textwidth]{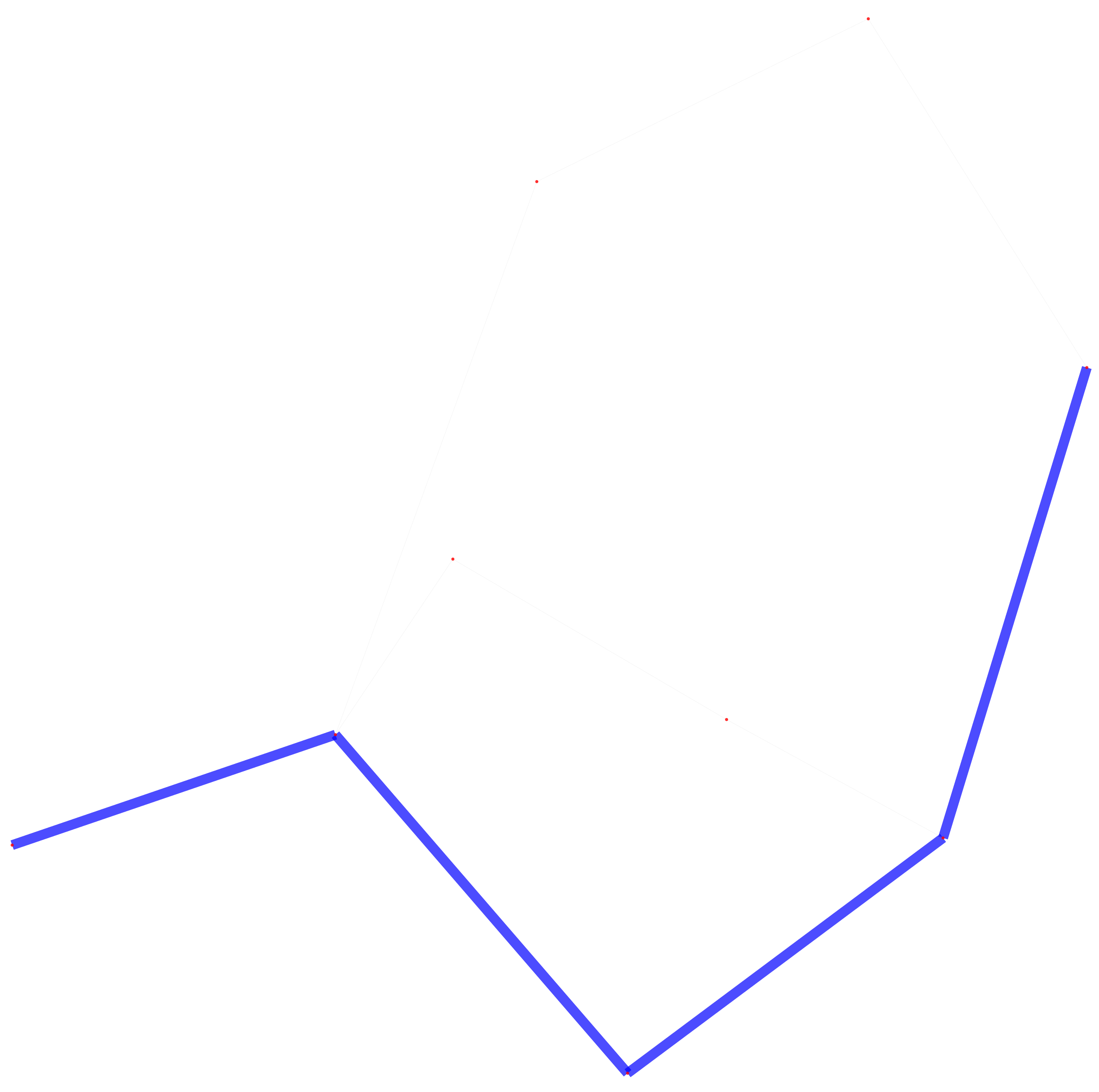}}\hspace{0mm}

\subfloat[  { [240, 0]} ]{\label{fig: path_1}\includegraphics[width=0.10\textwidth]{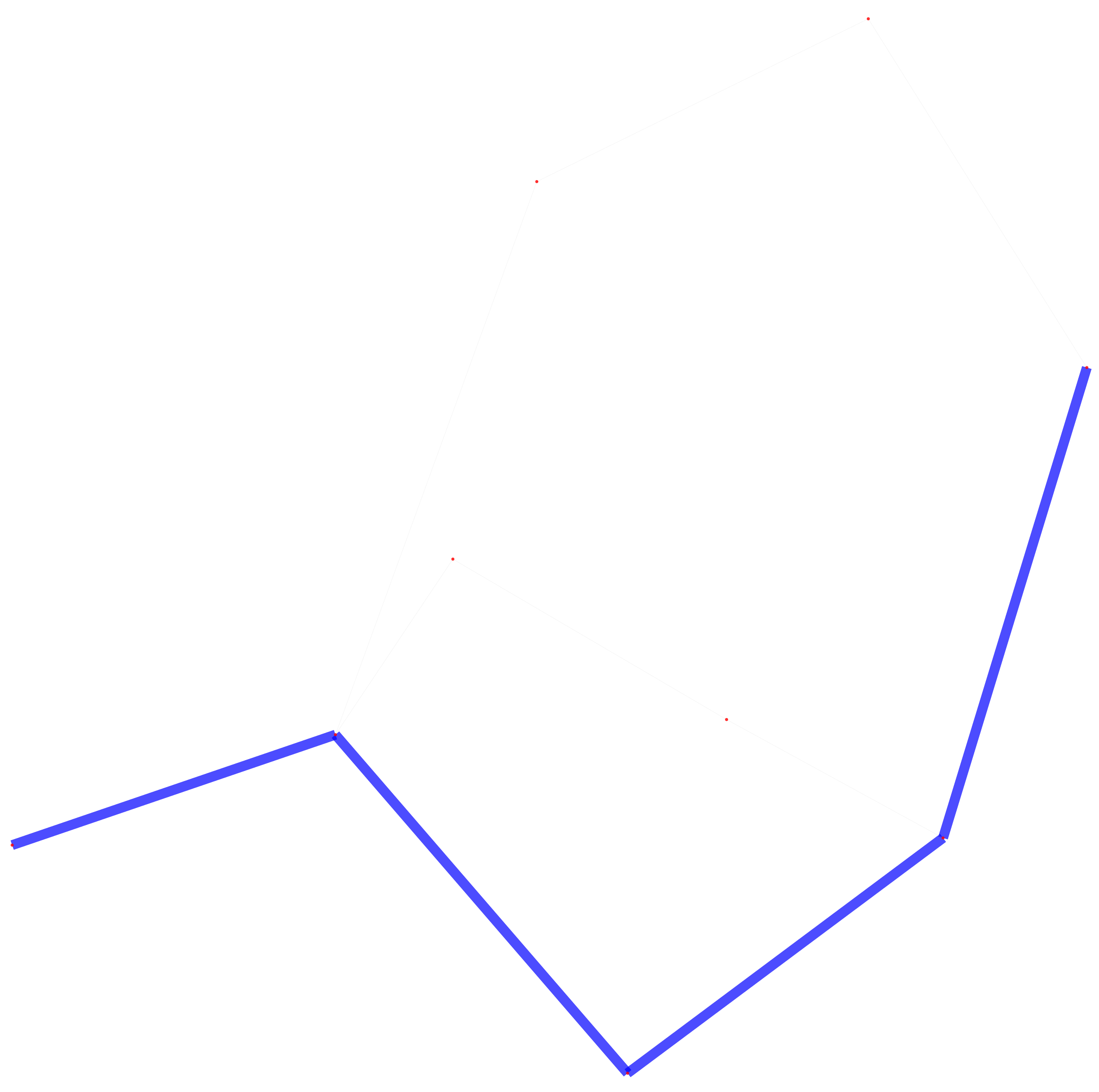}}\hspace{1mm}
\subfloat[  { [240, 1]} ]{\label{fig: path_1}\includegraphics[width=0.10\textwidth]{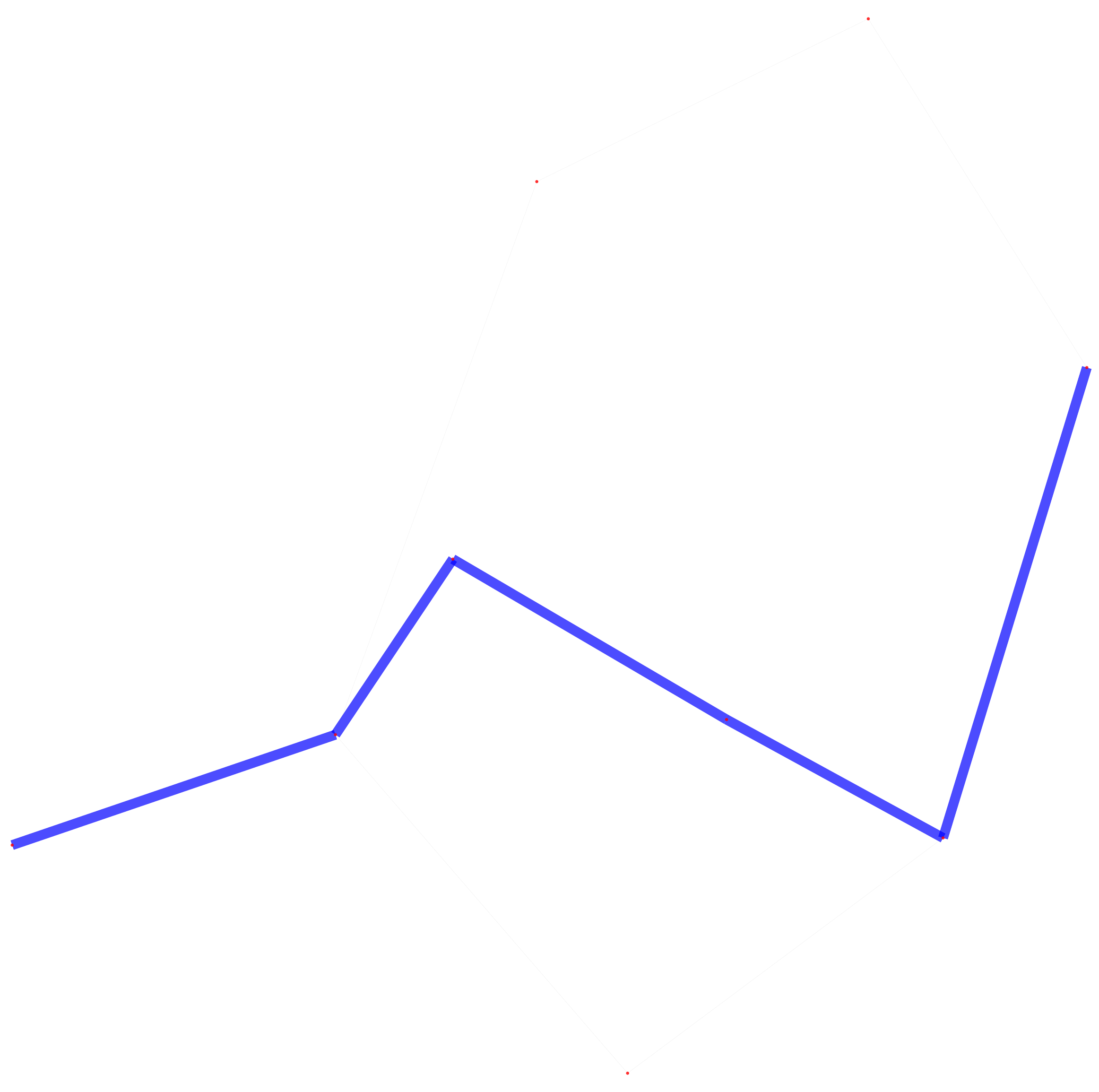}}\hspace{1mm}
\subfloat[  { [240, 4]} ]{\label{fig: path_1}\includegraphics[width=0.10\textwidth]{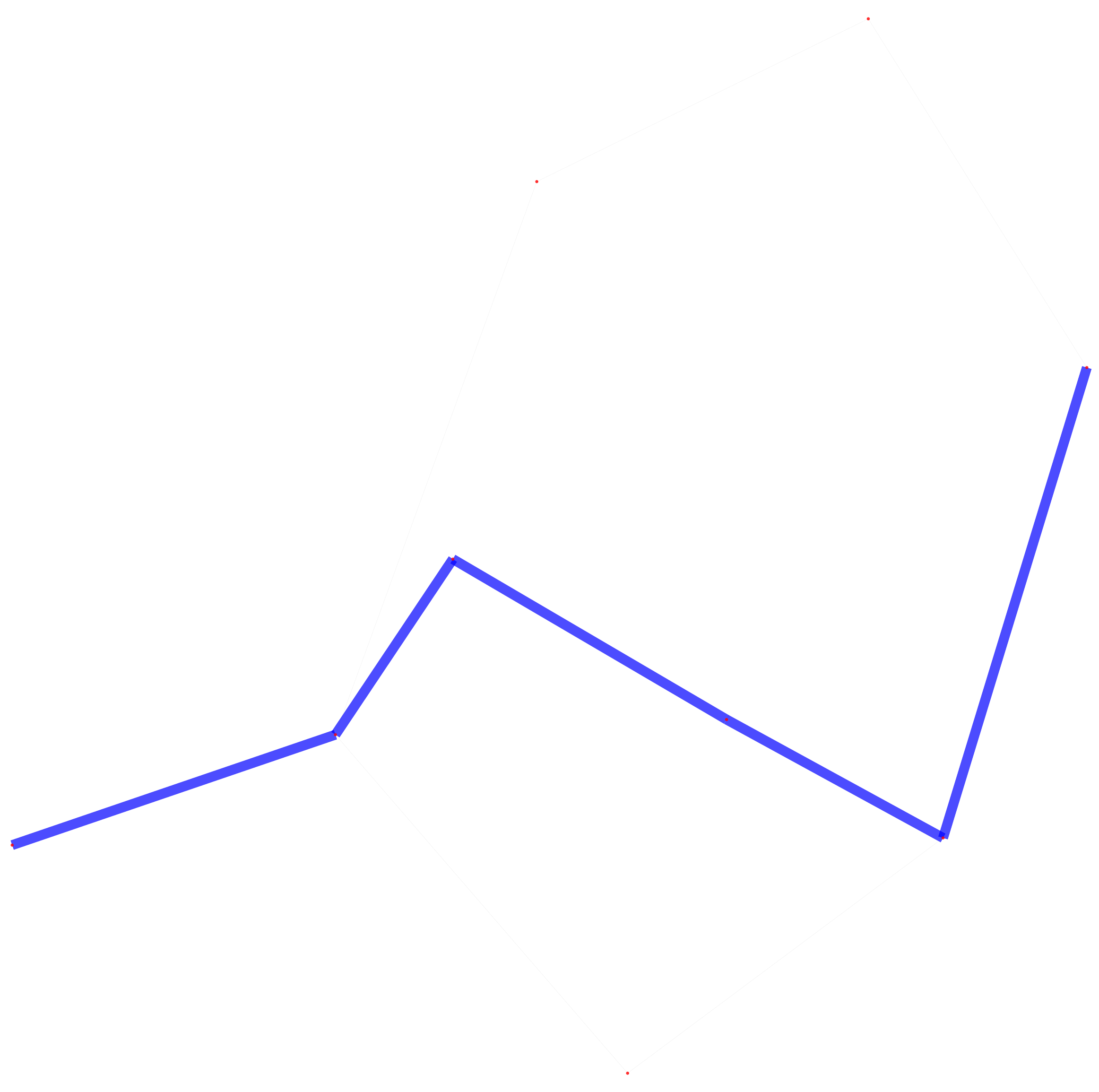}}\hspace{1mm}
\subfloat[  Target ]{\label{fig: path_1}\includegraphics[width=0.10\textwidth]{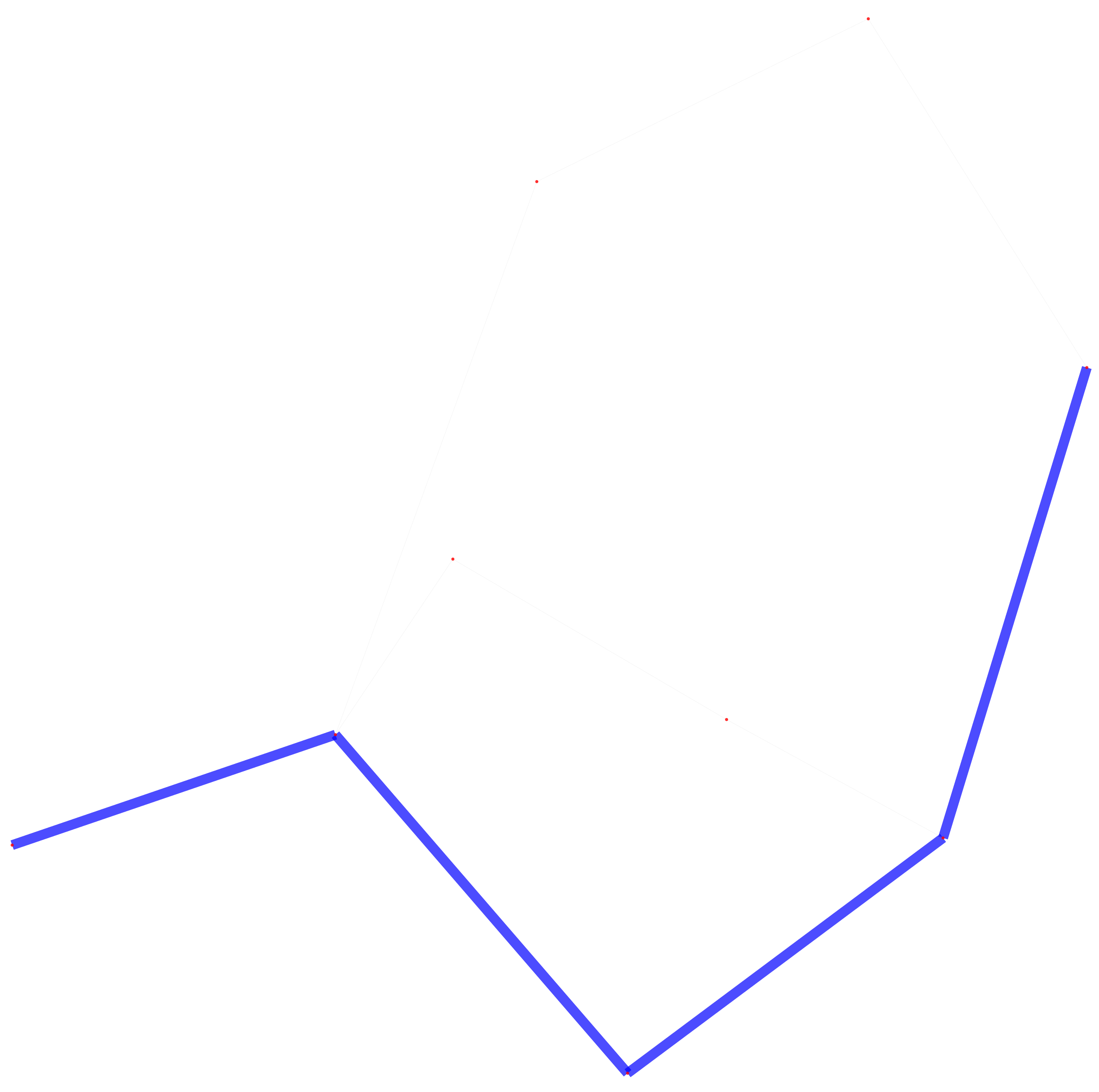}}\hspace{0mm}

\caption{\small Visualization of the results of QRTS-P for an example testing query for Tree-to-Path on Kro.}
\label{fig: more_path_1}
\end{figure}

\begin{figure}[h]
\centering
\captionsetup[subfloat]{labelfont=scriptsize,textfont=scriptsize,labelformat=empty}

\subfloat[  { [30, 0]} ]{\label{fig: path_1}\includegraphics[width=0.10\textwidth]{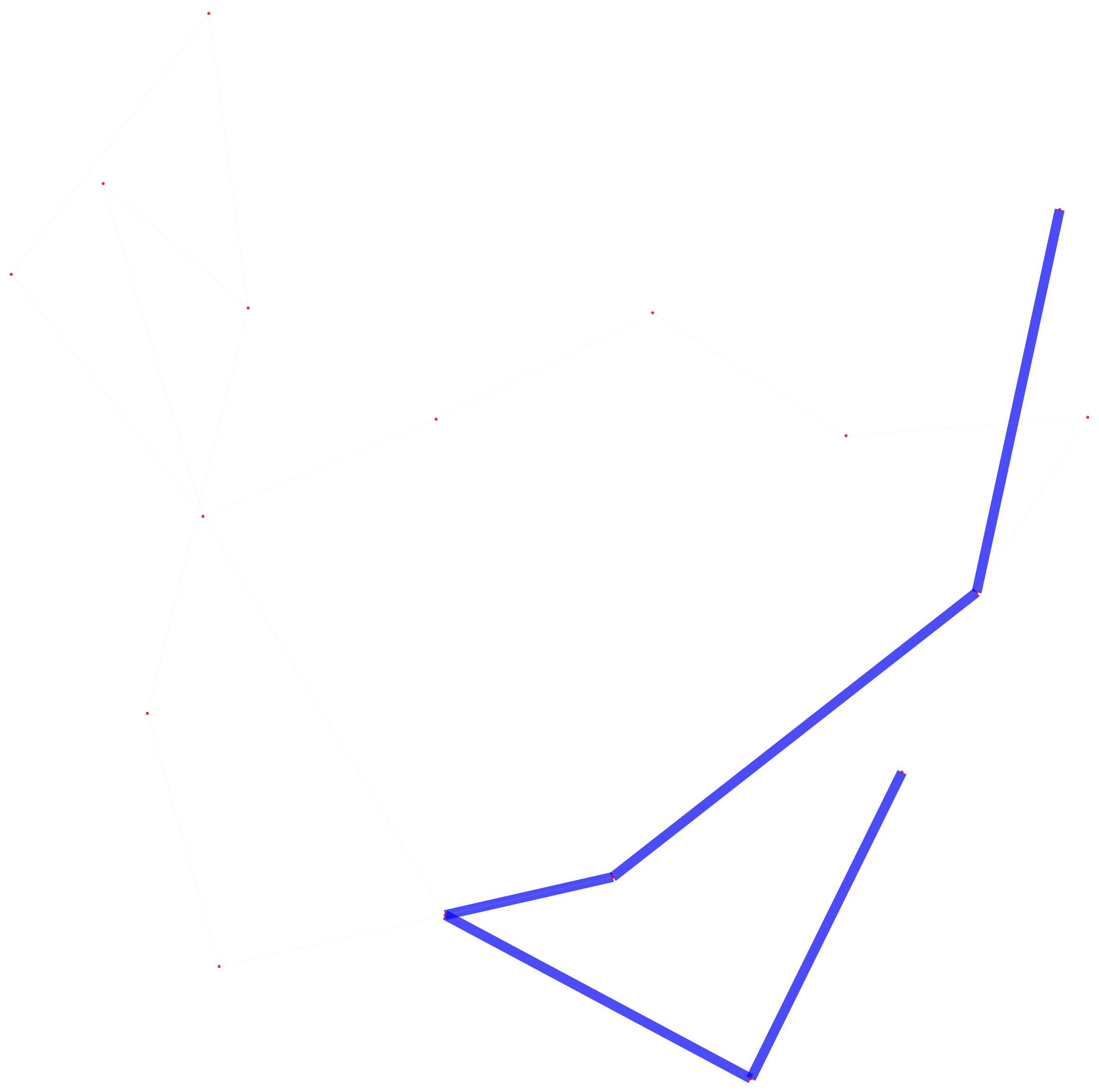}}\hspace{1mm}
\subfloat[  { [30, 1]} ]{\label{fig: path_1}\includegraphics[width=0.10\textwidth]{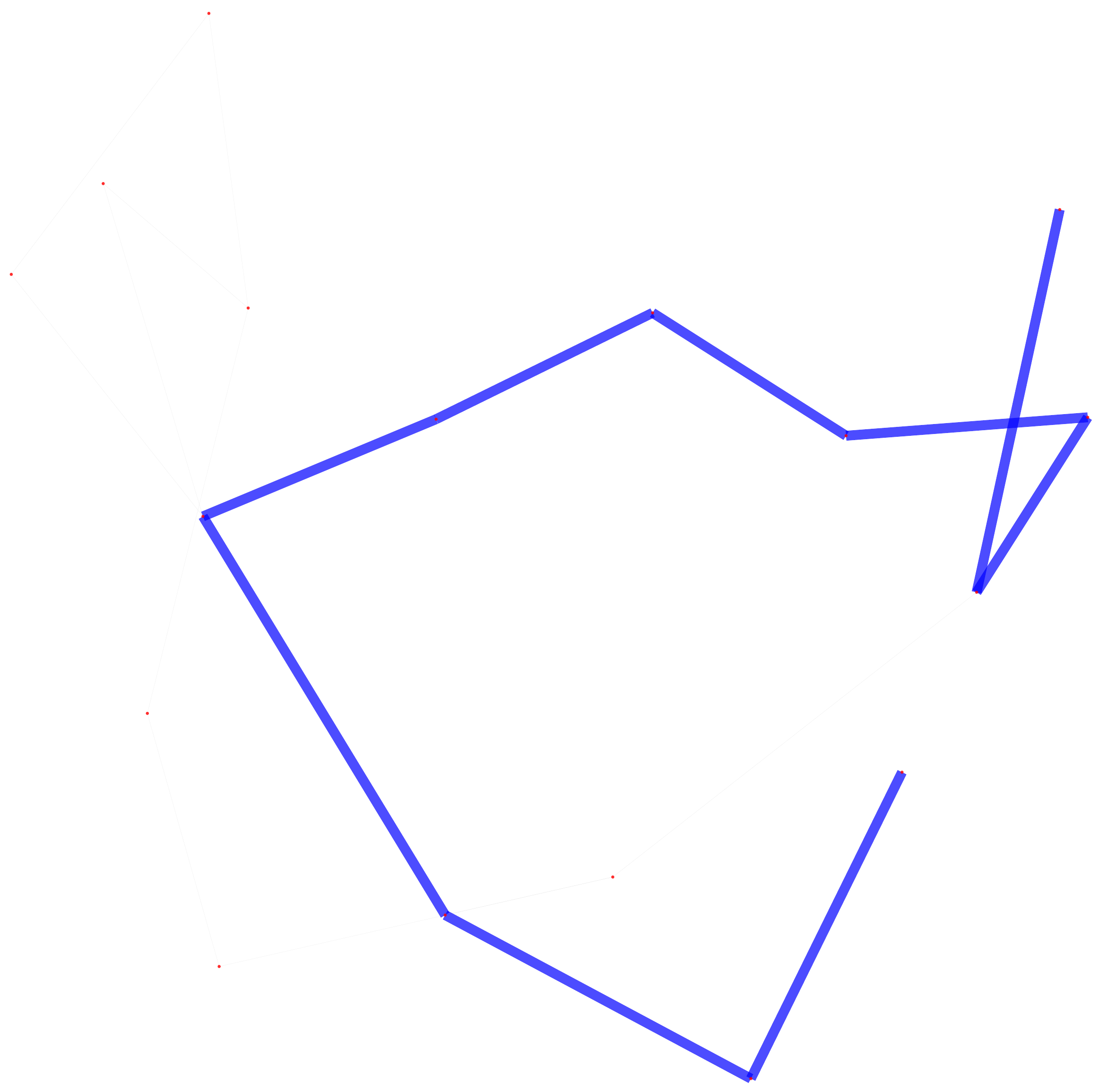}}\hspace{1mm}
\subfloat[  { [30, 4]} ]{\label{fig: path_1}\includegraphics[width=0.10\textwidth]{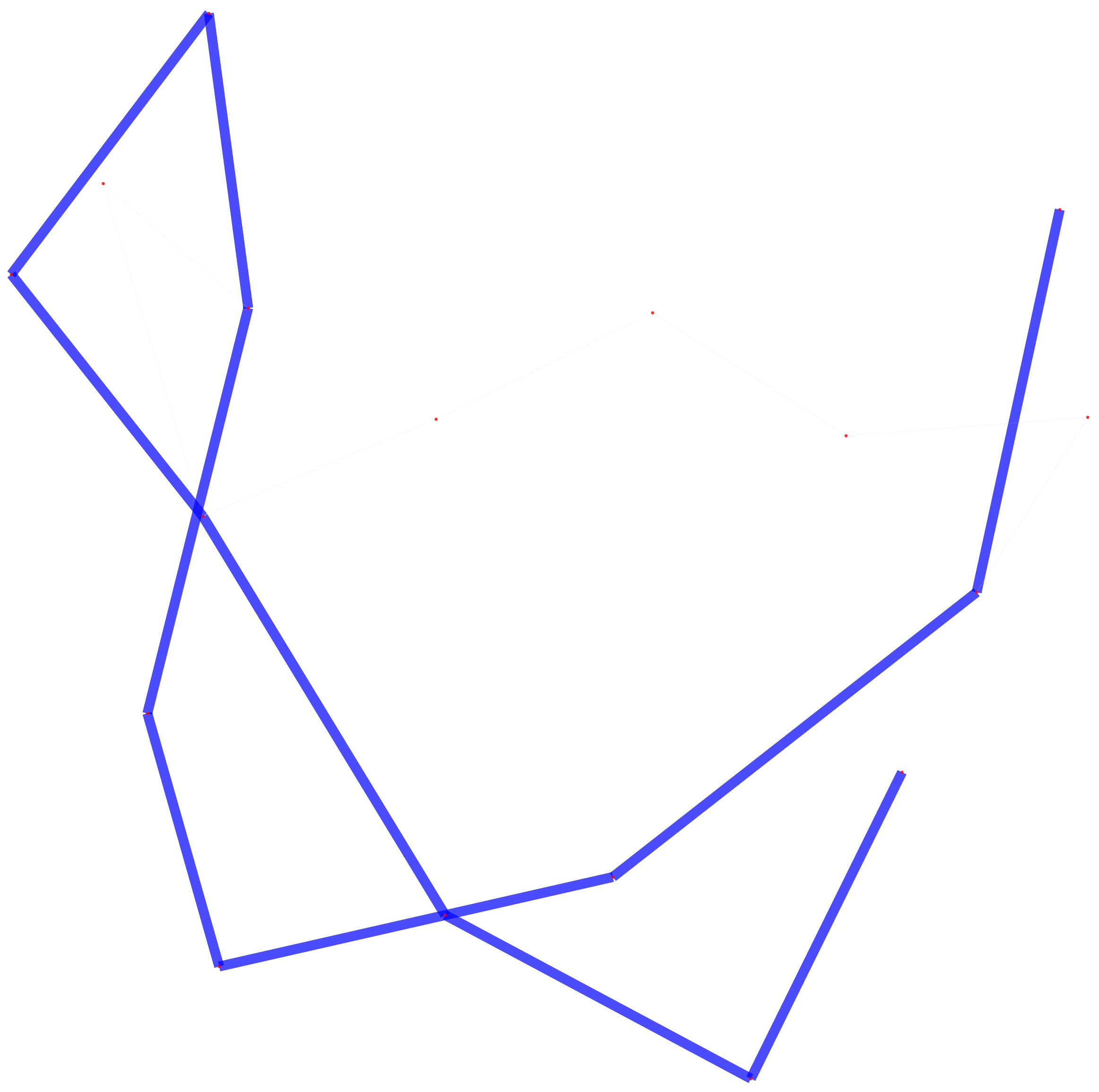}}\hspace{1mm}
\subfloat[  Target ]{\label{fig: path_1}\includegraphics[width=0.10\textwidth]{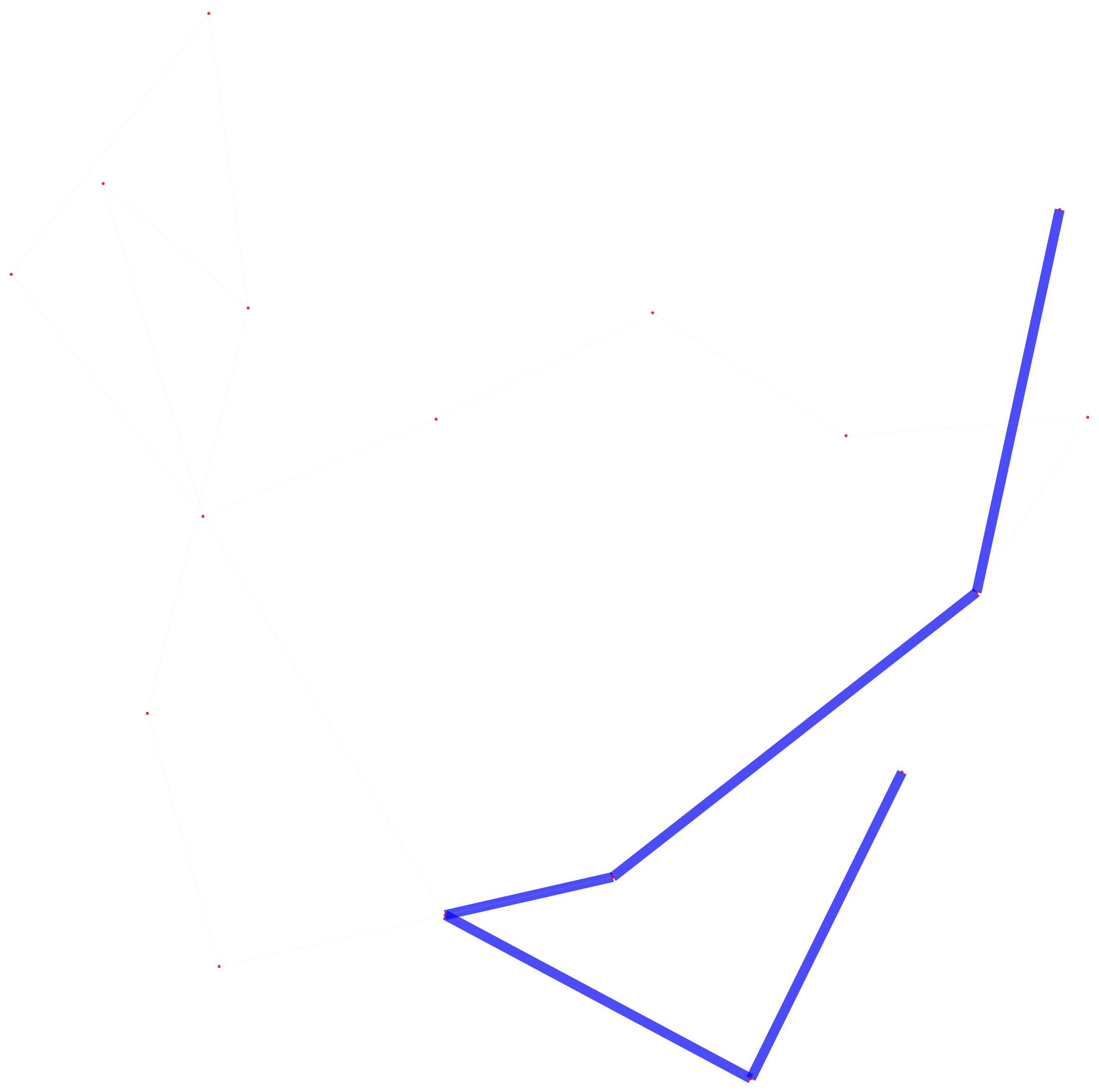}}\hspace{0mm}

\subfloat[  { [60, 0]} ]{\label{fig: path_1}\includegraphics[width=0.10\textwidth]{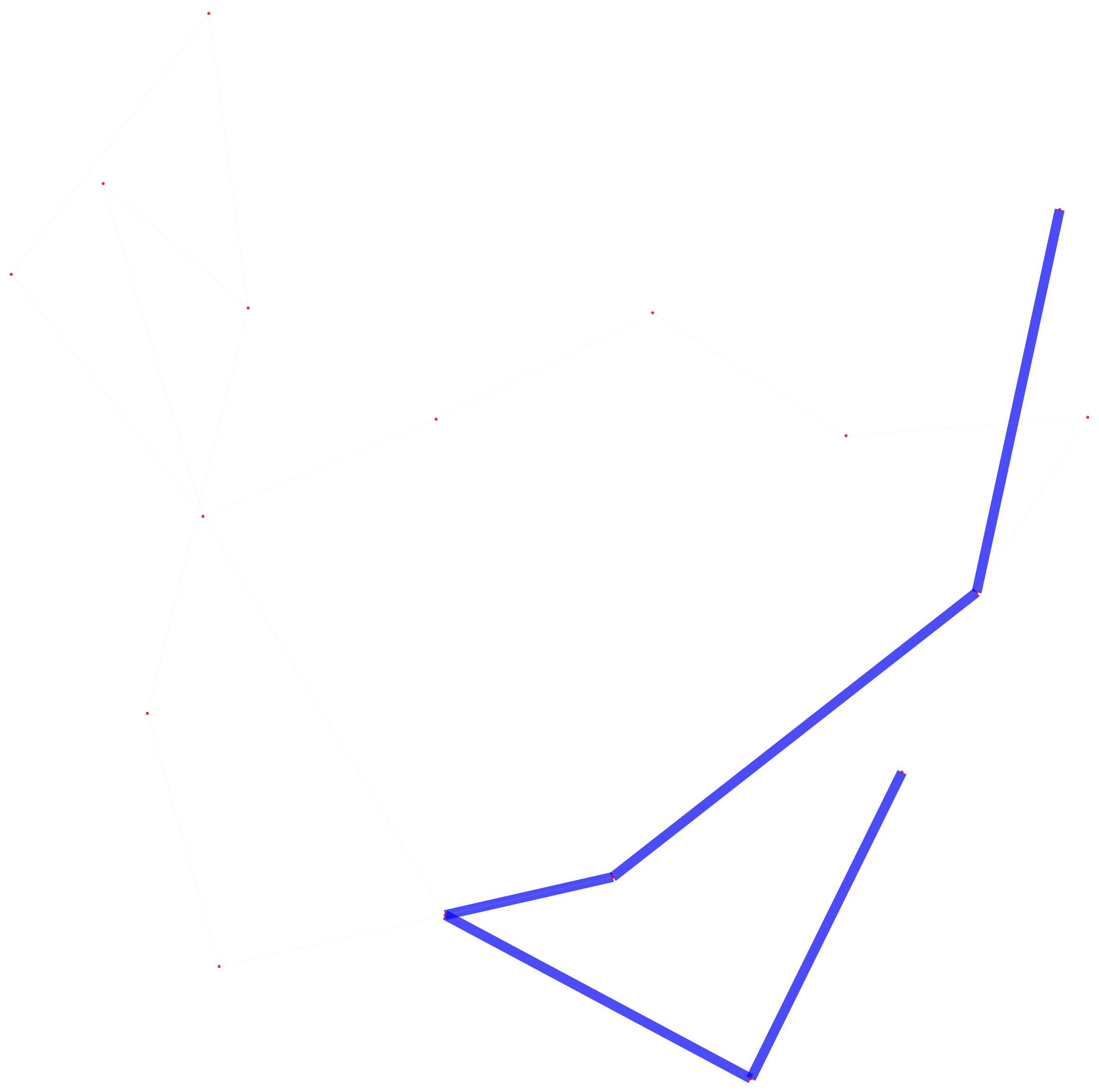}}\hspace{1mm}
\subfloat[  { [60, 1]} ]{\label{fig: path_1}\includegraphics[width=0.10\textwidth]{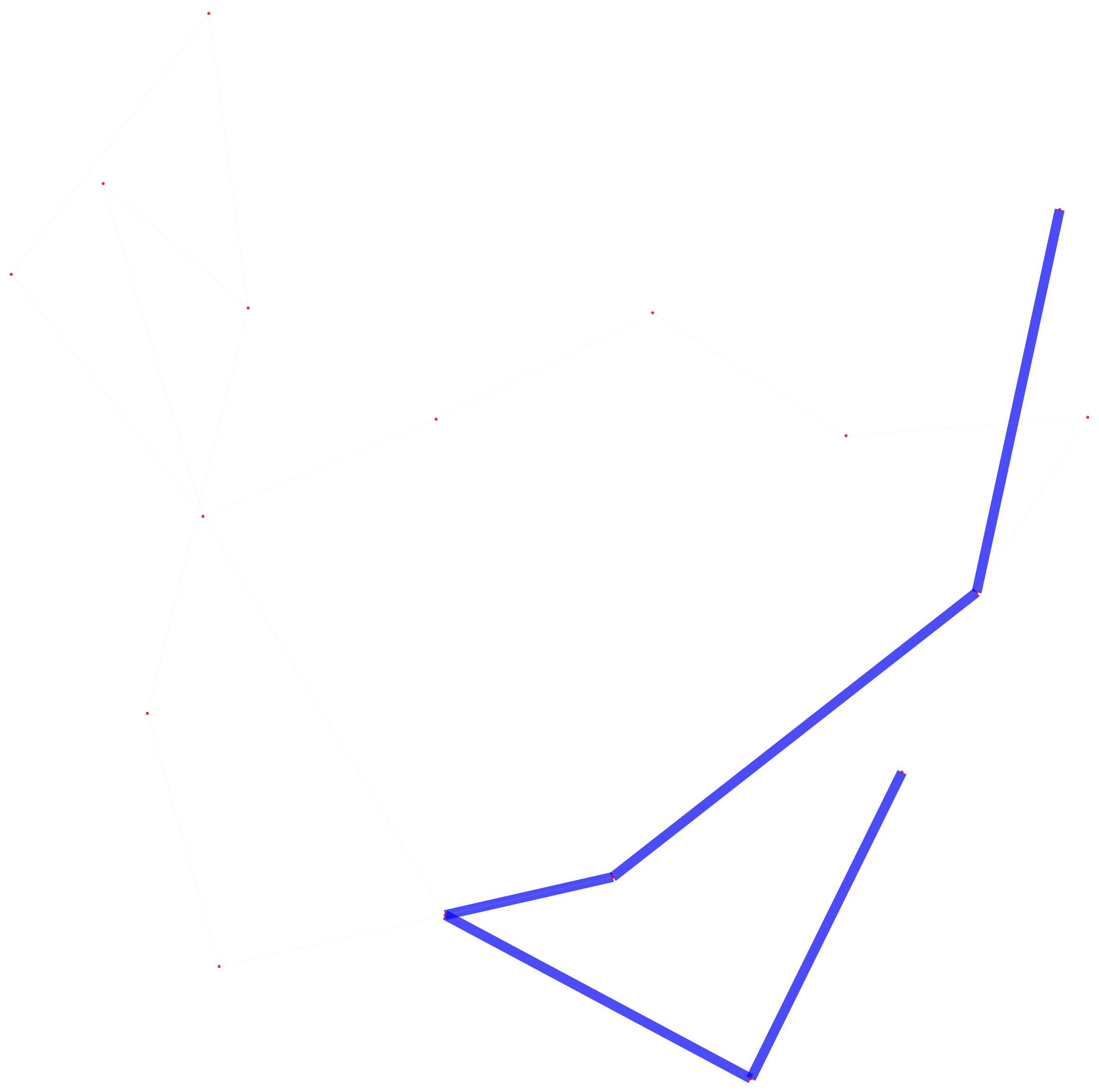}}\hspace{1mm}
\subfloat[  { [60, 4]} ]{\label{fig: path_1}\includegraphics[width=0.10\textwidth]{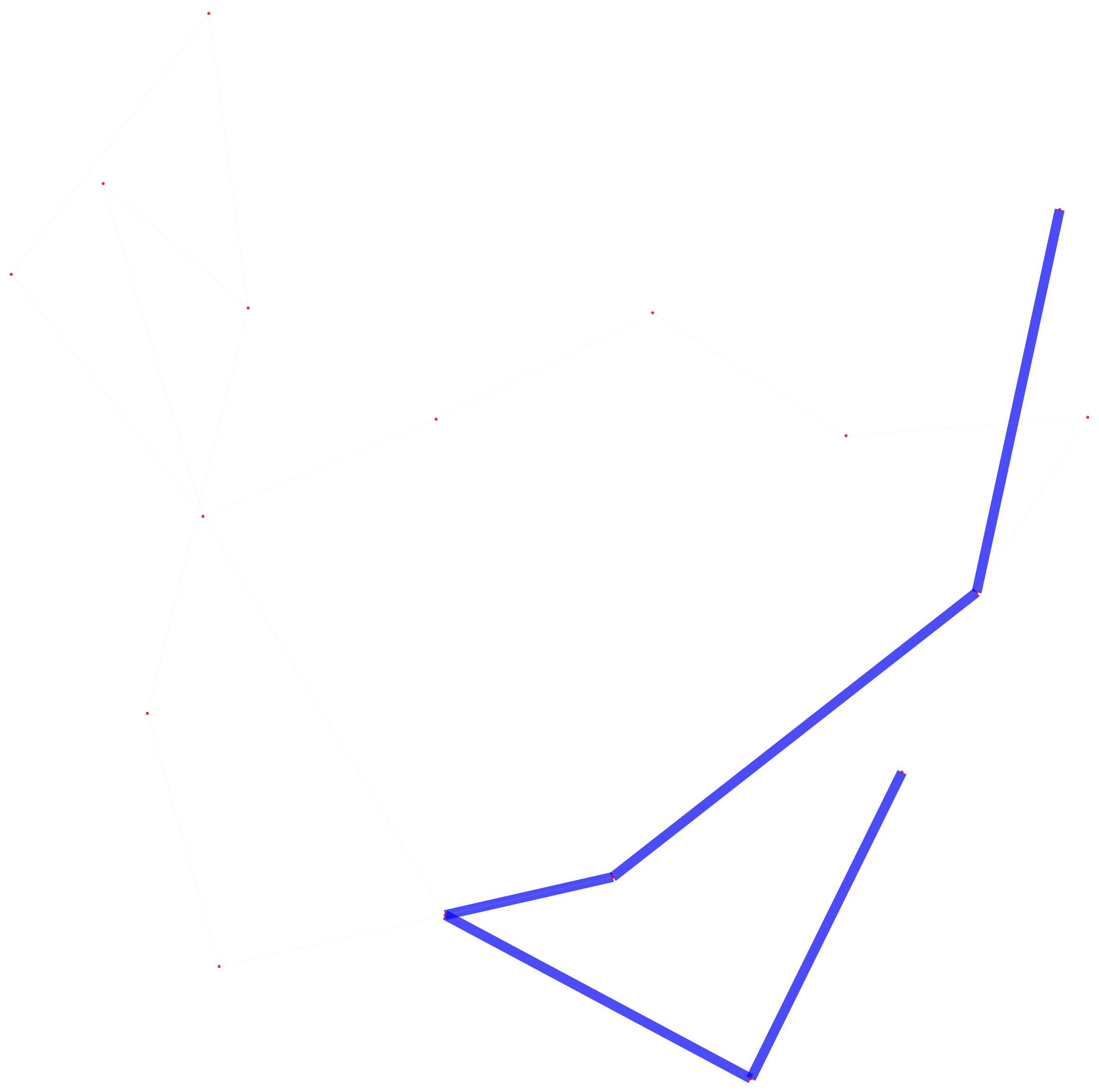}}\hspace{1mm}
\subfloat[  Target ]{\label{fig: path_1}\includegraphics[width=0.10\textwidth]{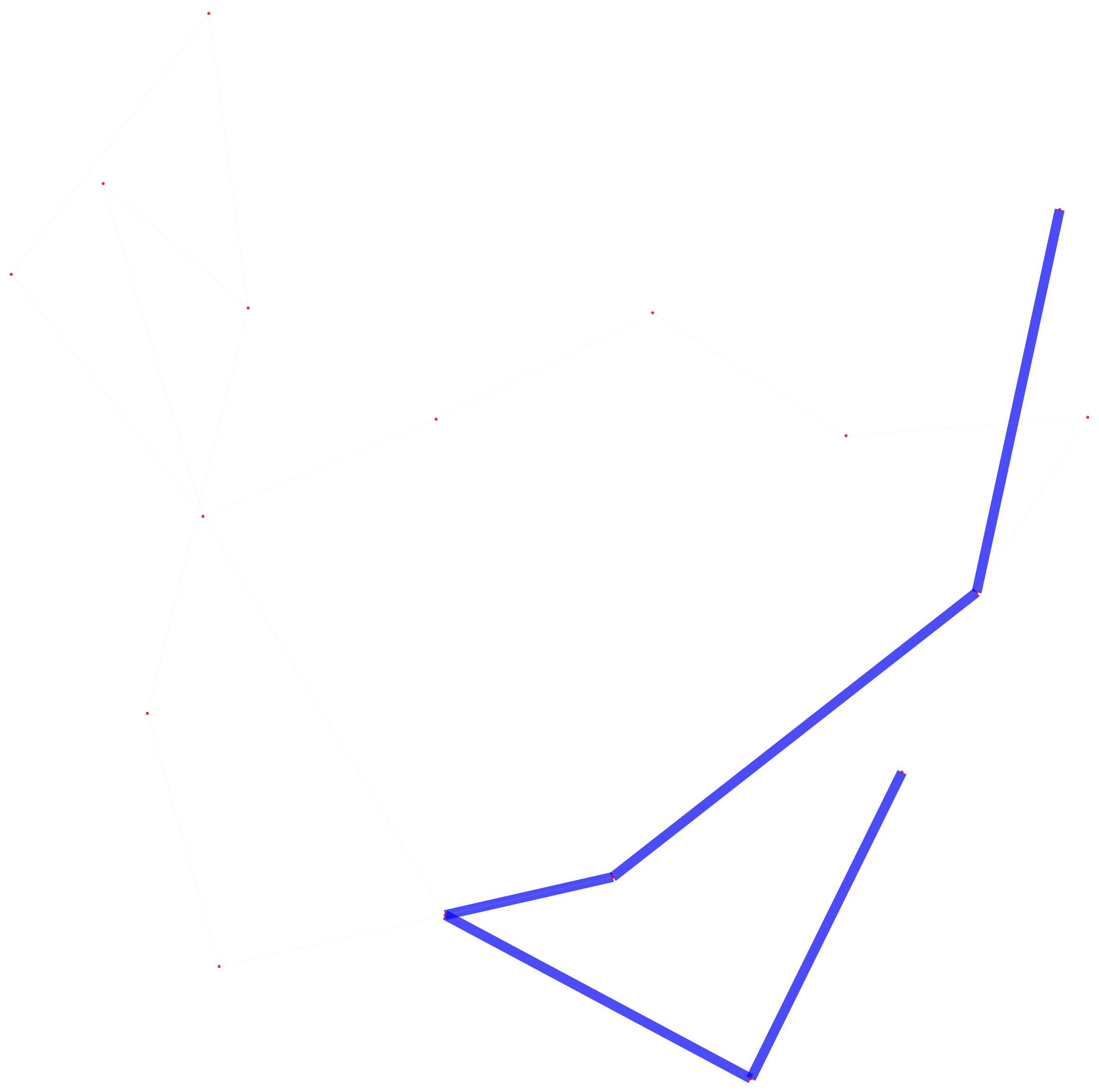}}\hspace{0mm}

\subfloat[  { [240, 0]} ]{\label{fig: path_1}\includegraphics[width=0.10\textwidth]{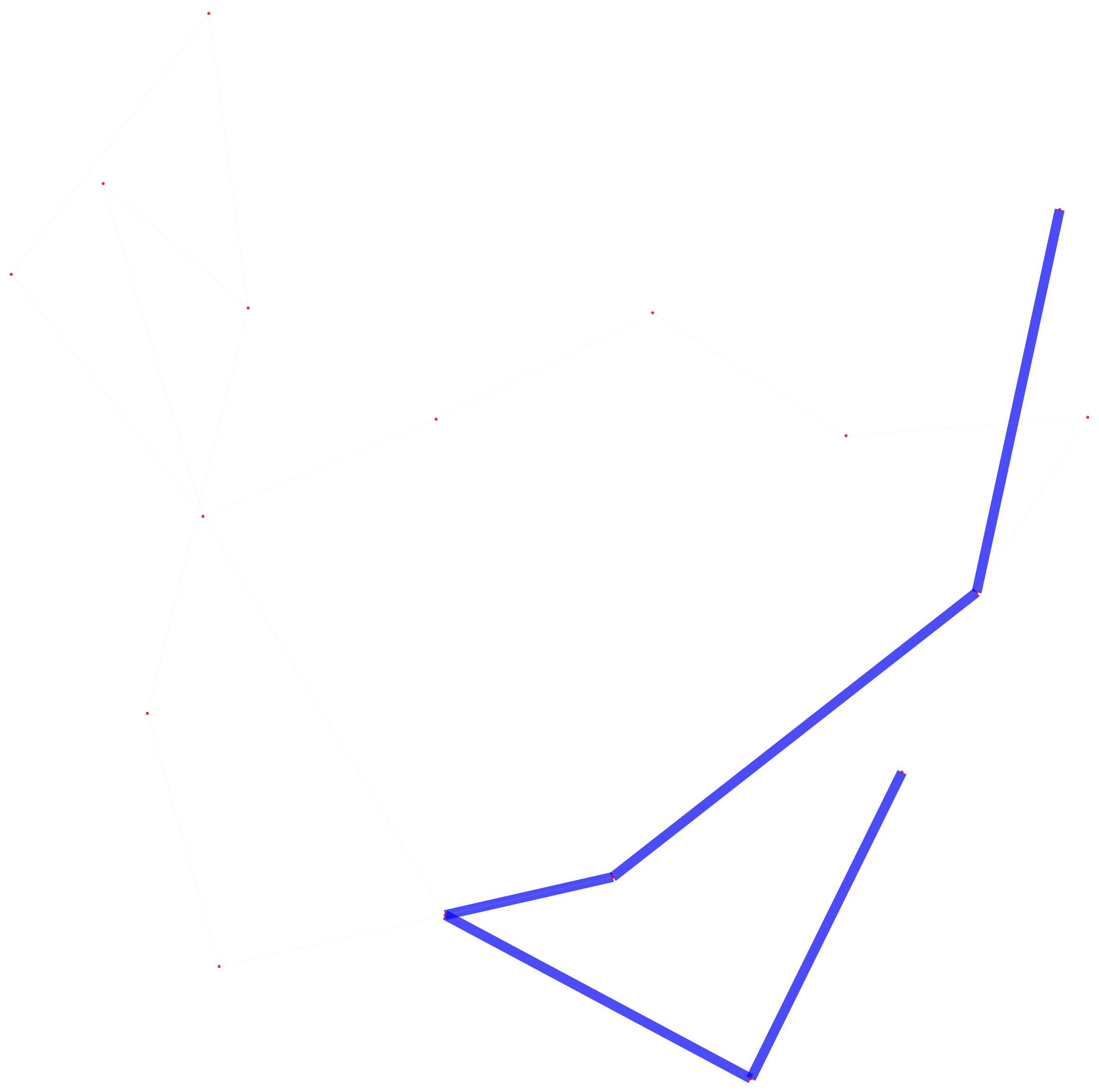}}\hspace{1mm}
\subfloat[  { [240, 1]} ]{\label{fig: path_1}\includegraphics[width=0.10\textwidth]{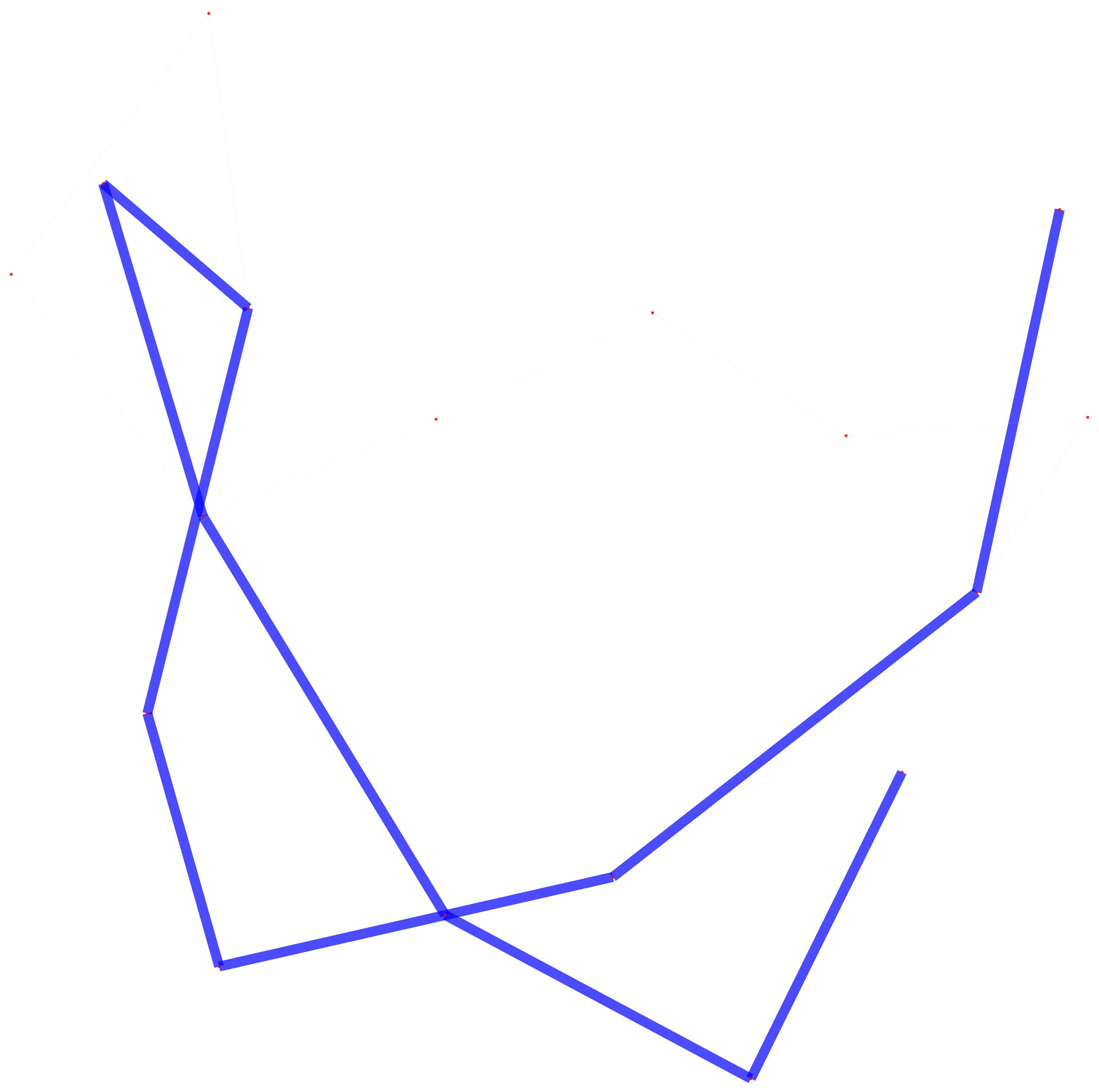}}\hspace{1mm}
\subfloat[  { [240, 4]} ]{\label{fig: path_1}\includegraphics[width=0.10\textwidth]{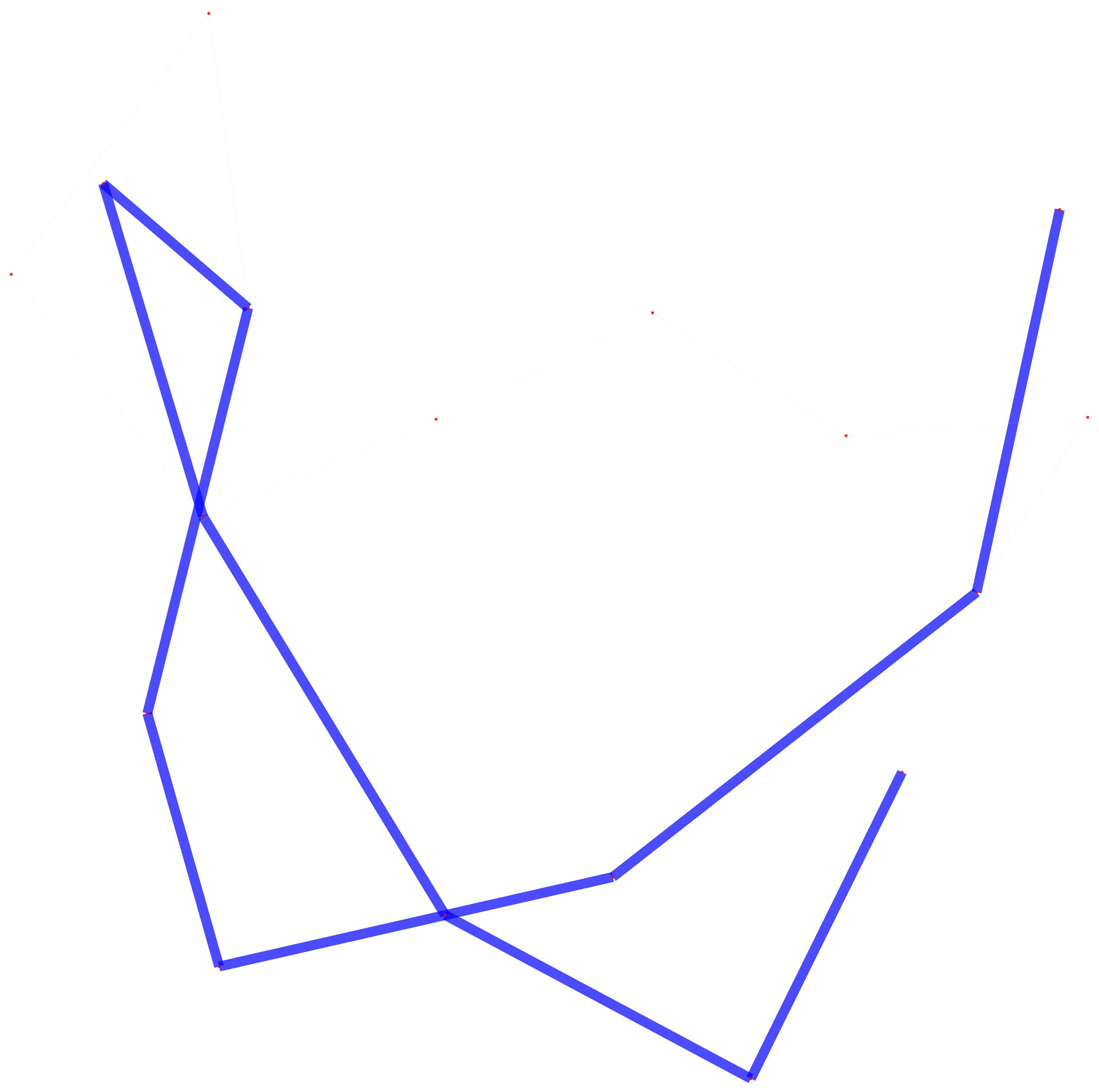}}\hspace{1mm}
\subfloat[  Target ]{\label{fig: path_1}\includegraphics[width=0.10\textwidth]{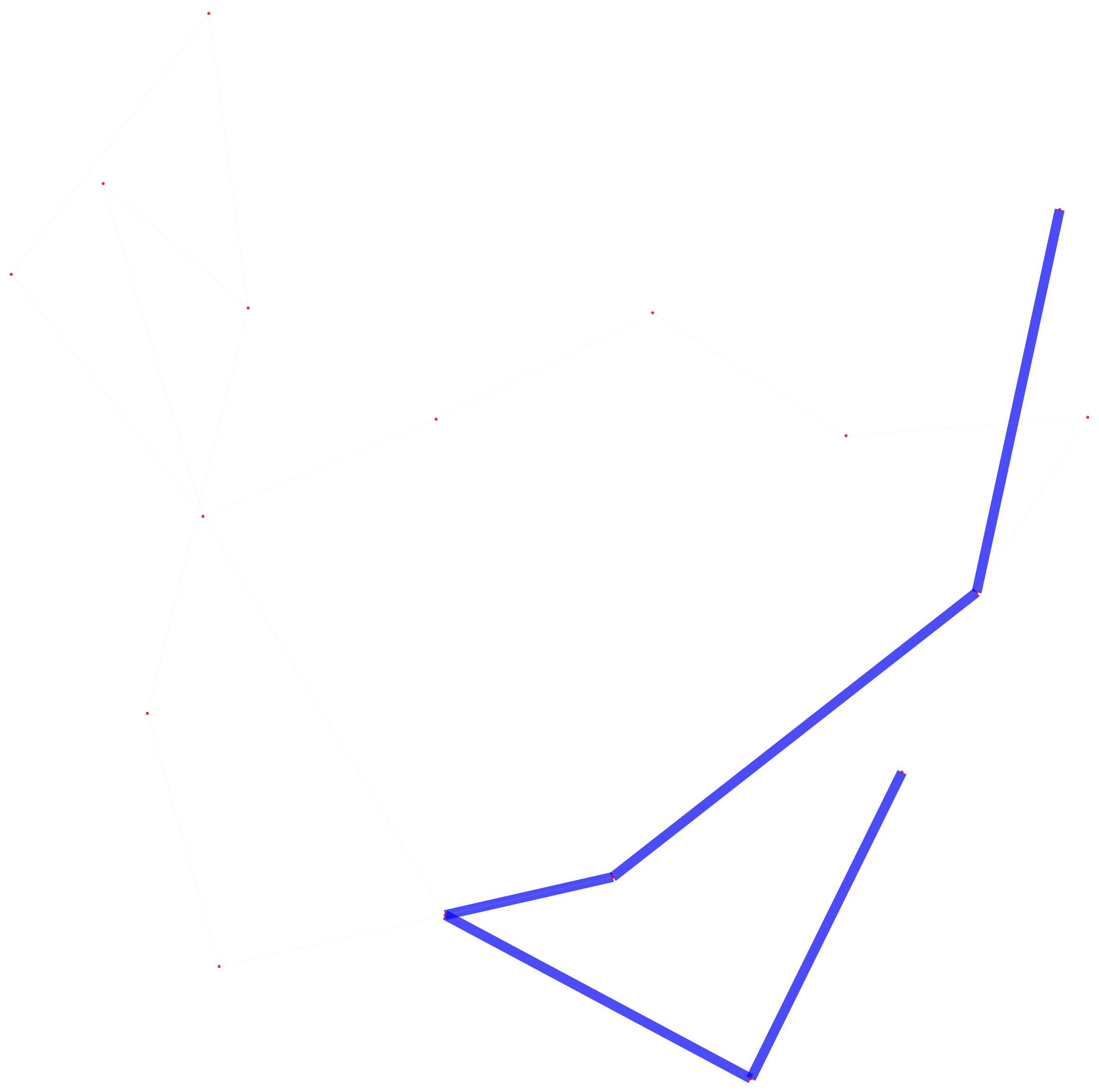}}\hspace{0mm}

\caption{\small Visualization of the results of QRTS-P for an example testing query for Tree-to-Path on Kro.}
\label{fig: more_path_2}
\end{figure}

\begin{figure}[h]
\centering
\captionsetup[subfloat]{labelfont=scriptsize,textfont=scriptsize,labelformat=empty}

\subfloat[  { [30, 0]} ]{\label{fig: path_1}\includegraphics[width=0.10\textwidth]{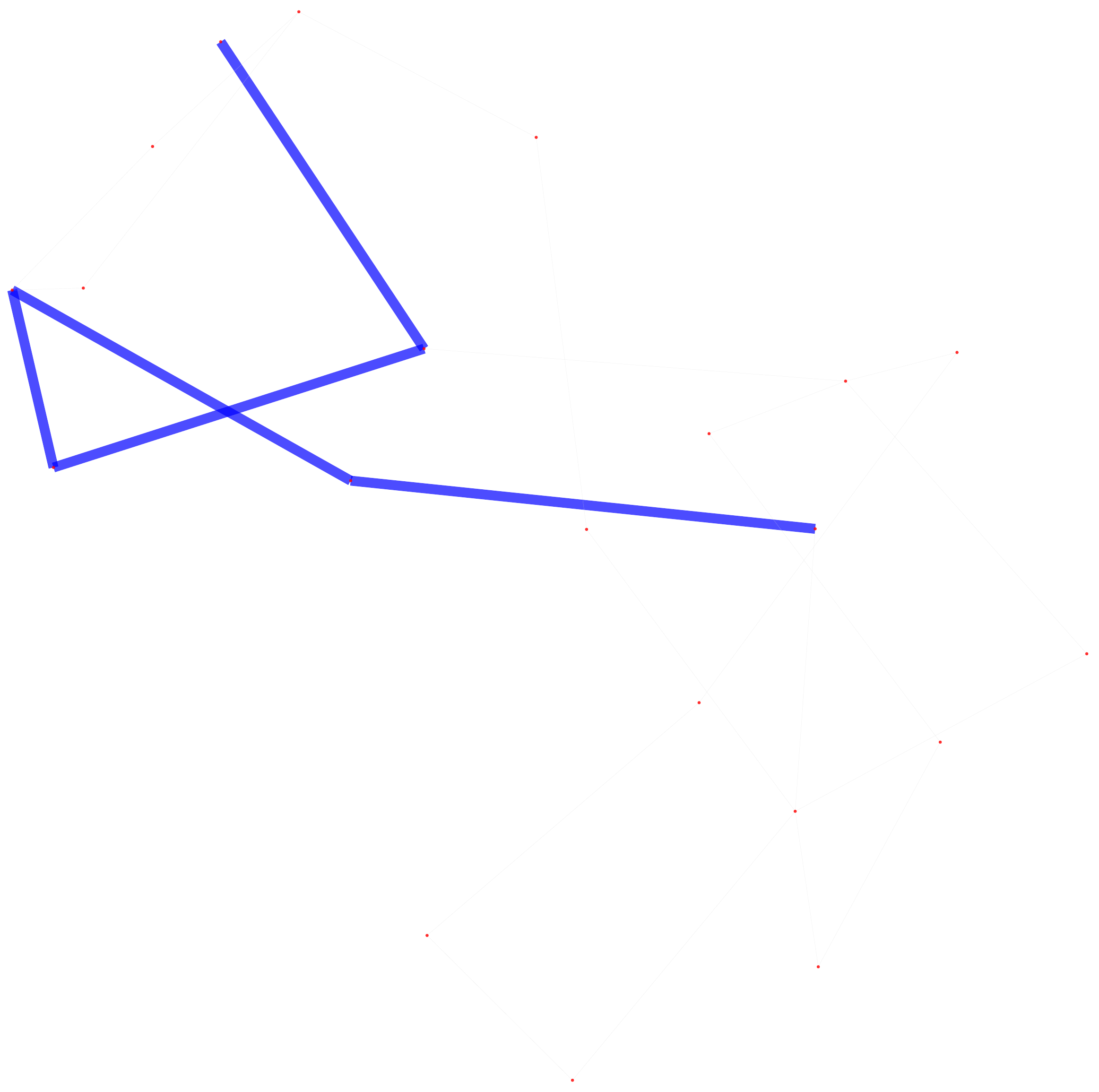}}\hspace{1mm}
\subfloat[  { [30, 1]} ]{\label{fig: path_1}\includegraphics[width=0.10\textwidth]{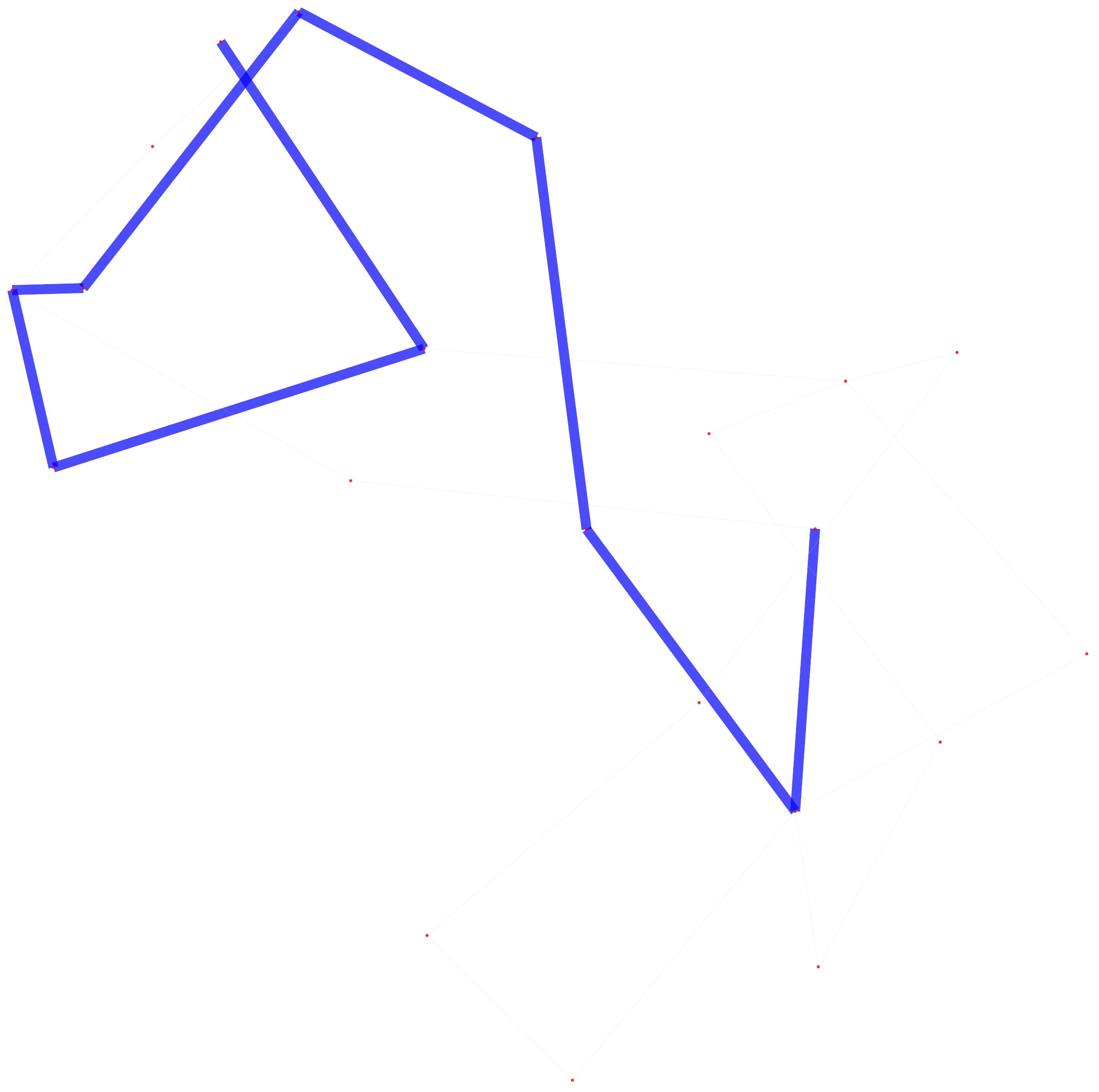}}\hspace{1mm}
\subfloat[  { [30, 4]} ]{\label{fig: path_1}\includegraphics[width=0.10\textwidth]{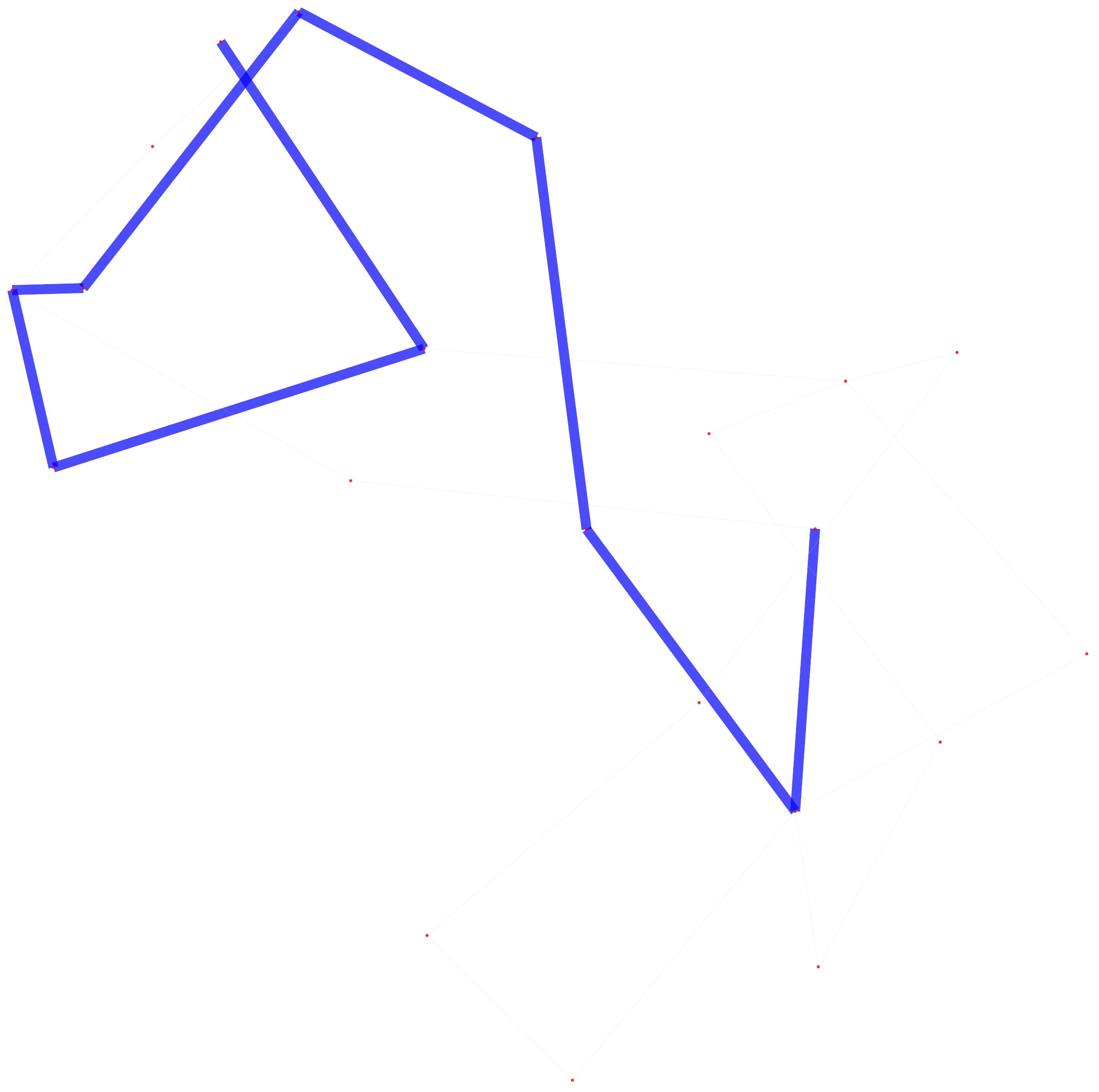}}\hspace{1mm}
\subfloat[  Target ]{\label{fig: path_1}\includegraphics[width=0.10\textwidth]{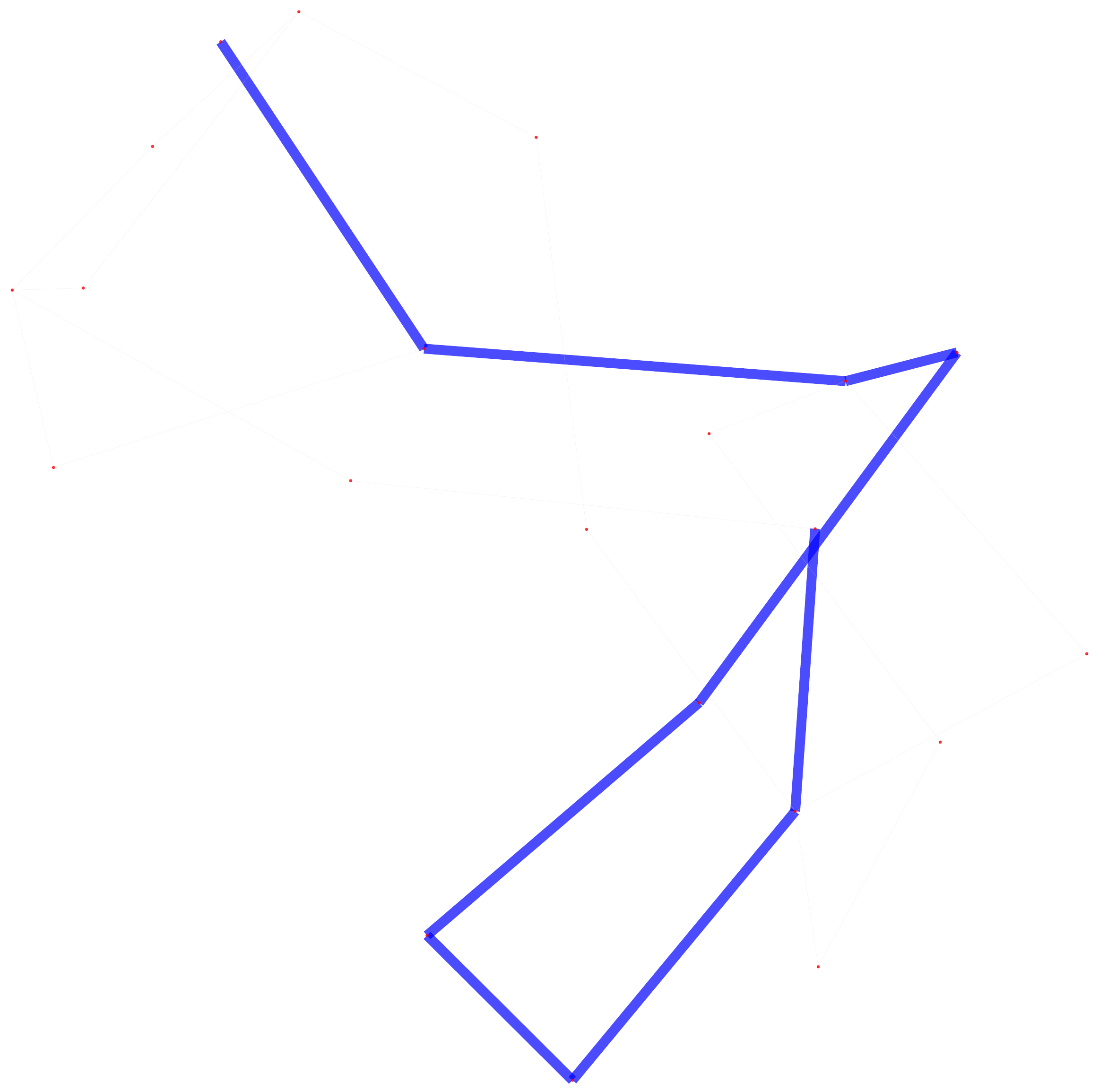}}\hspace{0mm}

\subfloat[  { [60, 0]} ]{\label{fig: path_1}\includegraphics[width=0.10\textwidth]{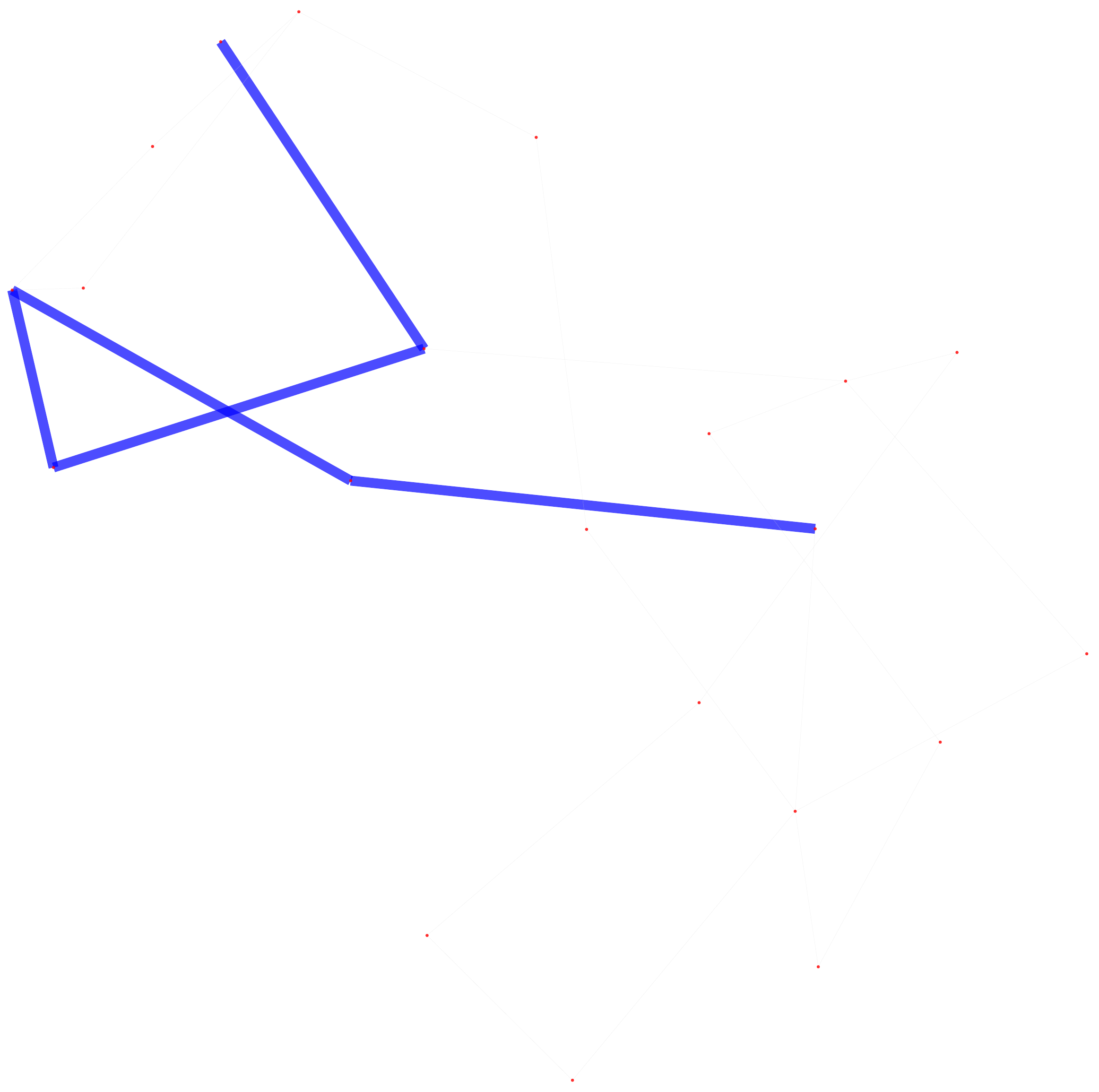}}\hspace{1mm}
\subfloat[  { [60, 1]} ]{\label{fig: path_1}\includegraphics[width=0.10\textwidth]{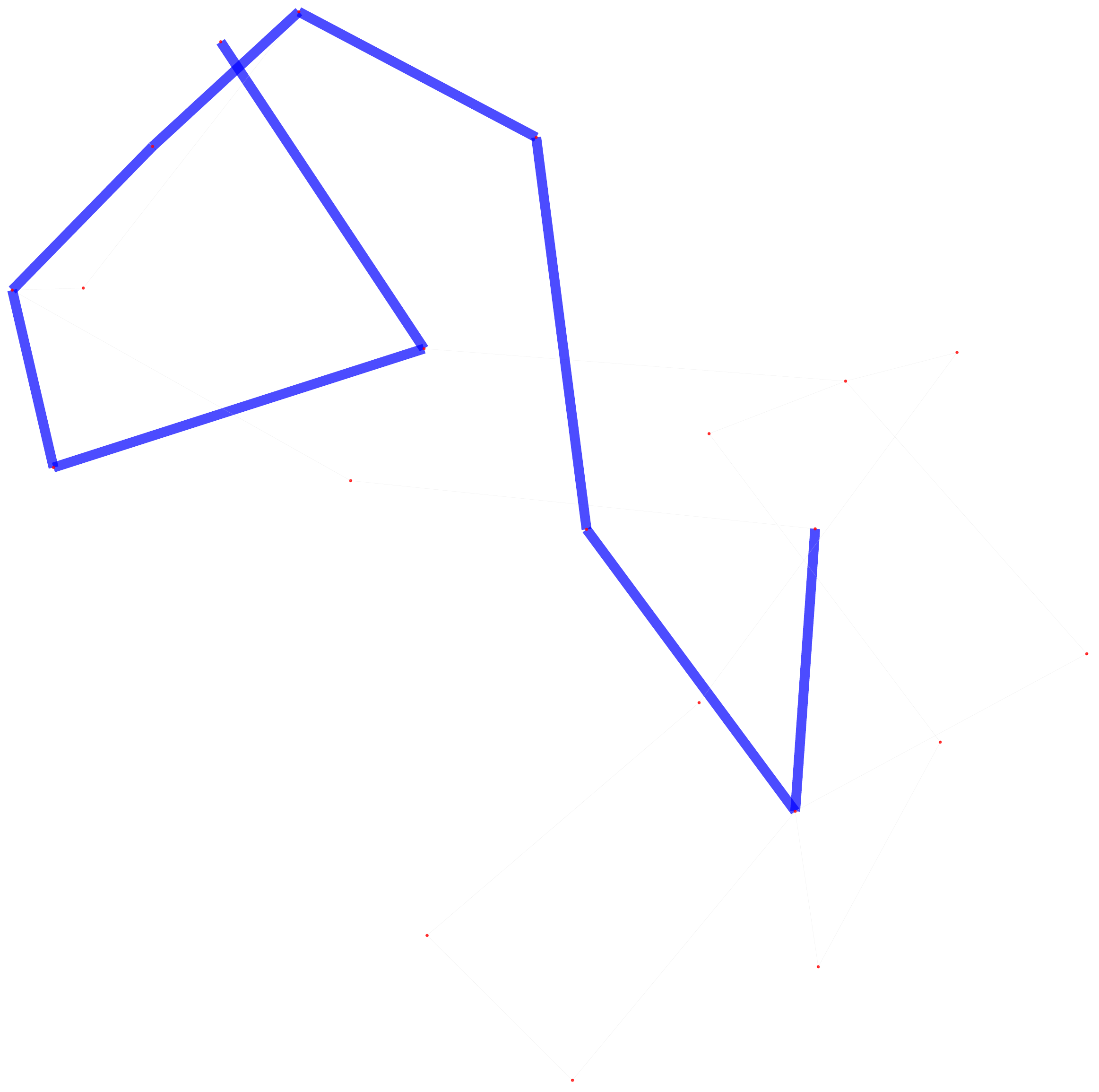}}\hspace{1mm}
\subfloat[  { [60, 4]} ]{\label{fig: path_1}\includegraphics[width=0.10\textwidth]{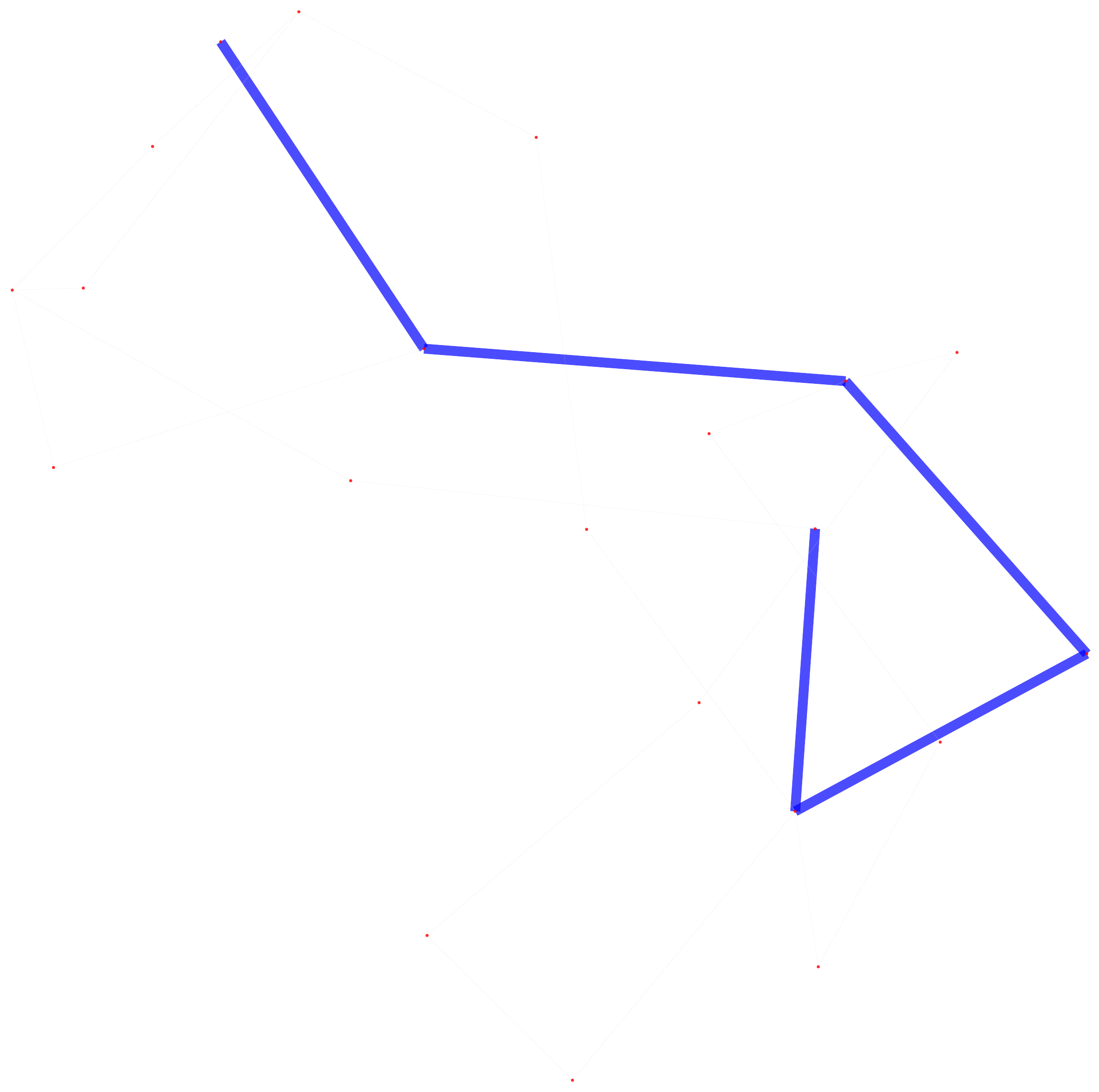}}\hspace{1mm}
\subfloat[  Target ]{\label{fig: path_1}\includegraphics[width=0.10\textwidth]{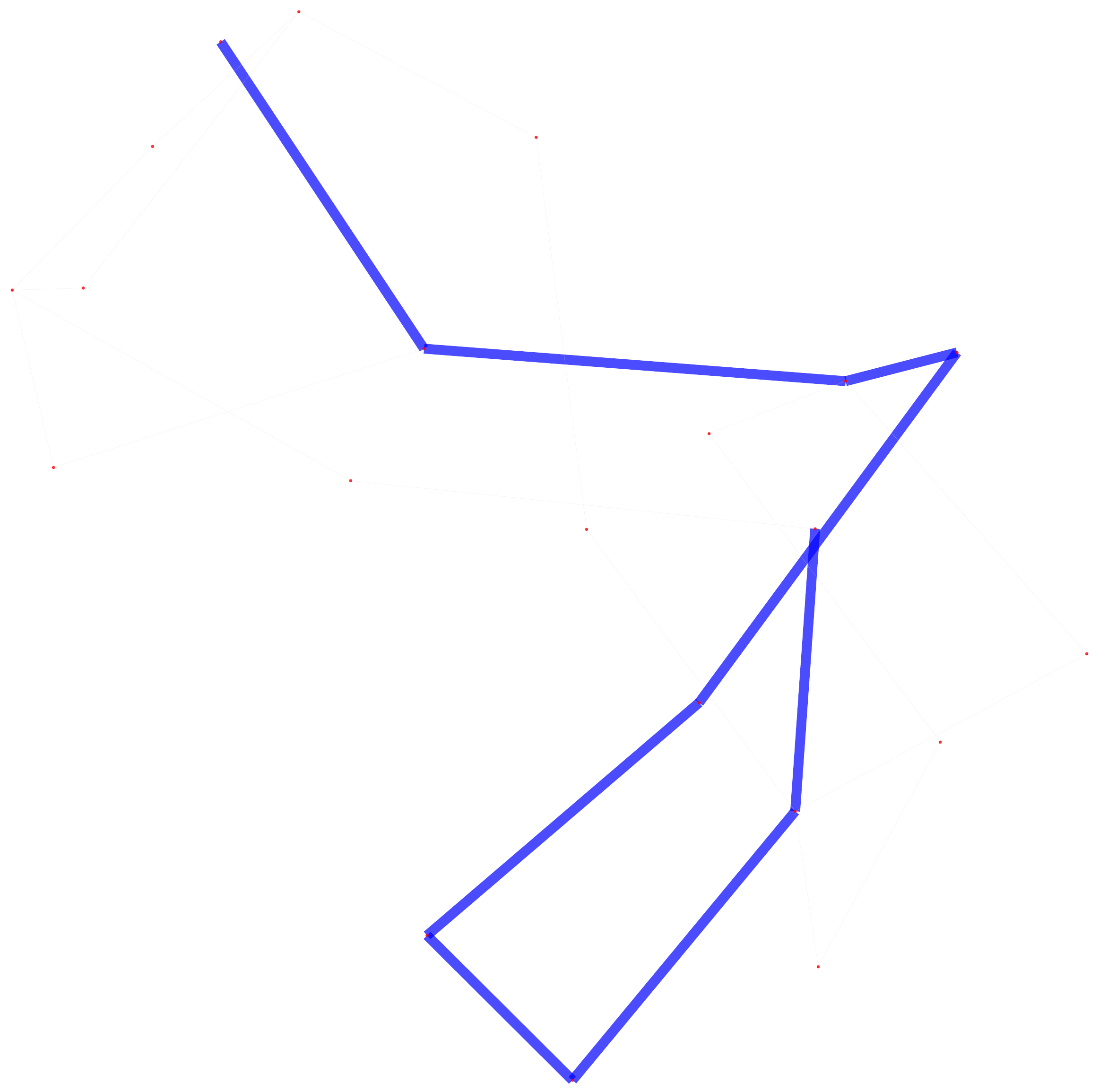}}\hspace{0mm}

\subfloat[  { [240, 0]} ]{\label{fig: path_1}\includegraphics[width=0.10\textwidth]{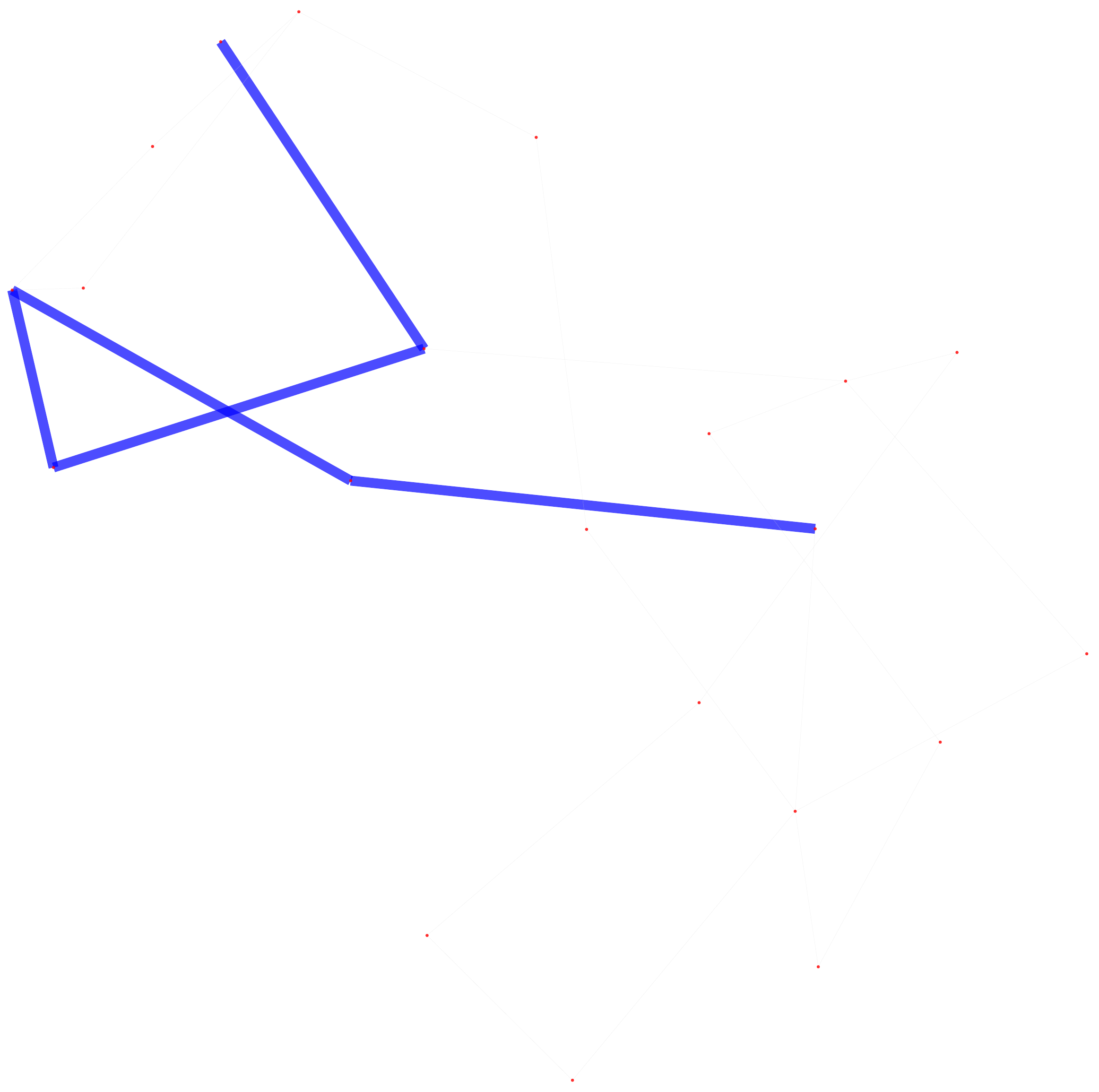}}\hspace{1mm}
\subfloat[  { [240, 1]} ]{\label{fig: path_1}\includegraphics[width=0.10\textwidth]{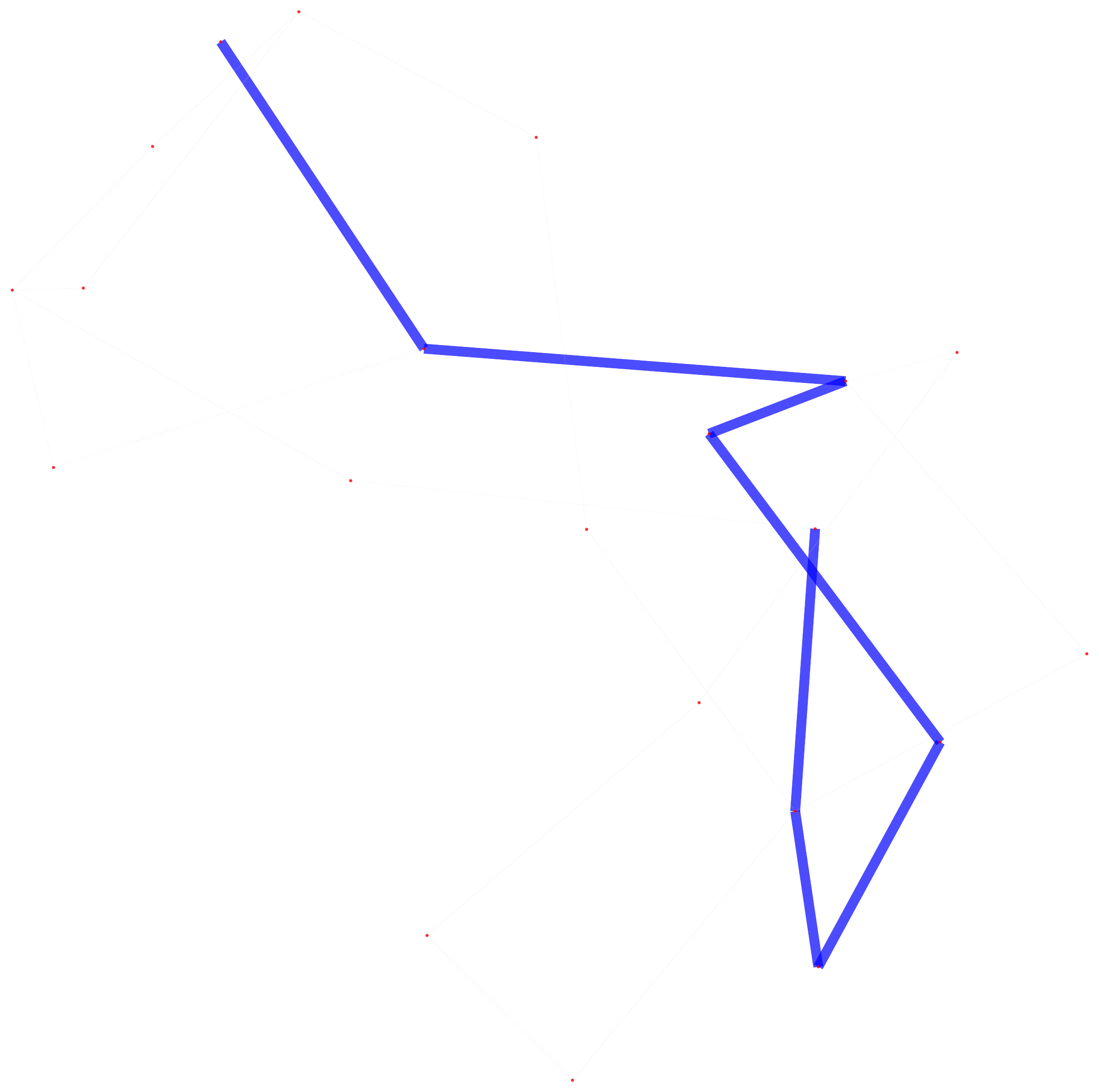}}\hspace{1mm}
\subfloat[  { [240, 4]} ]{\label{fig: path_1}\includegraphics[width=0.10\textwidth]{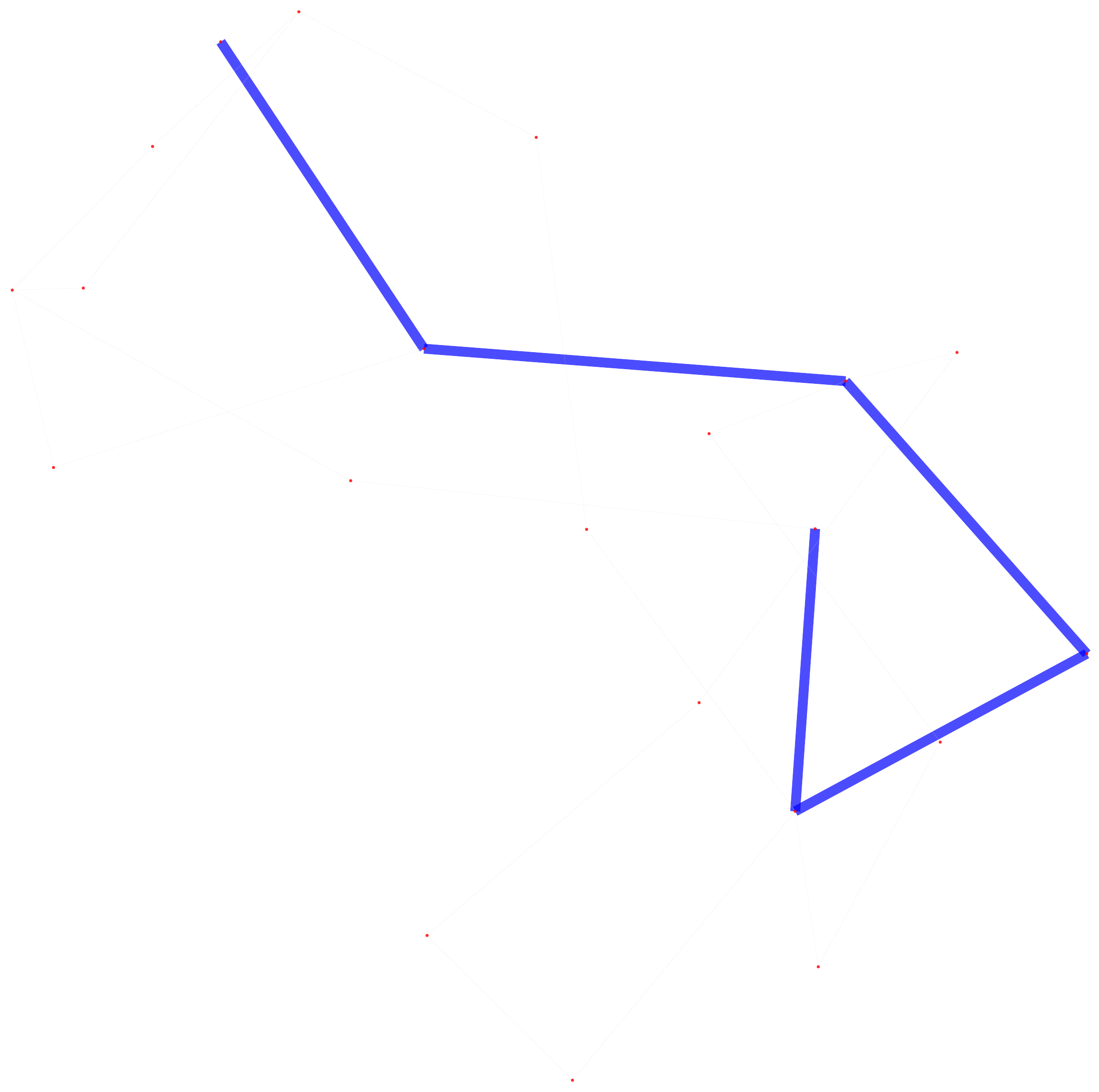}}\hspace{1mm}
\subfloat[  Target ]{\label{fig: path_1}\includegraphics[width=0.10\textwidth]{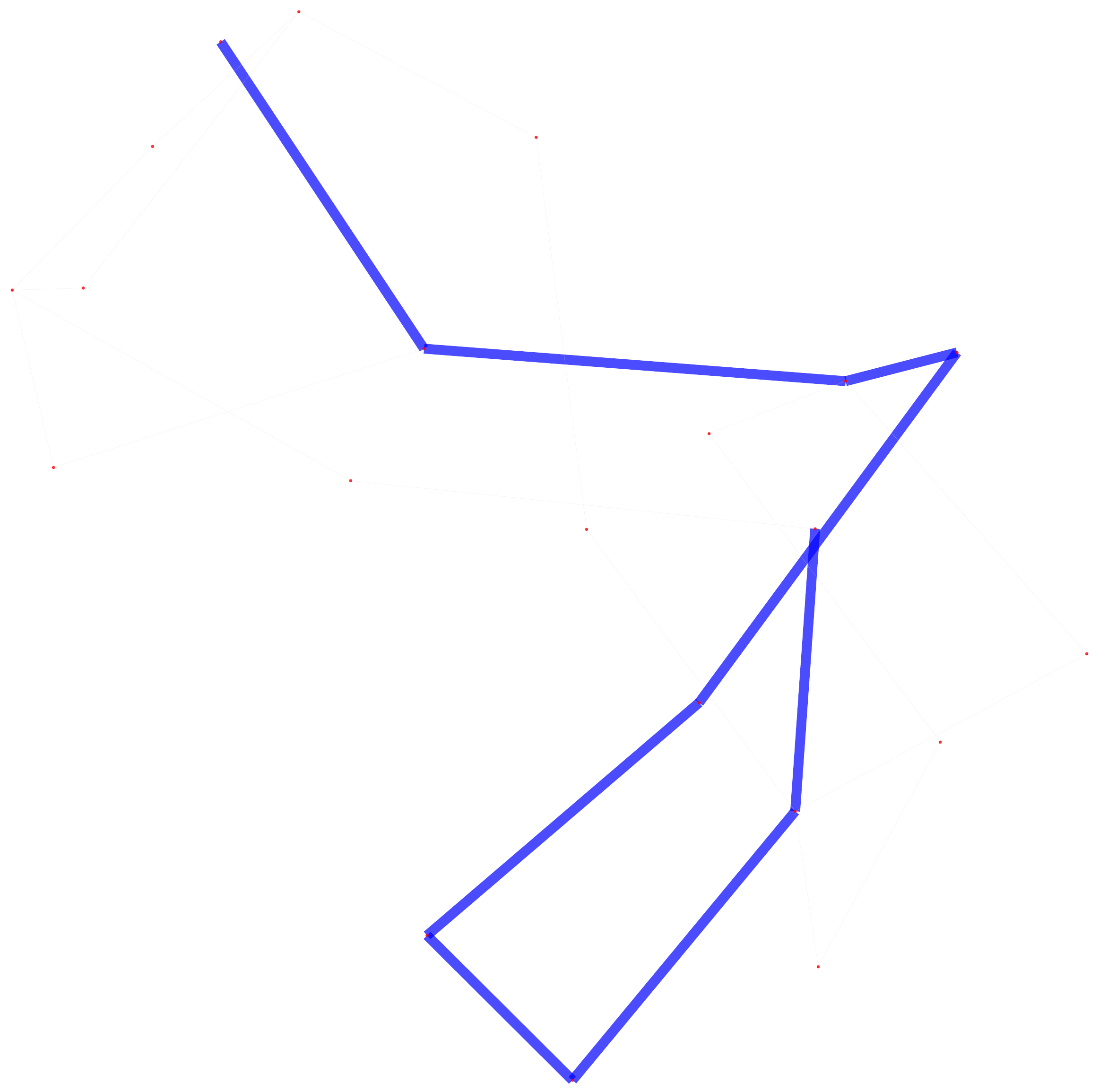}}\hspace{0mm}

\caption{\small Visualization of the results of QRTS-P for an example testing query for Tree-to-Path on Kro.}
\label{fig: more_path_3}
\end{figure}

\begin{figure}[h]
\centering
\captionsetup[subfloat]{labelfont=scriptsize,textfont=scriptsize,labelformat=empty}

\subfloat[  { [30, 0]} ]{\label{fig: path_1}\includegraphics[width=0.10\textwidth]{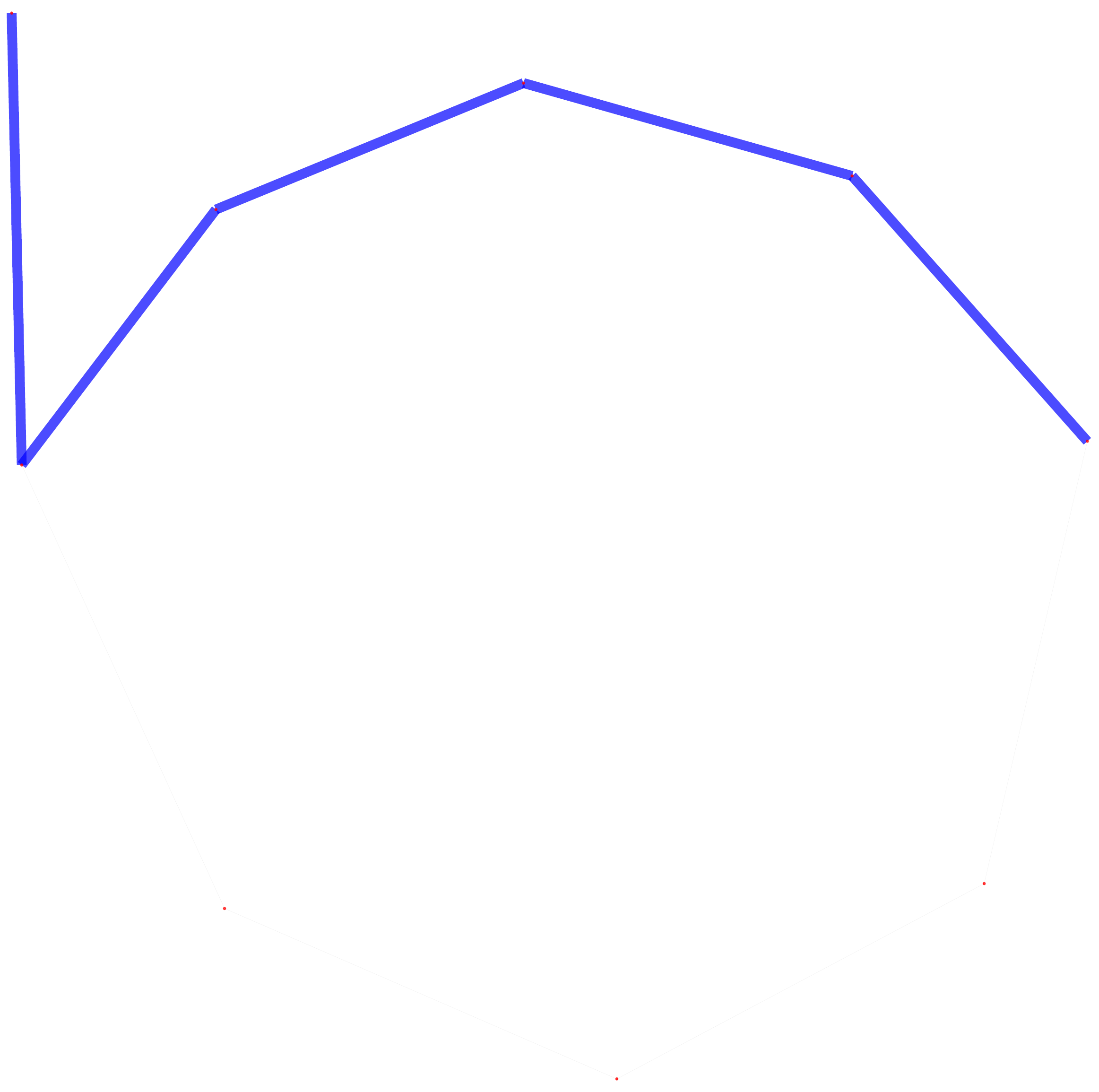}}\hspace{1mm}
\subfloat[  { [30, 1]} ]{\label{fig: path_1}\includegraphics[width=0.10\textwidth]{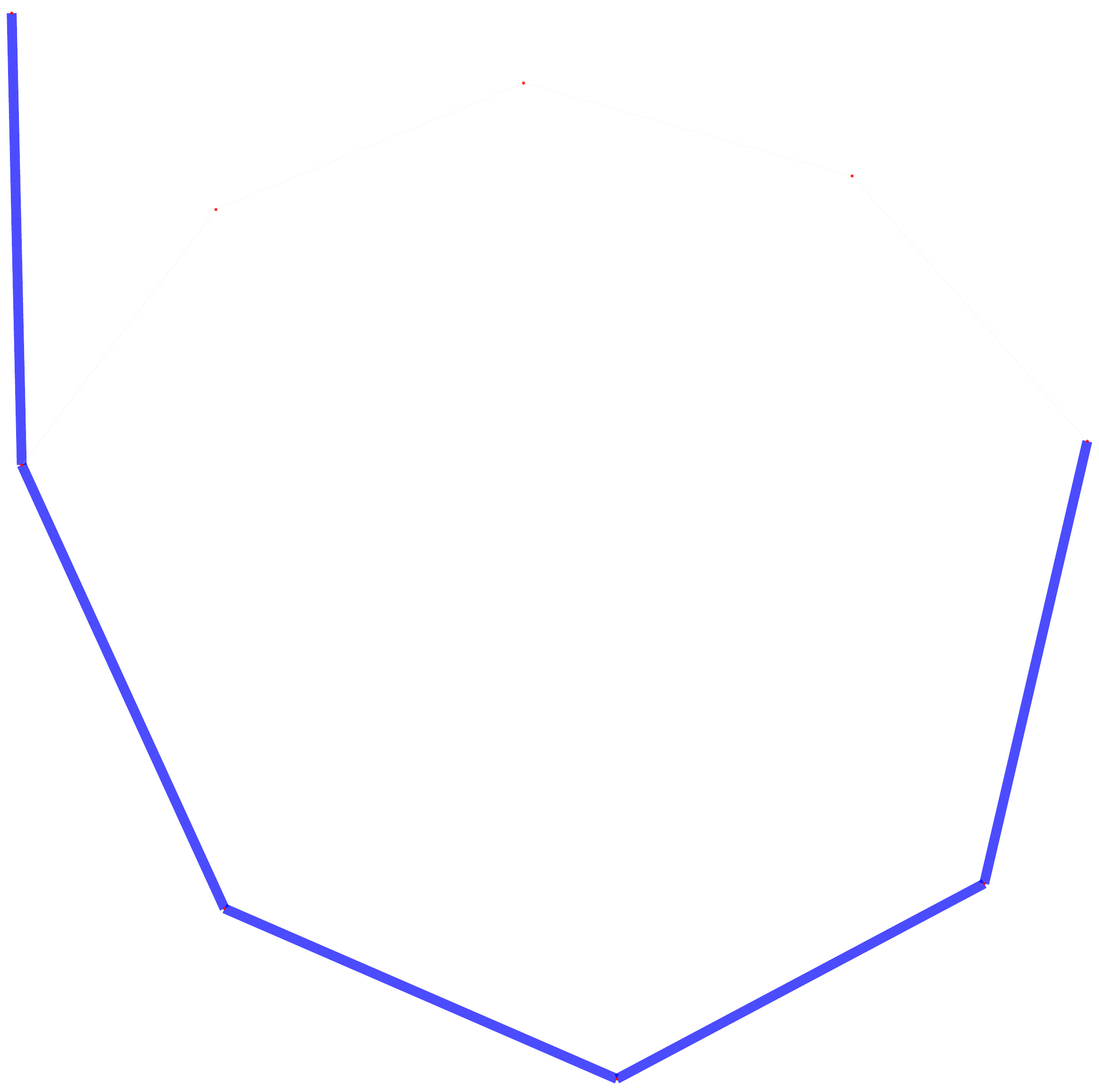}}\hspace{1mm}
\subfloat[  { [30, 4]} ]{\label{fig: path_1}\includegraphics[width=0.10\textwidth]{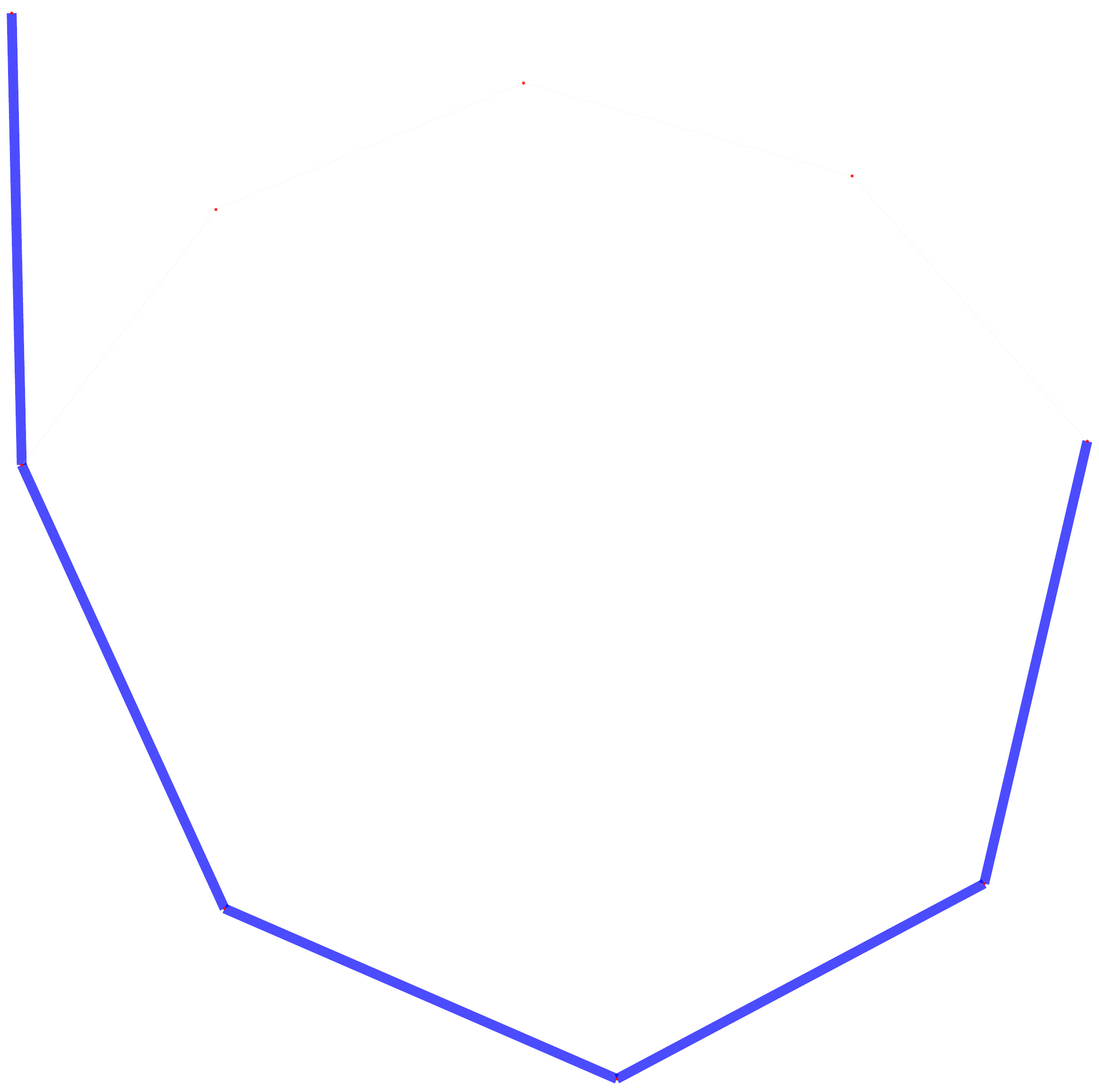}}\hspace{1mm}
\subfloat[  Target ]{\label{fig: path_1}\includegraphics[width=0.10\textwidth]{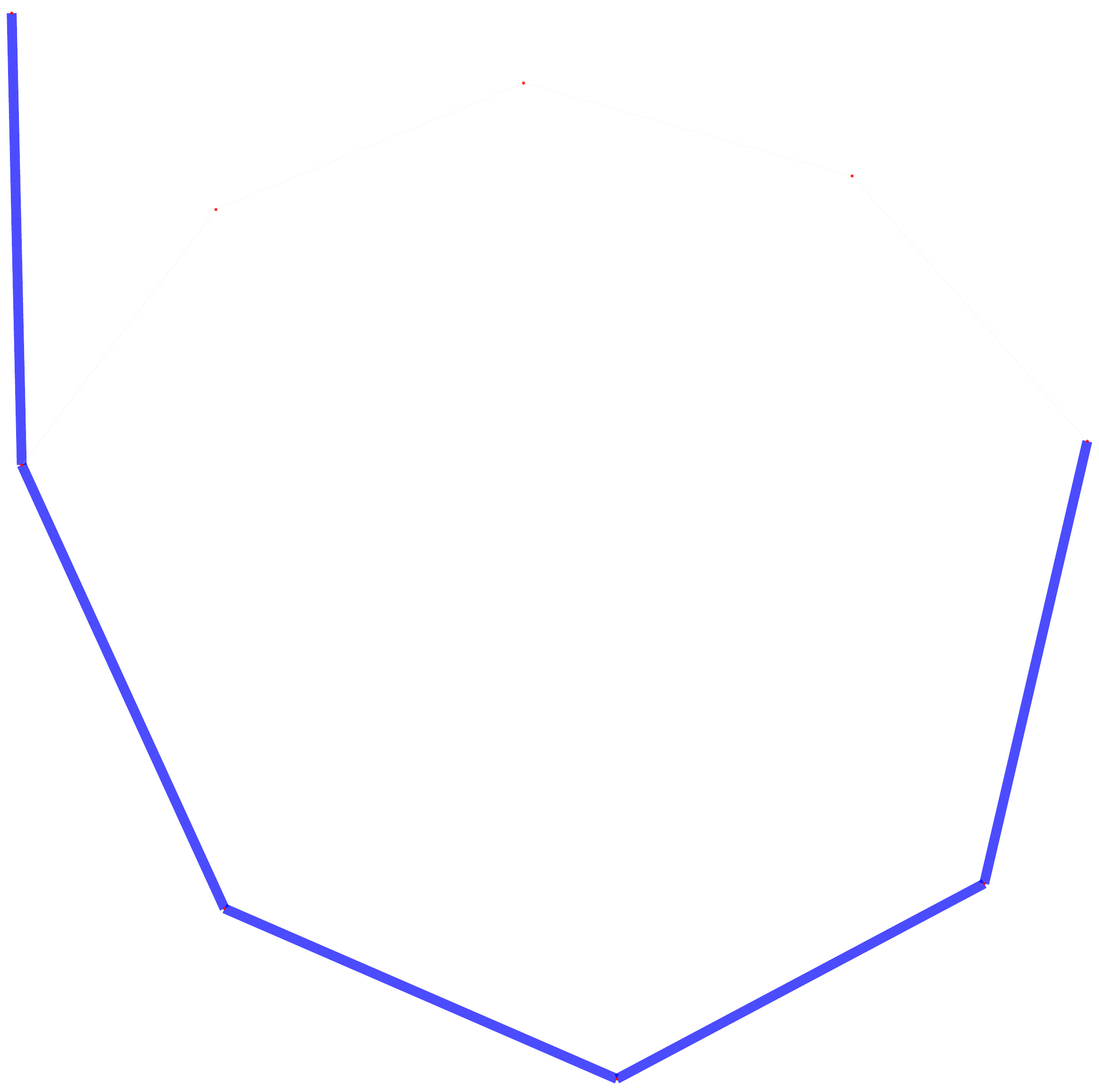}}\hspace{0mm}

\subfloat[  { [60, 0]} ]{\label{fig: path_1}\includegraphics[width=0.10\textwidth]{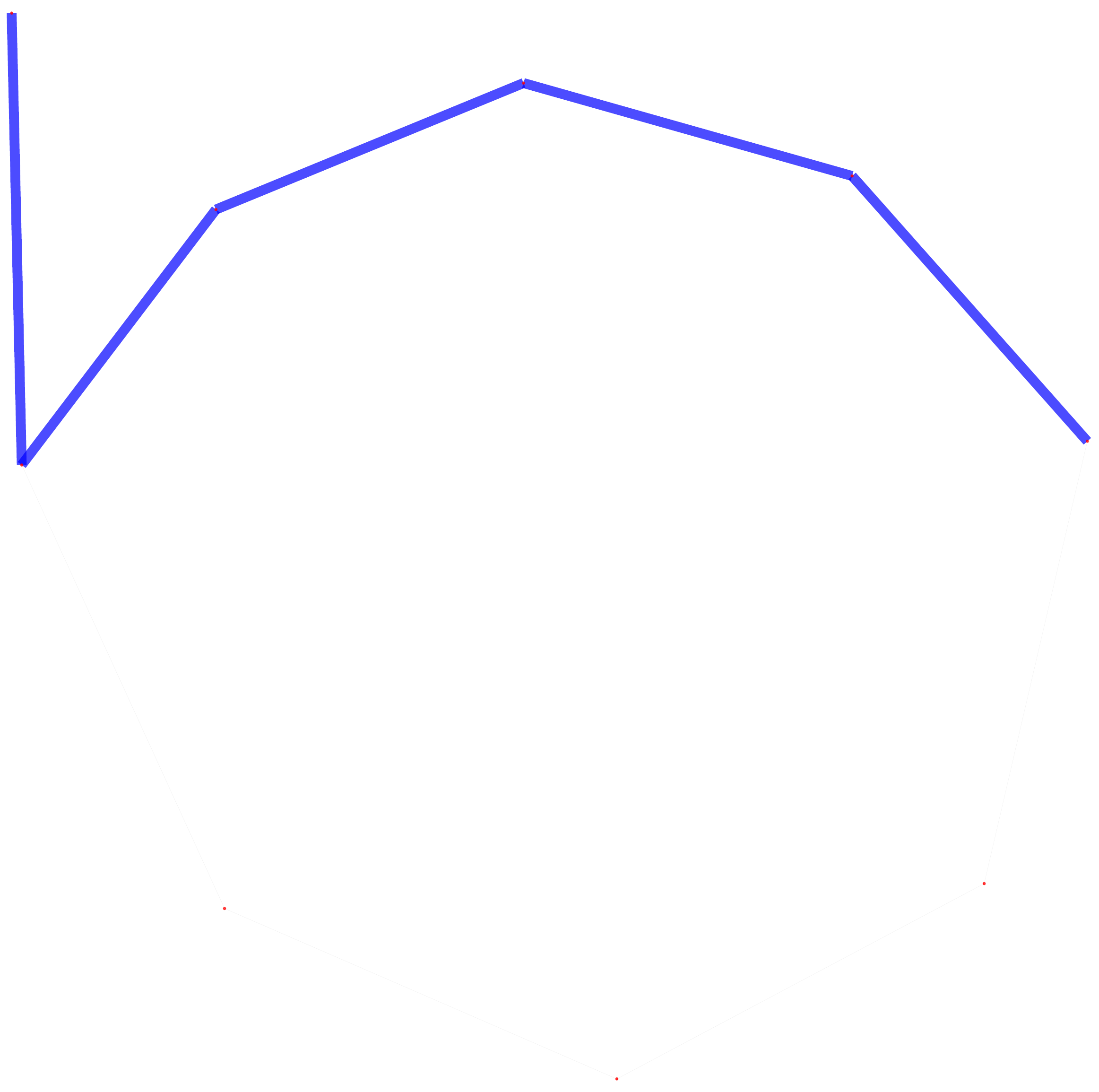}}\hspace{1mm}
\subfloat[  { [60, 1]} ]{\label{fig: path_1}\includegraphics[width=0.10\textwidth]{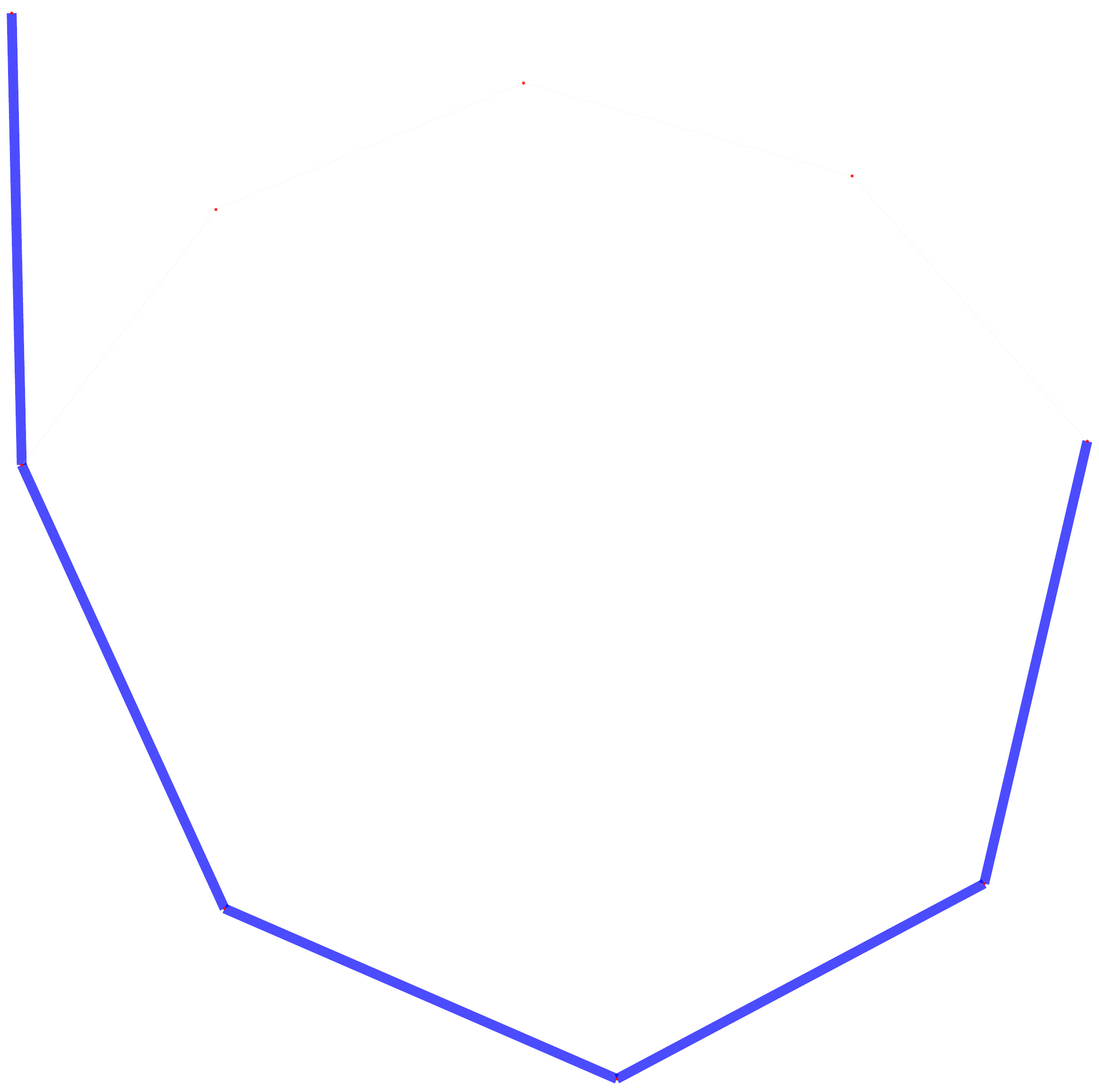}}\hspace{1mm}
\subfloat[  { [60, 4]} ]{\label{fig: path_1}\includegraphics[width=0.10\textwidth]{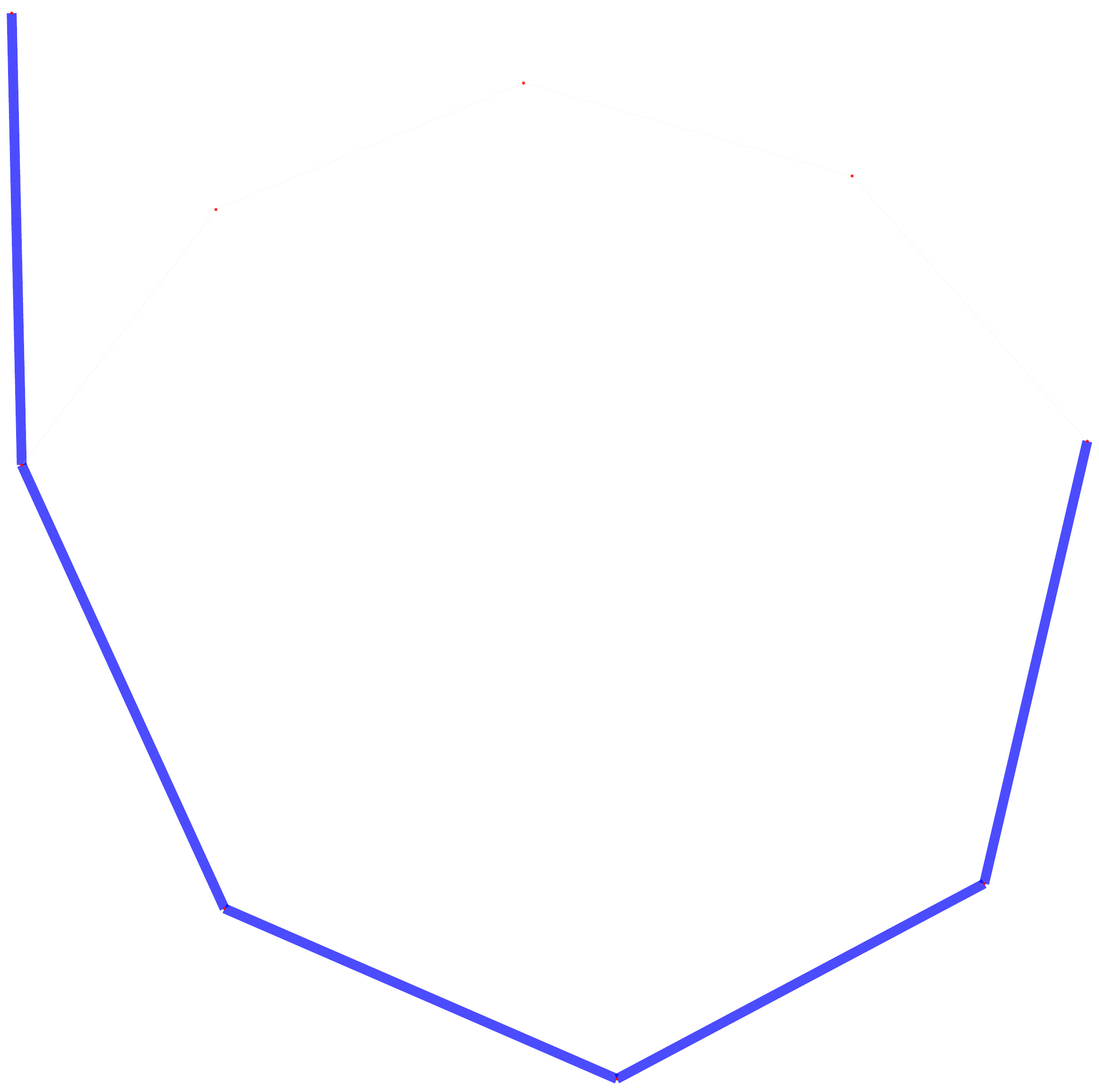}}\hspace{1mm}
\subfloat[  Target ]{\label{fig: path_1}\includegraphics[width=0.10\textwidth]{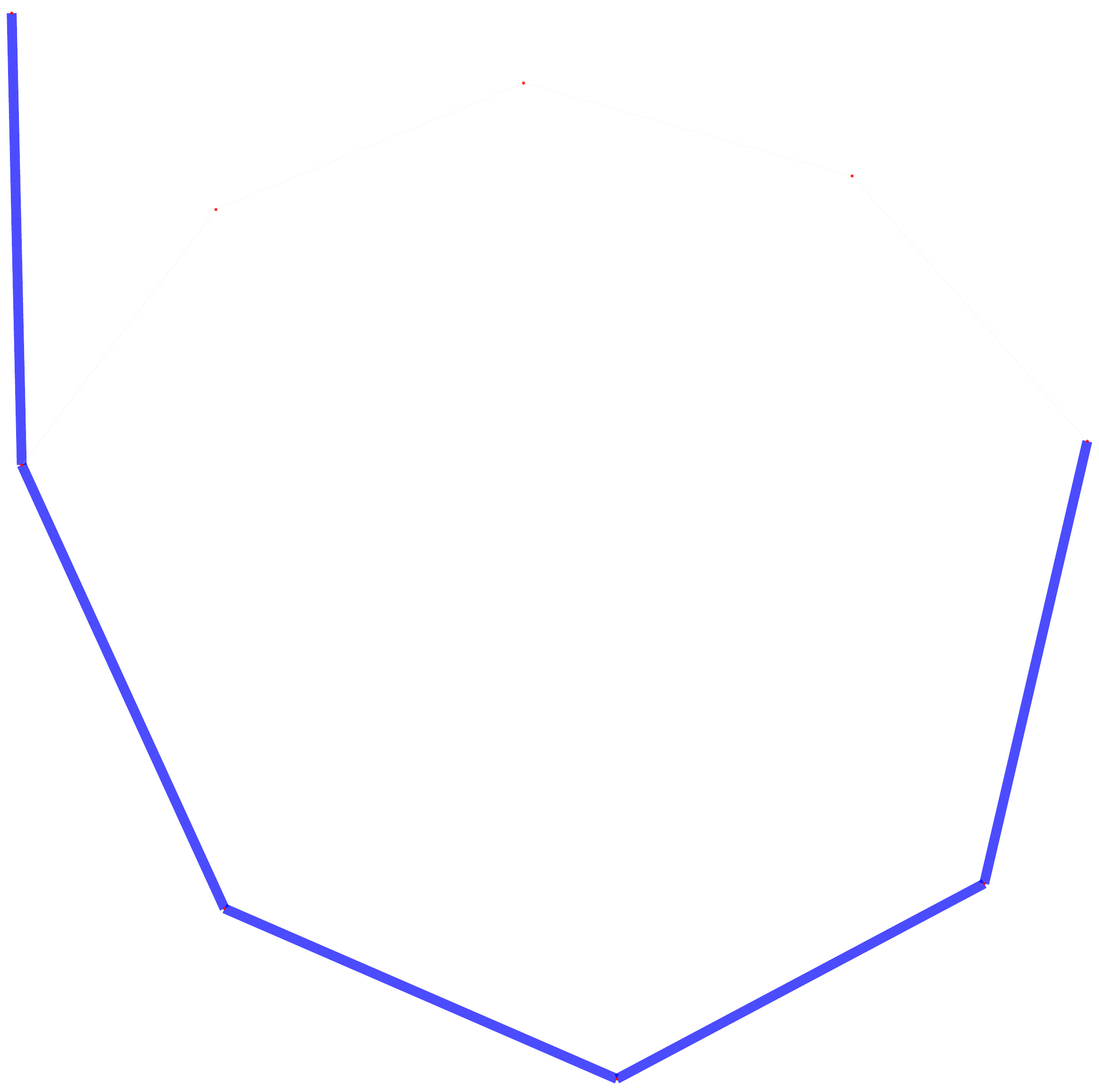}}\hspace{0mm}

\subfloat[  { [240, 0]} ]{\label{fig: path_1}\includegraphics[width=0.10\textwidth]{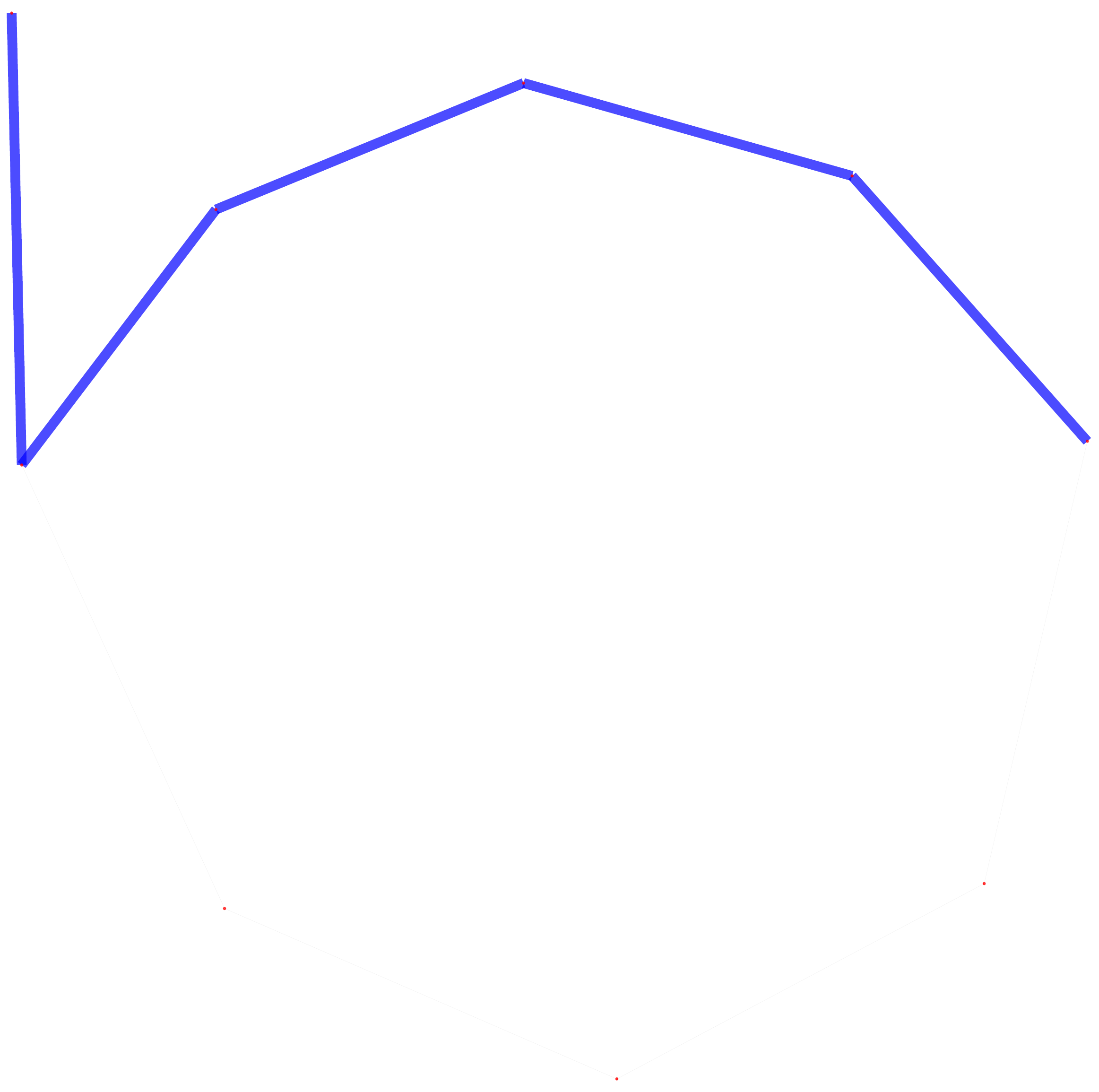}}\hspace{1mm}
\subfloat[  { [240, 1]} ]{\label{fig: path_1}\includegraphics[width=0.10\textwidth]{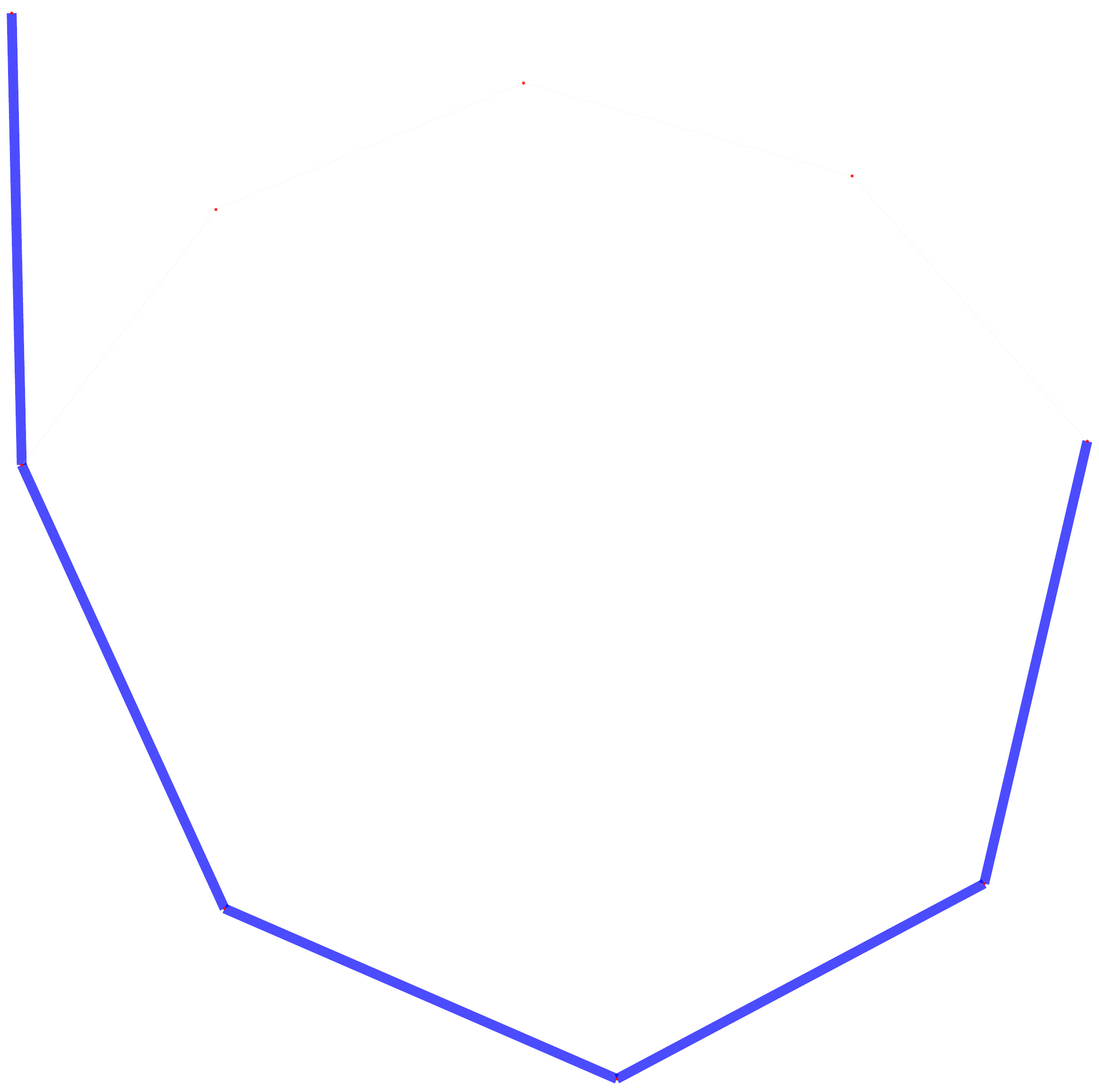}}\hspace{1mm}
\subfloat[  { [240, 4]} ]{\label{fig: path_1}\includegraphics[width=0.10\textwidth]{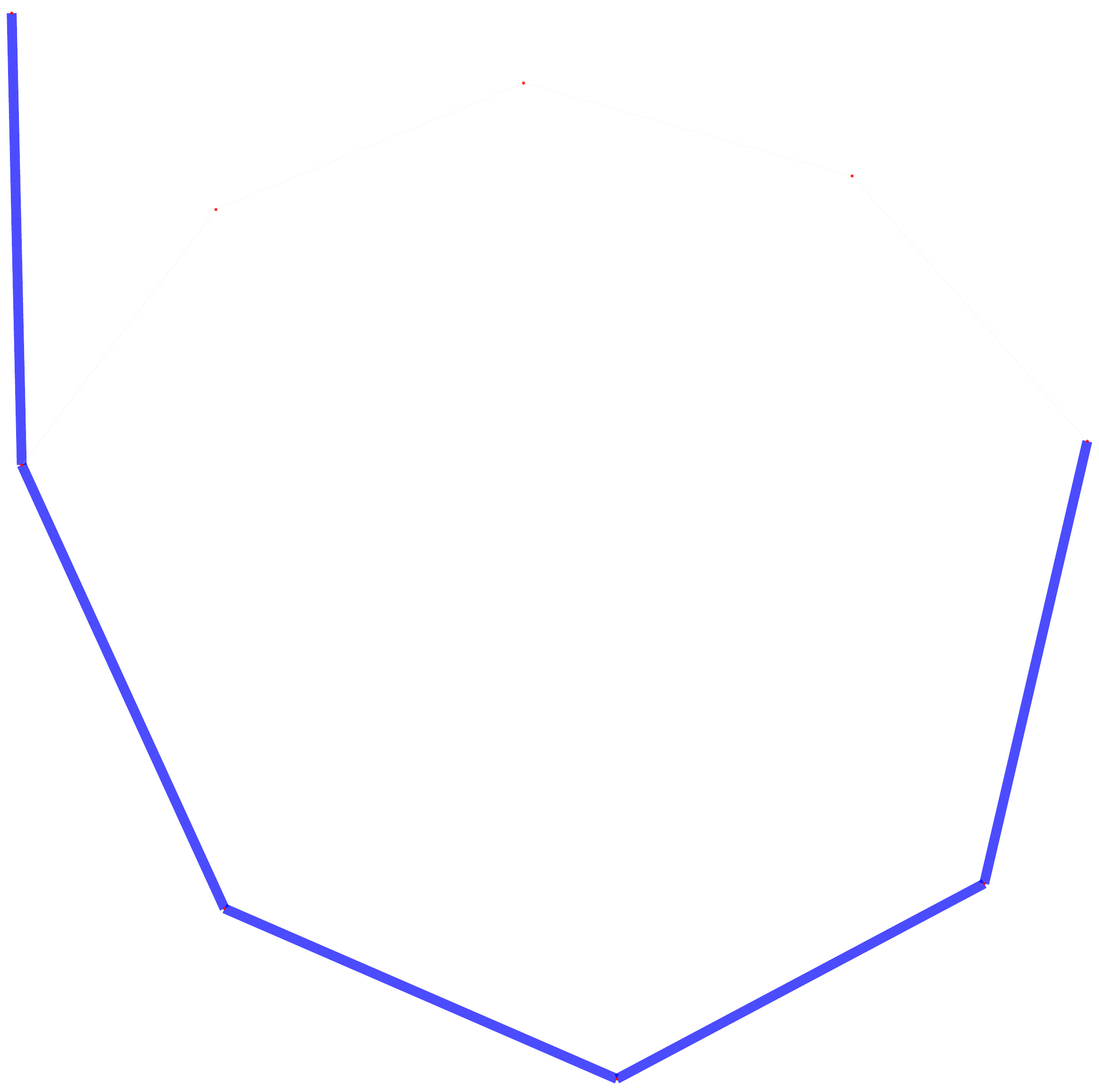}}\hspace{1mm}
\subfloat[  Target ]{\label{fig: path_1}\includegraphics[width=0.10\textwidth]{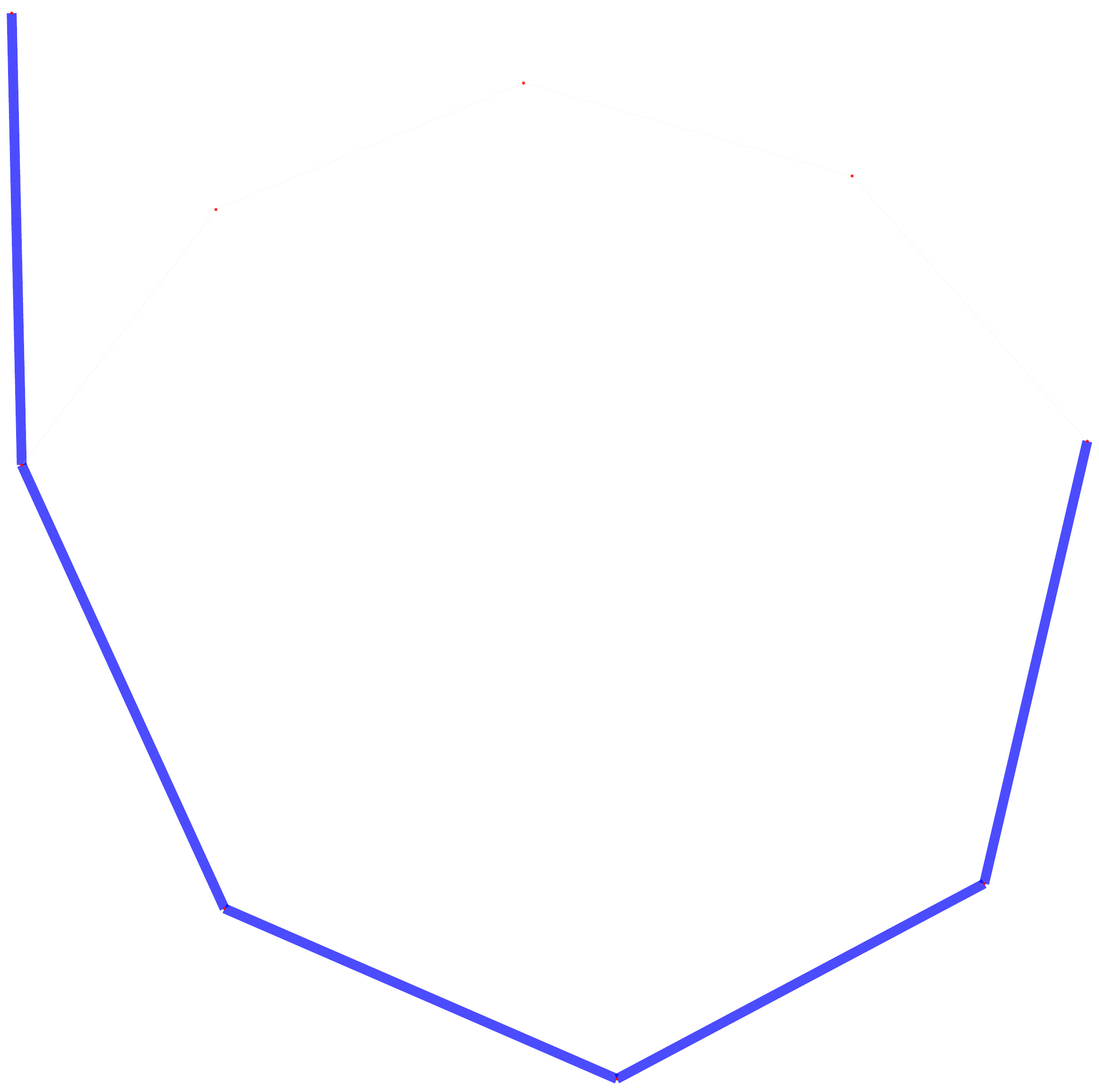}}\hspace{0mm}

\caption{\small Visualization of the results of QRTS-P for an example testing query for Tree-to-Path on Kro.}
\label{fig: more_path_4}
\end{figure}

\begin{figure}[h]
\centering
\captionsetup[subfloat]{labelfont=scriptsize,textfont=scriptsize,labelformat=empty}
\subfloat[  { [30, 0]} ]{\label{fig: path_1}\includegraphics[width=0.10\textwidth]{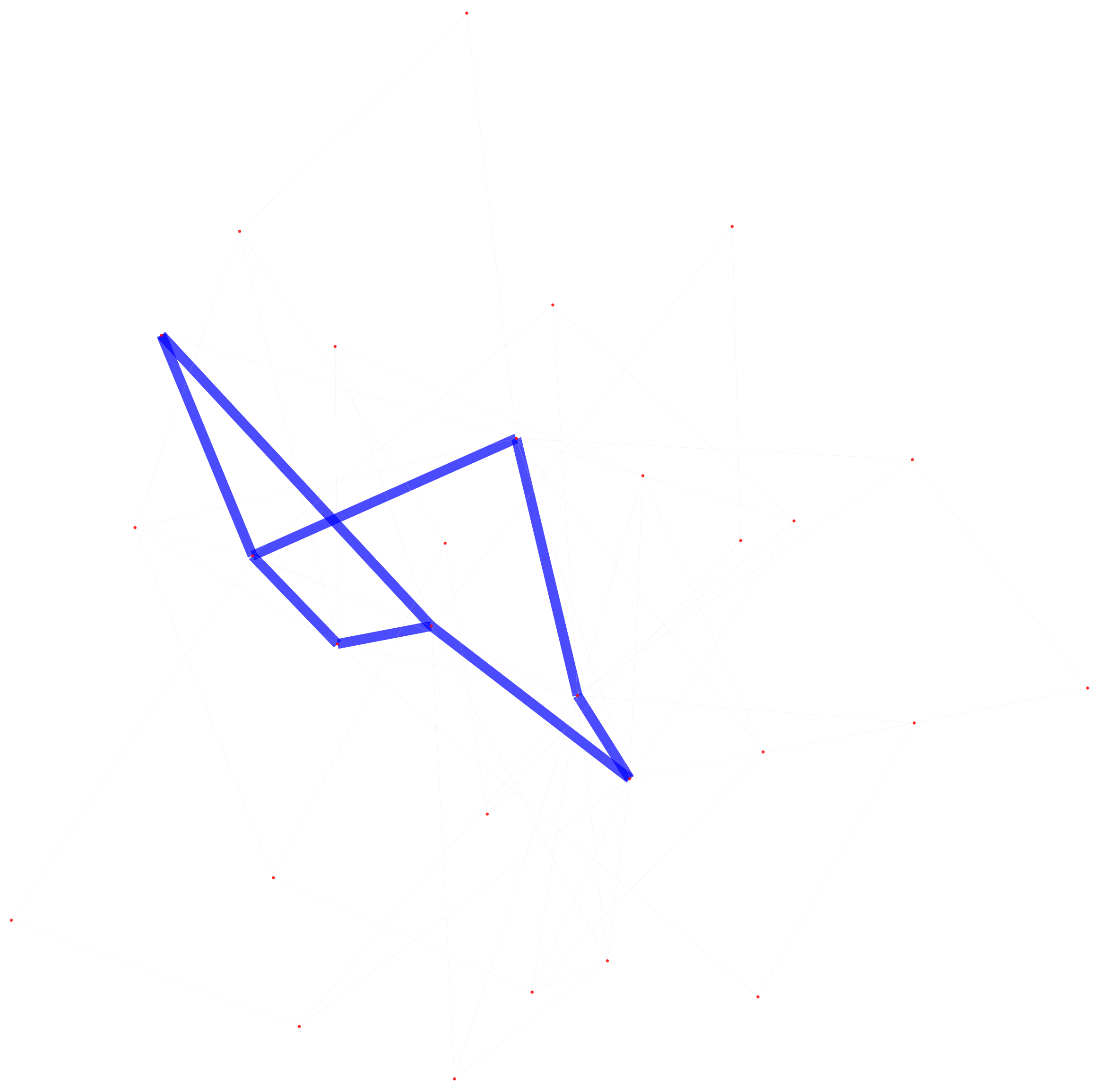}}\hspace{1mm}
\subfloat[  { [30, 1]} ]{\label{fig: path_1}\includegraphics[width=0.10\textwidth]{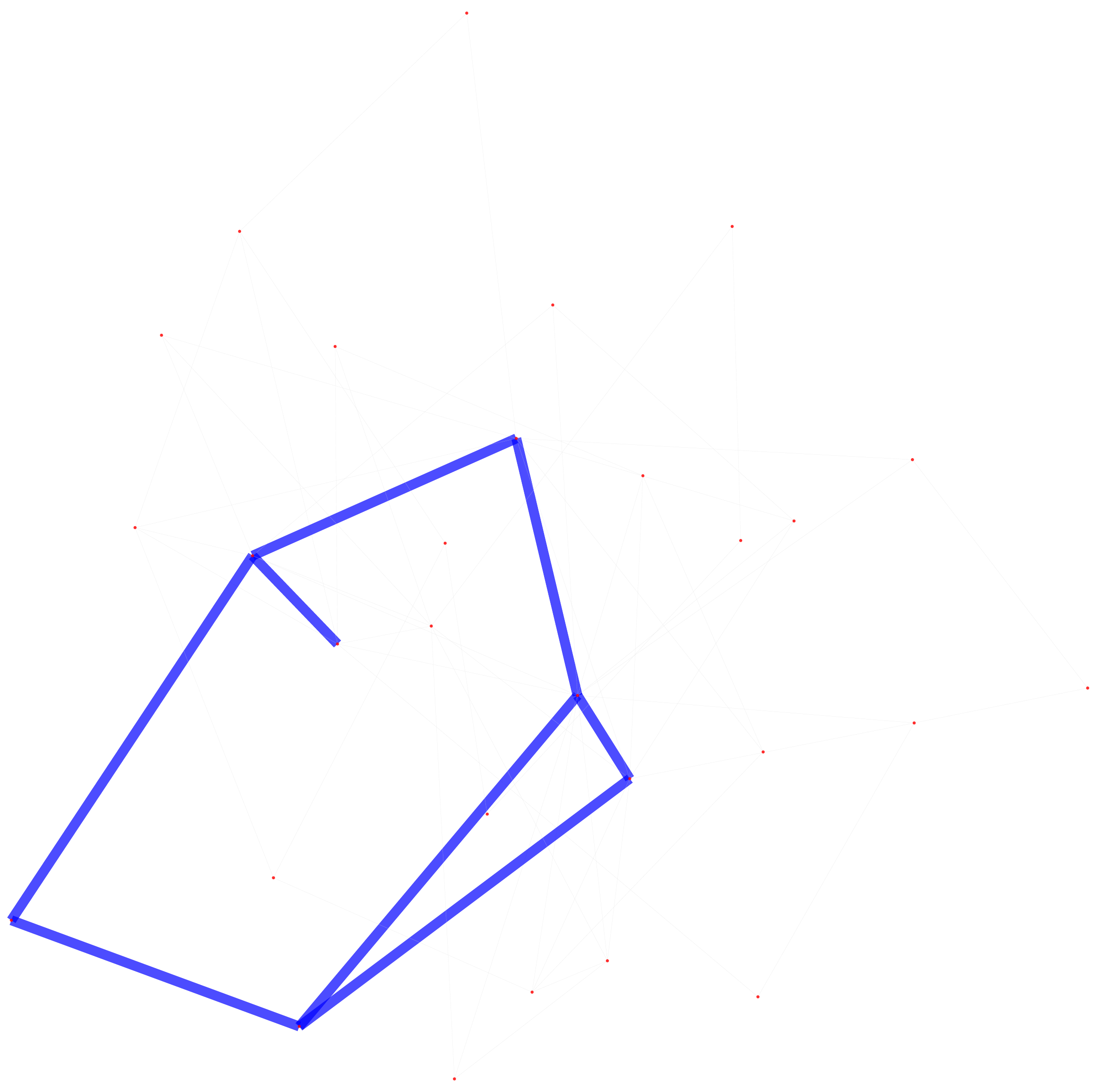}}\hspace{1mm}
\subfloat[  { [30, 2]} ]{\label{fig: path_1}\includegraphics[width=0.10\textwidth]{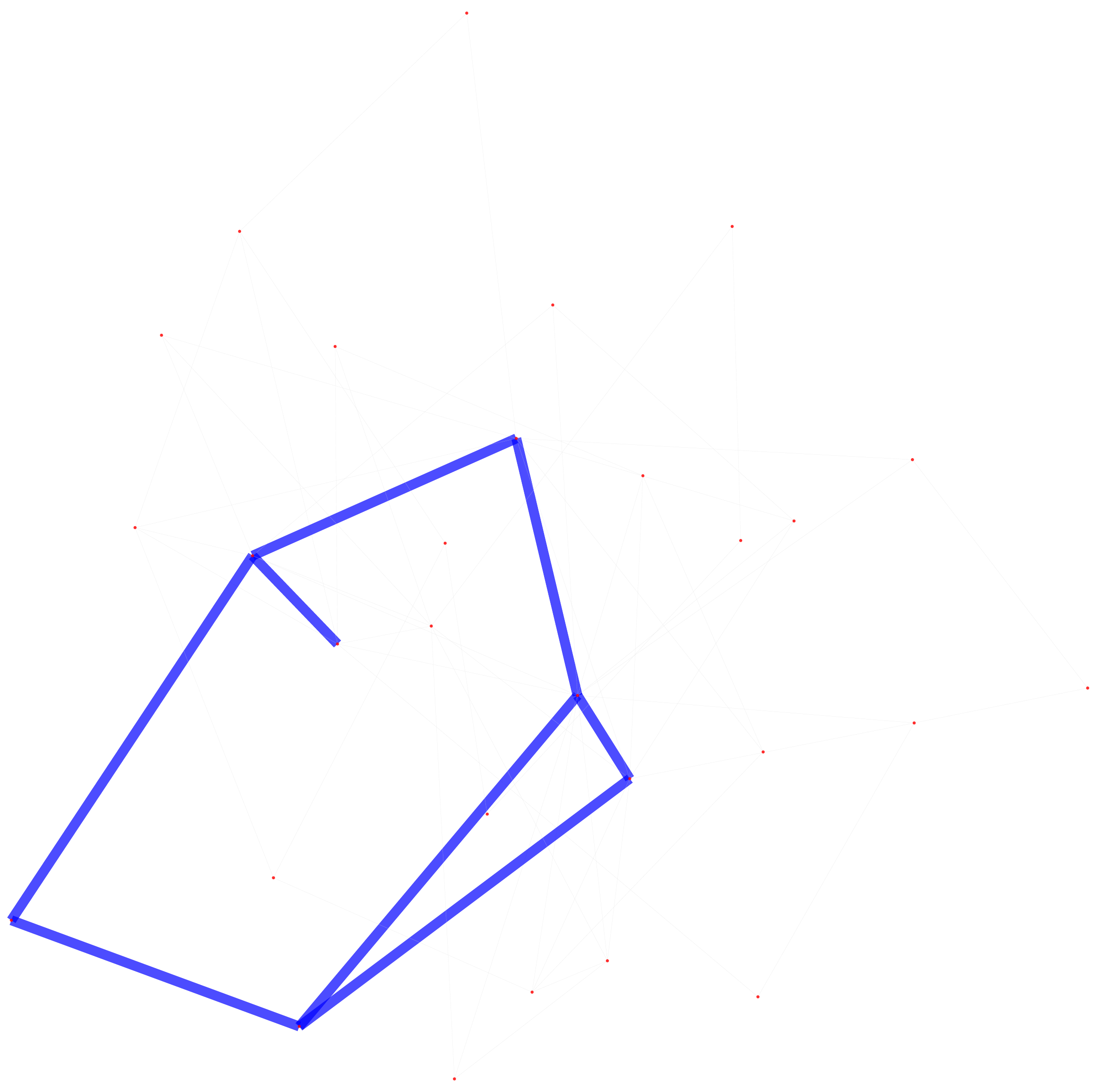}}\hspace{1mm}
\subfloat[  Target ]{\label{fig: path_1}\includegraphics[width=0.10\textwidth]{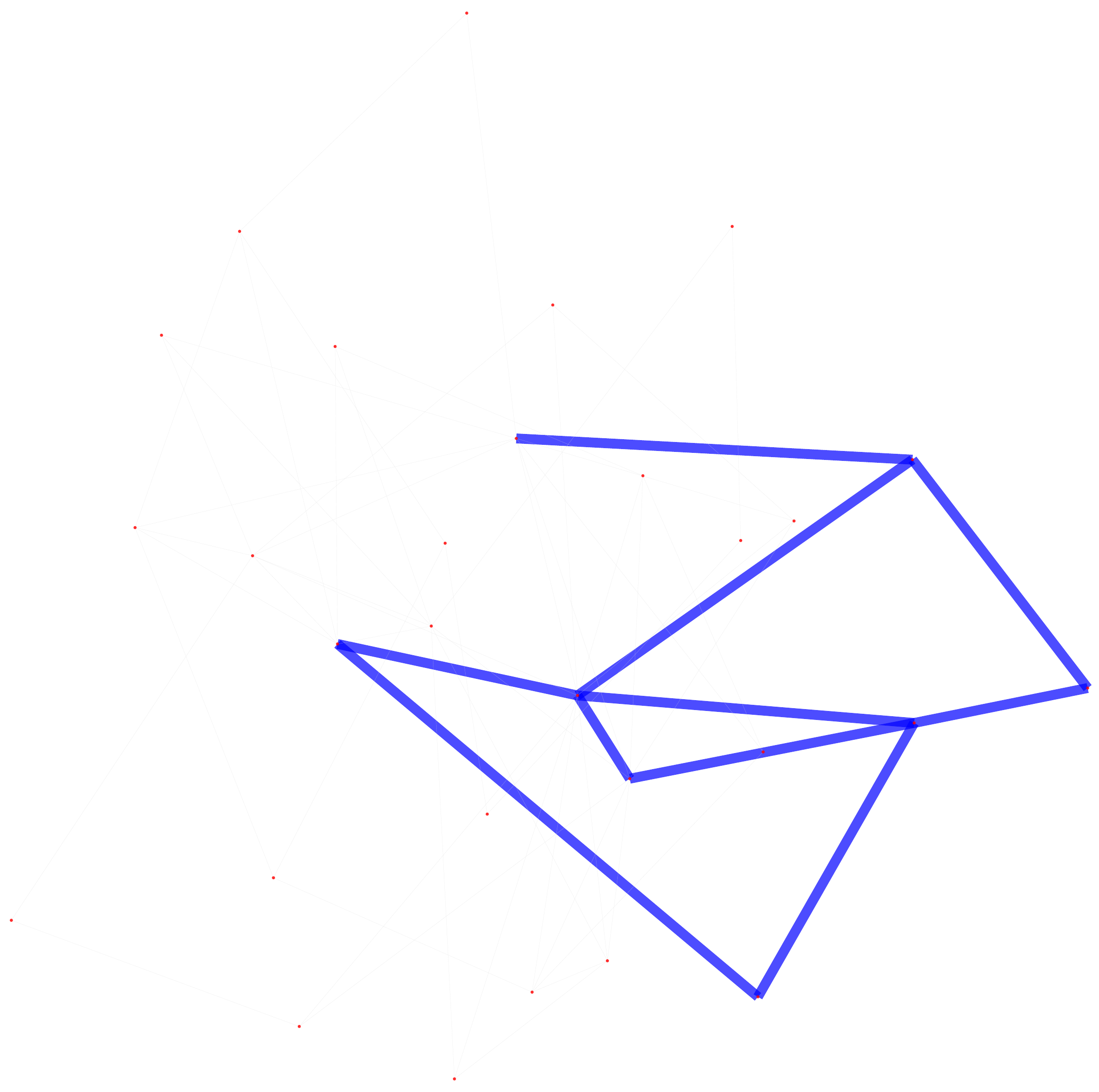}}\hspace{0mm}


\subfloat[  { [240, 0]} ]{\label{fig: path_1}\includegraphics[width=0.10\textwidth]{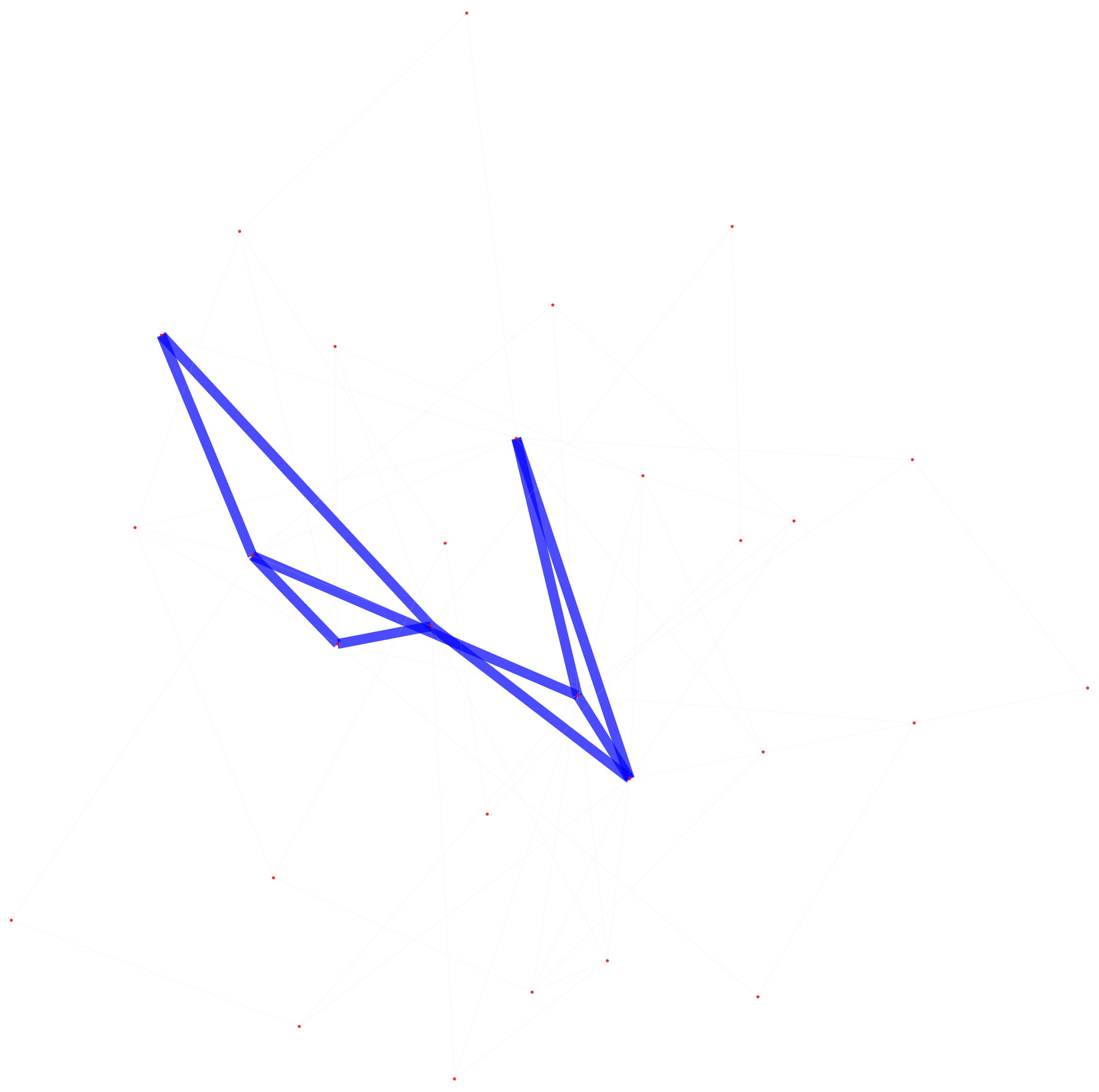}}\hspace{1mm}
\subfloat[  { [240, 1]} ]{\label{fig: path_1}\includegraphics[width=0.10\textwidth]{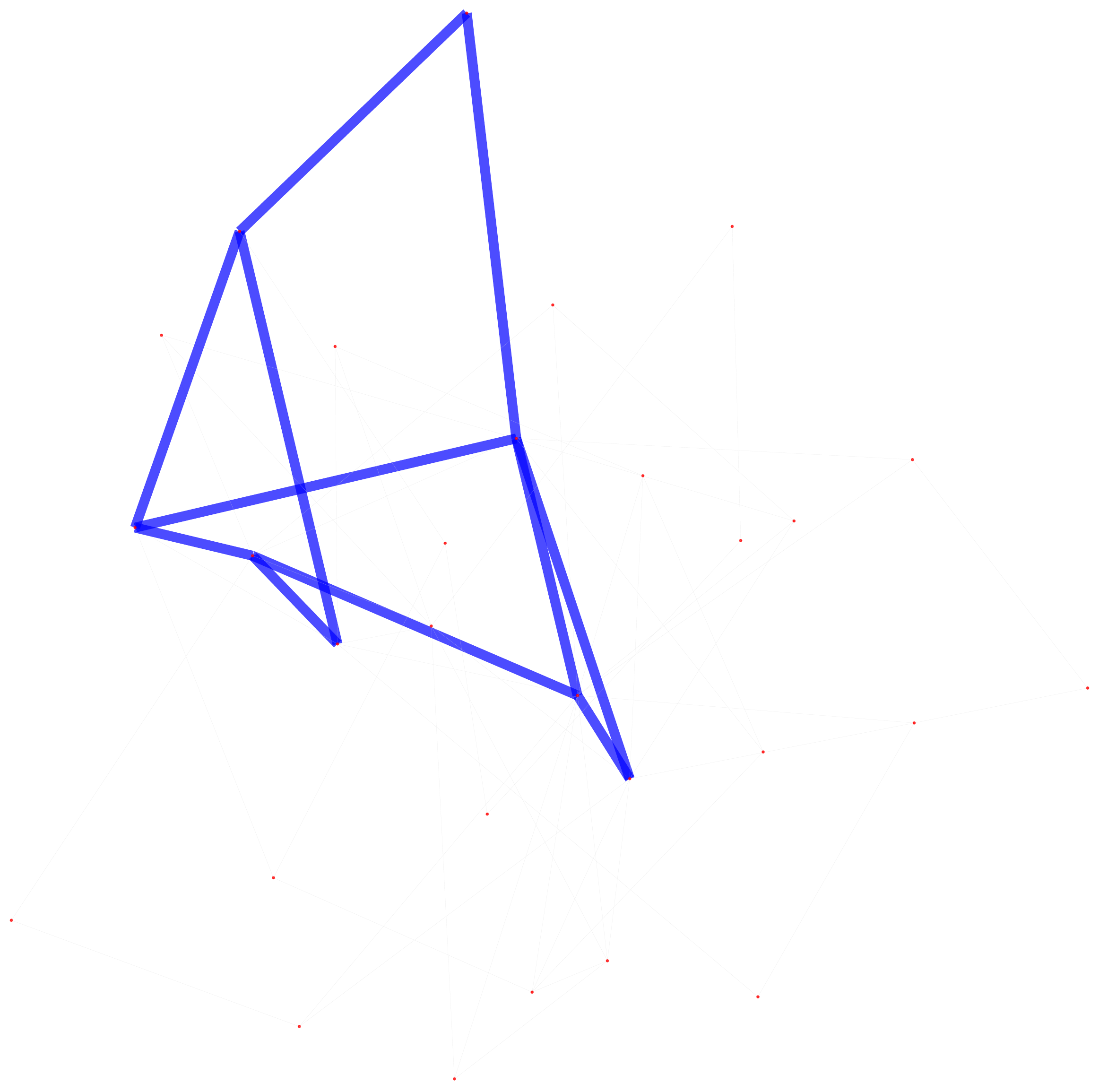}}\hspace{1mm}
\subfloat[  { [240, 2]} ]{\label{fig: path_1}\includegraphics[width=0.10\textwidth]{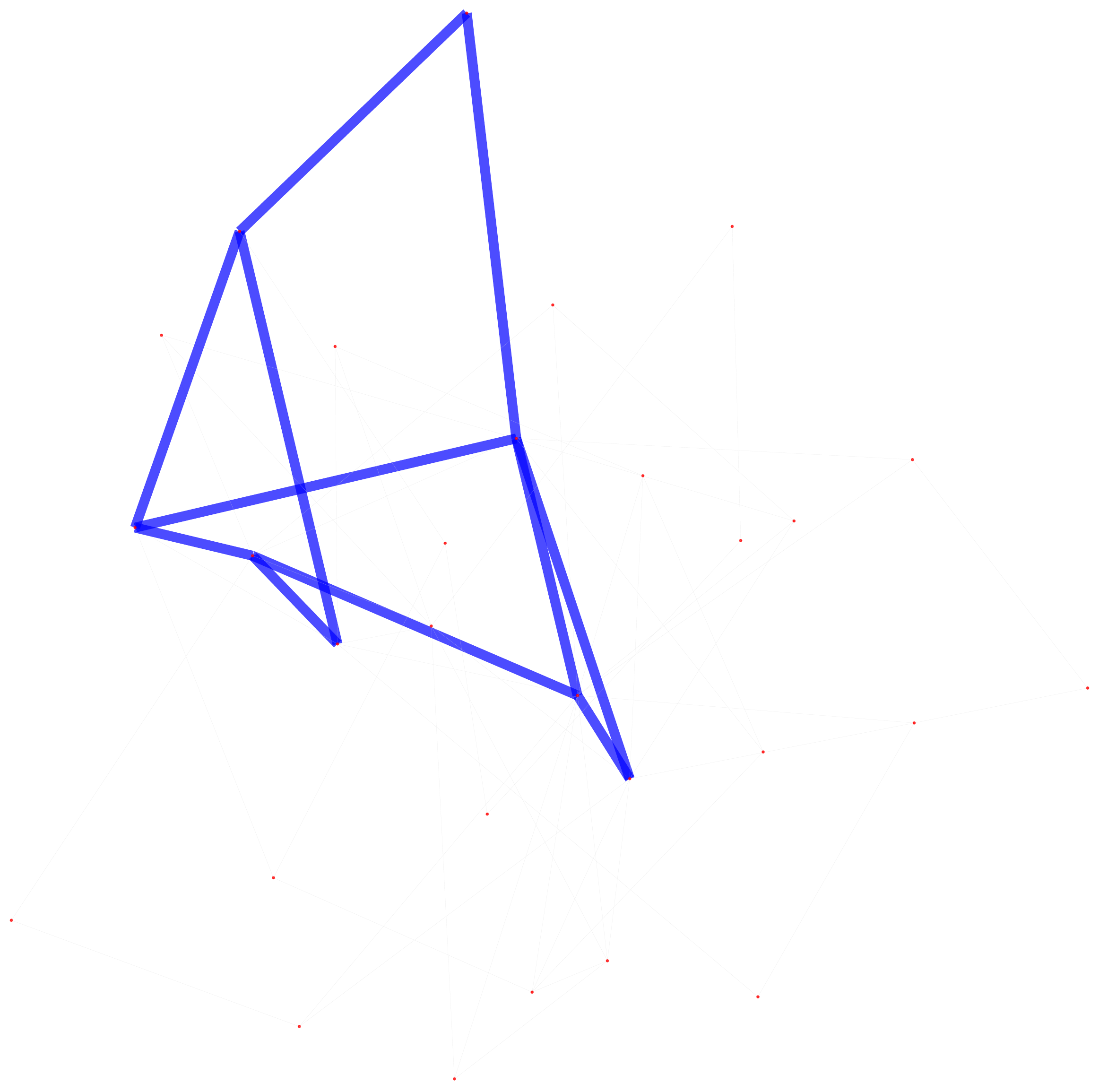}}\hspace{1mm}
\subfloat[  Target ]{\label{fig: path_1}\includegraphics[width=0.10\textwidth]{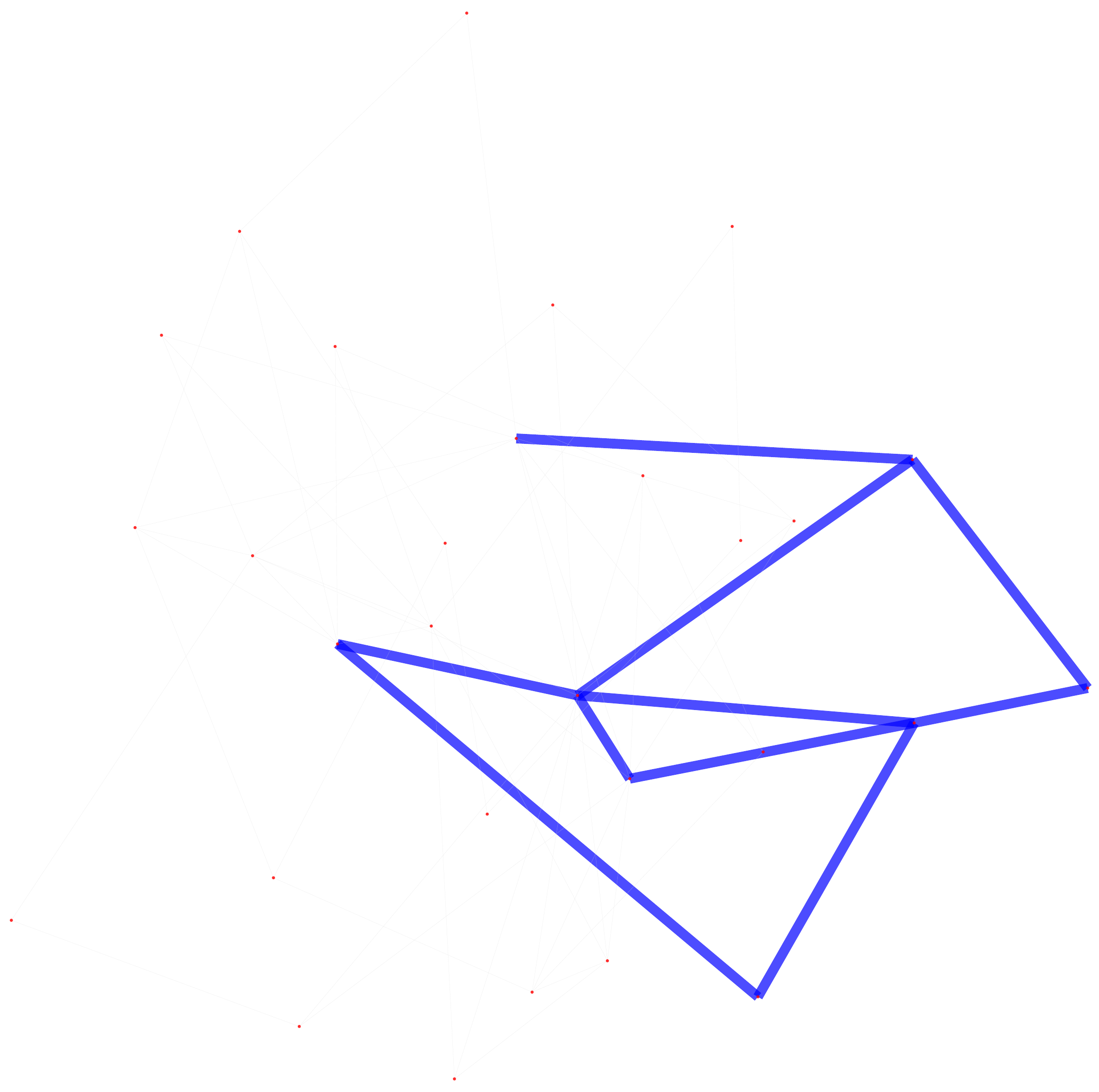}}\hspace{0mm}

\subfloat[  { [480, 0]} ]{\label{fig: path_1}\includegraphics[width=0.10\textwidth]{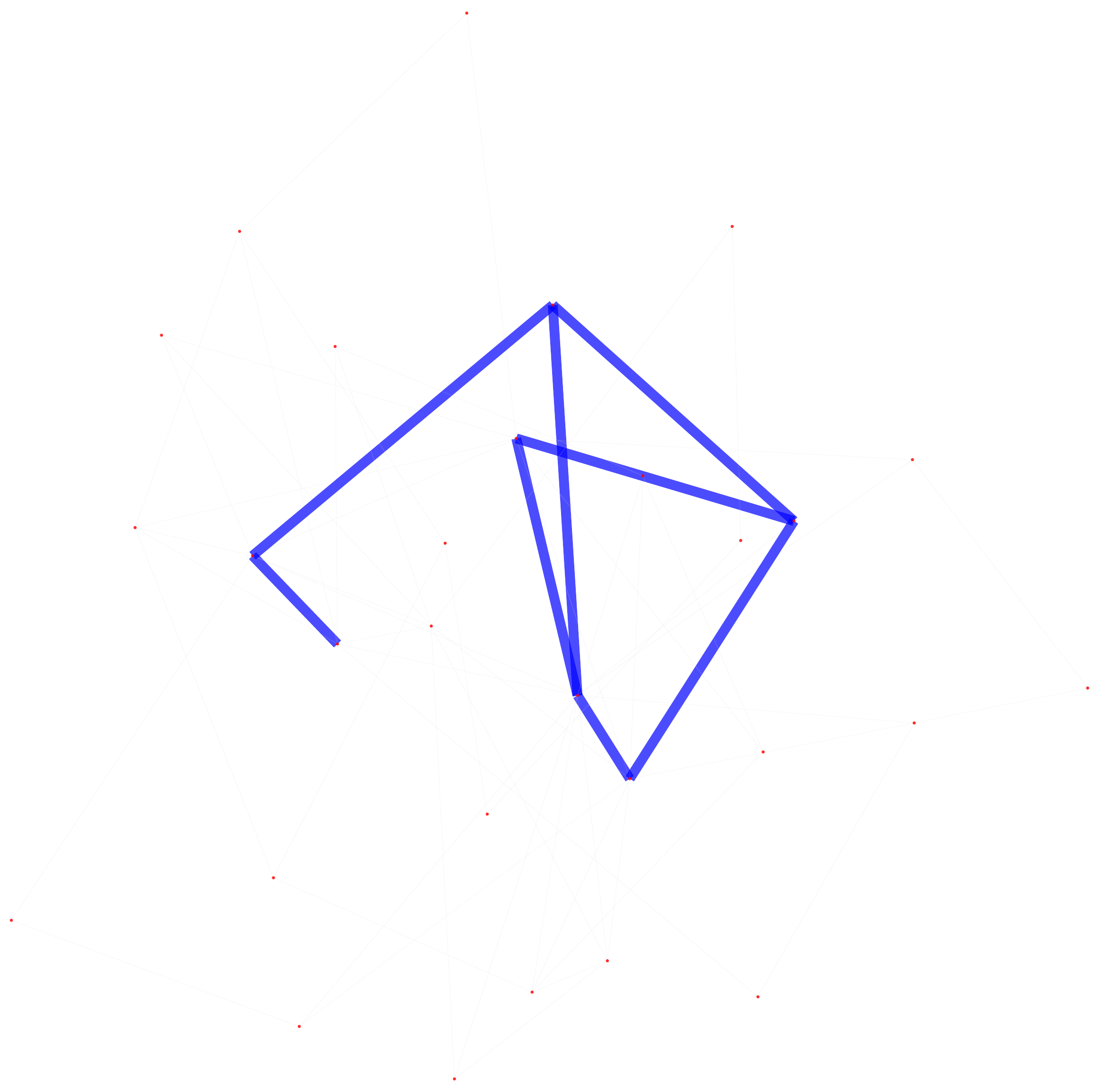}}\hspace{1mm}
\subfloat[  { [480, 1]} ]{\label{fig: path_1}\includegraphics[width=0.10\textwidth]{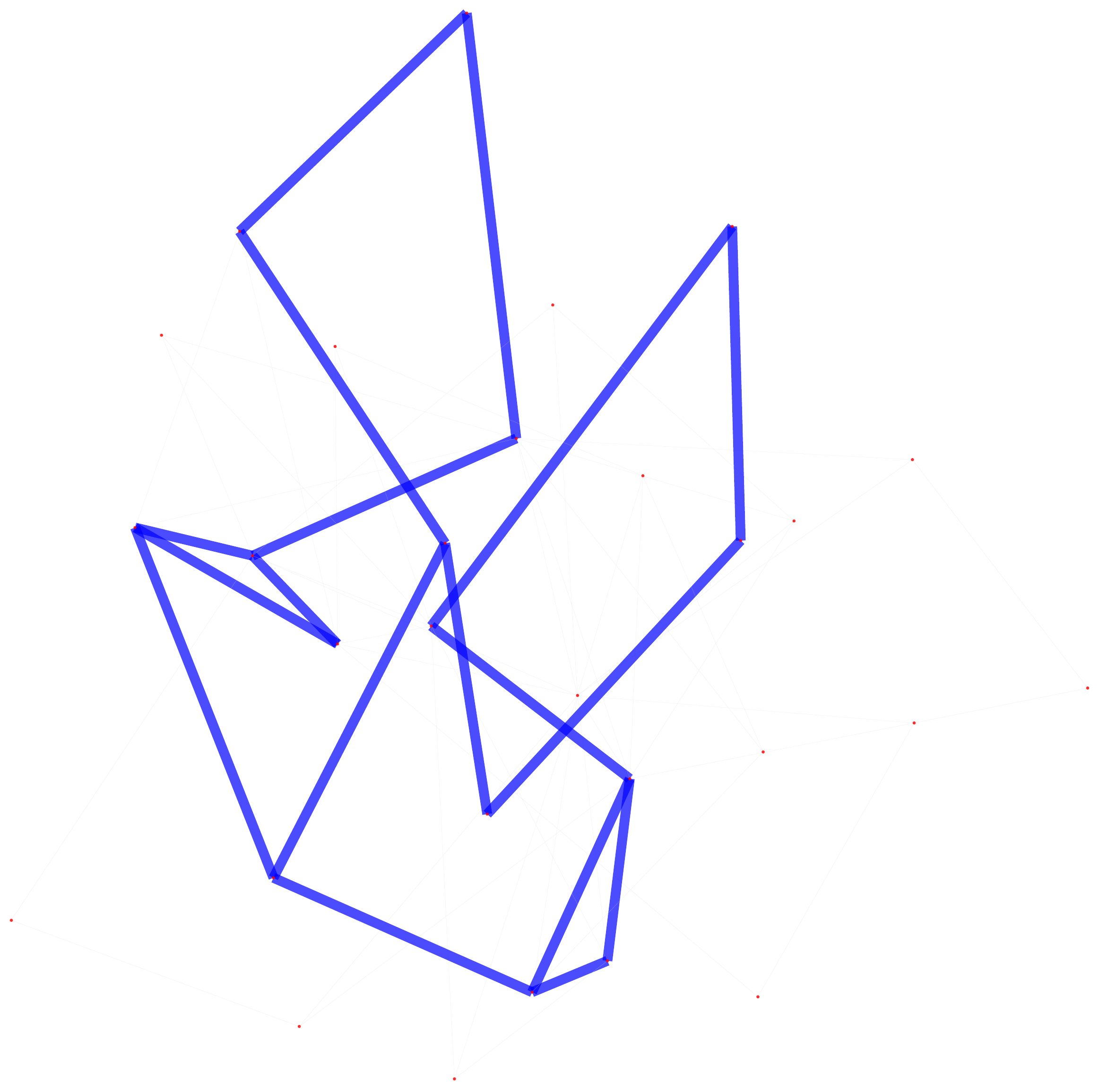}}\hspace{1mm}
\subfloat[  { [480, 2]} ]{\label{fig: path_1}\includegraphics[width=0.10\textwidth]{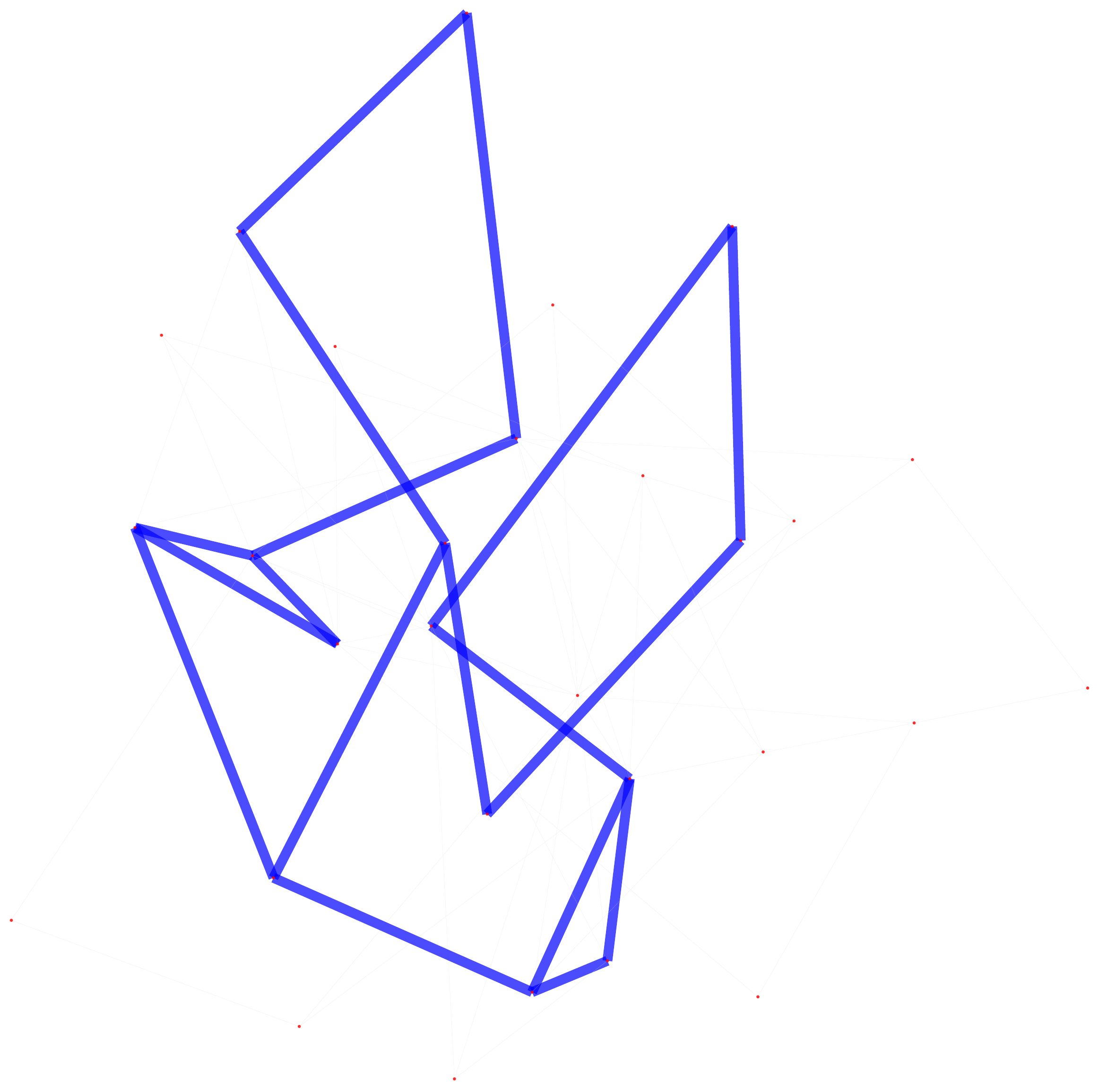}}\hspace{1mm}
\subfloat[  Target ]{\label{fig: path_1}\includegraphics[width=0.10\textwidth]{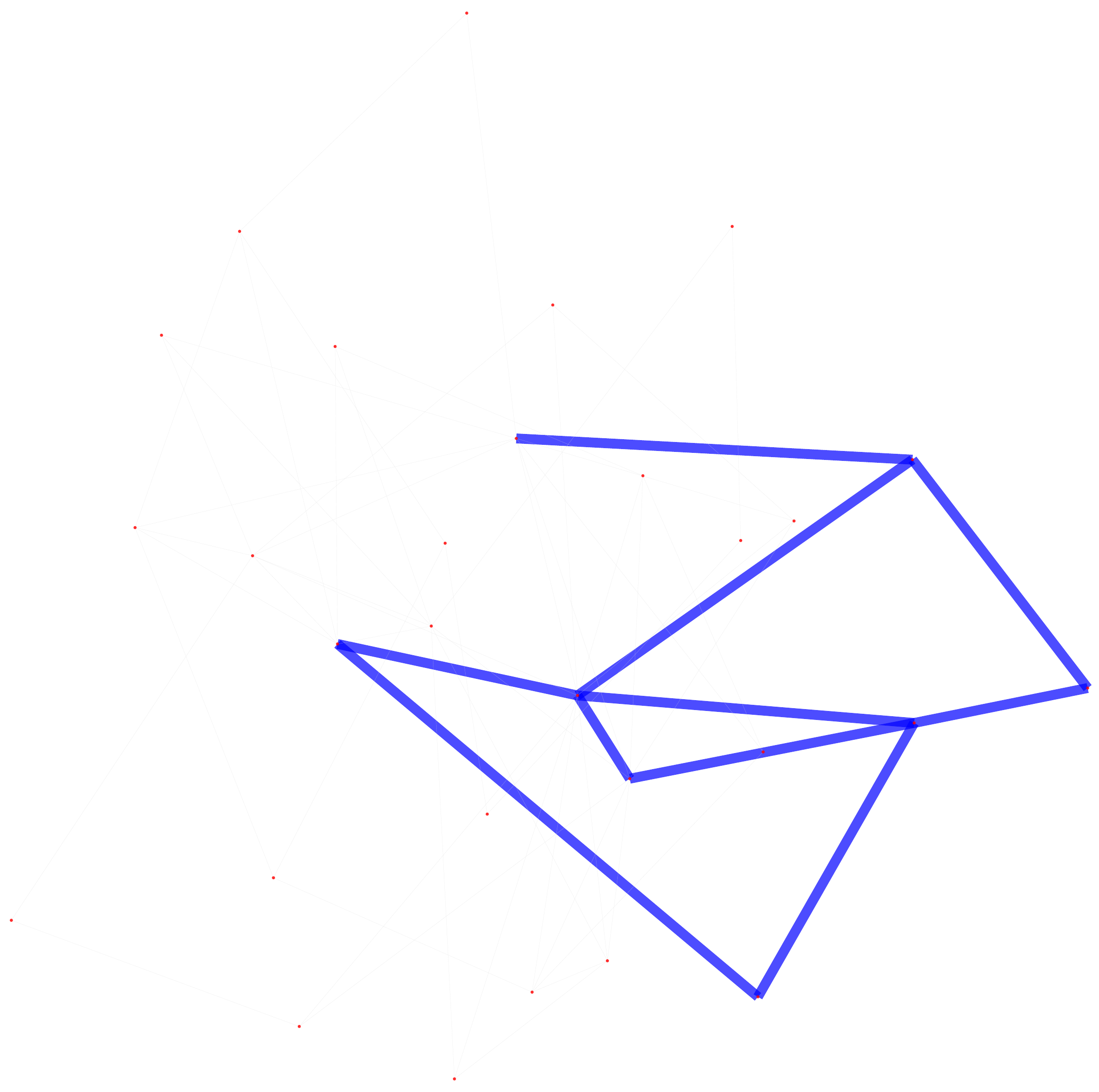}}\hspace{0mm}

\subfloat[  { [2400, 0]} ]{\label{fig: path_1}\includegraphics[width=0.10\textwidth]{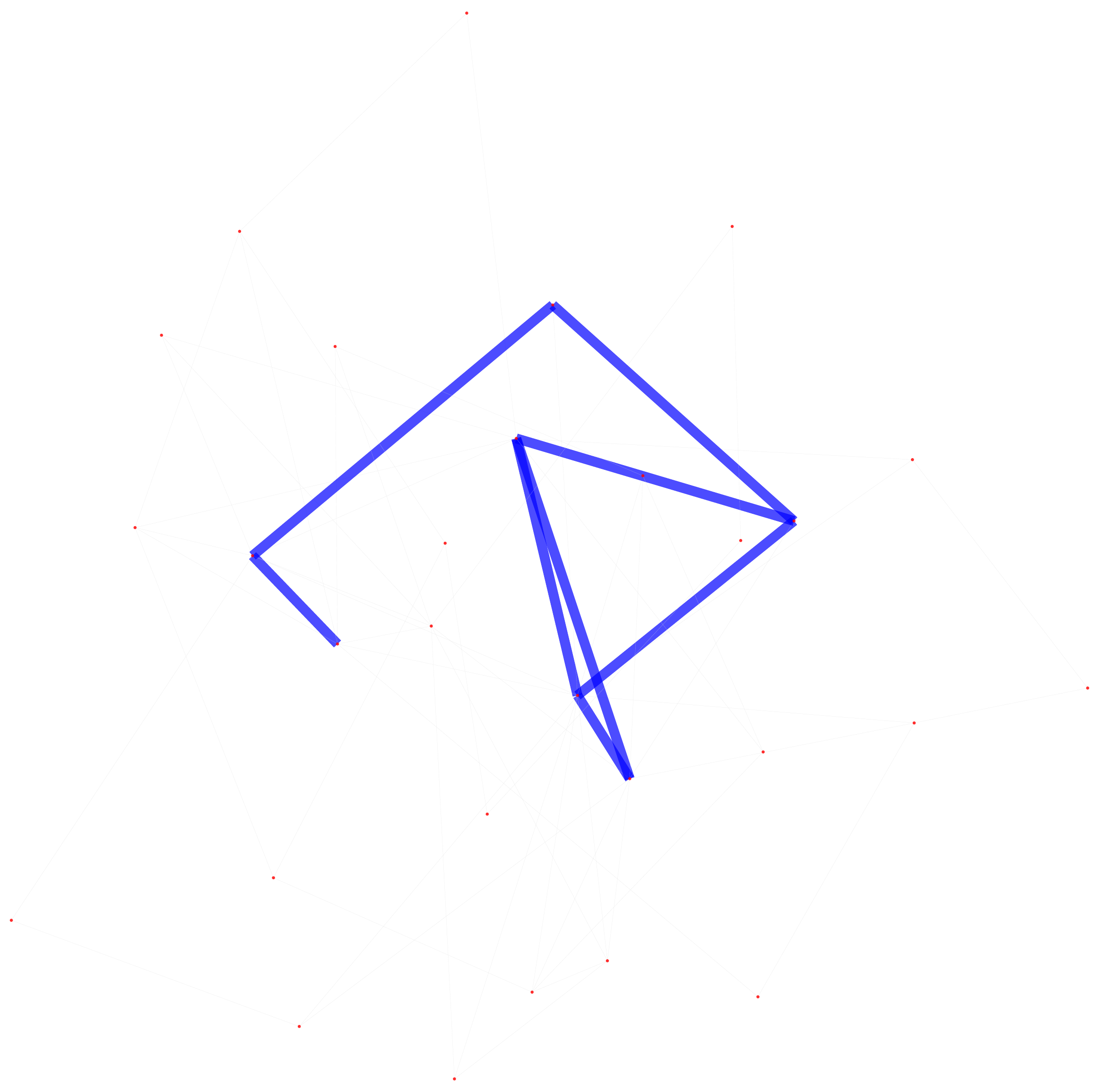}}\hspace{1mm}
\subfloat[  { [2400, 1]} ]{\label{fig: path_1}\includegraphics[width=0.10\textwidth]{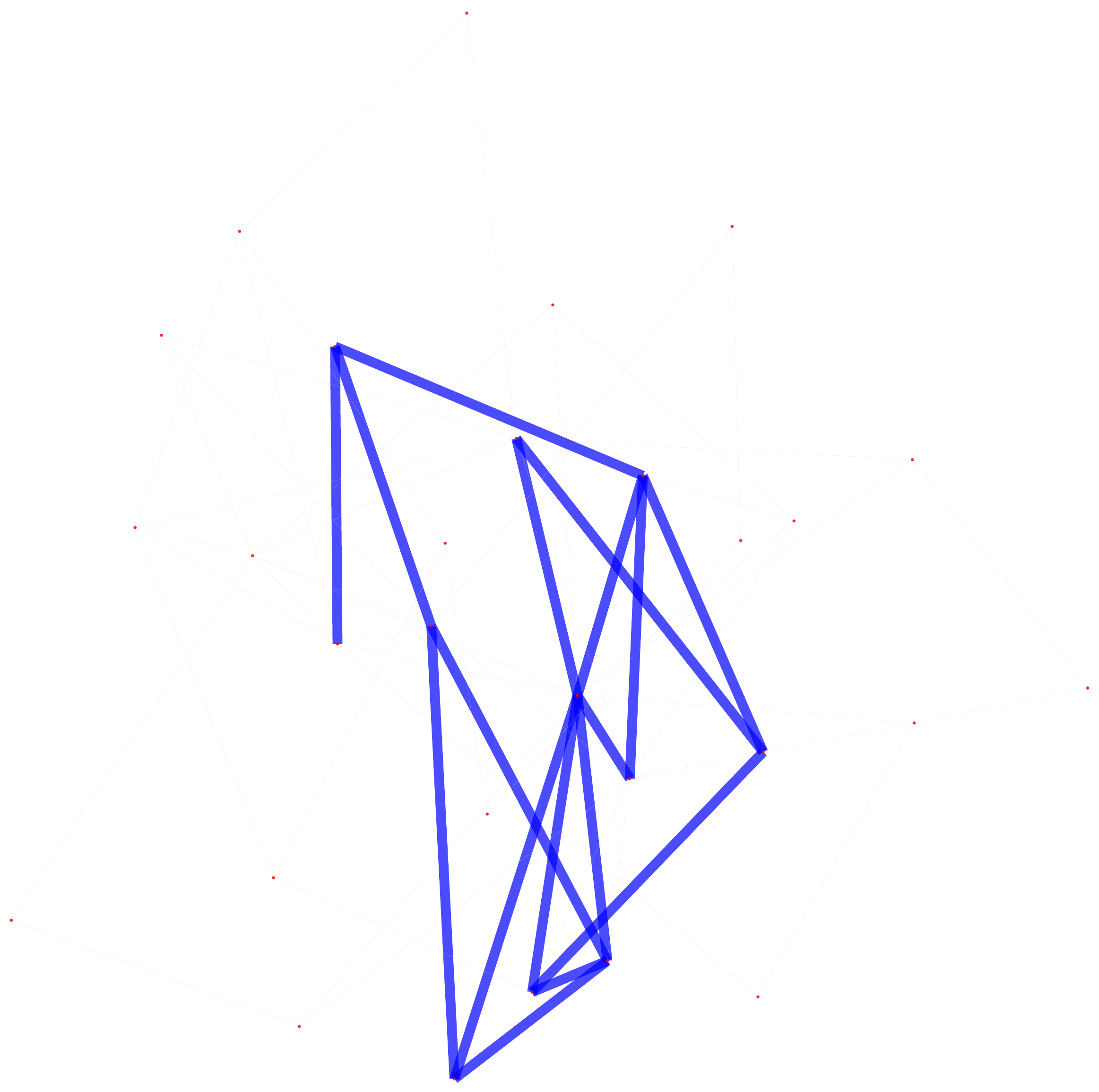}}\hspace{1mm}
\subfloat[  { [2400, 2]} ]{\label{fig: path_1}\includegraphics[width=0.10\textwidth]{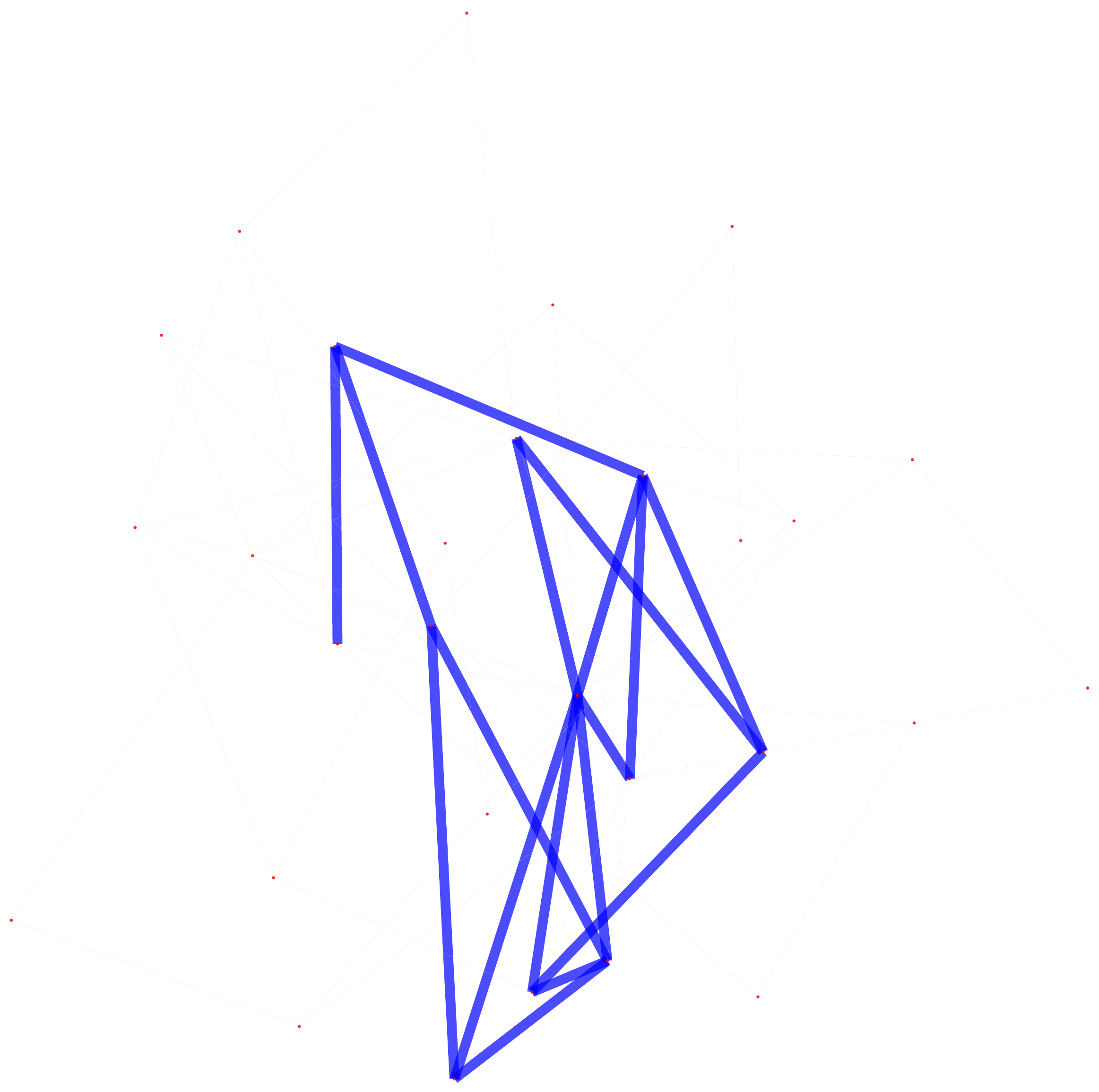}}\hspace{1mm}
\subfloat[  Target ]{\label{fig: path_1}\includegraphics[width=0.10\textwidth]{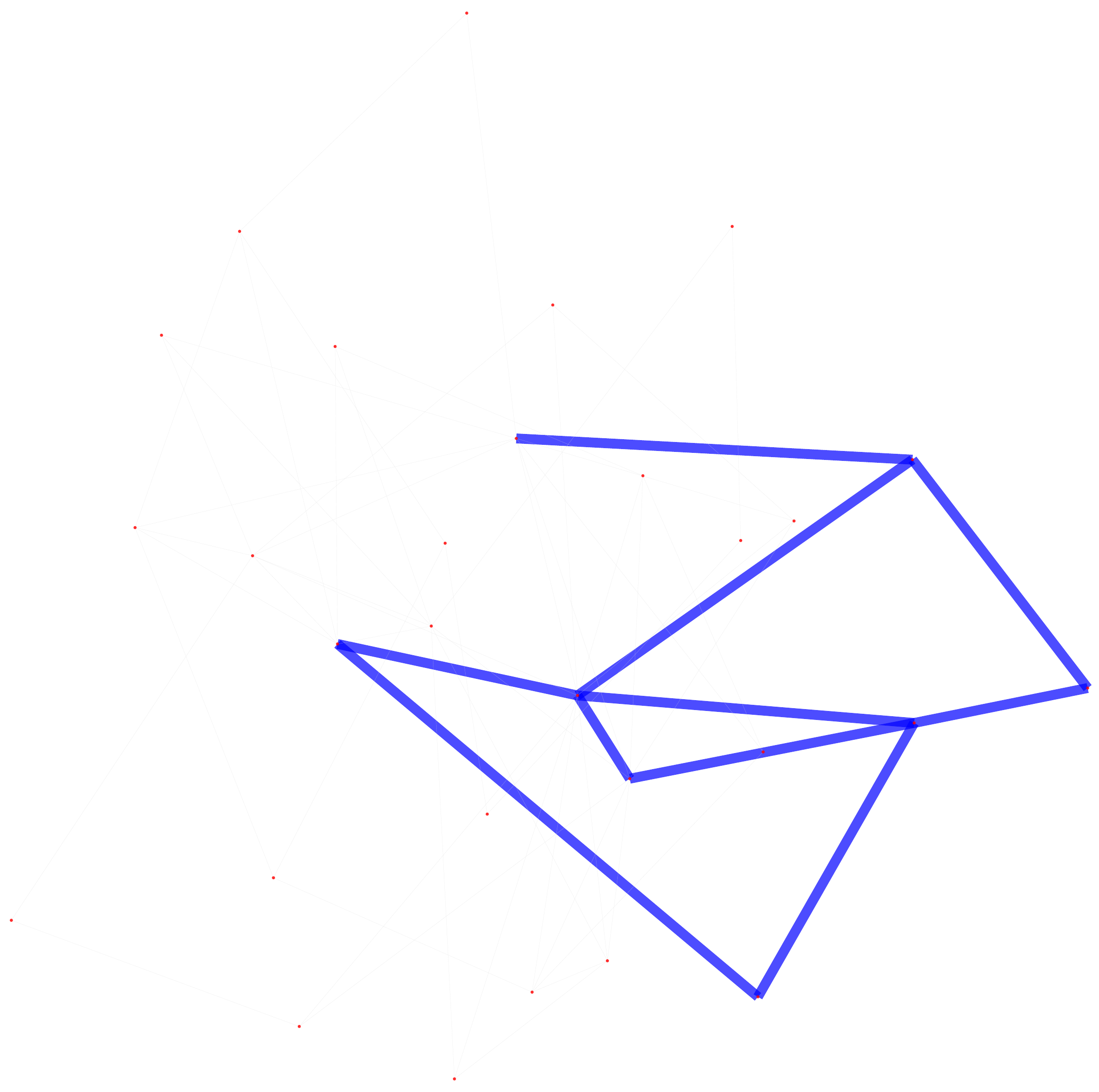}}\hspace{0mm}

\caption{\small Visualization of the results of QRTS-P for an example testing query for Path-to-Tree on Kro.}
\label{fig: more_tree_0}
\end{figure}

\begin{figure}[t]
\centering
\captionsetup[subfloat]{labelfont=scriptsize,textfont=scriptsize,labelformat=empty}
\subfloat[  { [30, 0]} ]{\label{fig: path_1}\includegraphics[width=0.10\textwidth]{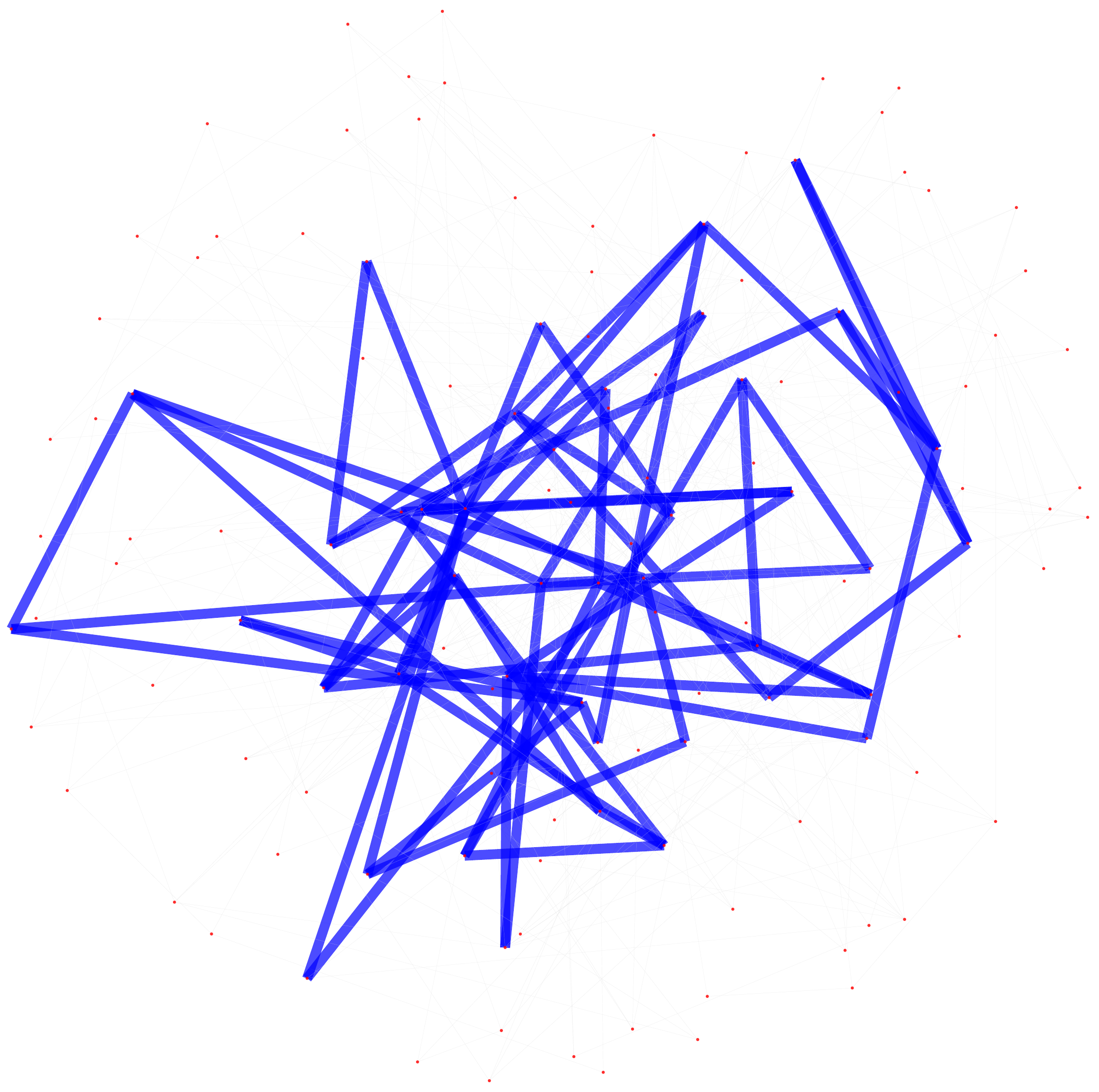}}\hspace{1mm}
\subfloat[  { [30, 1]} ]{\label{fig: path_1}\includegraphics[width=0.10\textwidth]{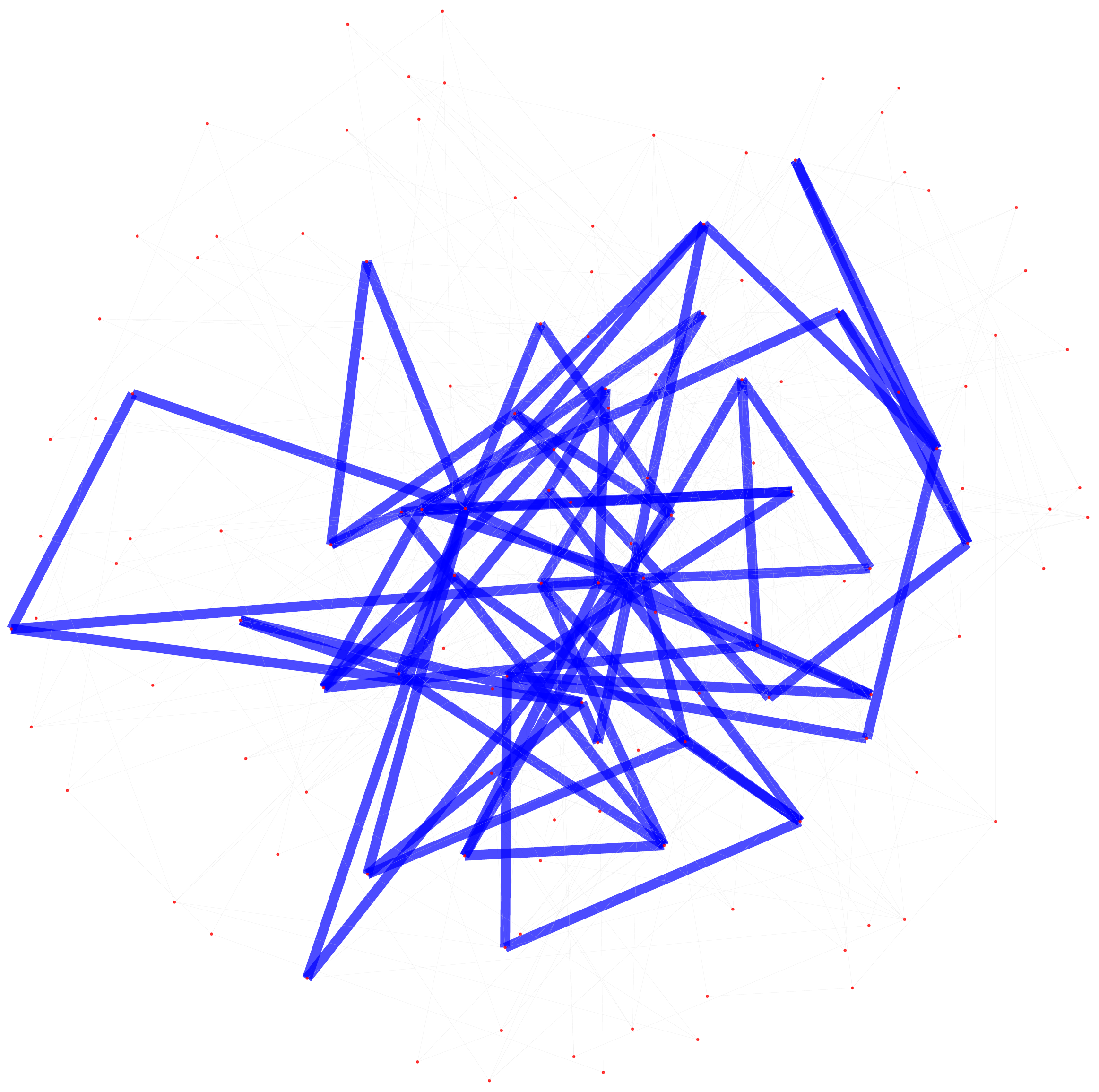}}\hspace{1mm}
\subfloat[  { [30, 2]} ]{\label{fig: path_1}\includegraphics[width=0.10\textwidth]{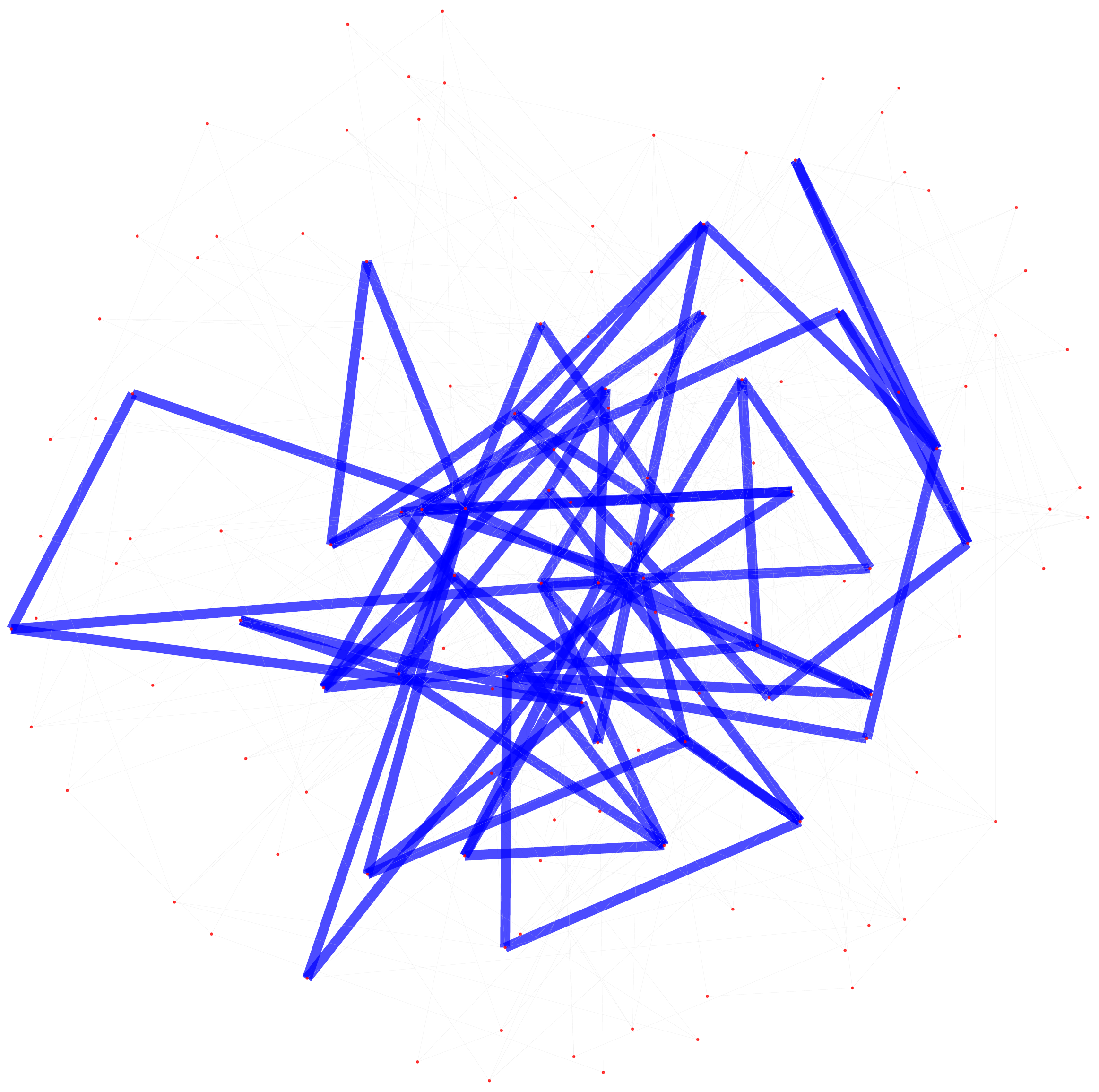}}\hspace{1mm}
\subfloat[  Target ]{\label{fig: path_1}\includegraphics[width=0.10\textwidth]{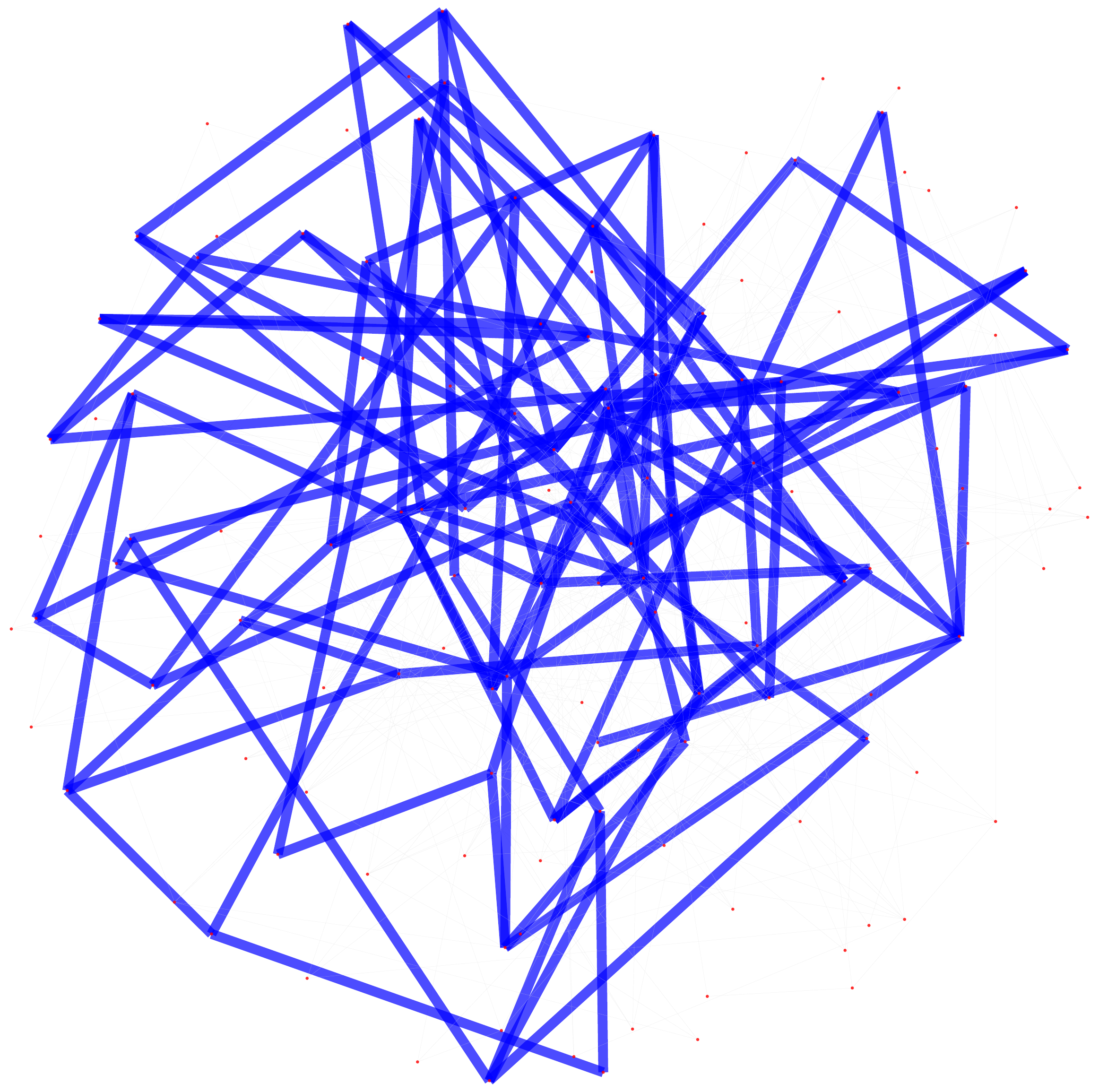}}\hspace{0mm}


\subfloat[  { [240, 0]} ]{\label{fig: path_1}\includegraphics[width=0.10\textwidth]{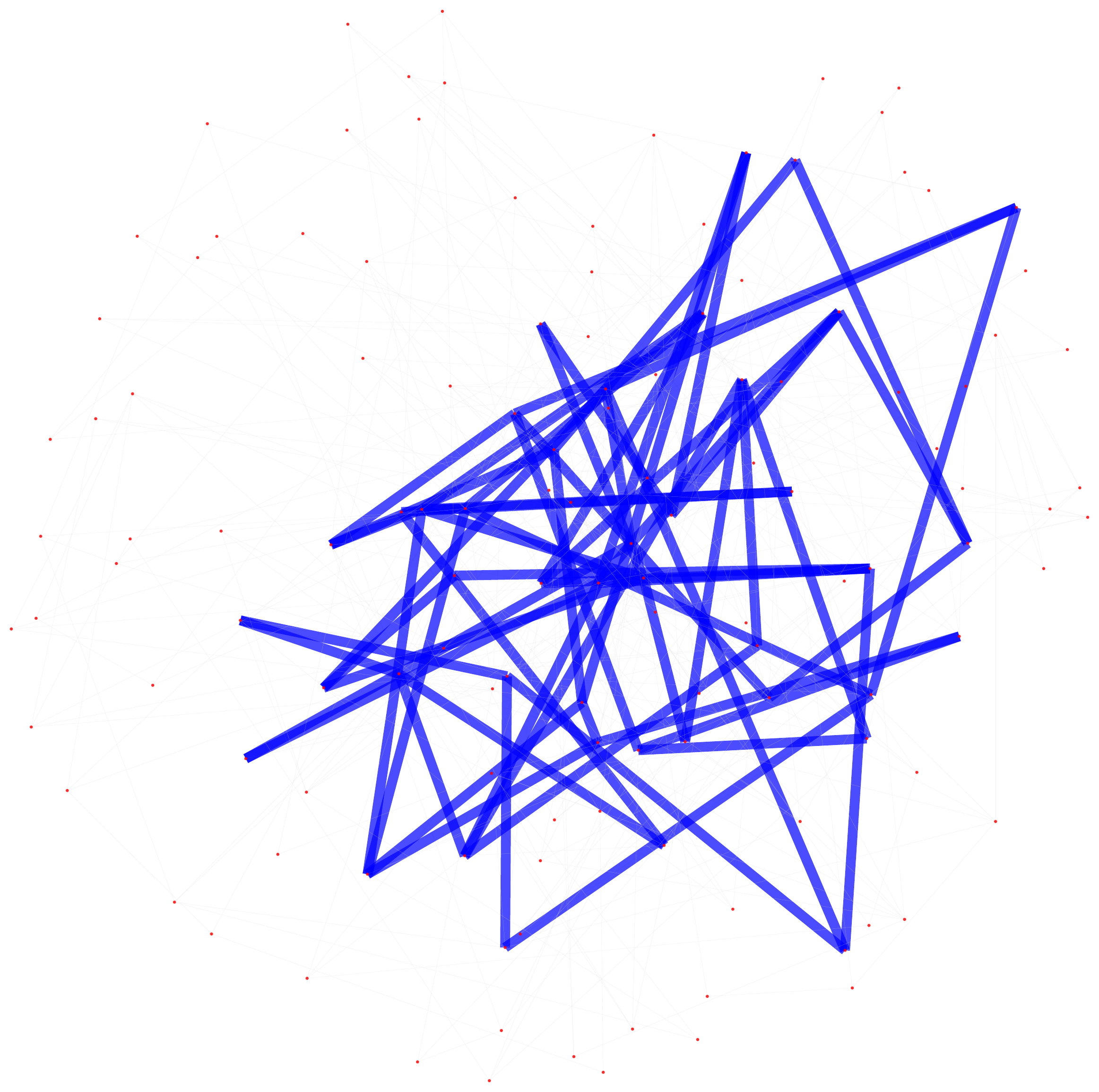}}\hspace{1mm}
\subfloat[  { [240, 1]} ]{\label{fig: path_1}\includegraphics[width=0.10\textwidth]{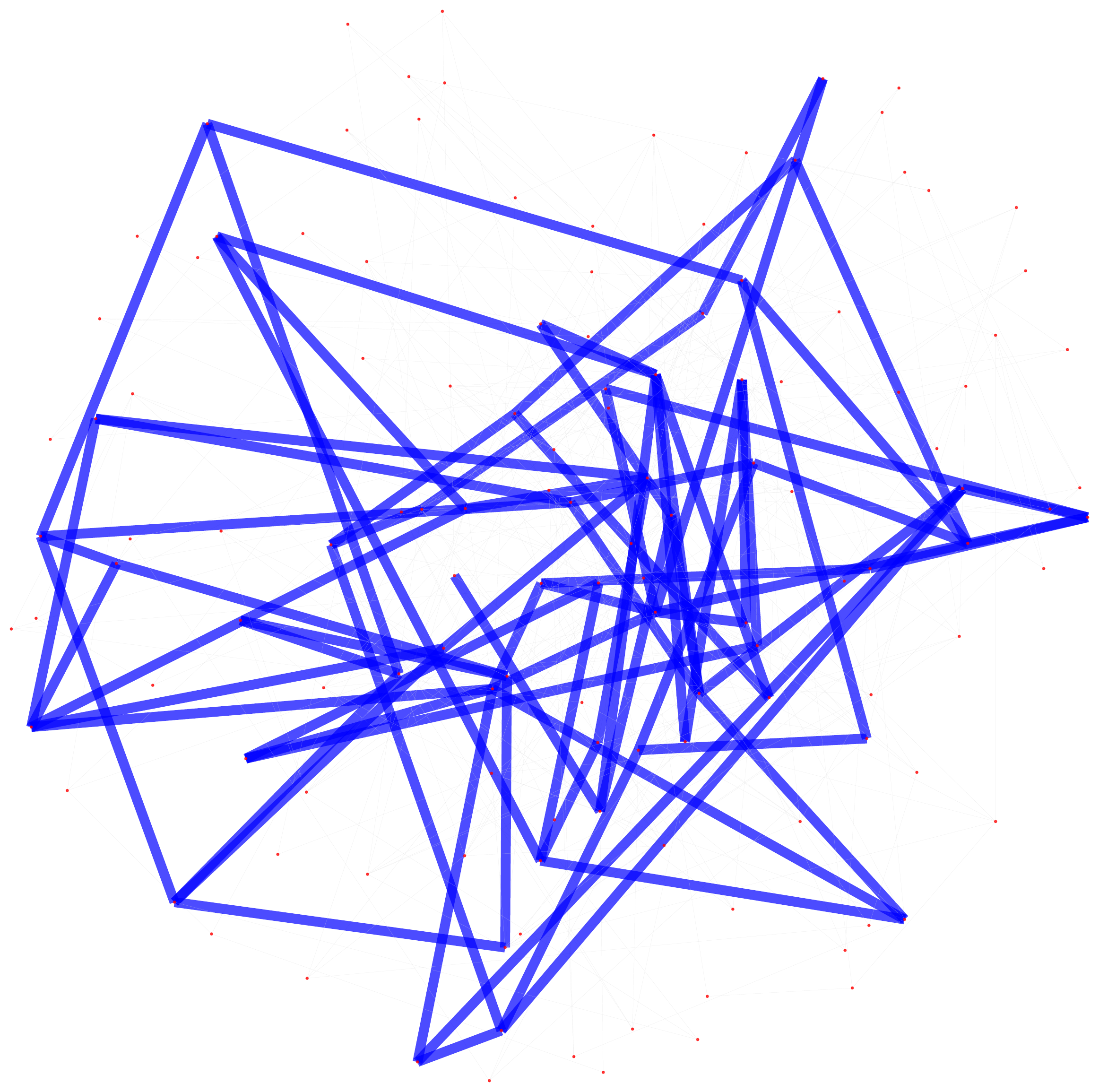}}\hspace{1mm}
\subfloat[  { [240, 2]} ]{\label{fig: path_1}\includegraphics[width=0.10\textwidth]{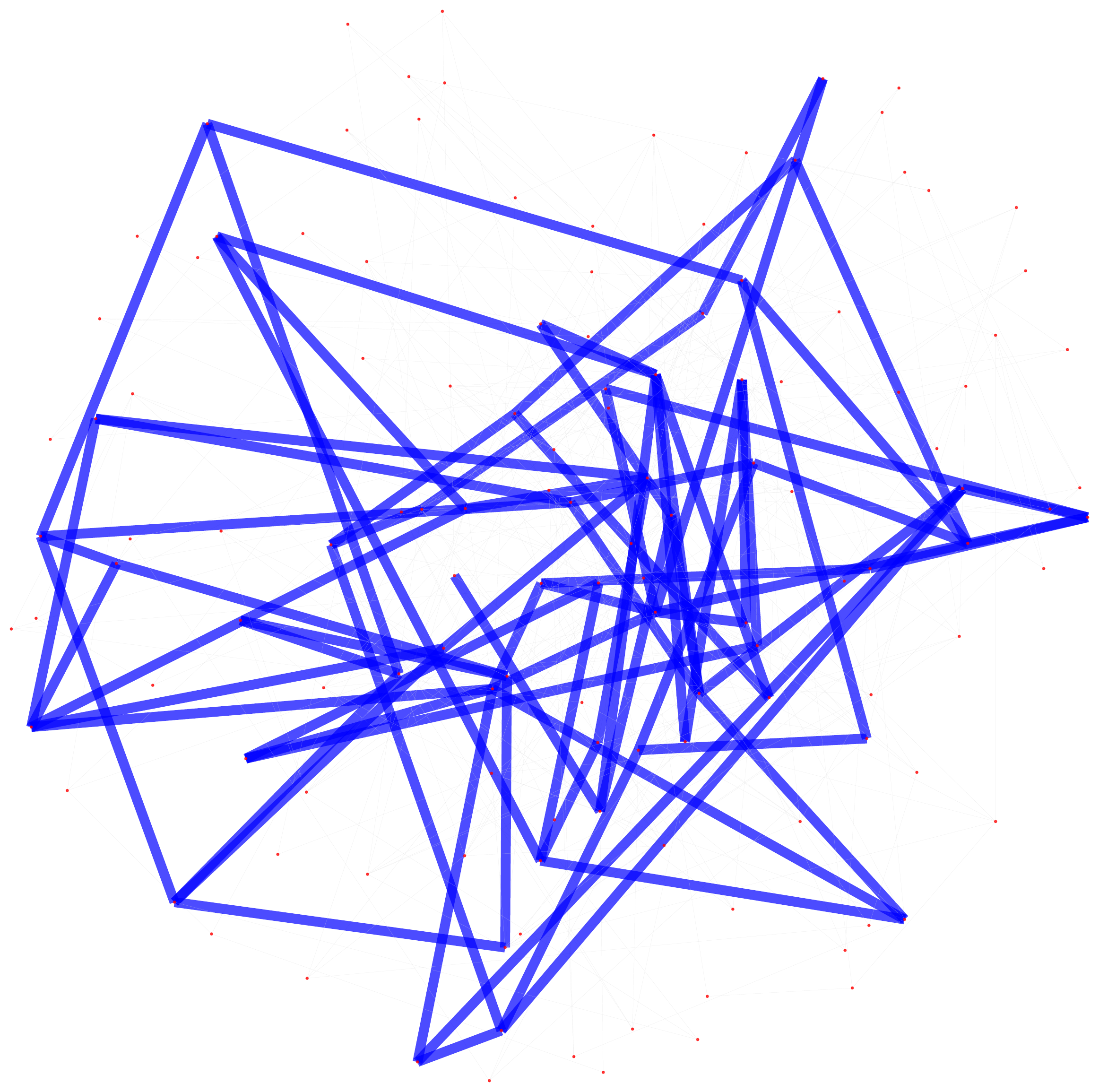}}\hspace{1mm}
\subfloat[  Target ]{\label{fig: path_1}\includegraphics[width=0.10\textwidth]{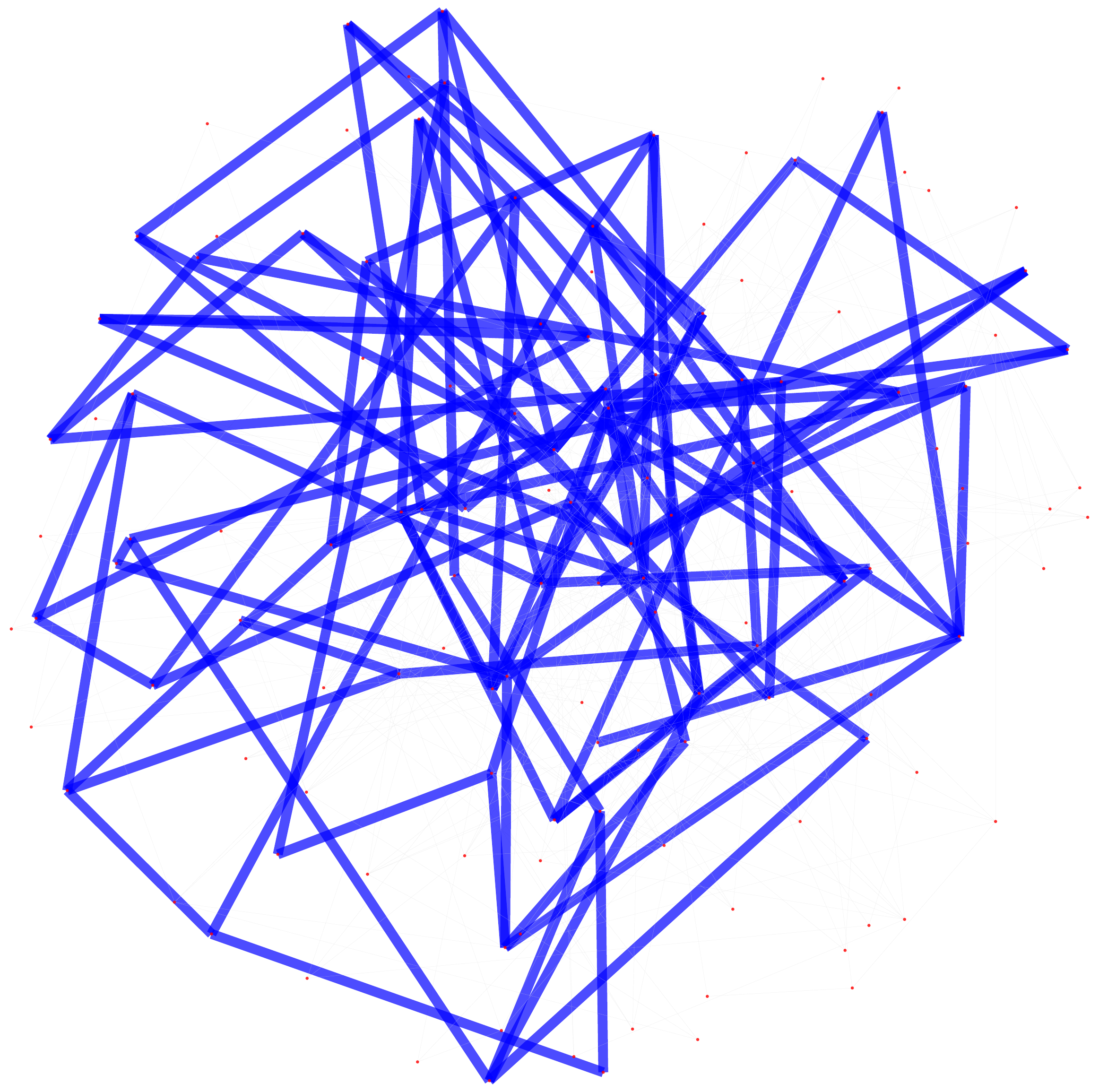}}\hspace{0mm}

\subfloat[  { [480, 0]} ]{\label{fig: path_1}\includegraphics[width=0.10\textwidth]{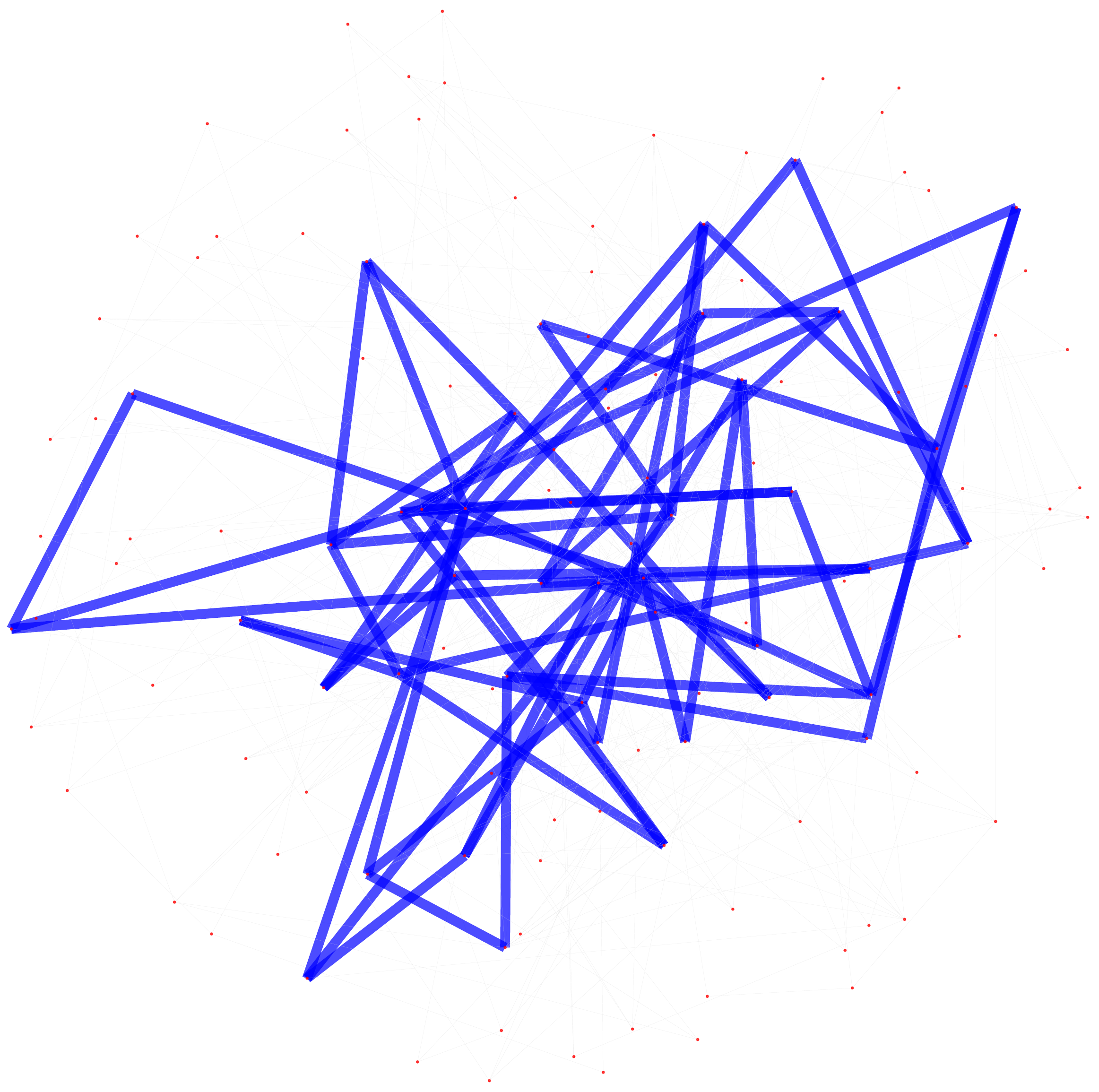}}\hspace{1mm}
\subfloat[  { [480, 1]} ]{\label{fig: path_1}\includegraphics[width=0.10\textwidth]{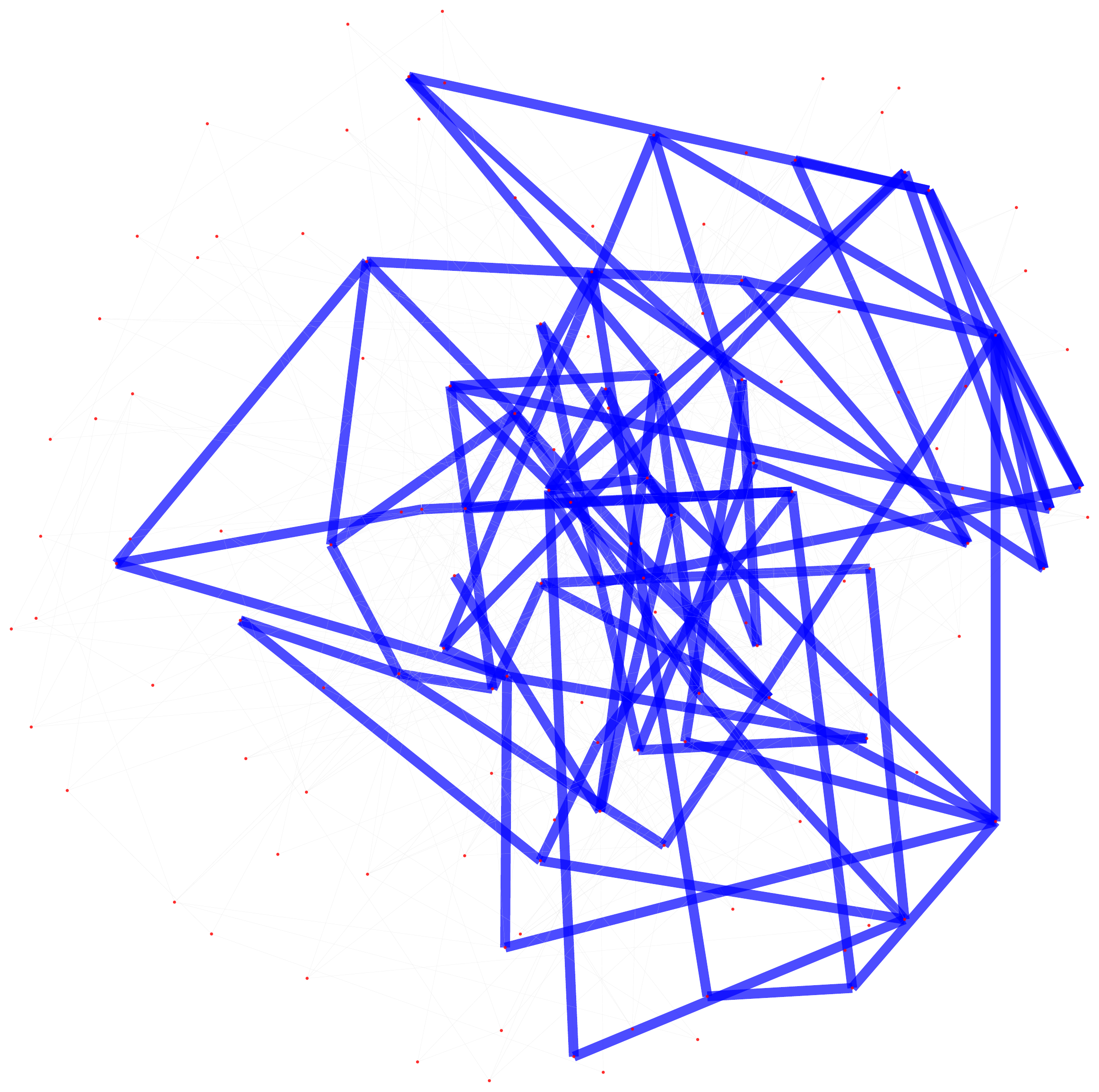}}\hspace{1mm}
\subfloat[  { [480, 2]} ]{\label{fig: path_1}\includegraphics[width=0.10\textwidth]{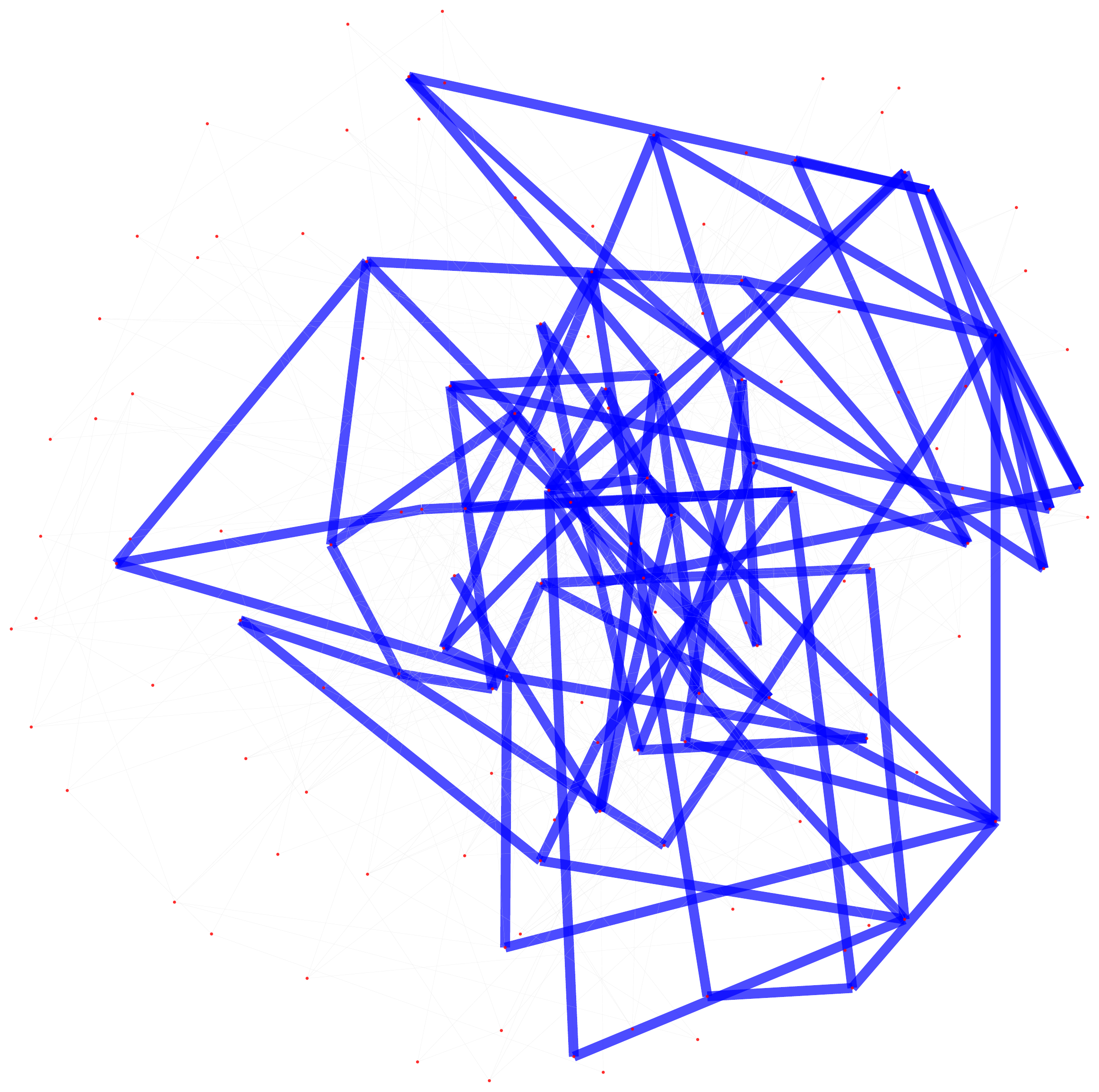}}\hspace{1mm}
\subfloat[  Target ]{\label{fig: path_1}\includegraphics[width=0.10\textwidth]{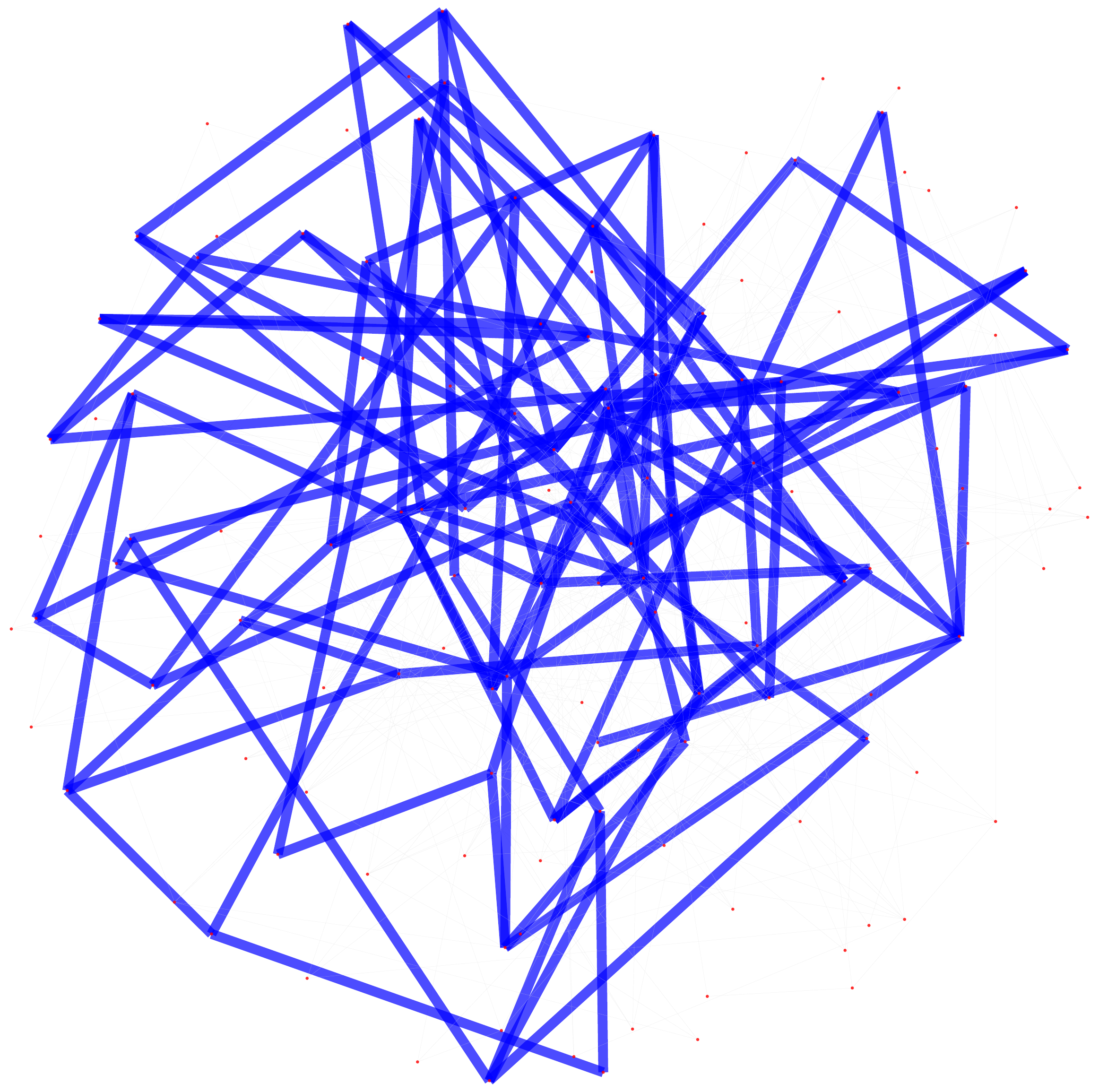}}\hspace{0mm}

\subfloat[  { [2400, 0]} ]{\label{fig: path_1}\includegraphics[width=0.10\textwidth]{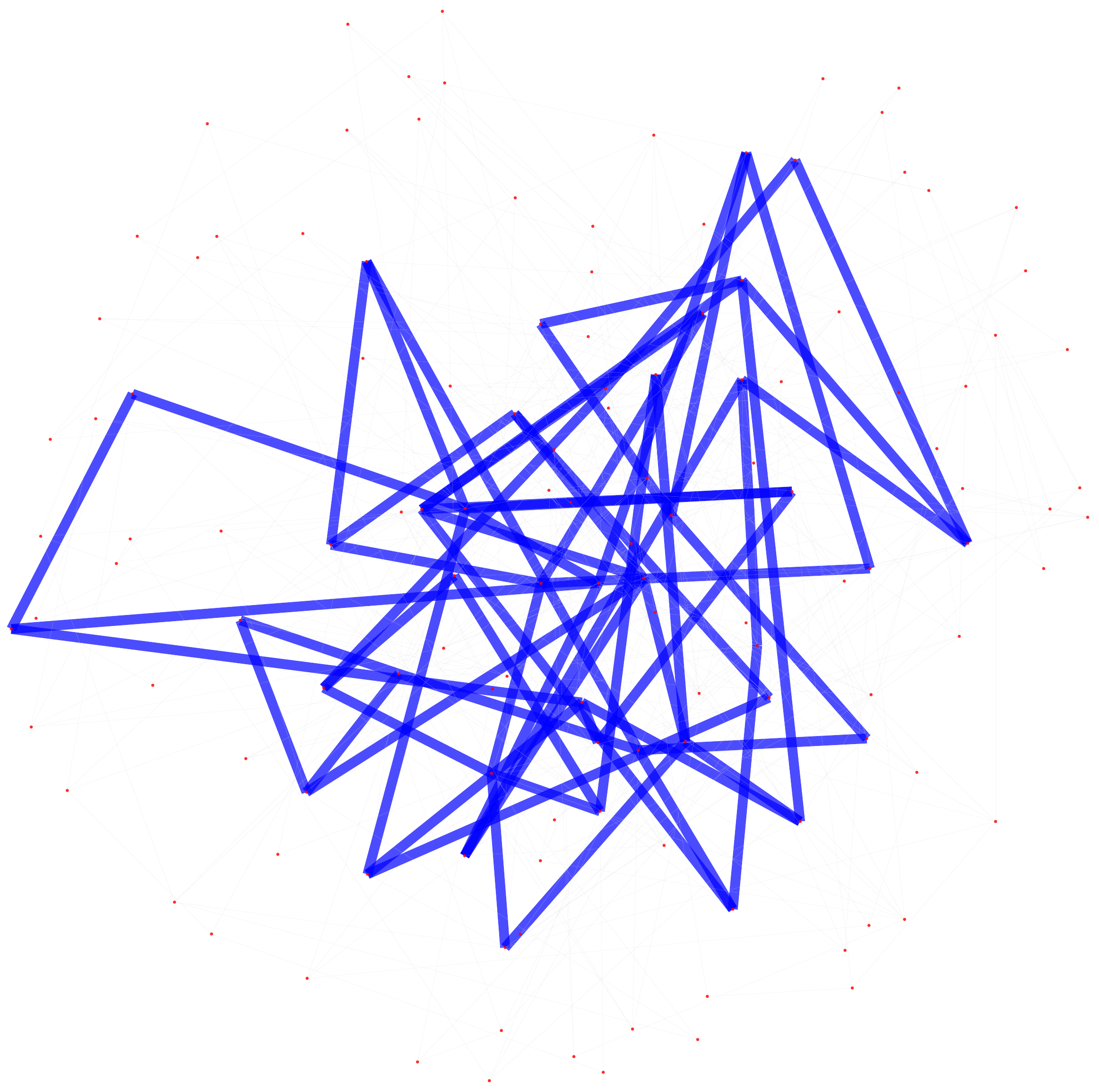}}\hspace{1mm}
\subfloat[  { [2400, 1]} ]{\label{fig: path_1}\includegraphics[width=0.10\textwidth]{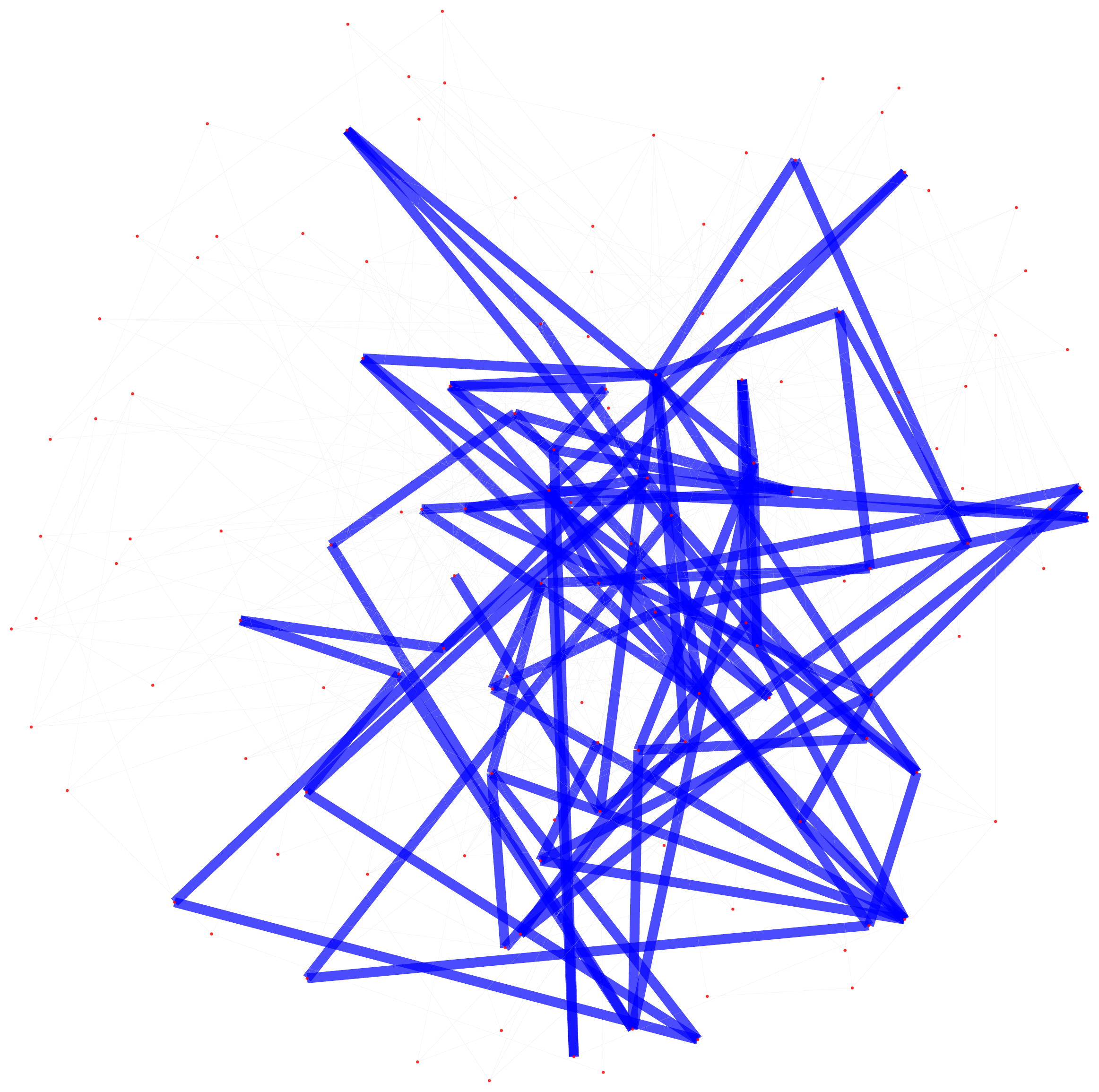}}\hspace{1mm}
\subfloat[  { [2400, 2]} ]{\label{fig: path_1}\includegraphics[width=0.10\textwidth]{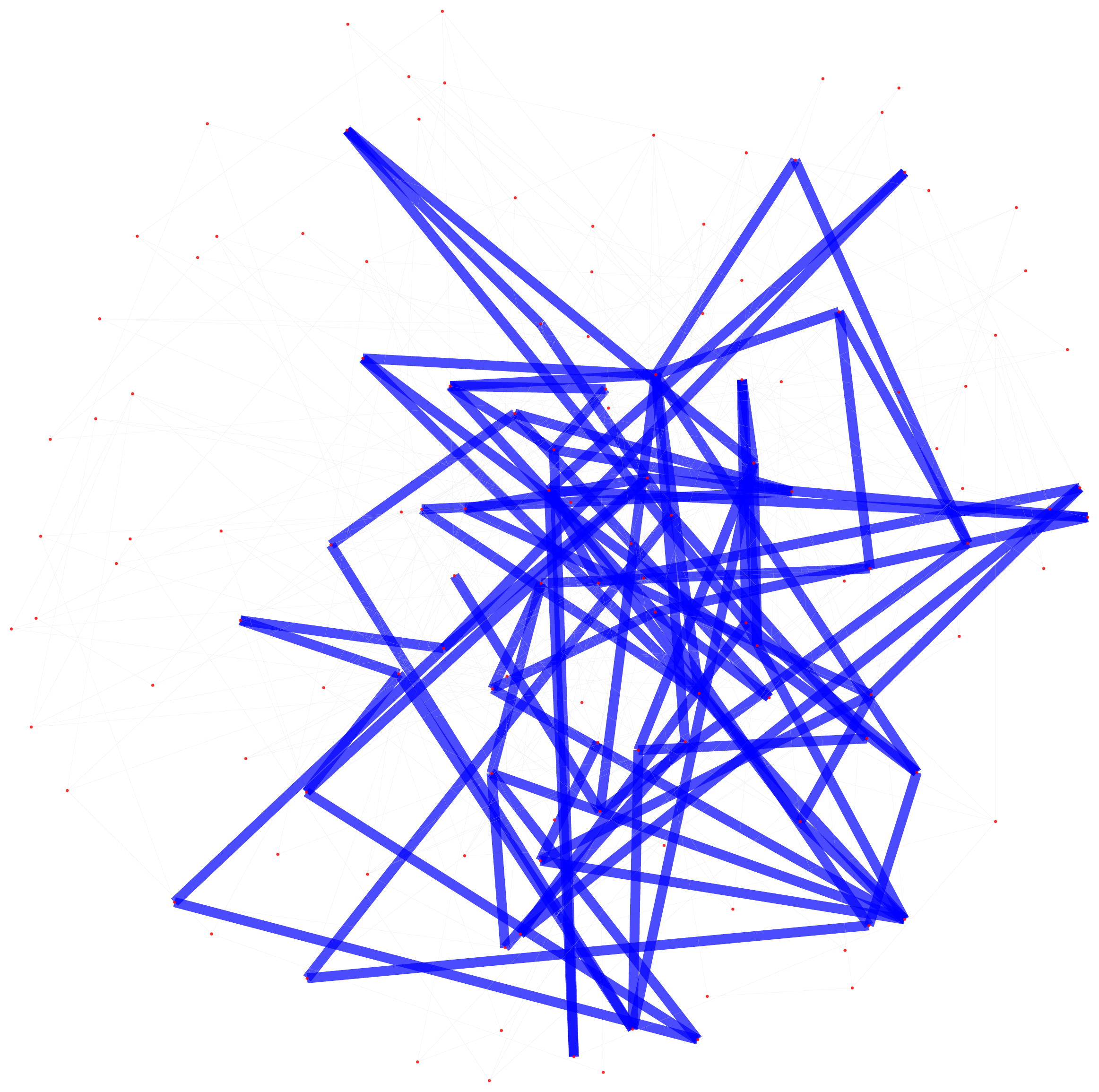}}\hspace{1mm}
\subfloat[  Target ]{\label{fig: path_1}\includegraphics[width=0.10\textwidth]{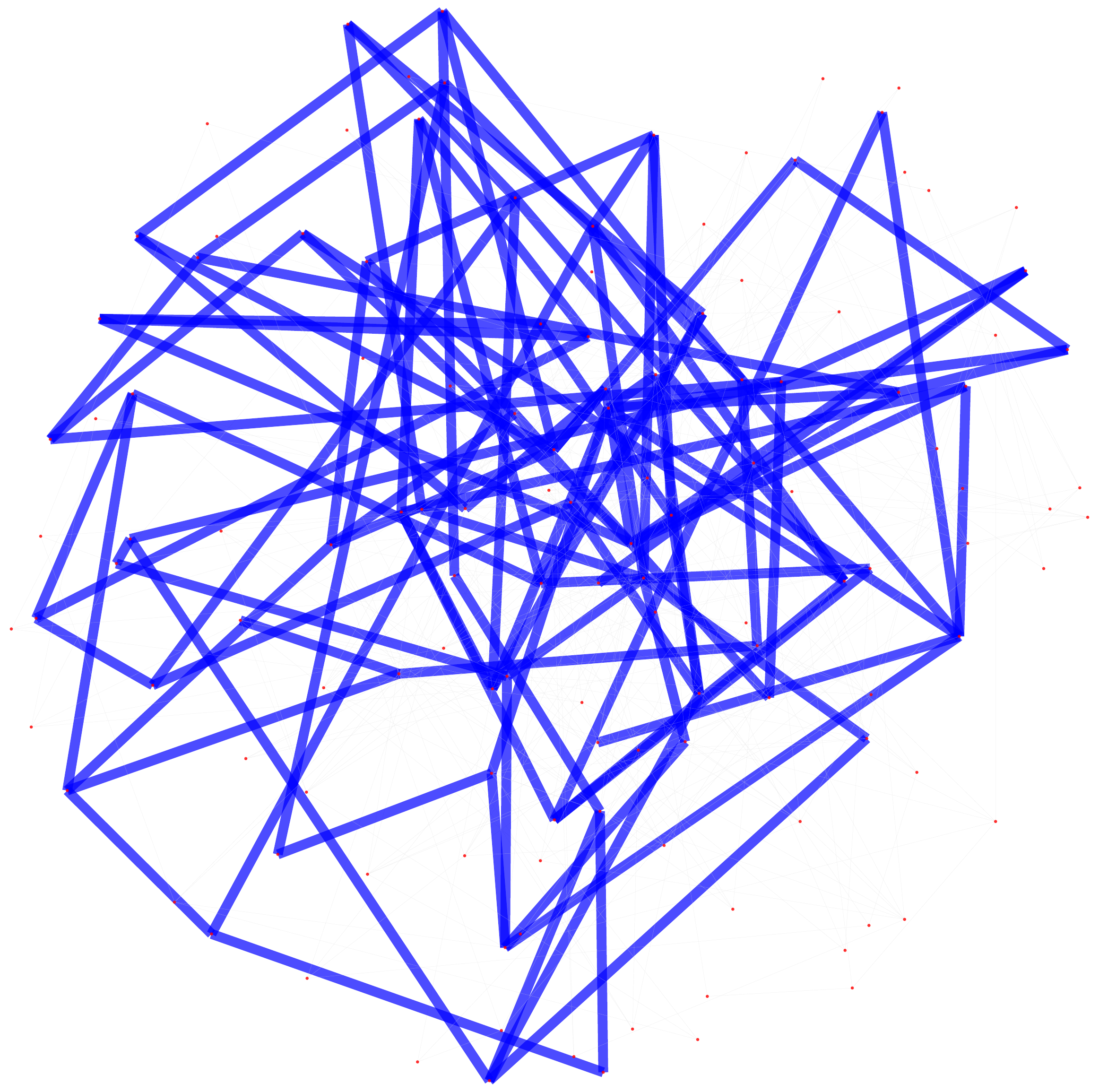}}\hspace{0mm}

\caption{\small Visualization of the results of QRTS-P for an example testing query for Path-to-Tree on Kro.}
\label{fig: more_tree_1}
\end{figure}

\begin{figure}[t]
\centering
\captionsetup[subfloat]{labelfont=scriptsize,textfont=scriptsize,labelformat=empty}
\subfloat[  { [30, 0]} ]{\label{fig: path_1}\includegraphics[width=0.10\textwidth]{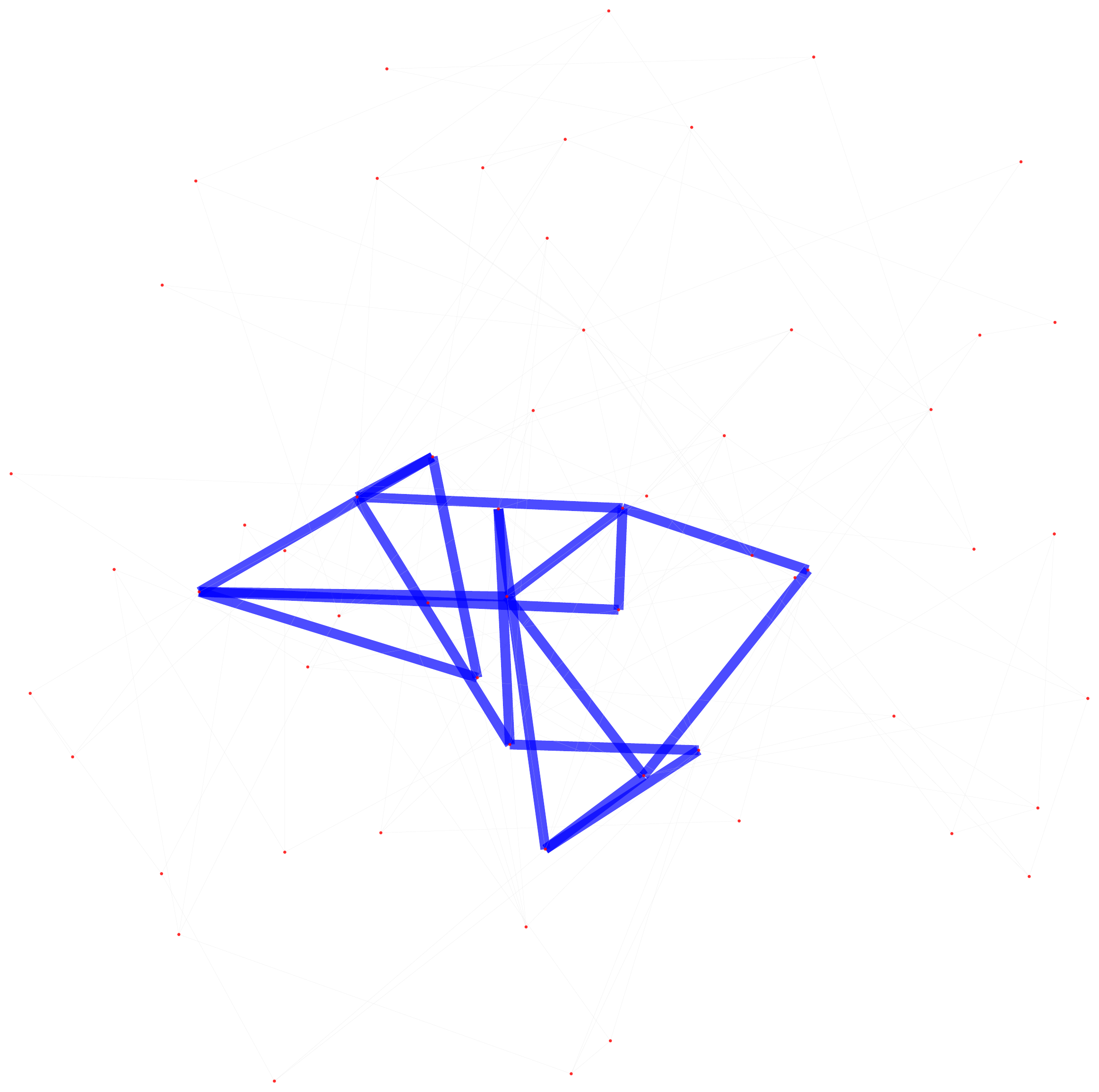}}\hspace{1mm}
\subfloat[  { [30, 1]} ]{\label{fig: path_1}\includegraphics[width=0.10\textwidth]{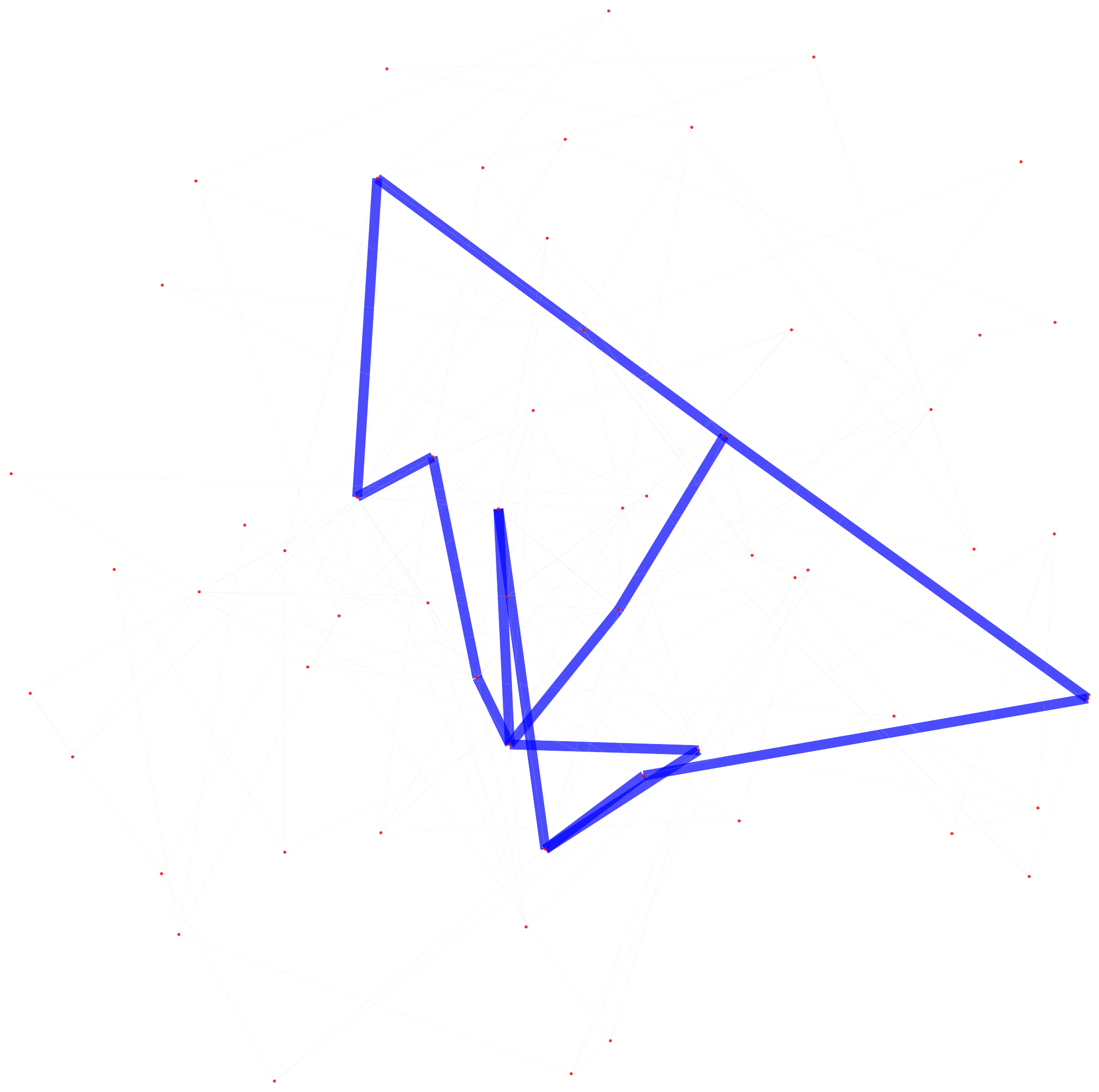}}\hspace{1mm}
\subfloat[  { [30, 2]} ]{\label{fig: path_1}\includegraphics[width=0.10\textwidth]{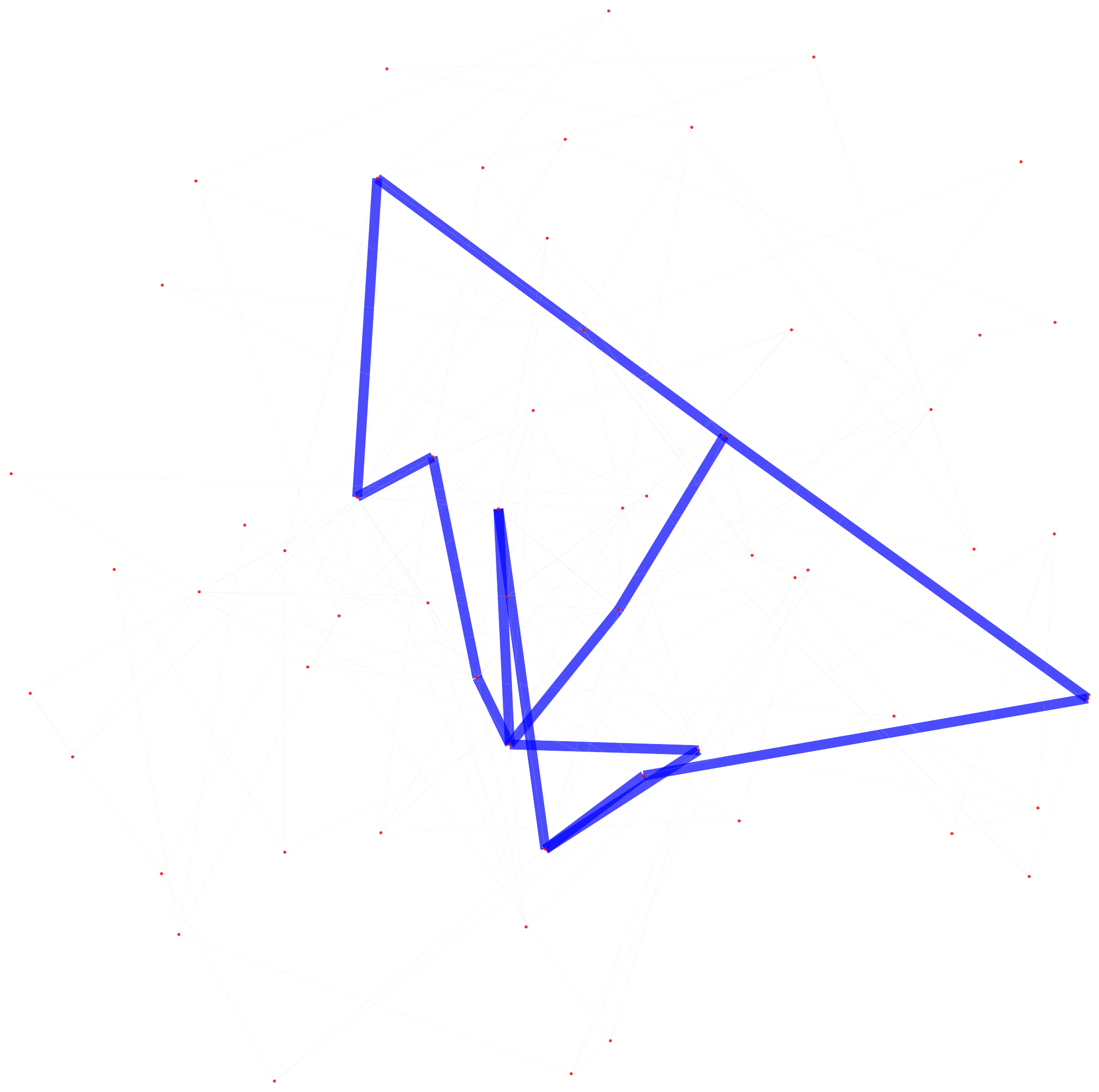}}\hspace{1mm}
\subfloat[  Target ]{\label{fig: path_1}\includegraphics[width=0.10\textwidth]{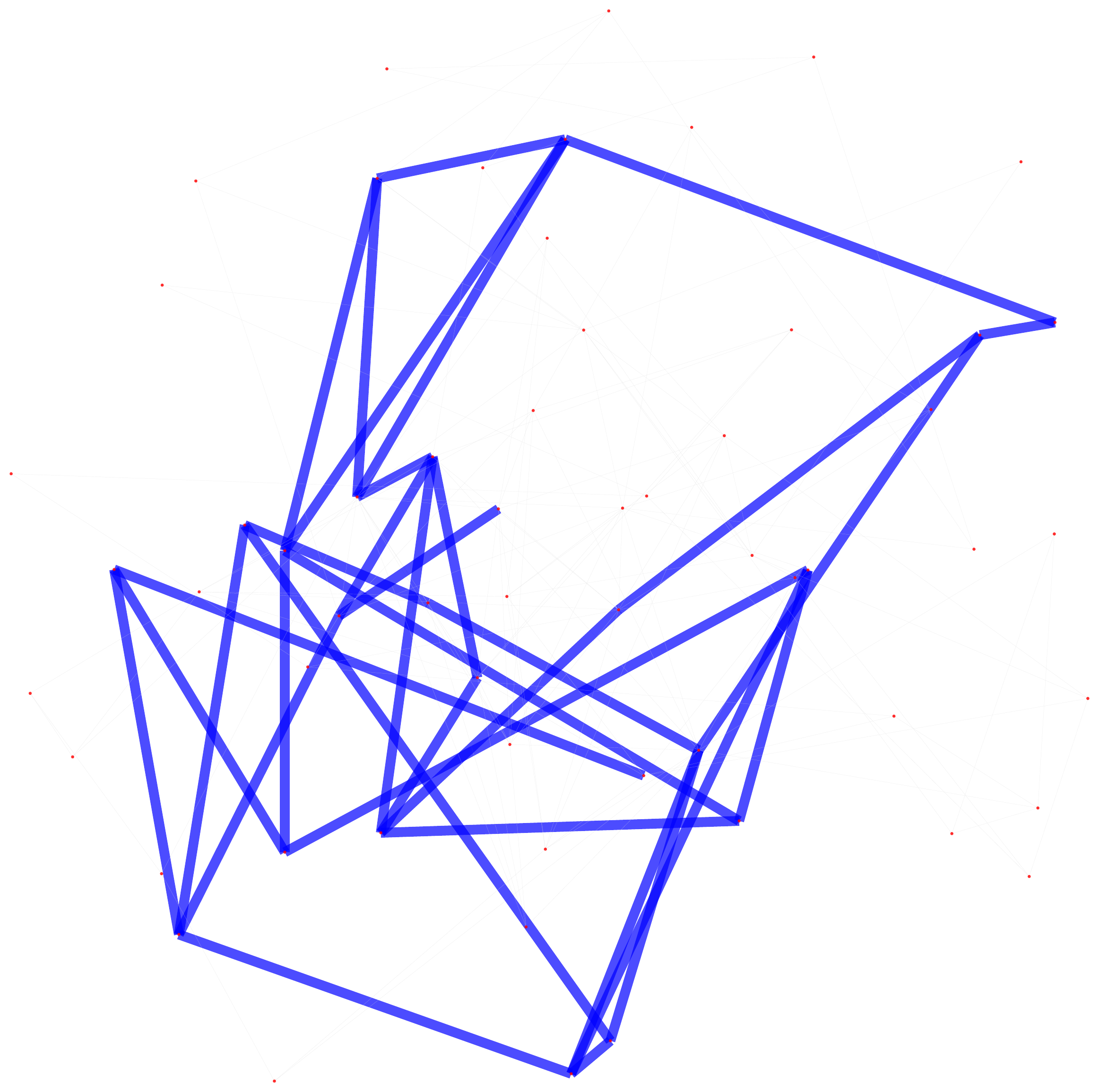}}\hspace{0mm}


\subfloat[  { [240, 0]} ]{\label{fig: path_1}\includegraphics[width=0.10\textwidth]{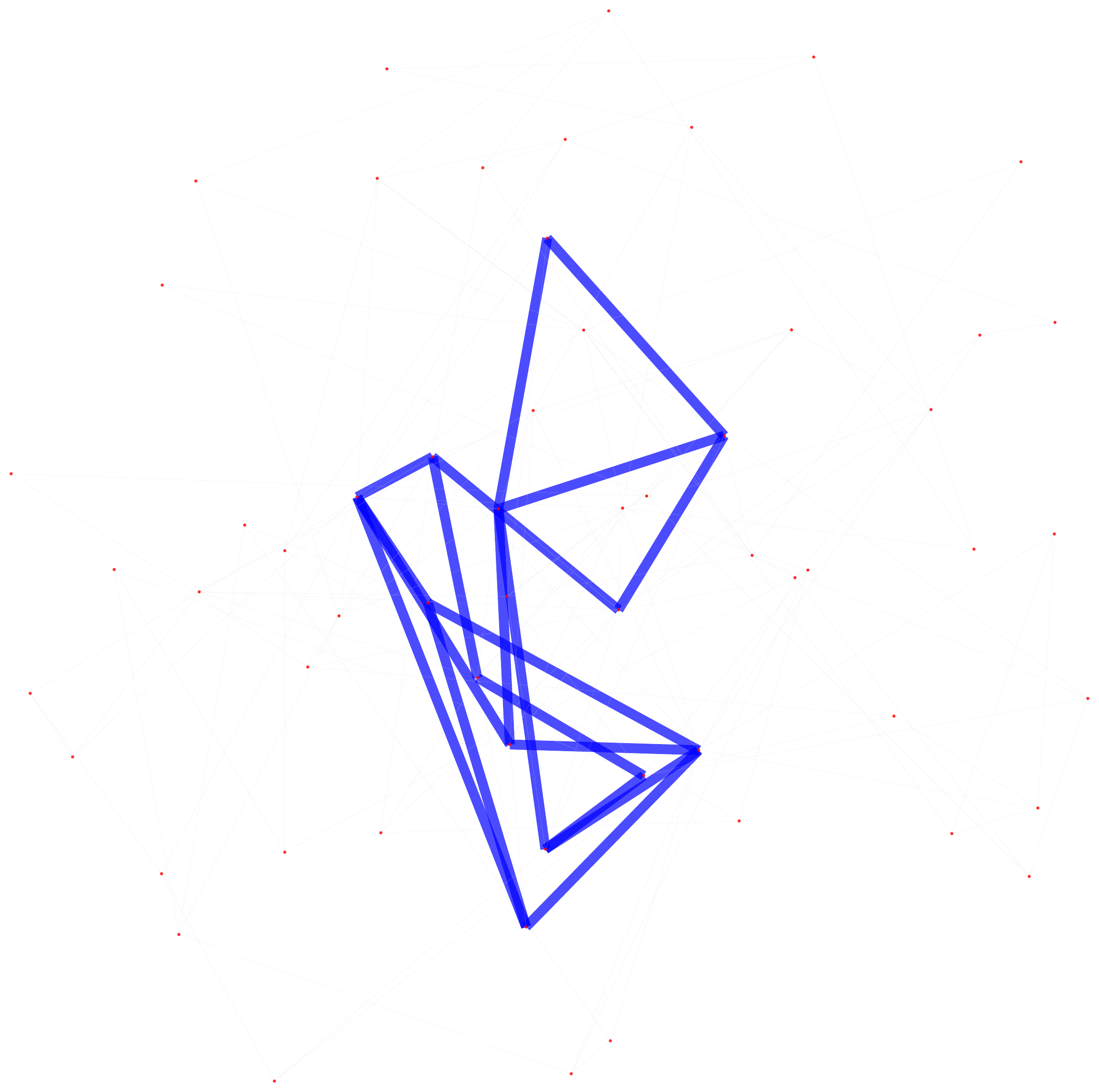}}\hspace{1mm}
\subfloat[  { [240, 1]} ]{\label{fig: path_1}\includegraphics[width=0.10\textwidth]{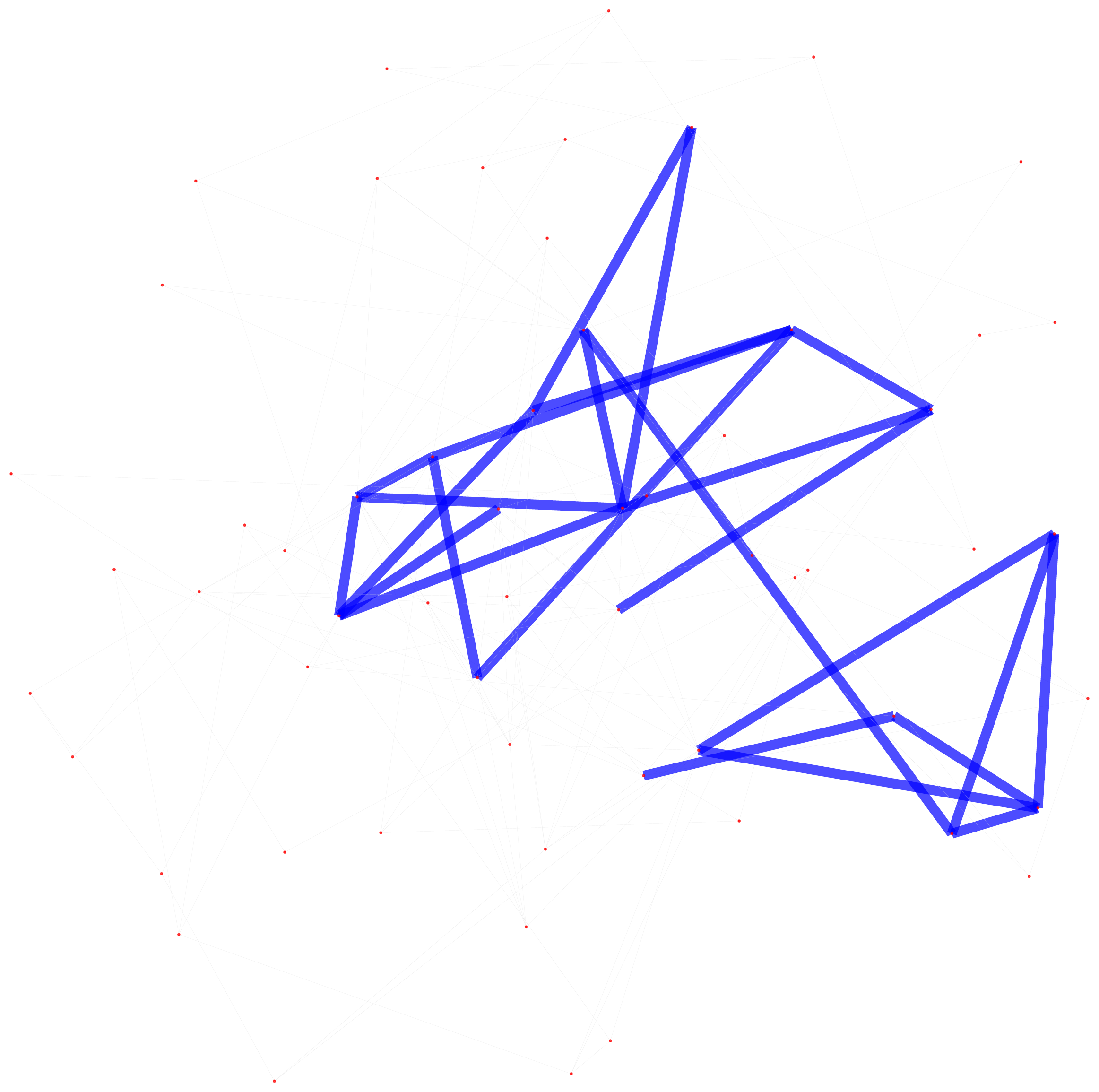}}\hspace{1mm}
\subfloat[  { [240, 2]} ]{\label{fig: path_1}\includegraphics[width=0.10\textwidth]{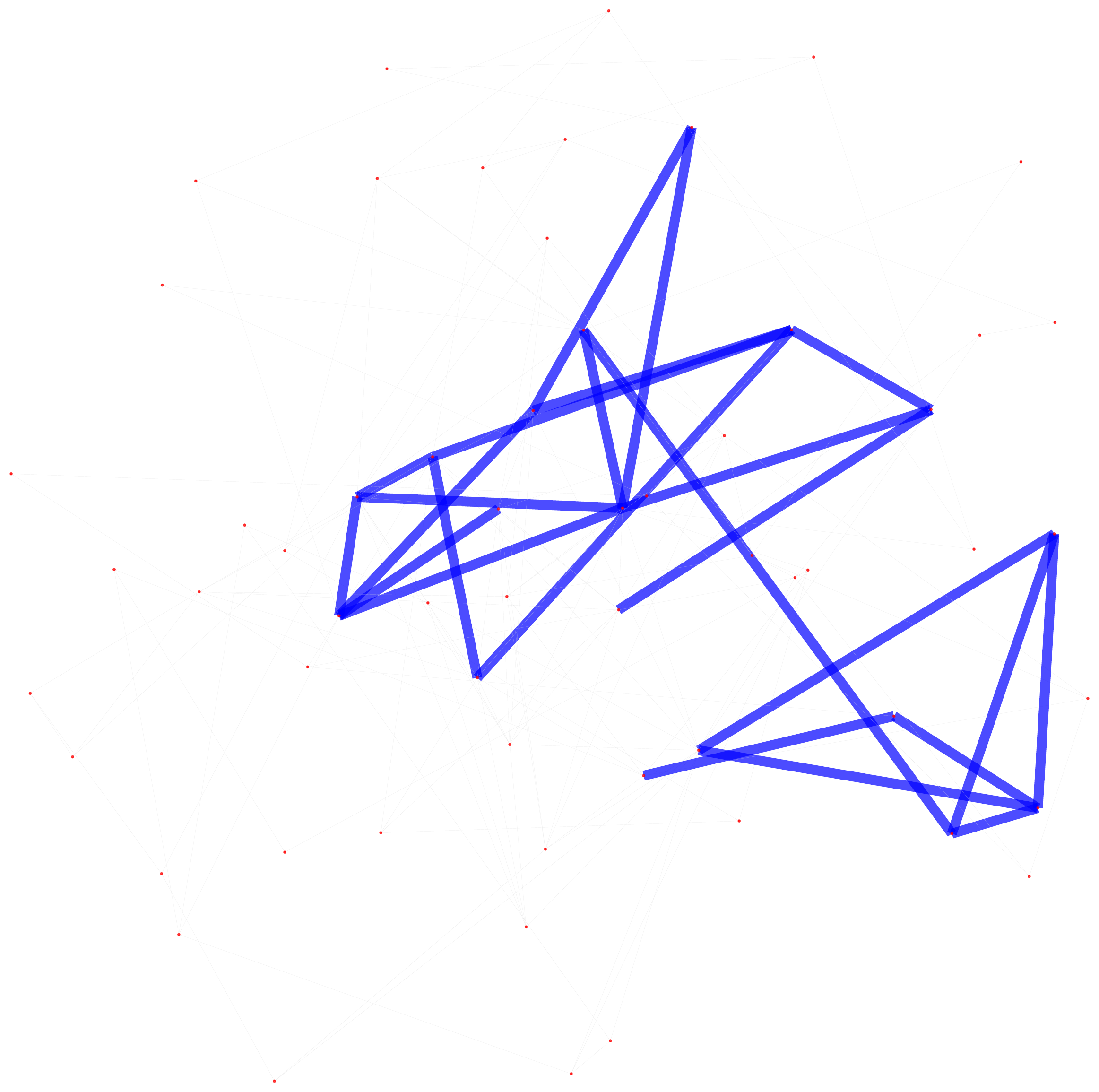}}\hspace{1mm}
\subfloat[  Target ]{\label{fig: path_1}\includegraphics[width=0.10\textwidth]{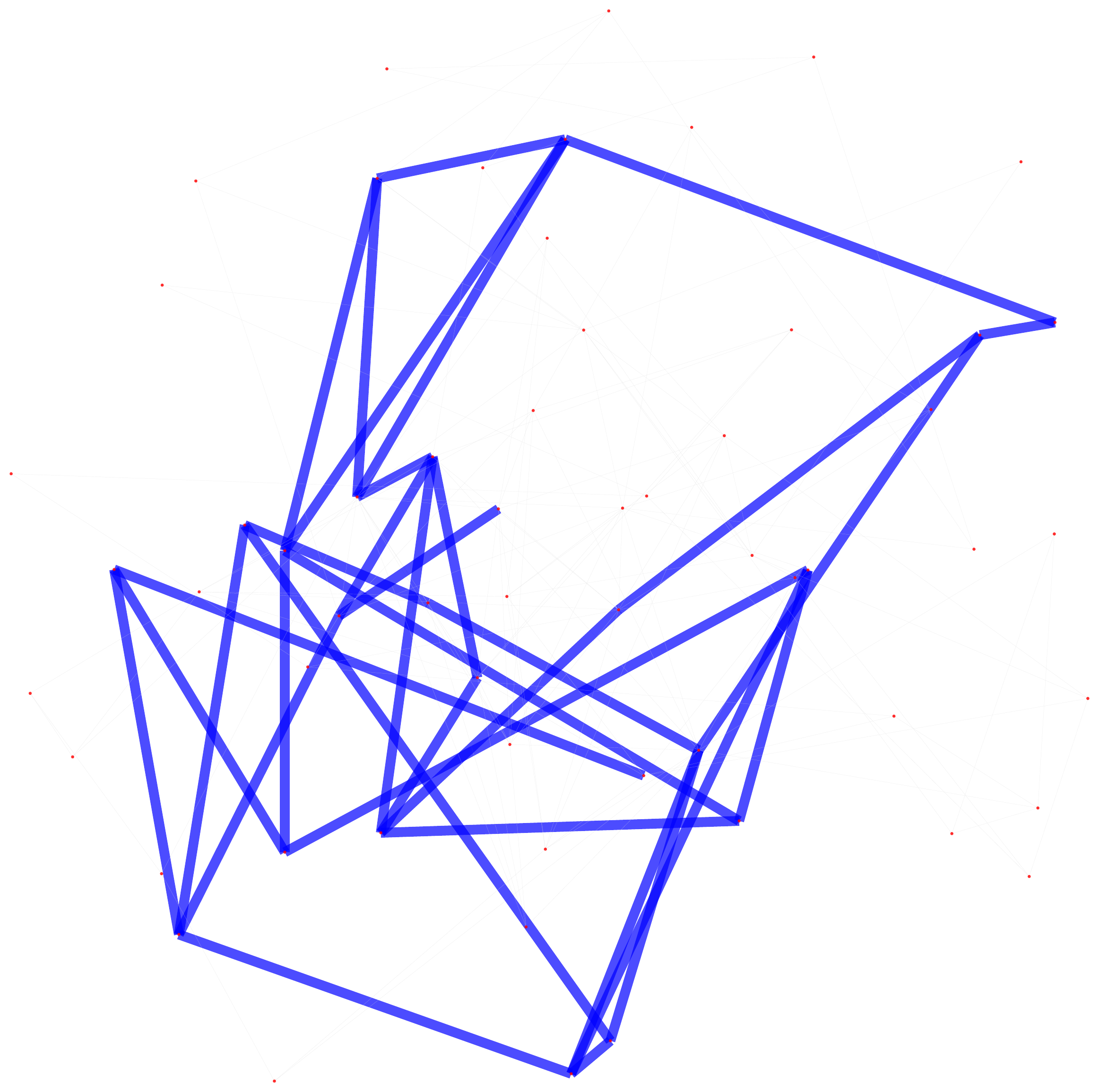}}\hspace{0mm}

\subfloat[  { [480, 0]} ]{\label{fig: path_1}\includegraphics[width=0.10\textwidth]{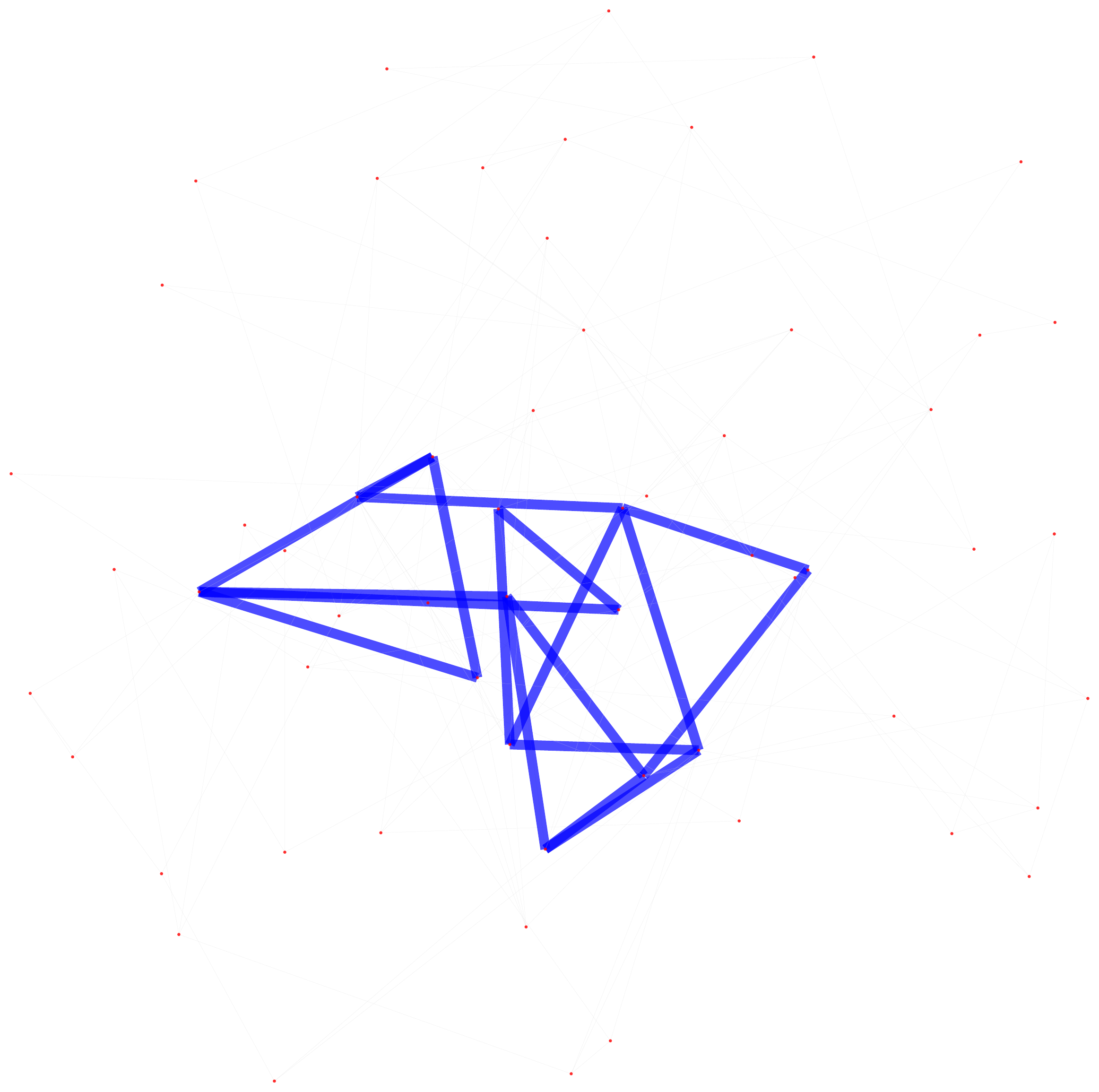}}\hspace{1mm}
\subfloat[  { [480, 1]} ]{\label{fig: path_1}\includegraphics[width=0.10\textwidth]{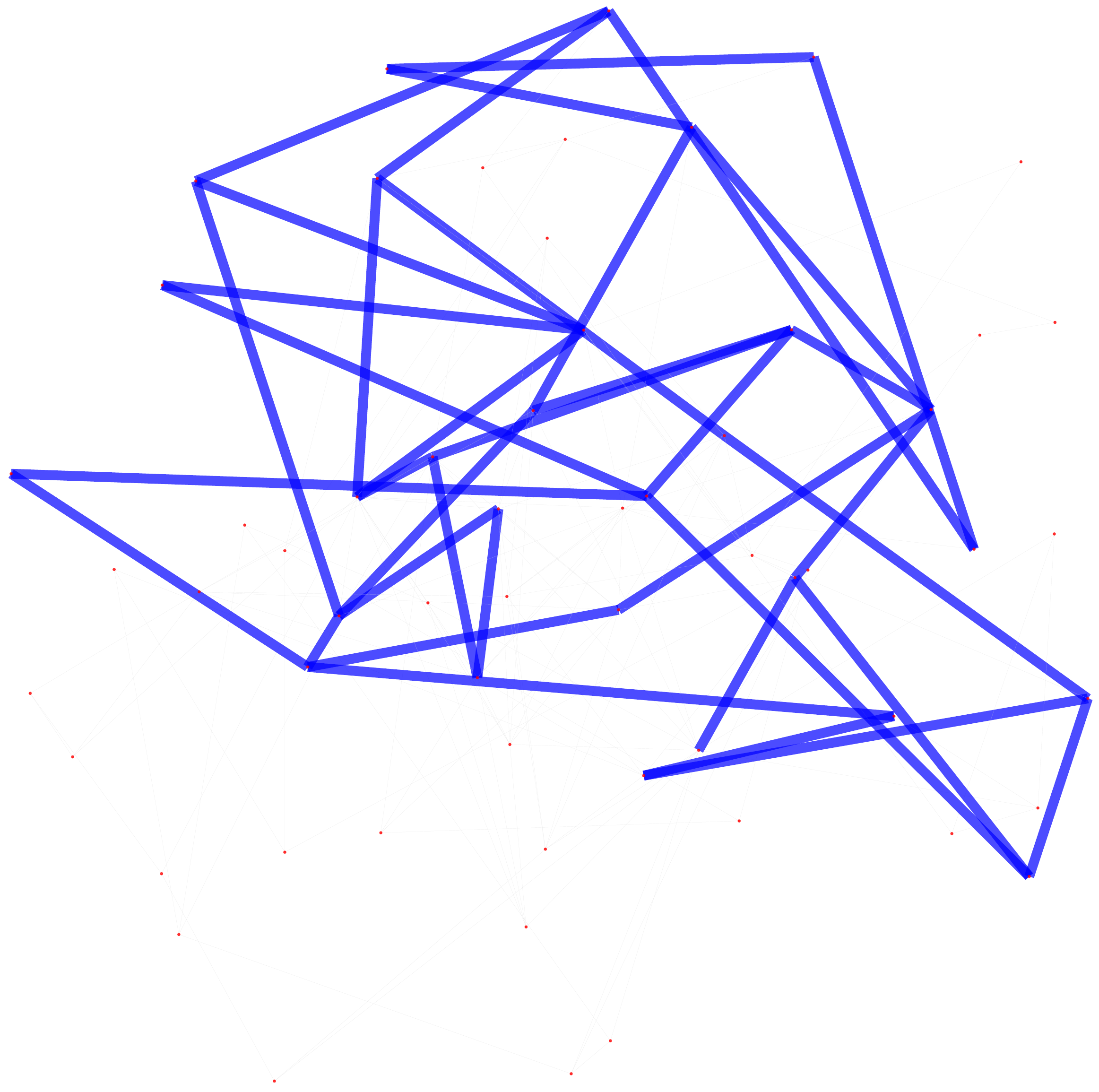}}\hspace{1mm}
\subfloat[  { [480, 2]} ]{\label{fig: path_1}\includegraphics[width=0.10\textwidth]{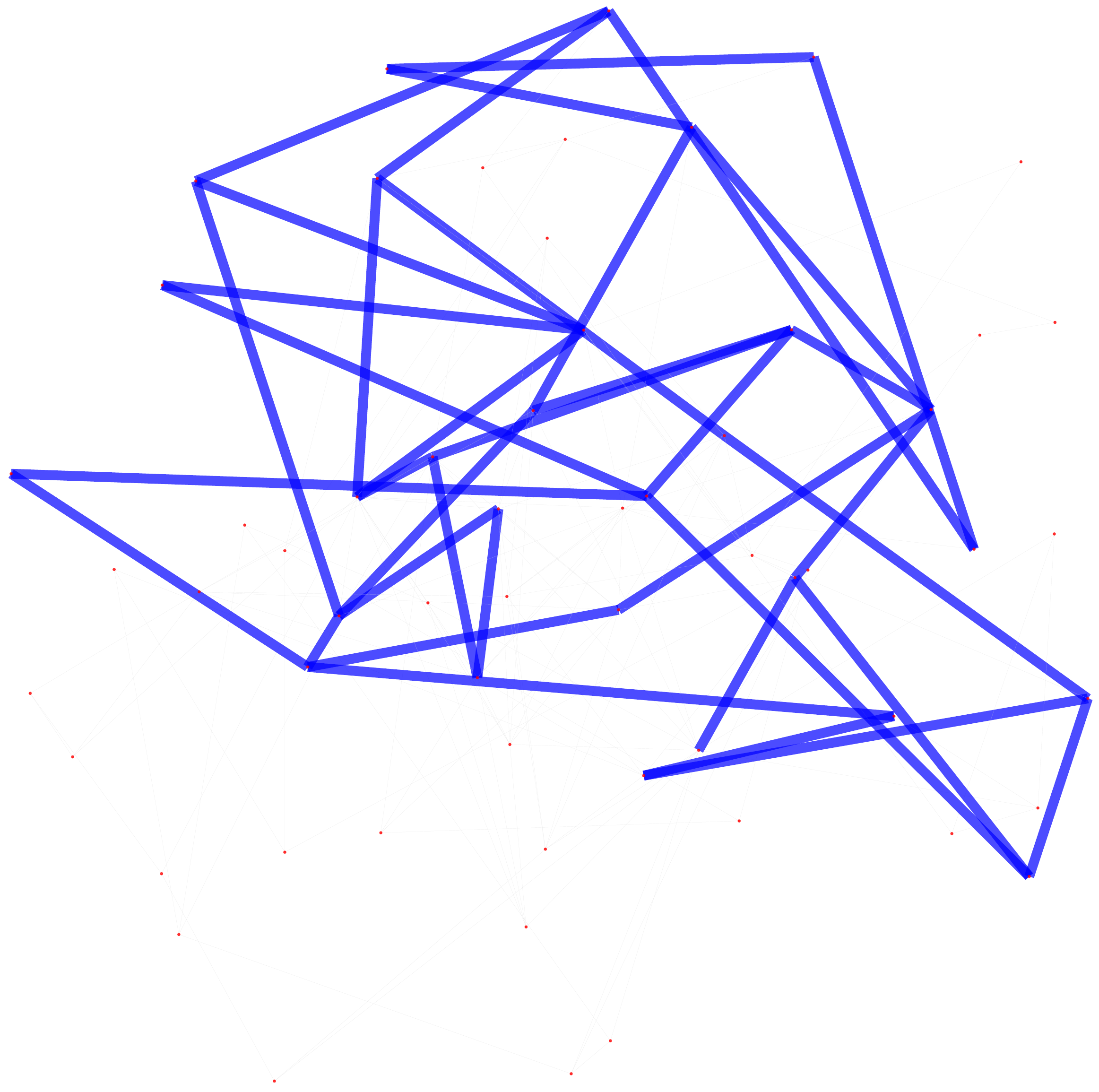}}\hspace{1mm}
\subfloat[  Target ]{\label{fig: path_1}\includegraphics[width=0.10\textwidth]{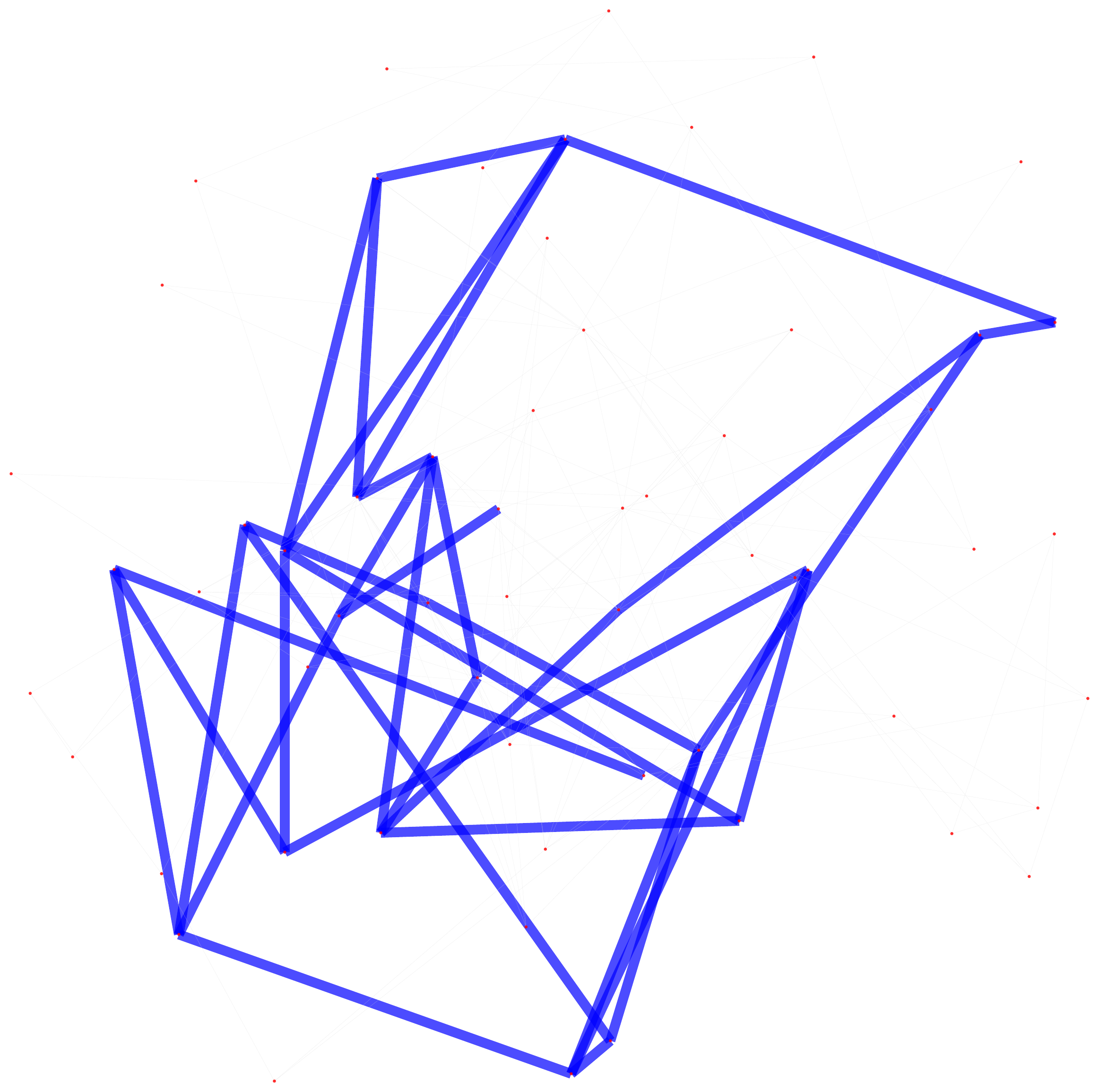}}\hspace{0mm}

\subfloat[  { [2400, 0]} ]{\label{fig: path_1}\includegraphics[width=0.10\textwidth]{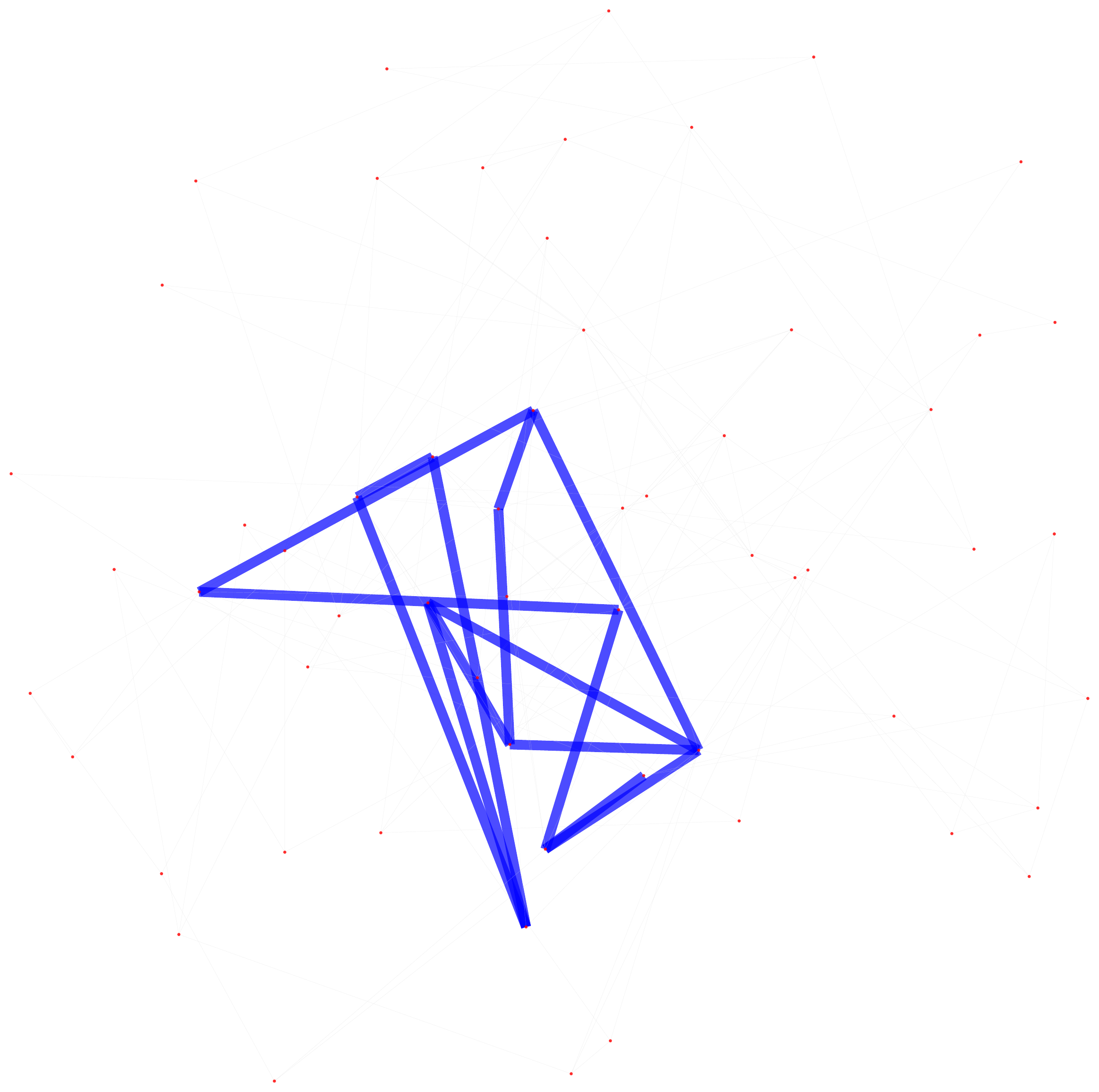}}\hspace{1mm}
\subfloat[  { [2400, 1]} ]{\label{fig: path_1}\includegraphics[width=0.10\textwidth]{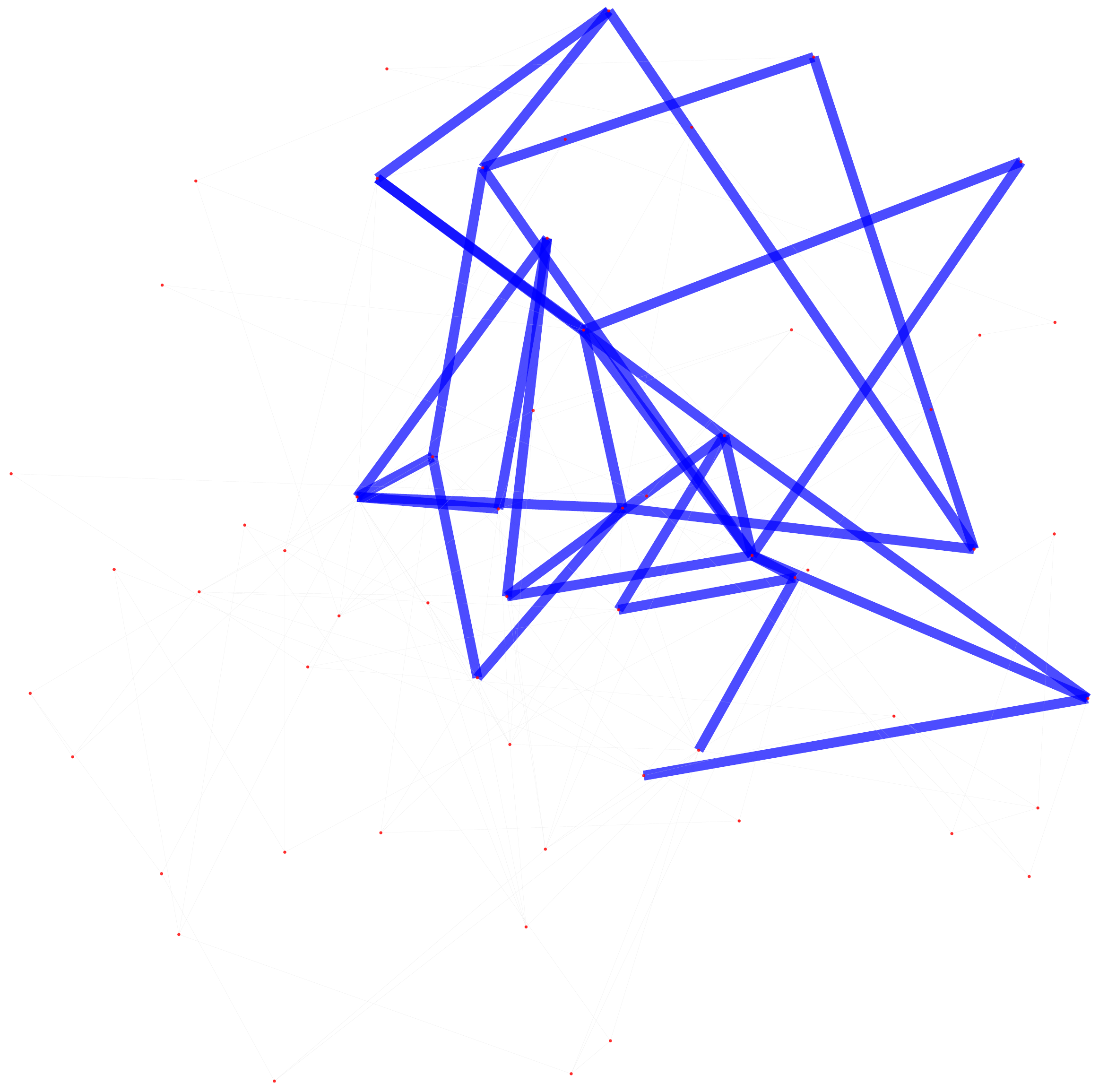}}\hspace{1mm}
\subfloat[  { [2400, 2]} ]{\label{fig: path_1}\includegraphics[width=0.10\textwidth]{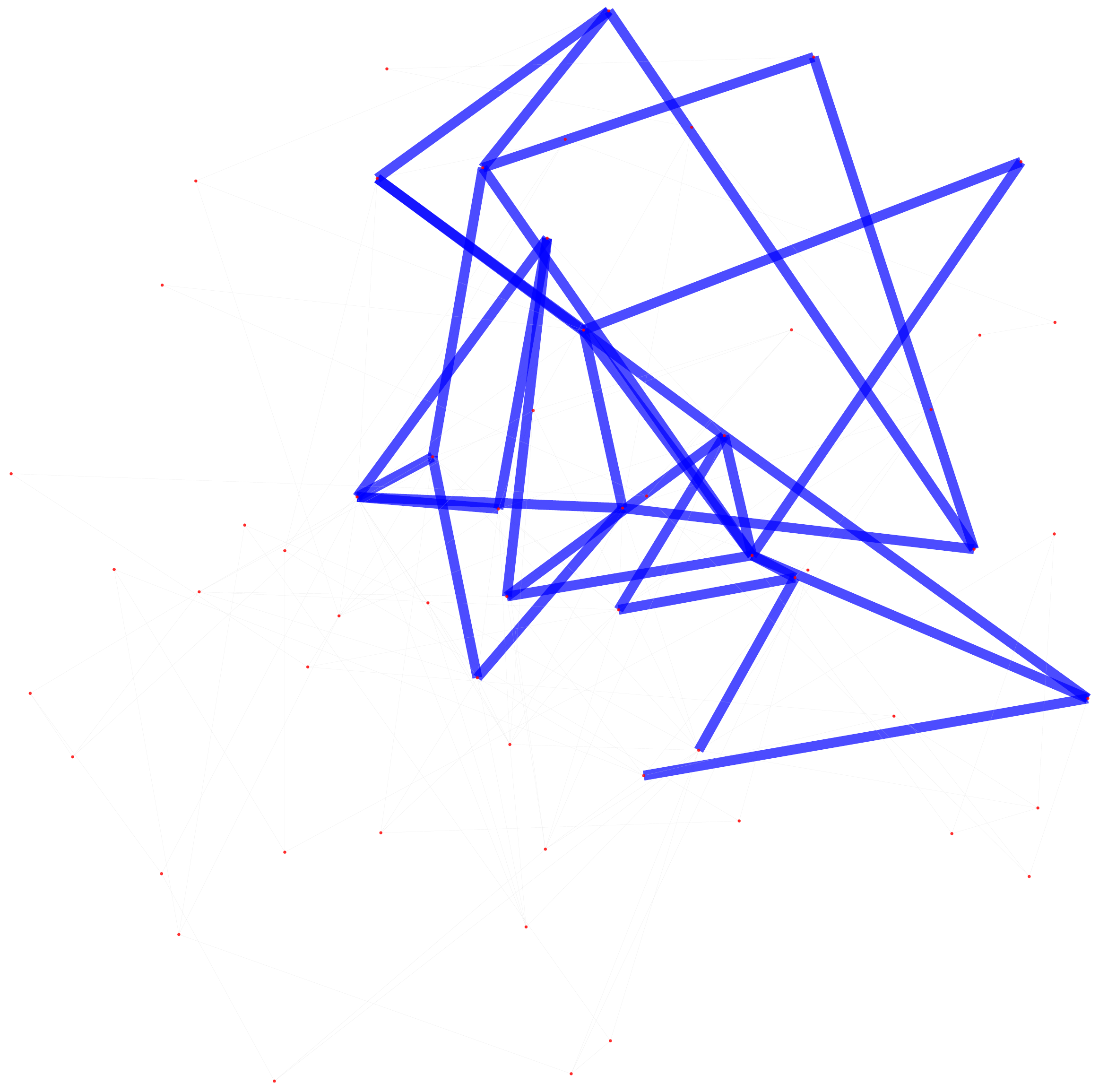}}\hspace{1mm}
\subfloat[  Target ]{\label{fig: path_1}\includegraphics[width=0.10\textwidth]{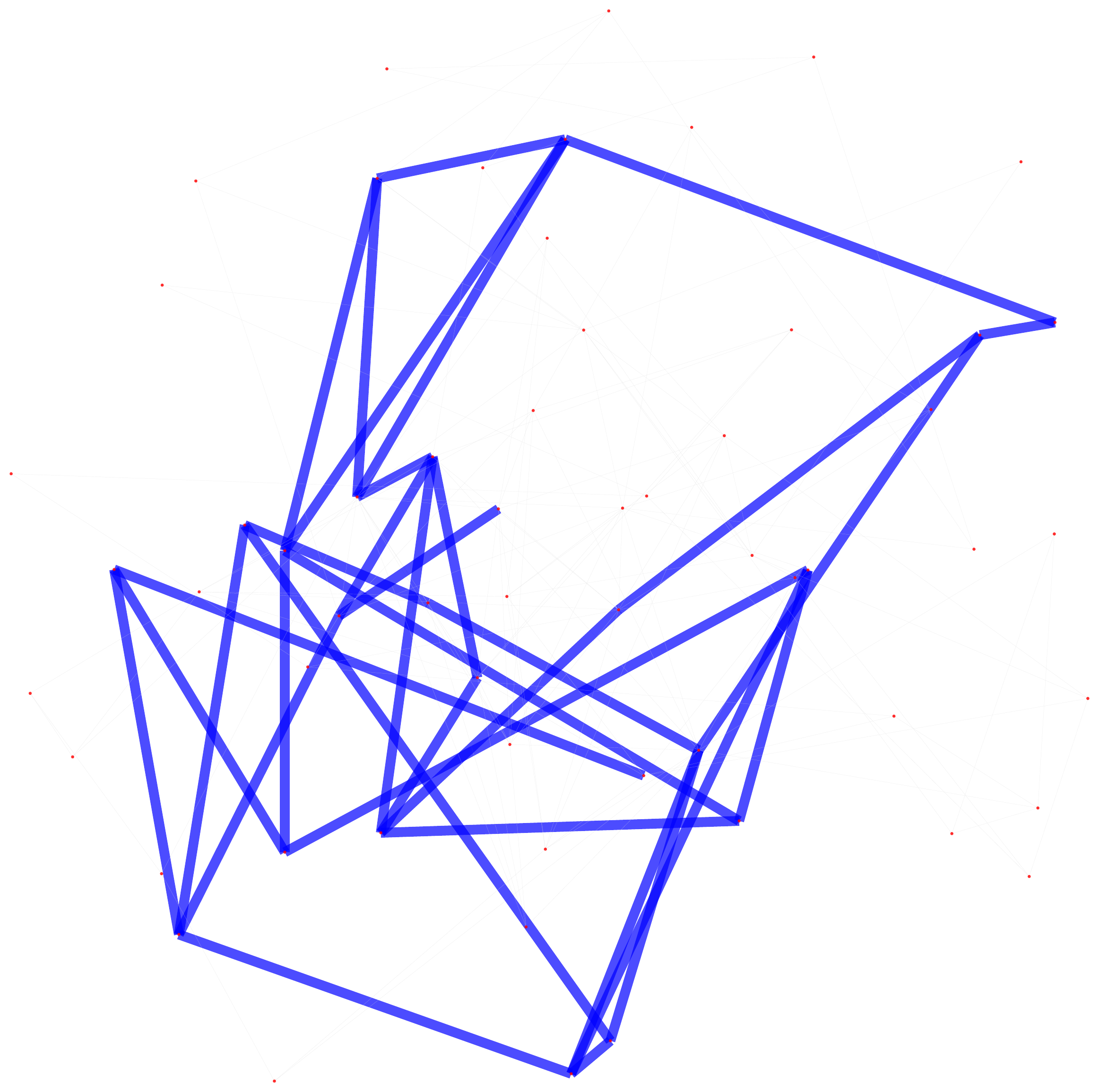}}\hspace{0mm}

\caption{\small Visualization of the results of QRTS-P for an example testing query for Path-to-Tree on Kro.}
\label{fig: more_tree_2}
\end{figure}

\begin{figure}[t]
\centering
\captionsetup[subfloat]{labelfont=scriptsize,textfont=scriptsize,labelformat=empty}
\subfloat[  { [30, 0]} ]{\label{fig: path_1}\includegraphics[width=0.10\textwidth]{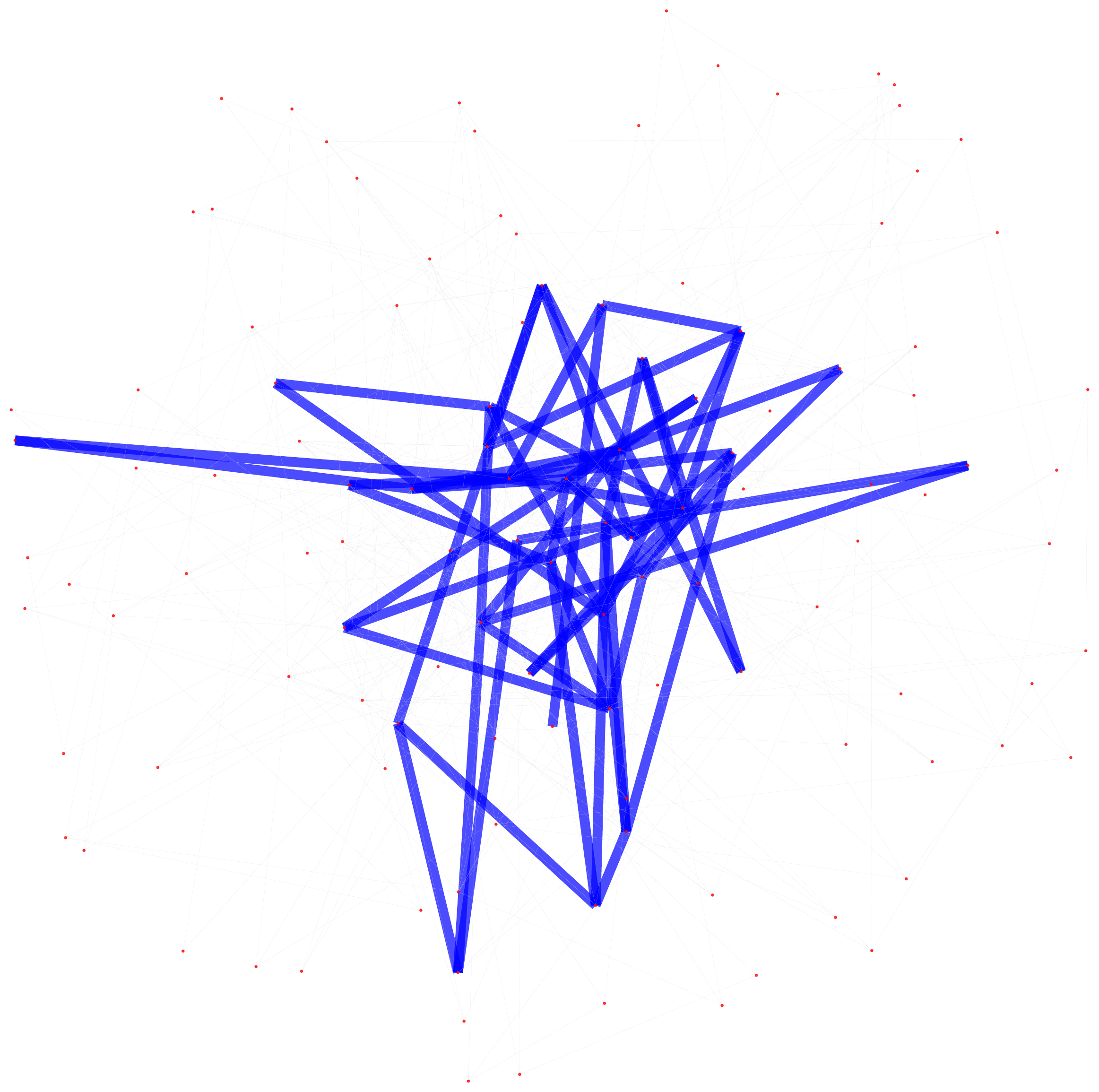}}\hspace{1mm}
\subfloat[  { [30, 1]} ]{\label{fig: path_1}\includegraphics[width=0.10\textwidth]{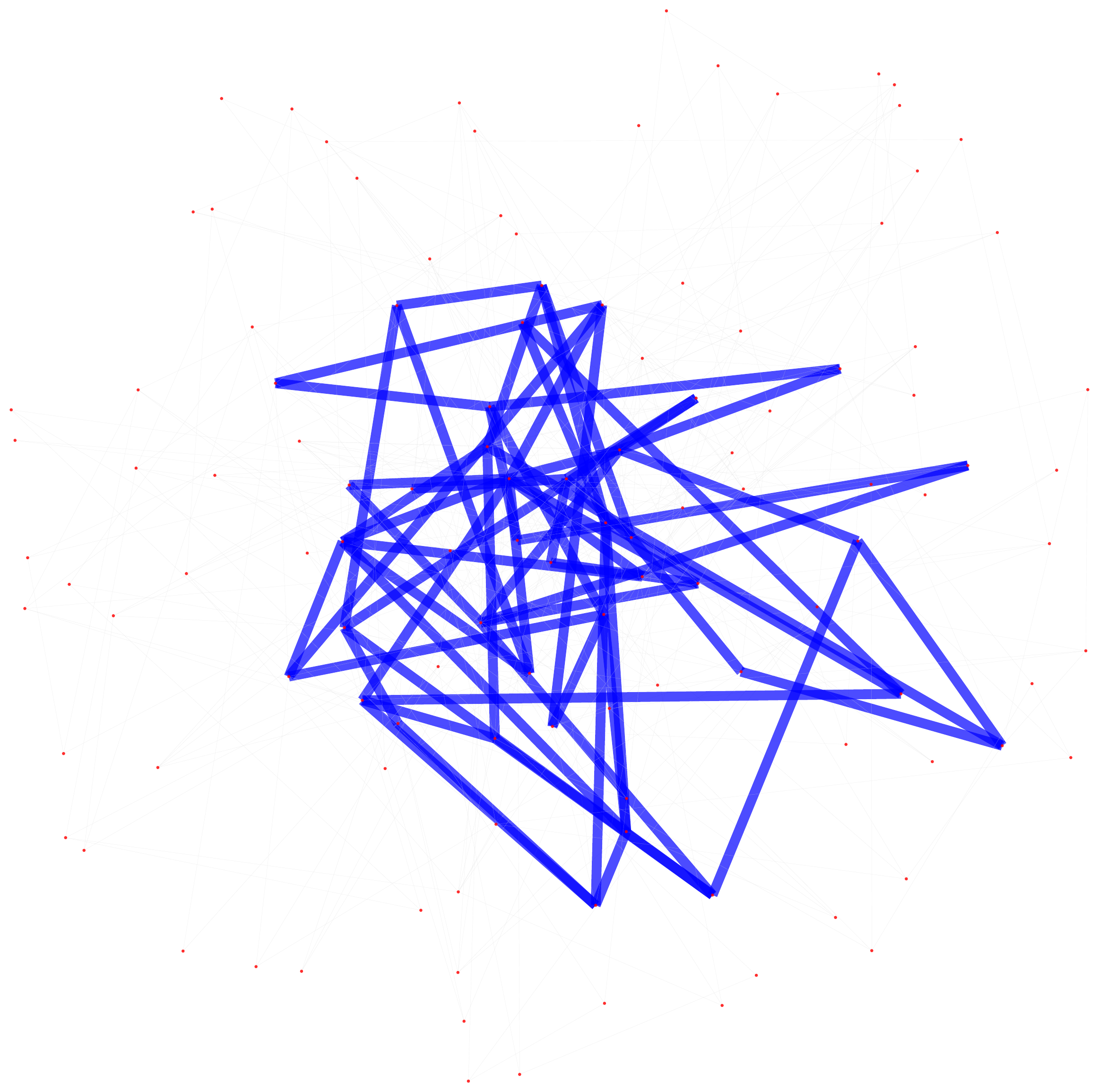}}\hspace{1mm}
\subfloat[  { [30, 2]} ]{\label{fig: path_1}\includegraphics[width=0.10\textwidth]{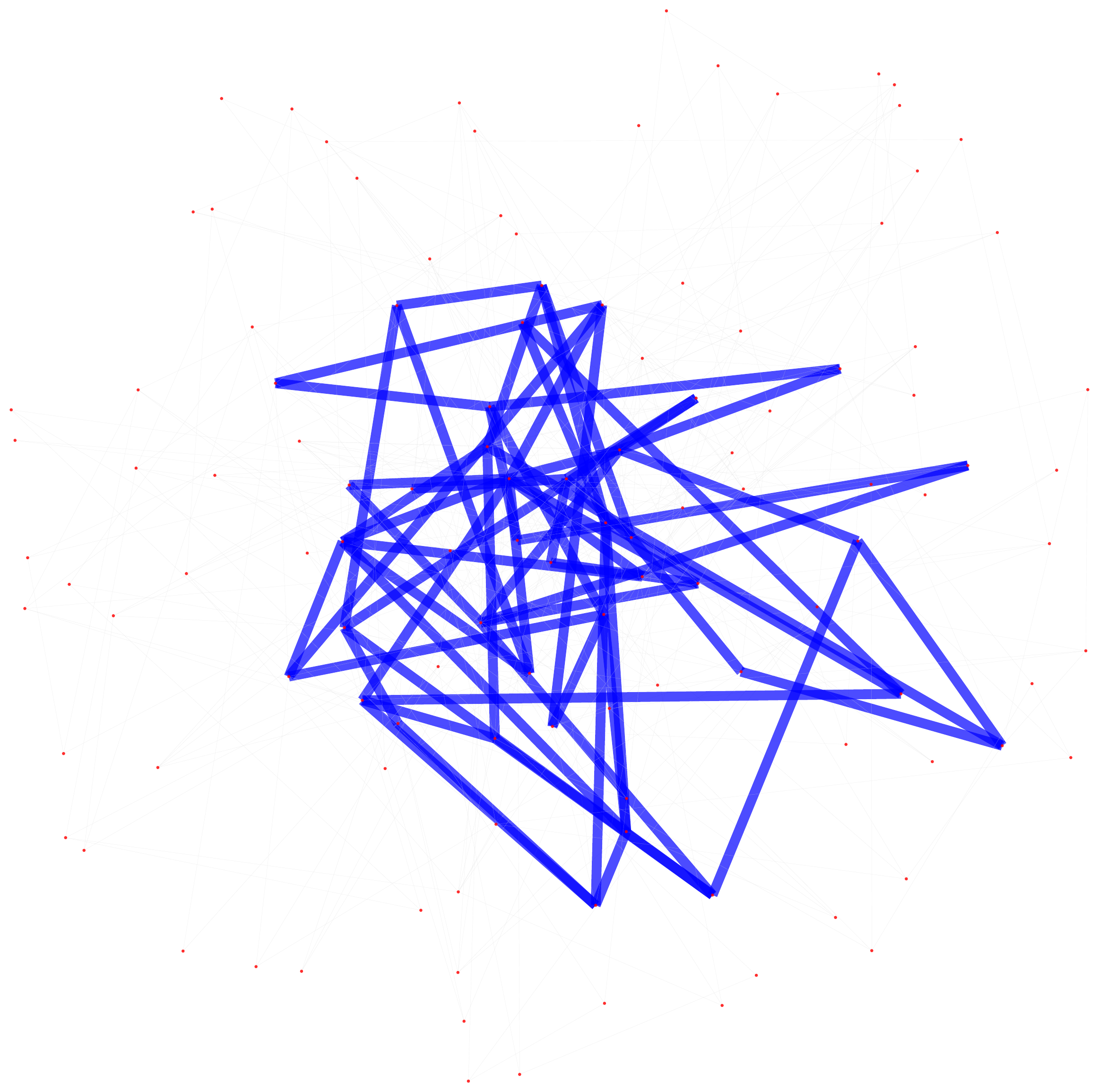}}\hspace{1mm}
\subfloat[  Target ]{\label{fig: path_1}\includegraphics[width=0.10\textwidth]{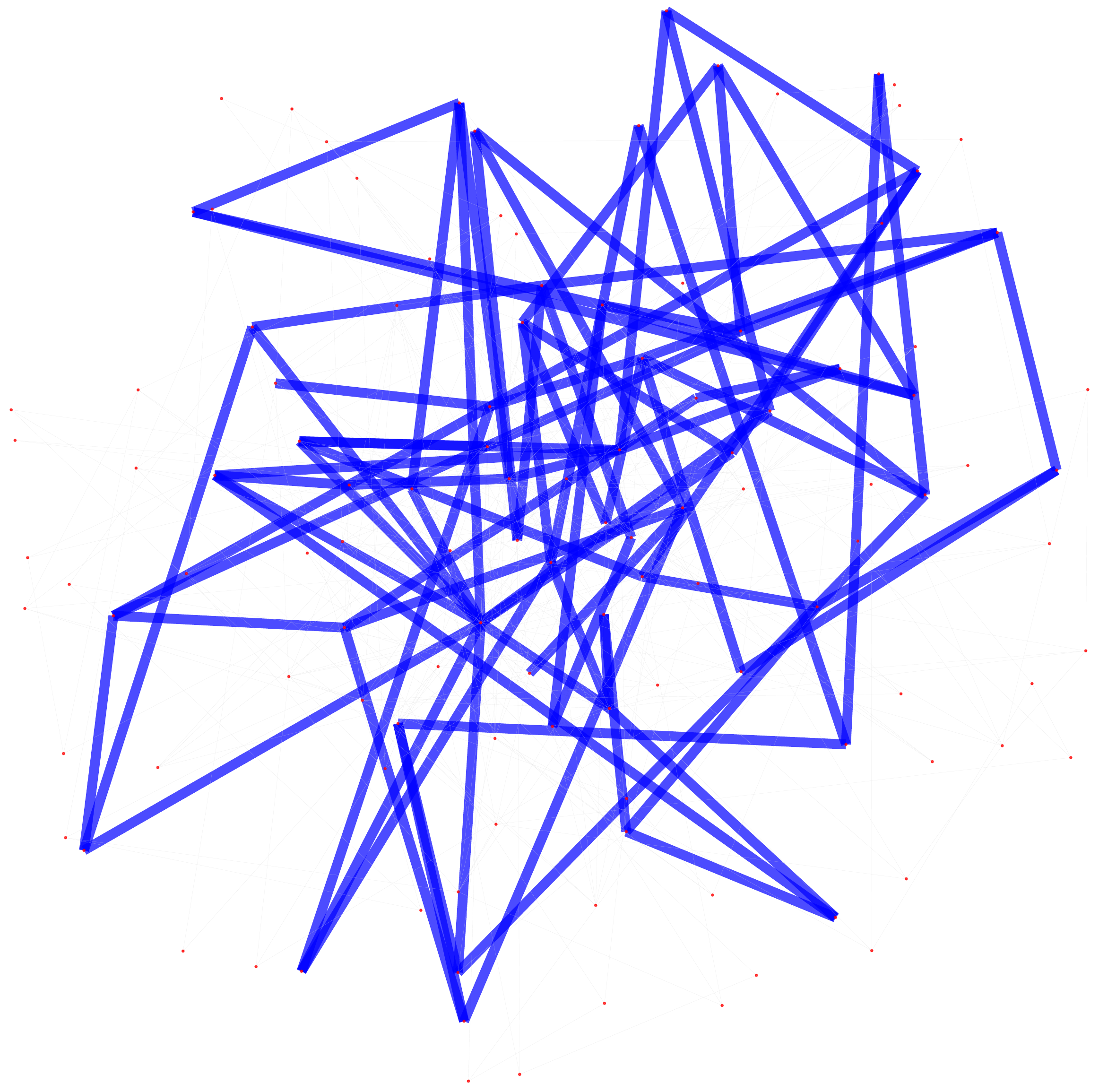}}\hspace{0mm}


\subfloat[  { [240, 0]} ]{\label{fig: path_1}\includegraphics[width=0.10\textwidth]{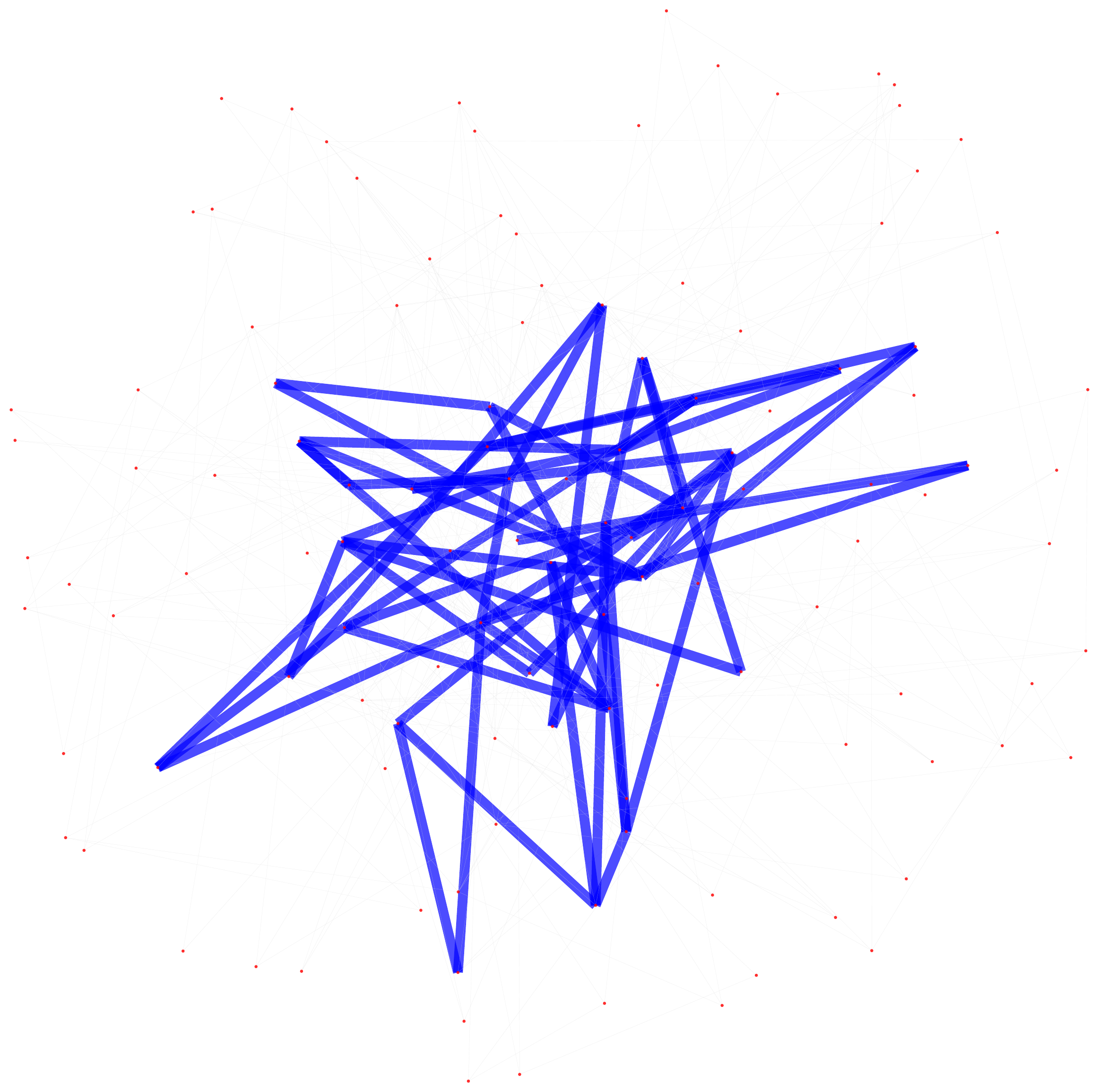}}\hspace{1mm}
\subfloat[  { [240, 1]} ]{\label{fig: path_1}\includegraphics[width=0.10\textwidth]{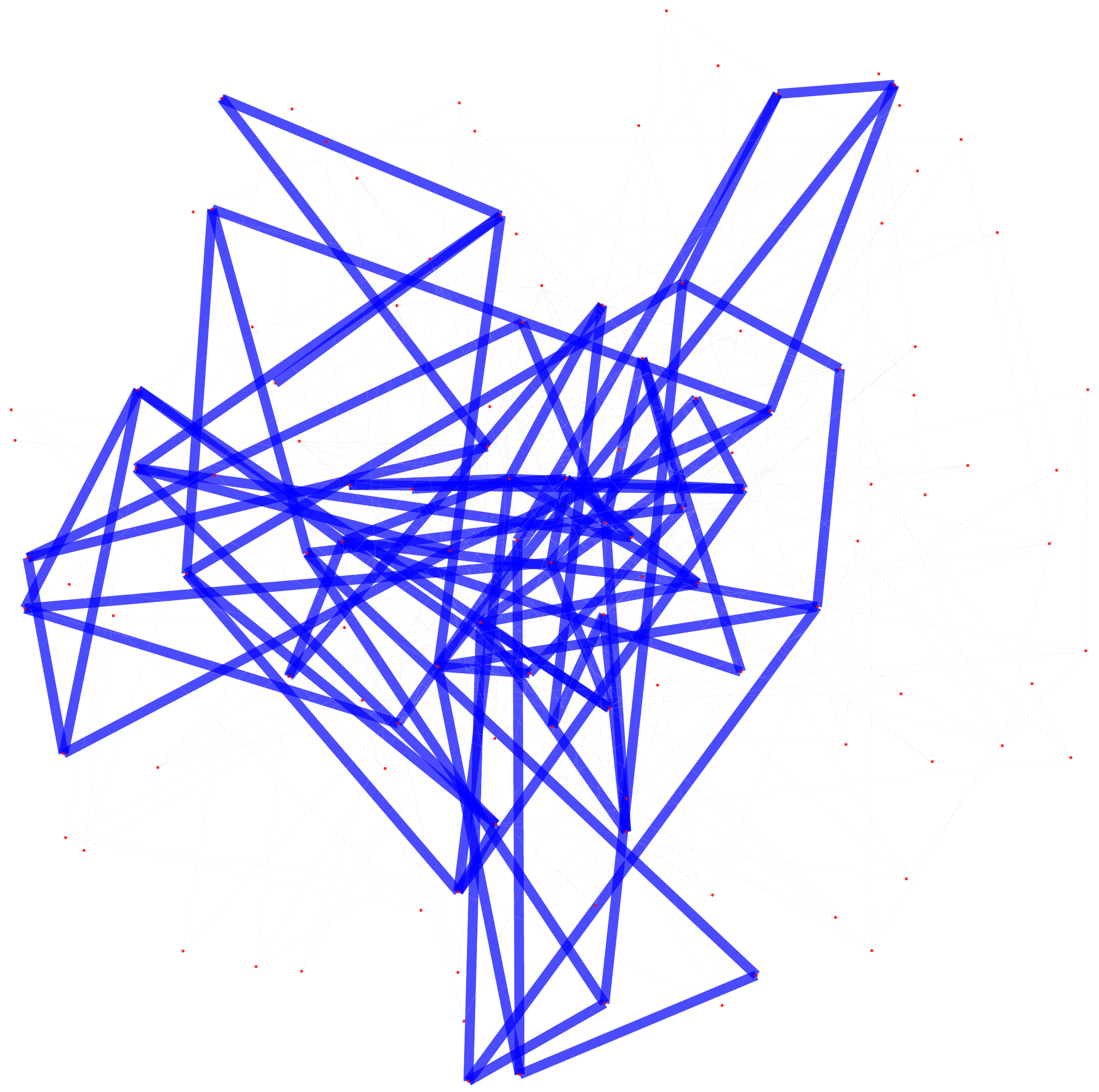}}\hspace{1mm}
\subfloat[  { [240, 2]} ]{\label{fig: path_1}\includegraphics[width=0.10\textwidth]{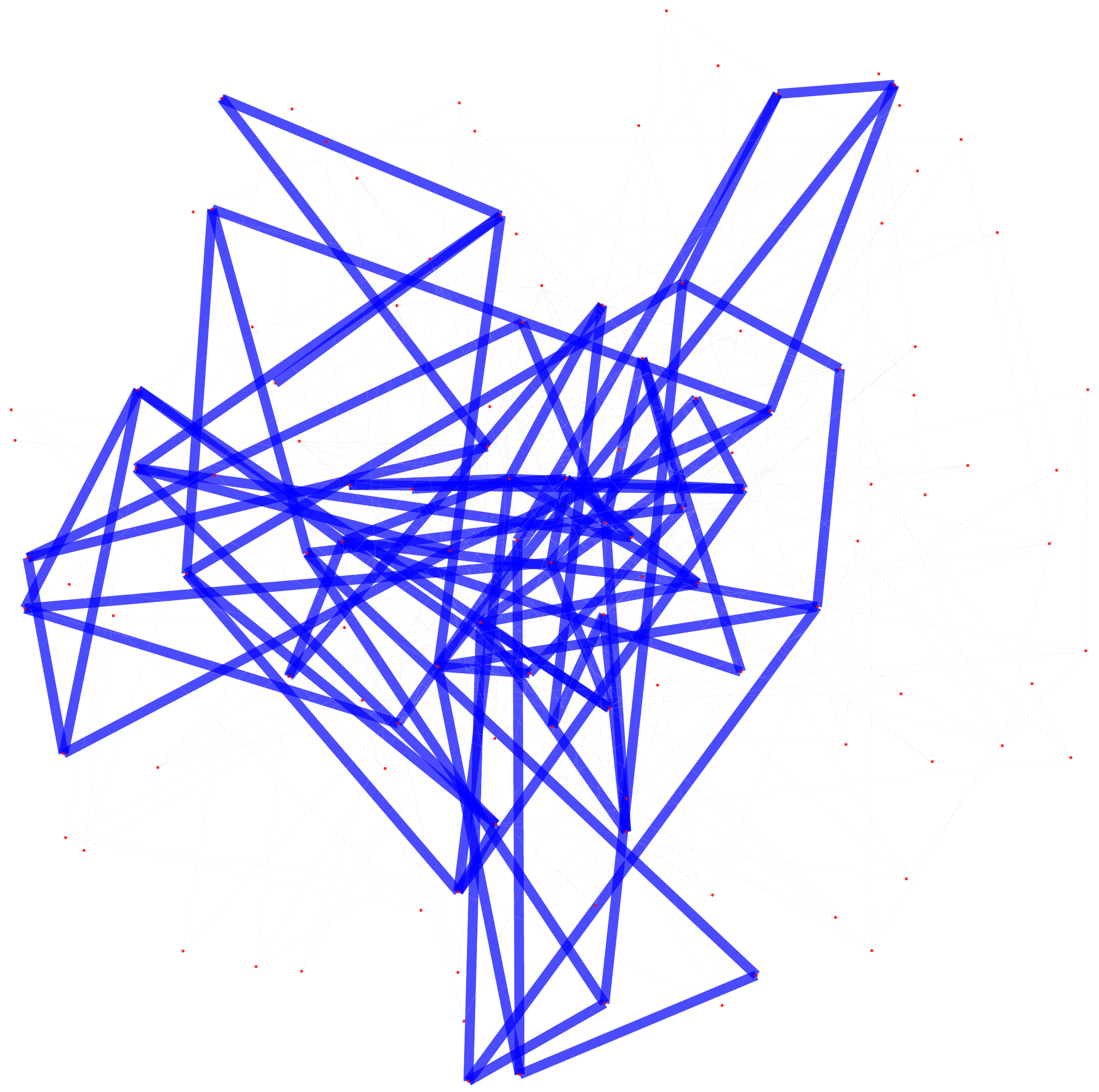}}\hspace{1mm}
\subfloat[  Target ]{\label{fig: path_1}\includegraphics[width=0.10\textwidth]{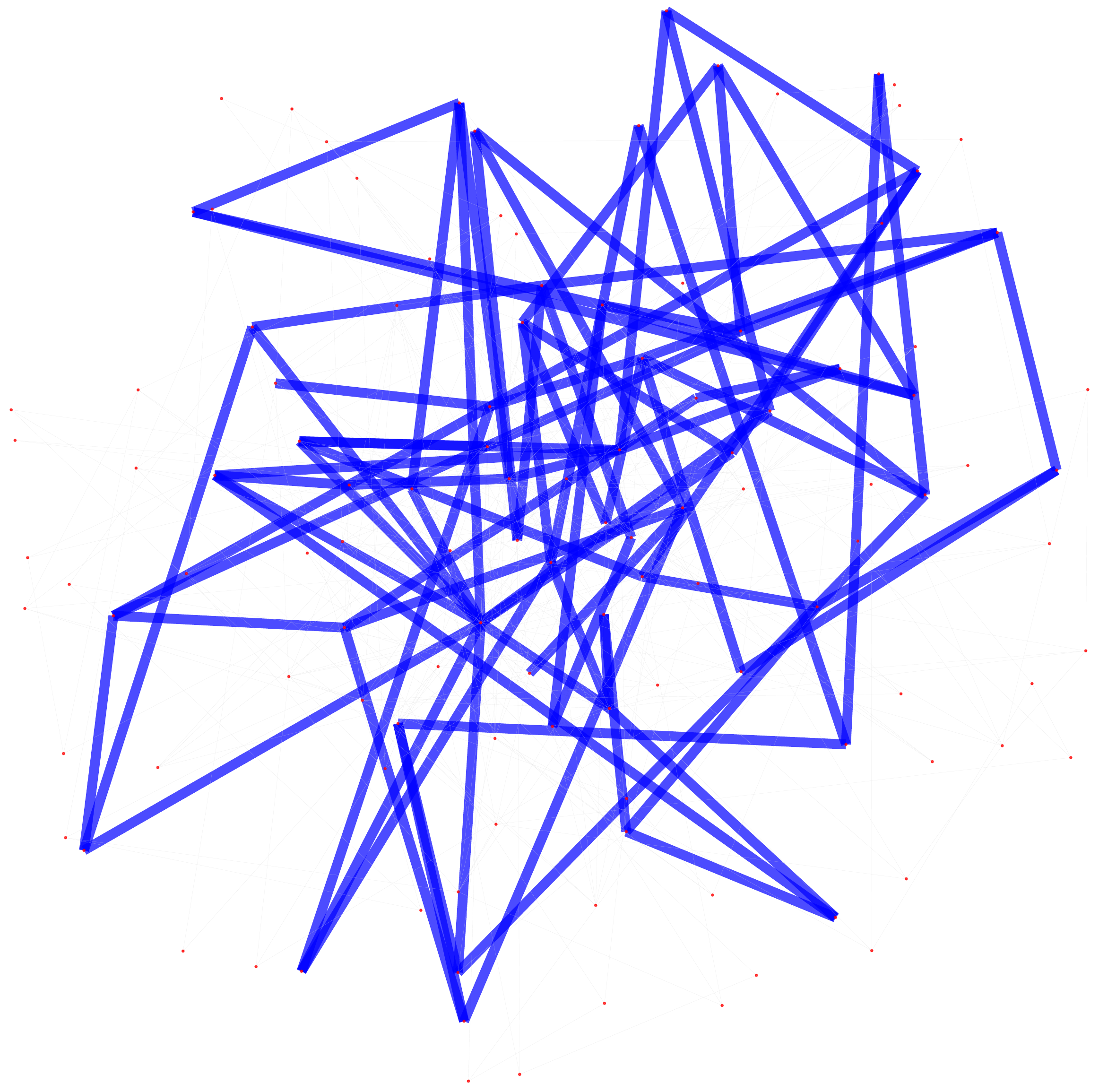}}\hspace{0mm}

\subfloat[  { [480, 0]} ]{\label{fig: path_1}\includegraphics[width=0.10\textwidth]{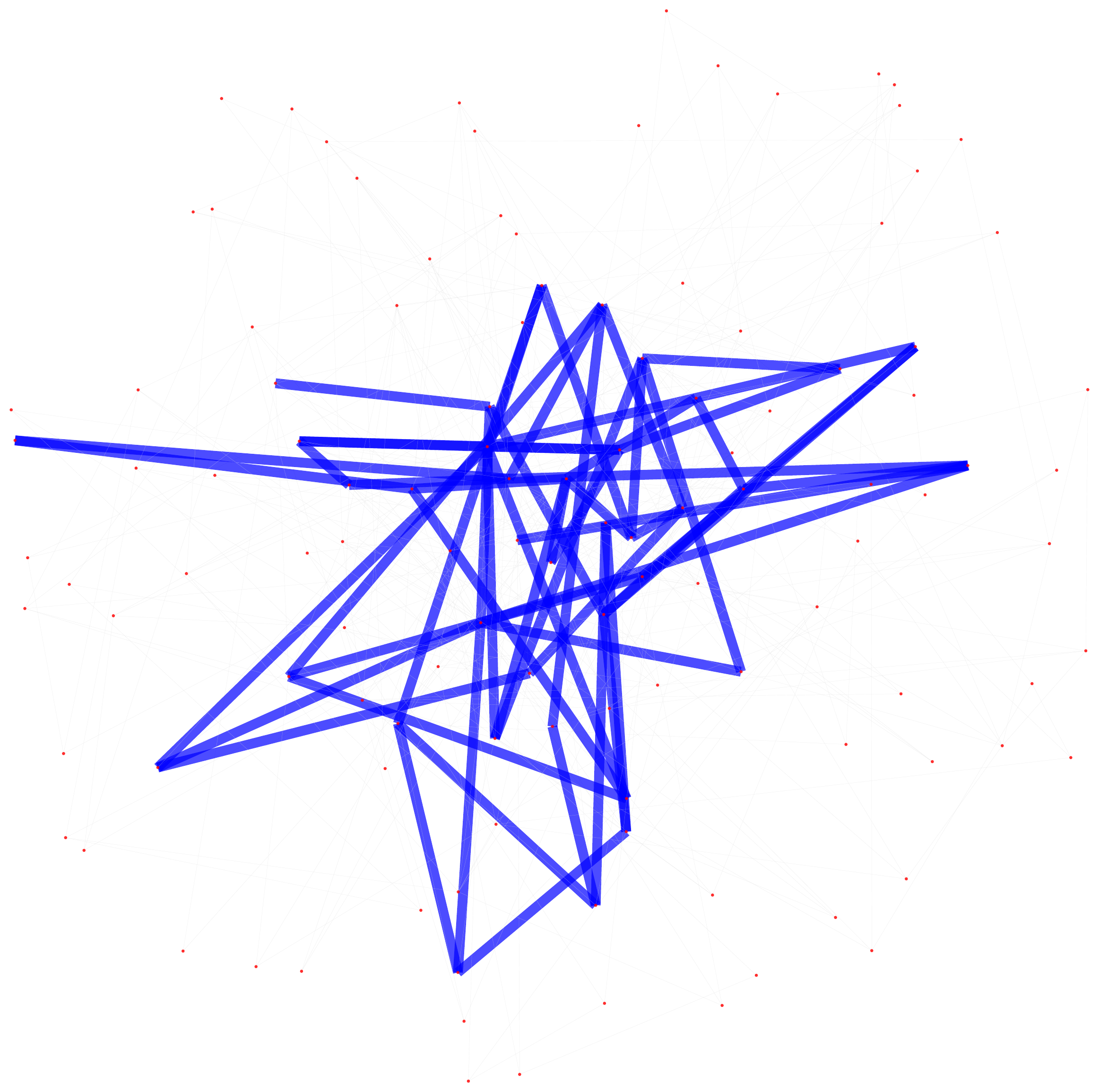}}\hspace{1mm}
\subfloat[  { [480, 1]} ]{\label{fig: path_1}\includegraphics[width=0.10\textwidth]{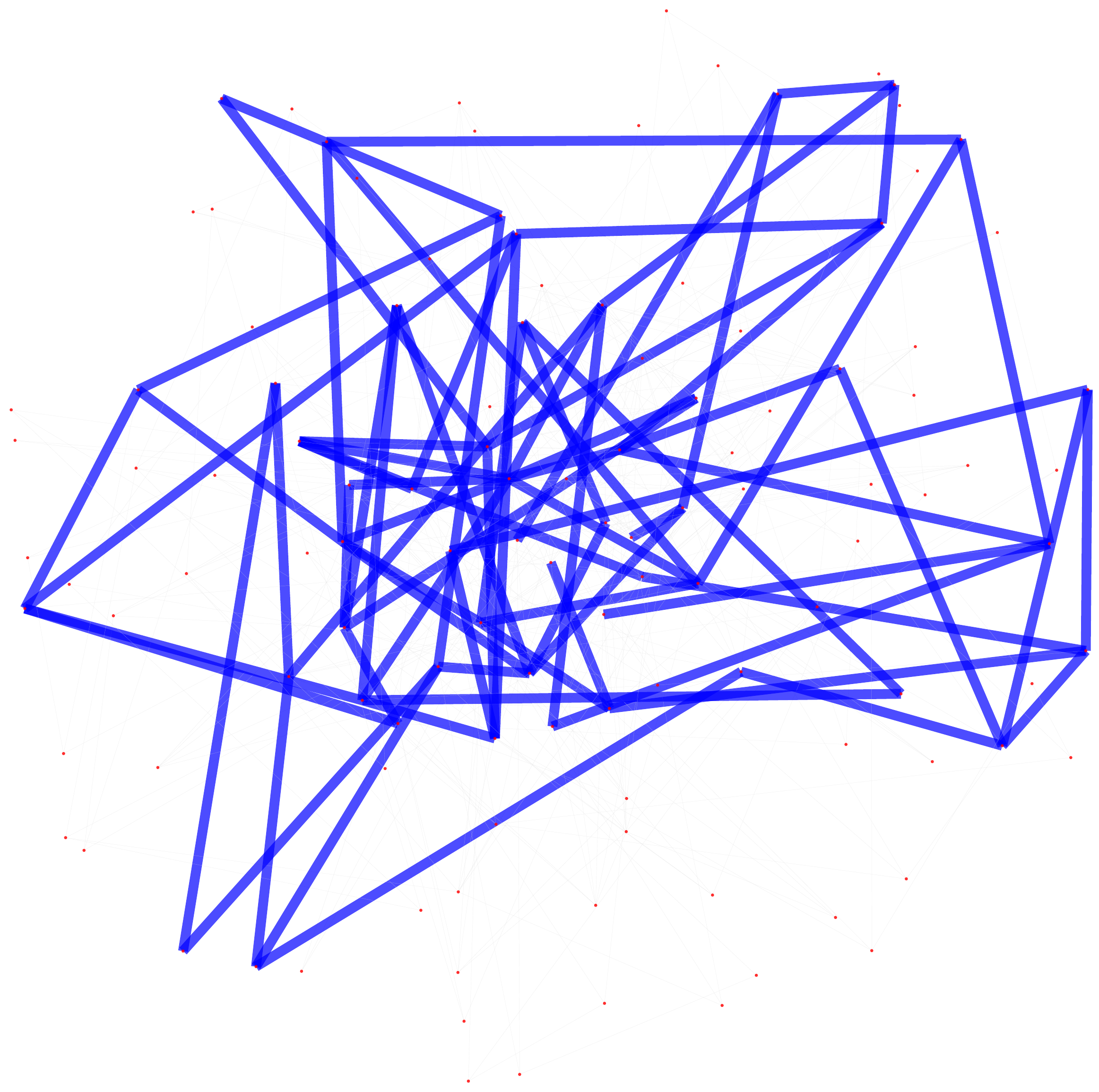}}\hspace{1mm}
\subfloat[  { [480, 2]} ]{\label{fig: path_1}\includegraphics[width=0.10\textwidth]{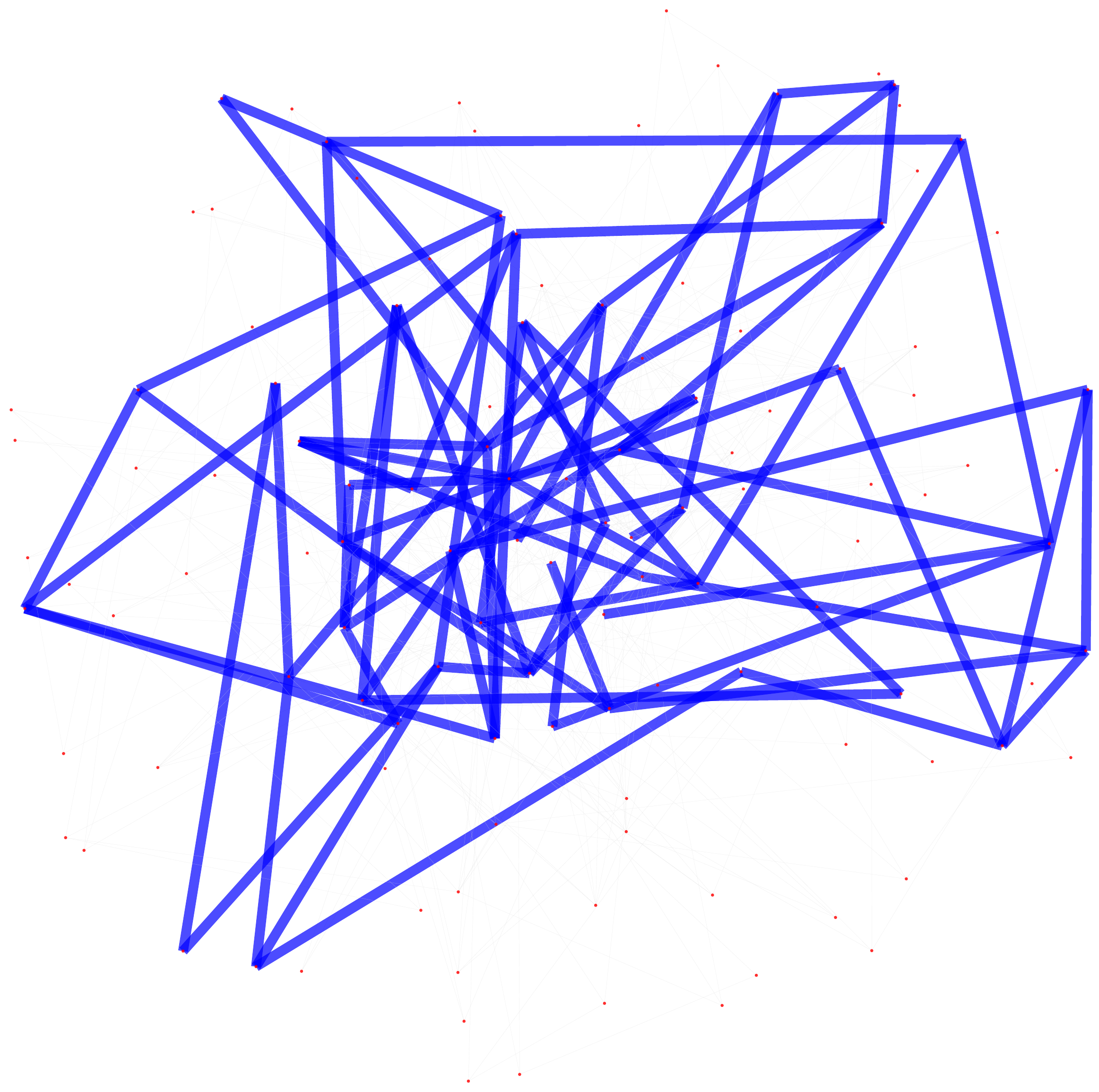}}\hspace{1mm}
\subfloat[  Target ]{\label{fig: path_1}\includegraphics[width=0.10\textwidth]{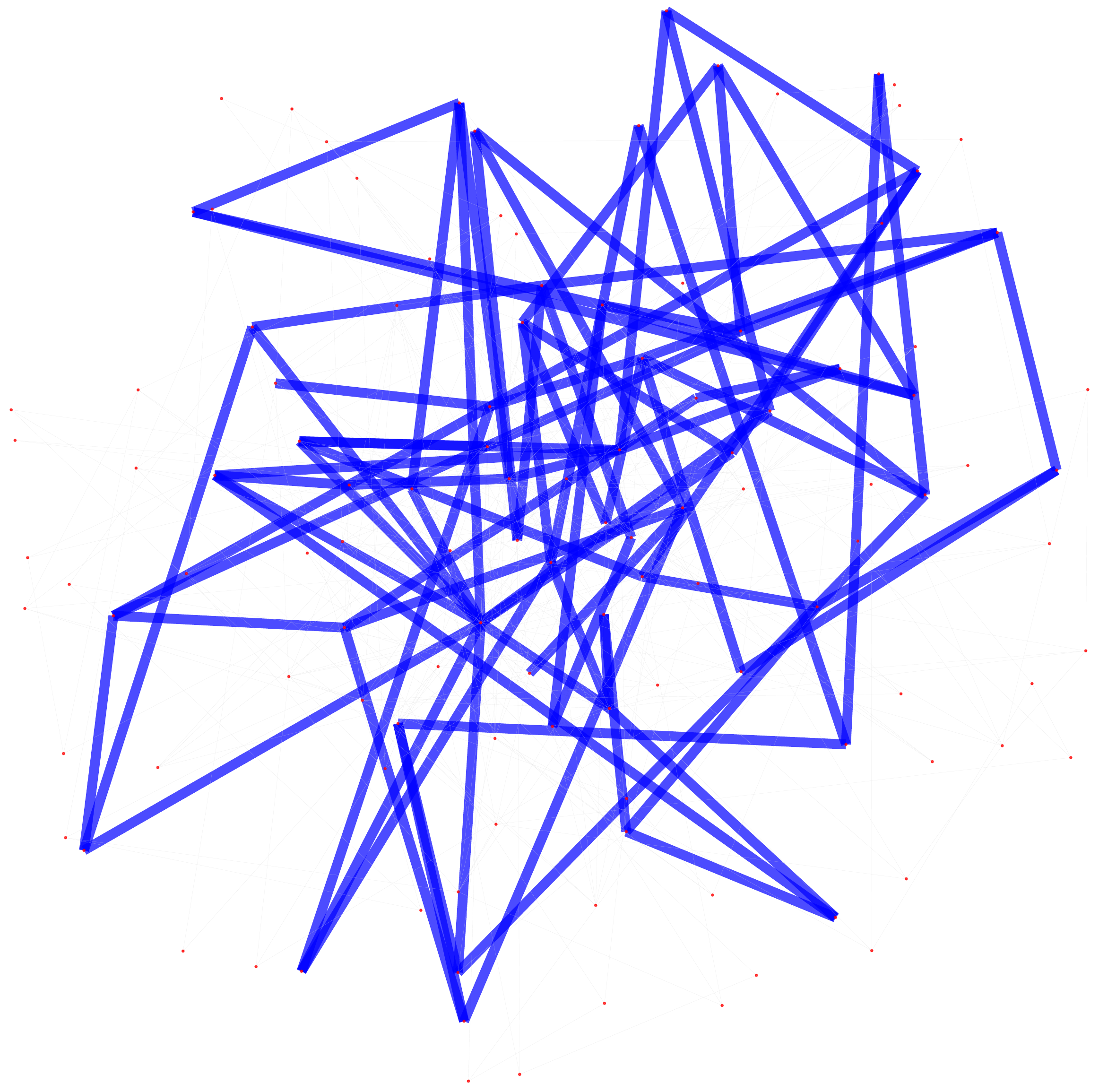}}\hspace{0mm}

\subfloat[  { [2400, 0]} ]{\label{fig: path_1}\includegraphics[width=0.10\textwidth]{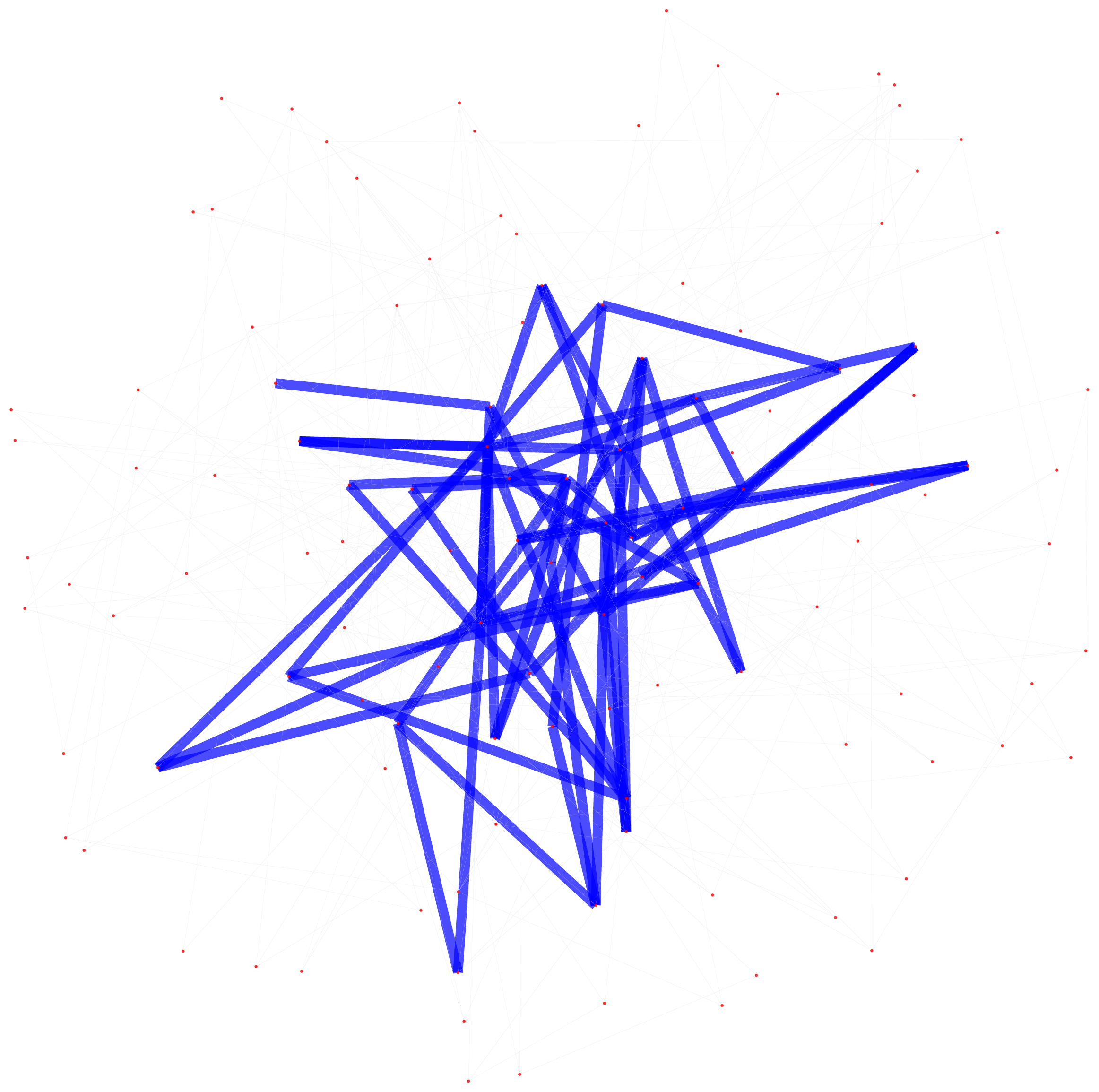}}\hspace{1mm}
\subfloat[  { [2400, 1]} ]{\label{fig: path_1}\includegraphics[width=0.10\textwidth]{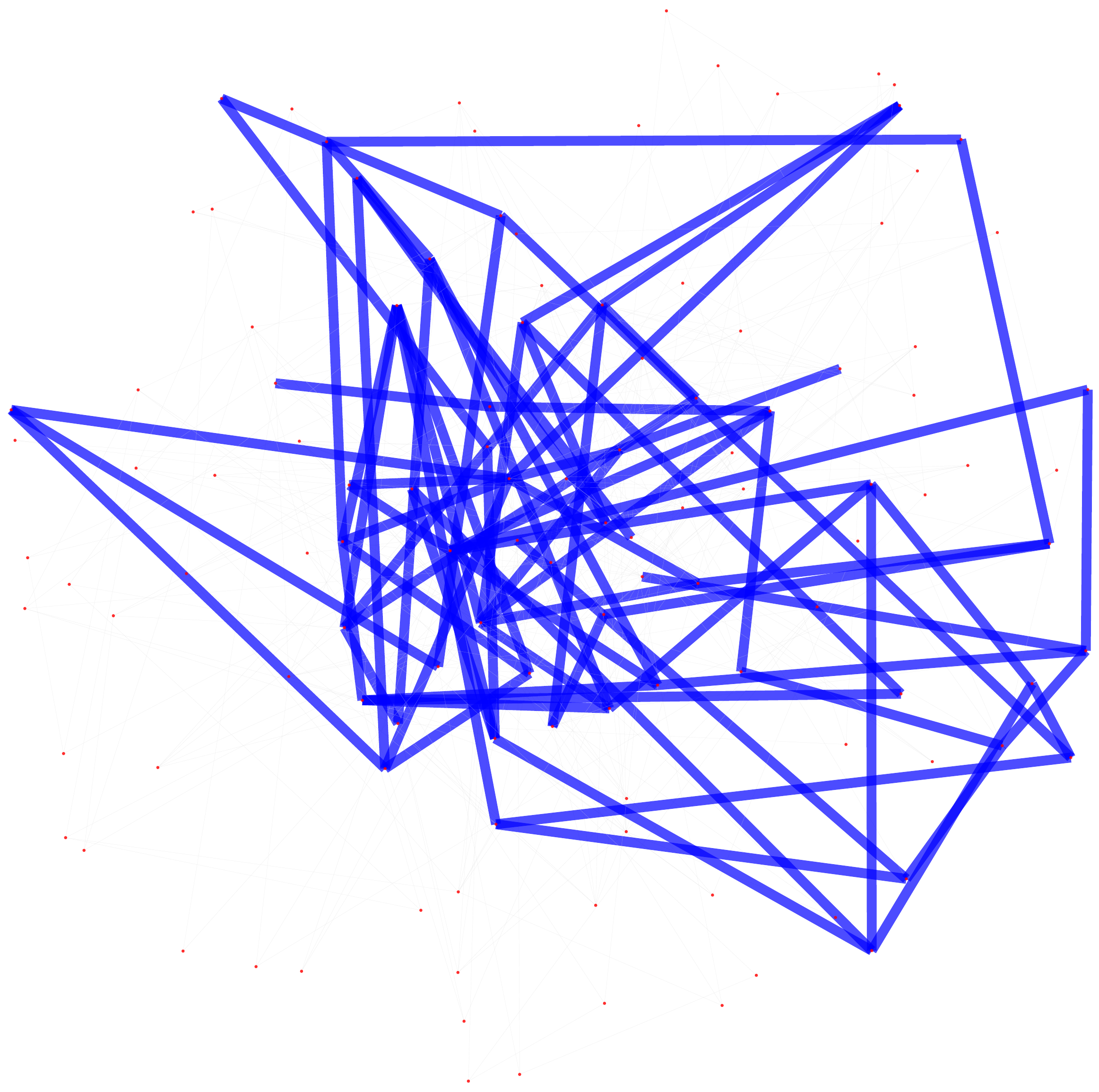}}\hspace{1mm}
\subfloat[  { [2400, 2]} ]{\label{fig: path_1}\includegraphics[width=0.10\textwidth]{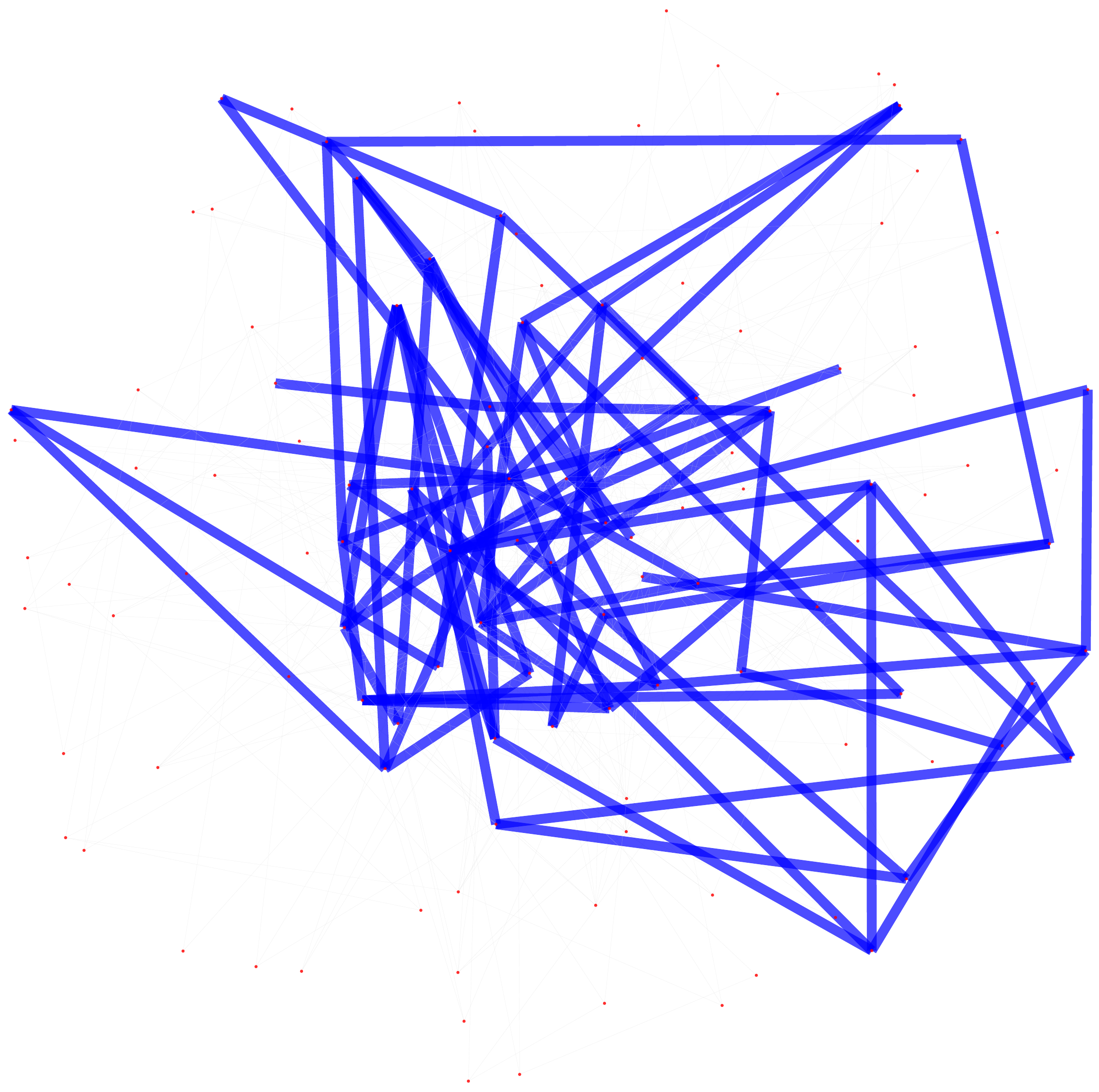}}\hspace{1mm}
\subfloat[  Target ]{\label{fig: path_1}\includegraphics[width=0.10\textwidth]{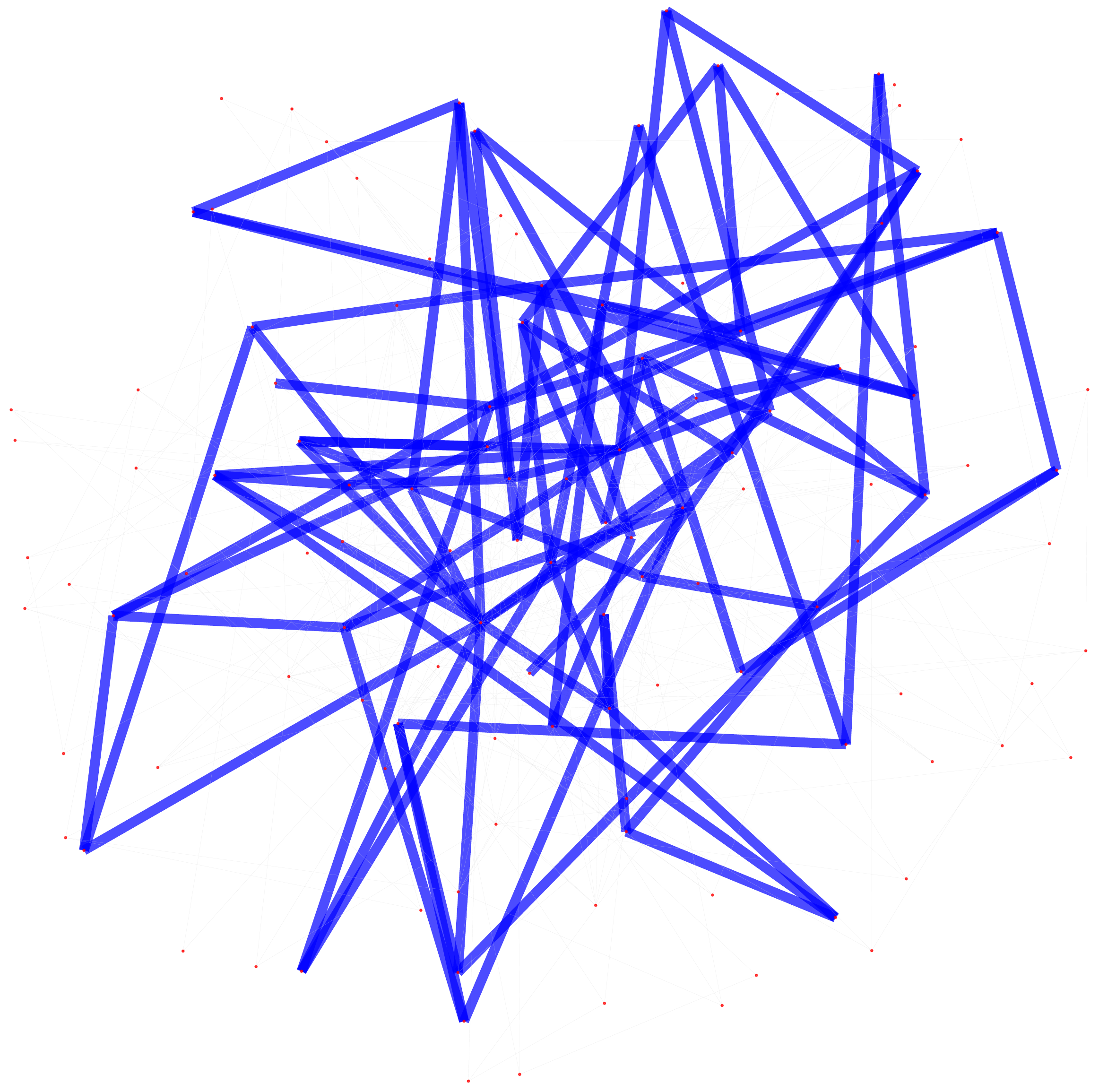}}\hspace{0mm}

\caption{\small Visualization of the results of QRTS-P for an example testing query for Path-to-Tree on Kro.}
\label{fig: more_tree_3}
\end{figure}

\begin{figure}[t]
\centering
\captionsetup[subfloat]{labelfont=scriptsize,textfont=scriptsize,labelformat=empty}
\subfloat[  { [30, 0]} ]{\label{fig: path_1}\includegraphics[width=0.10\textwidth]{images/more/tree_4/4_30_0.pdf}}\hspace{1mm}
\subfloat[  { [30, 1]} ]{\label{fig: path_1}\includegraphics[width=0.10\textwidth]{images/more/tree_4/4_30_1.pdf}}\hspace{1mm}
\subfloat[  { [30, 2]} ]{\label{fig: path_1}\includegraphics[width=0.10\textwidth]{images/more/tree_4/4_30_2.pdf}}\hspace{1mm}
\subfloat[  Target ]{\label{fig: path_1}\includegraphics[width=0.10\textwidth]{images/more/tree_4/4_30_t.pdf}}\hspace{0mm}


\subfloat[  { [240, 0]} ]{\label{fig: path_1}\includegraphics[width=0.10\textwidth]{images/more/tree_4/4_240_0.pdf}}\hspace{1mm}
\subfloat[  { [240, 1]} ]{\label{fig: path_1}\includegraphics[width=0.10\textwidth]{images/more/tree_4/4_240_1.pdf}}\hspace{1mm}
\subfloat[  { [240, 2]} ]{\label{fig: path_1}\includegraphics[width=0.10\textwidth]{images/more/tree_4/4_240_2.pdf}}\hspace{1mm}
\subfloat[  Target ]{\label{fig: path_1}\includegraphics[width=0.10\textwidth]{images/more/tree_4/4_240_t.pdf}}\hspace{0mm}

\subfloat[  { [480, 0]} ]{\label{fig: path_1}\includegraphics[width=0.10\textwidth]{images/more/tree_4/4_480_0.pdf}}\hspace{1mm}
\subfloat[  { [480, 1]} ]{\label{fig: path_1}\includegraphics[width=0.10\textwidth]{images/more/tree_4/4_480_1.pdf}}\hspace{1mm}
\subfloat[  { [480, 2]} ]{\label{fig: path_1}\includegraphics[width=0.10\textwidth]{images/more/tree_4/4_480_2.pdf}}\hspace{1mm}
\subfloat[  Target ]{\label{fig: path_1}\includegraphics[width=0.10\textwidth]{images/more/tree_4/4_480_t.pdf}}\hspace{0mm}

\subfloat[  { [2400, 0]} ]{\label{fig: path_1}\includegraphics[width=0.10\textwidth]{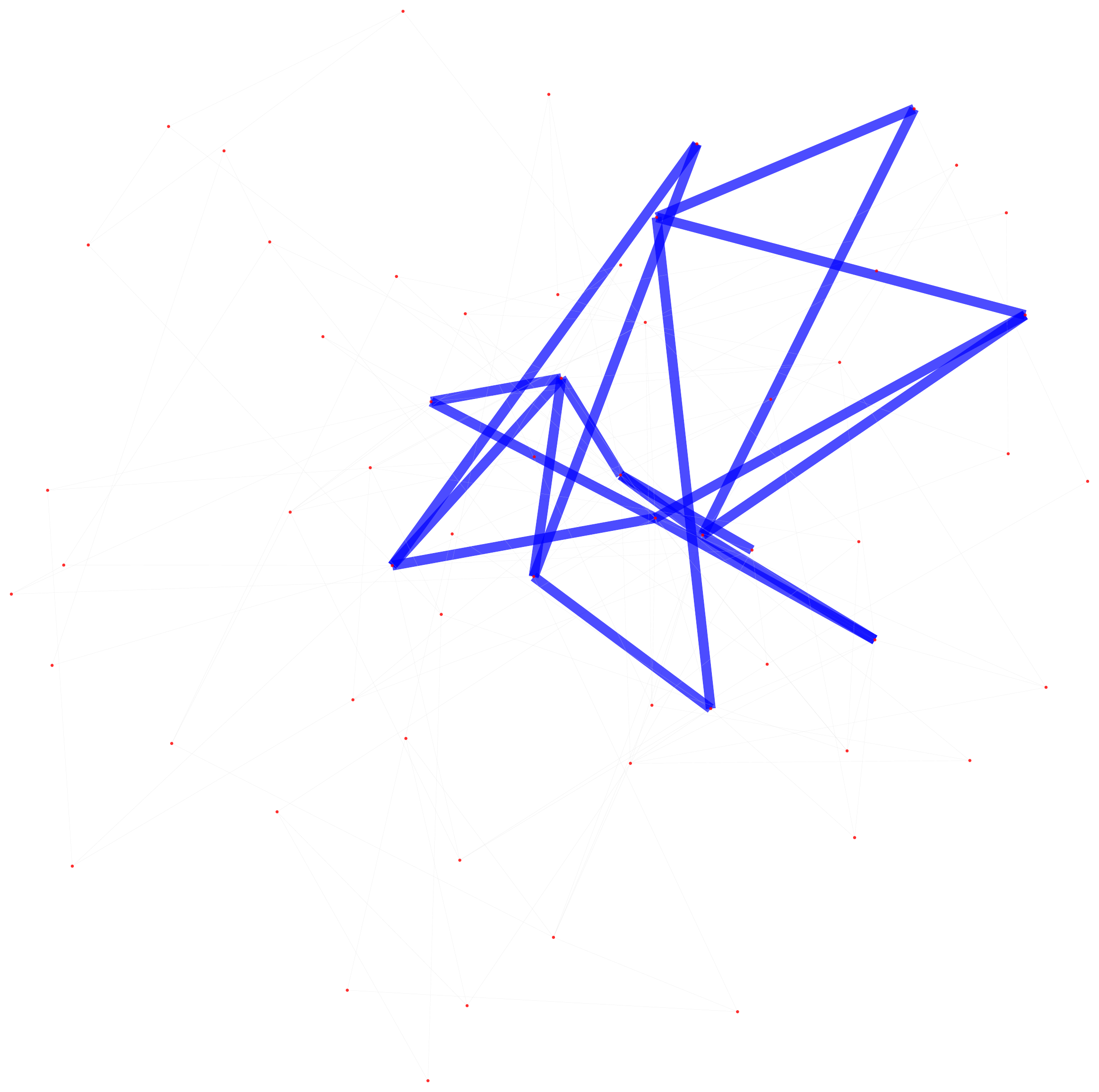}}\hspace{1mm}
\subfloat[  { [2400, 1]} ]{\label{fig: path_1}\includegraphics[width=0.10\textwidth]{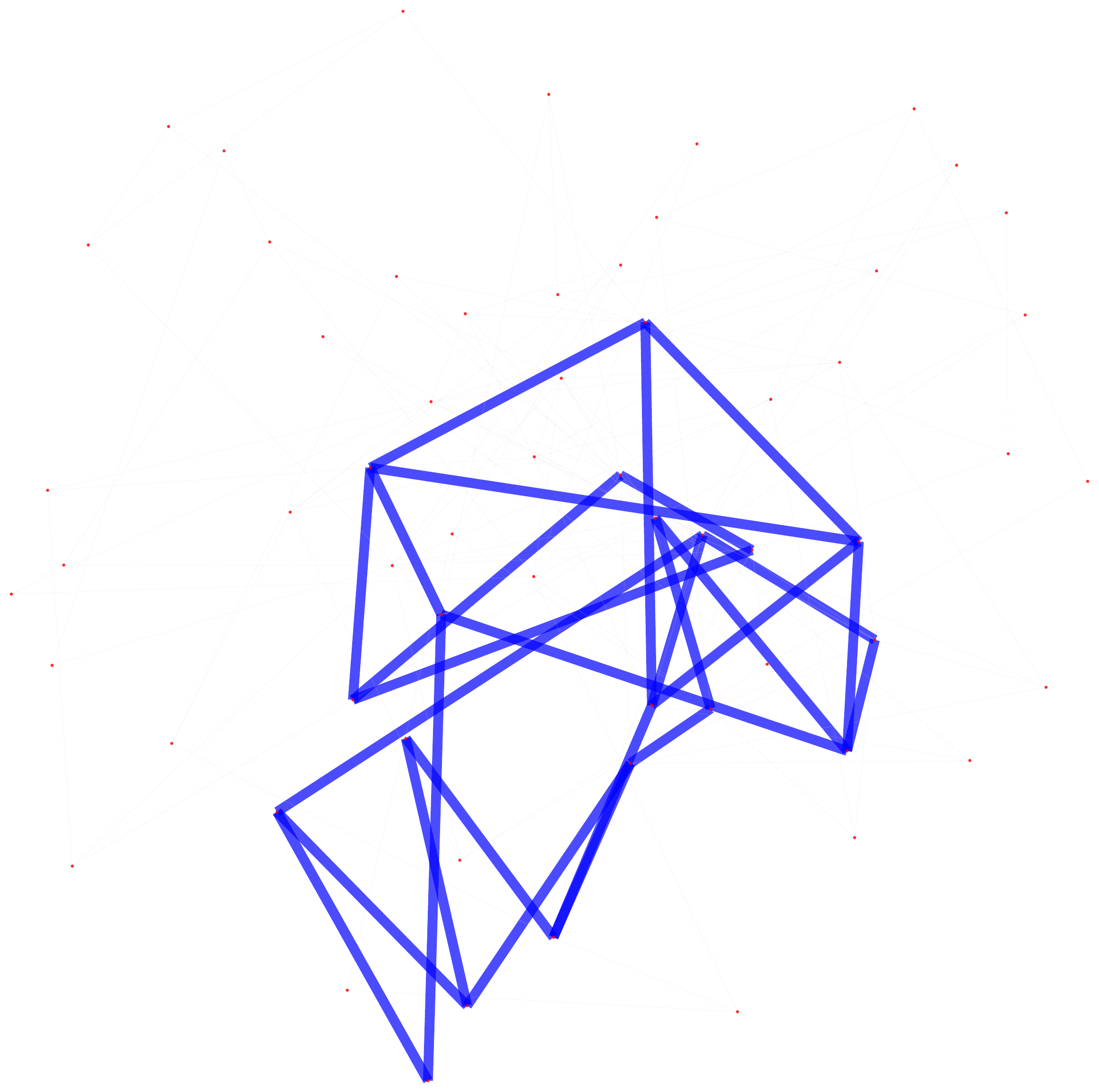}}\hspace{1mm}
\subfloat[  { [2400, 2]} ]{\label{fig: path_1}\includegraphics[width=0.10\textwidth]{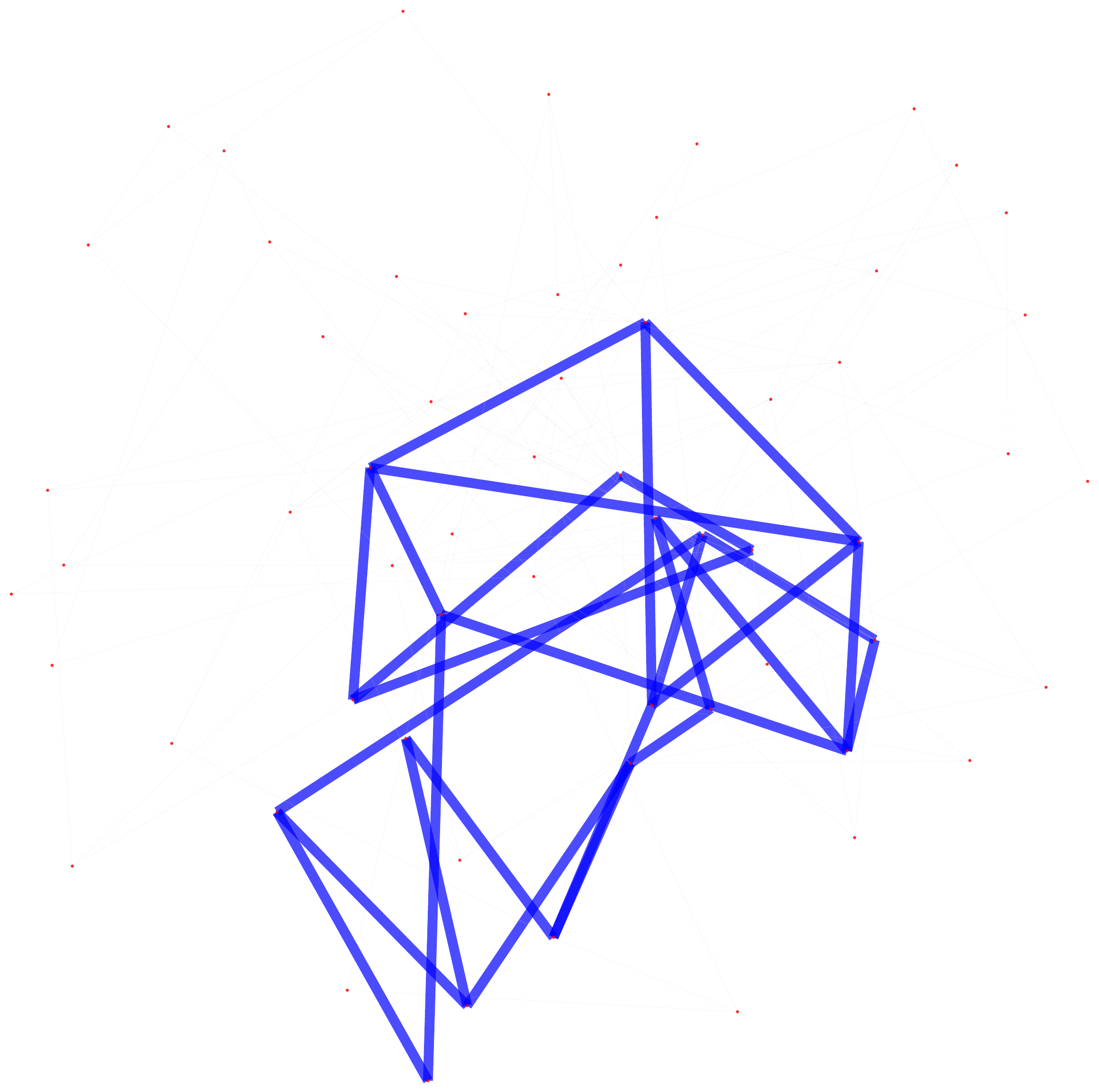}}\hspace{1mm}
\subfloat[  Target ]{\label{fig: path_1}\includegraphics[width=0.10\textwidth]{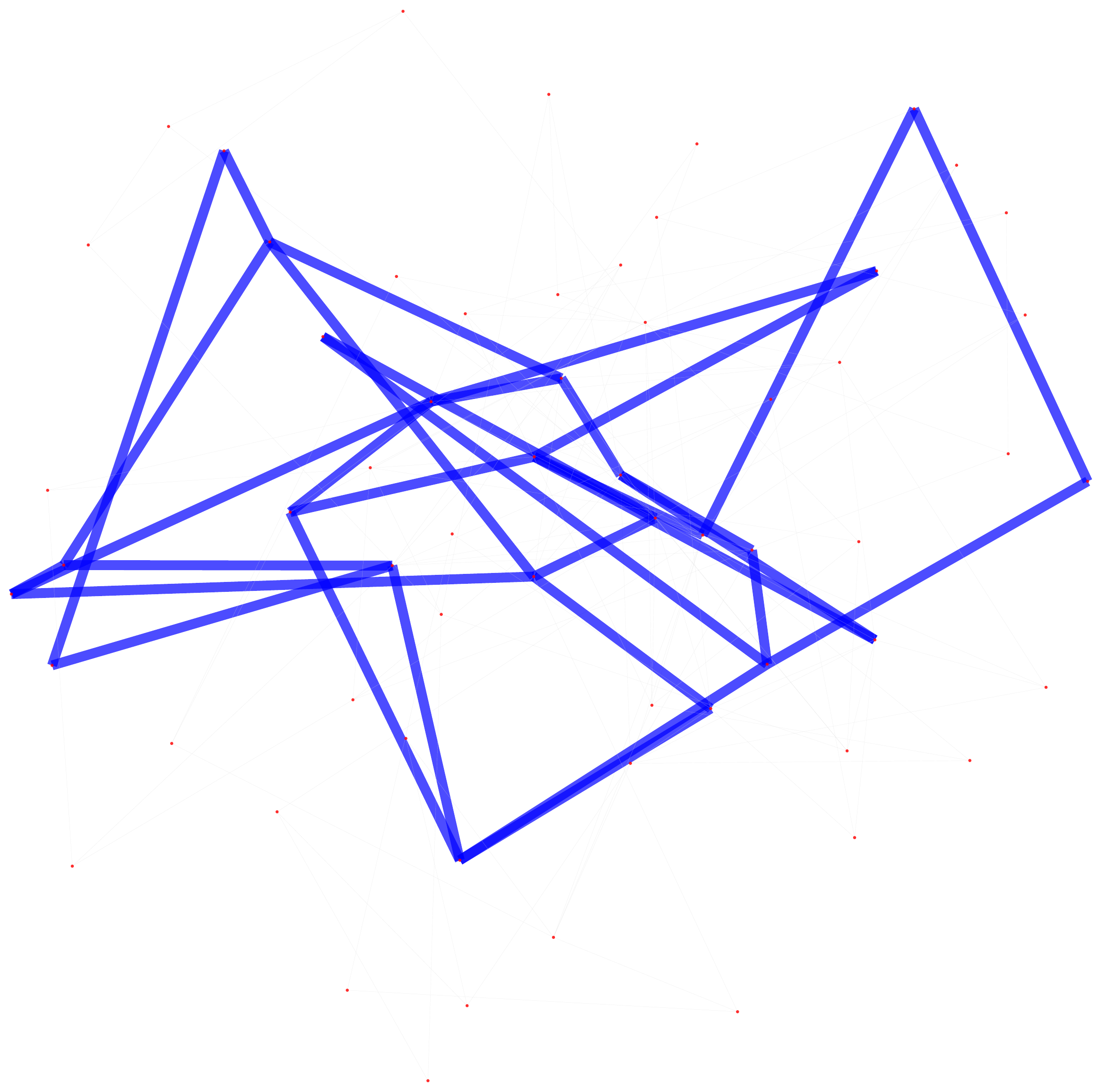}}\hspace{0mm}

\caption{\small Visualization of the results of QRTS-P for an example testing query for Path-to-Tree on Kro.}
\label{fig: more_tree_4}
\end{figure}

\newpage
\end{document}